\documentclass[english, oneside, dvipsnames]{book}
\usepackage{amsthm} %
\usepackage{amsmath}
\usepackage{mathtools} %
\usepackage{tcolorbox} %
\usepackage{xcolor} %
\usepackage{xspace}
\usepackage{quiver}
\usepackage{float} %
\usepackage{stmaryrd} %
\usepackage{xparse} %
\usepackage{thmtools} %
\usepackage{tikzit, tikz-cd, circuitikz} %
\usepackage{epigraph} %
\usepackage{hyperref} %
\usepackage{todonotes} %
\usepackage{graphicx} %
\usepackage{multicol} %
\usepackage{titling} %
\usepackage{setspace} %
\usepackage[left=4cm,right=2.5cm,top=2cm,bottom=4cm,includehead,includefoot,headheight=15pt]{geometry} %
\usepackage{scalefnt} %
\usepackage{subcaption} %
\usepackage{csquotes} %
\usepackage{pgfplots} %
\usepackage{etoolbox} %
\usepackage{bibentry} %
\nobibliography*      %

\usepackage[capitalise]{cleveref} %

\crefname{proposition}{Prop.}{Props.}

\hypersetup{
  pdftitle={Fundamental Components of Deep Learning: A category-theoretic approach},
  pdfauthor={Bruno Gavranovi\'c},
  colorlinks = true,
  linkcolor = NavyBlue,
  citecolor = Purple,
  urlcolor = NavyBlue,
}

\newcommand{\myref}[2]{\hyperref[#2]{#1}}

\pgfplotsset{compat=1.17}
\def\identityfn(#1){#1}
\def\sigmoidfn(#1){1/(1 + exp(-#1))}
\def\tanhfn(#1){(exp(#1) - exp(-#1))/(exp(#1) + exp(-#1))}
\def\relufn(#1){max(0, #1)}
\def\leakyrelufn(#1){max(0.01 * #1, #1)}
\def\gelufn(#1){#1 * \sigmoidfn(1.702 * #1)}

\usetikzlibrary{decorations.markings} %
\usetikzlibrary{external}
\NewDocumentCommand{\scaletikzfig}{ O{1} O{1} m }{
  \begin{center}
     \tikzsetnextfilename{#3}
     \scalefont{#2}\scalebox{#1}{\input{#3.tikz}}
   \end{center}
}

\NewDocumentCommand{\scaletikzfignocenter}{ O{1} O{1} m }{
    \tikzsetnextfilename{#3}
    \scalefont{#2}\scalebox{#1}{\input{#3.tikz}}
}

\graphicspath{ {./figures/} {./figures/tikz} }

\newcounter{todocounter}

\newcommand{\forlater}[2][]{
}

\tikzstyle{Medium box}=[fill=white, draw=black, shape=rectangle, tikzit shape=rectangle, minimum width=1.5cm, minimum height=1.5cm]
\tikzstyle{Large box}=[fill=white, draw=black, shape=rectangle, minimum width=5cm, minimum height=5cm]
\tikzstyle{small}=[fill=white, draw=black, shape=rectangle, minimum width=1cm, minimum height=0.5]
\tikzstyle{loss}=[fill=white, draw=black, shape=rectangle, minimum width=1cm, minimum height=3.5cm]
\tikzstyle{Medium}=[fill=white, draw=green, shape=rectangle, minimum width=1.5cm, minimum height=1.5cm, line width=0.8]
\tikzstyle{blue}=[fill=white, draw=blue, shape=rectangle, minimum width=1.5cm, minimum height=1.5cm]
\tikzstyle{red}=[fill=white, draw=red, shape=rectangle, minimum width=1.5cm, minimum height=1.5cm]
\tikzstyle{Long}=[fill=white, draw=black, shape=rectangle, minimum height=4.5cm, minimum width=1.5cm]
\tikzstyle{blackcircle}=[fill=black, draw=black, shape=circle, minimum width=0.2cm, inner sep=0pt]
\tikzstyle{Small box}=[fill=white, draw=black, shape=rectangle, minimum height=0.5cm, minimum width=0.5cm]
\tikzstyle{elongated}=[fill=white, draw=black, shape=rectangle, minimum height=1cm, minimum width=3cm]
\tikzstyle{circ}=[fill=white, draw=black, shape=circle]
\tikzstyle{blank}=[fill=white, draw=white, shape=circle]
\tikzstyle{none}=[fill=none, draw=none]
\tikzstyle{copy}=[fill=white, draw=black, shape=circle, minimum height=0.2cm, inner sep=0]
\tikzstyle{varCopy}=[fill=black, draw=black, shape=circle, minimum height=0.2cm, inner sep=0]
\tikzstyle{copy2}=[fill=black, draw=black, shape=circle, minimum height=0.2cm, inner sep=0]
\tikzstyle{1morph1}=[fill=white, draw=black, shape=rectangle, minimum width=1cm, minimum height=1cm]
\tikzstyle{1morph}=[fill=white, draw=black, shape=rectangle, minimum width=0.75cm, minimum height=0.75cm, inner sep=0.1cm]
\tikzstyle{2morph2}=[fill=white, draw=black, shape=rectangle, minimum width=1cm, minimum height=2cm]
\tikzstyle{2morph}=[fill=white, draw=black, shape=rectangle, minimum width=1cm, minimum height=1.25cm, inner sep=0.1cm]
\tikzstyle{nmorph}=[fill=white, draw=black, shape=rectangle, minimum height=6cm, minimum width=1cm, inner sep=0.1cm]
\tikzstyle{1state}=[fill=white, draw=black, regular polygon, regular polygon sides=3, minimum height=0.5cm, regular polygon rotate=-30]
\tikzstyle{dbox}=[fill=white, draw=black, dashed, shape=rectangle, minimum width=2cm, minimum height=1cm, inner sep=0.1cm]
\tikzstyle{vdbox}=[fill=white, draw=black, dashed, shape=rectangle, minimum width=2cm, minimum height=1.5cm, inner sep=0.1cm]
\tikzstyle{bigbox}=[fill=white, draw=black, dashed, shape=rectangle, minimum width=2cm, minimum height=4cm, inner sep=0.1cm]
\tikzstyle{2state}=[inner sep=0.05cm, fill=white, draw=black, isosceles triangle, minimum width=1.25cm, isosceles triangle apex angle=90, shape border rotate=180]
\tikzstyle{var2state}=[inner sep=0.05cm, fill=white, draw=black, isosceles triangle, minimum width=1.25cm, isosceles triangle apex angle=60, shape border rotate=180]
\tikzstyle{g2state}=[inner sep=0.05cm, fill=white, draw=black, isosceles triangle, minimum width=6cm, isosceles triangle apex angle=110, shape border rotate=180]
\tikzstyle{bigstate}=[inner sep=0.05cm, fill=white, draw=black, isosceles triangle, minimum width=3cm, isosceles triangle apex angle=110, shape border rotate=180]
\tikzstyle{bigeffect}=[inner sep=0.05cm, fill=white, draw=black, isosceles triangle, minimum width=3cm, isosceles triangle apex angle=110]
\tikzstyle{g2effect}=[inner sep=0.05cm, fill=white, draw=black, isosceles triangle, minimum width=6cm, isosceles triangle apex angle=110]
\tikzstyle{2effect}=[inner sep=0.05cm, fill=white, draw=black, isosceles triangle, minimum width=1.25cm, isosceles triangle apex angle=90]
\tikzstyle{b2effect}=[inner sep=0.05cm, fill=white, draw=black, isosceles triangle, minimum width=2cm, isosceles triangle apex angle=90]
\tikzstyle{node}=[fill=black, draw=black, shape=circle, scale=0.5]
\tikzstyle{GroundLeft}=[fill=white, draw=black, shape=tlground, rotate=-90]
\tikzstyle{GroundRight}=[fill=white, draw=black, shape=tlground, rotate=90]
\tikzstyle{GroundDown}=[fill=white, draw=black, shape=tlground]
\tikzstyle{GroundUp}=[fill=white, draw=black, shape=tlground, rotate=180]

\tikzstyle{Black arrow}=[->]
\tikzstyle{Red line}=[-, draw=red, line width=0.7]
\tikzstyle{Red arrow}=[draw=red, ->]
\tikzstyle{Gray line}=[-, draw={rgb,255: red,191; green,191; blue,191}, line width=0.8]
\tikzstyle{Gray arrow}=[->, draw={rgb,255: red,191; green,191; blue,191}]
\tikzstyle{Blue line}=[-, draw=blue]
\tikzstyle{Blue arrow}=[->, draw=blue]
\tikzstyle{bluearrow}=[->, fill=none, draw={rgb,255: red,29; green,206; blue,255}, thick]
\tikzstyle{midArrow}=[-, decoration={{markings,mark=at position .5 with {\arrow{>}}}}, postaction=decorate]
\tikzstyle{arrow}=[->]
\tikzstyle{pointy}=[->]
\tikzstyle{lightnone}=[-, draw={rgb,255: red,191; green,191; blue,191}]
\tikzstyle{Filled edge shape}=[-, fill=white, draw=black]
\tikzstyle{Thin Grey Arrow}=[->, draw={rgb,255: red,191; green,191; blue,191}, line width=0.3]
\tikzstyle{Thin Grey Line}=[-, draw={rgb,255: red,191; green,191; blue,191}, line width=0.3]
\tikzstyle{Black arrow crossing over}=[->, decoration={{crossing over}}]

\input{categories_macro} %

\providecommand\newthought[1]{%
   \addvspace{1.0\baselineskip plus 0.5ex minus 0.2ex}%
   \noindent\textsc{#1} %
}

\newtcolorbox[auto counter,number within=section,crefname={Box}{Boxes}]{mybox}[3][]
{
  fonttitle=\bfseries\selectfont, %
  colframe = #2!40!black,
  colback  = #2!10,
  coltitle = white,  
  title    = {#3 \hfill Box~\thetcbcounter},
  #1,
}

\theoremstyle{definition}
\newtheorem{definition}[equation]{Definition}
\newtheorem{remark}[equation]{Remark}
\newtheorem{example}[equation]{Example}
\newtheorem{notation}[equation]{Notation}

\theoremstyle{plain}
\newtheorem{corollary}[equation]{Corollary}
\newtheorem{lemma}[equation]{Lemma}
\newtheorem{proposition}[equation]{Proposition}
\newtheorem{theorem}[equation]{Theorem}
\newtheorem{conjecture}[equation]{Conjecture}

\newtheorem*{contributions}{Contributions}
\newtheorem*{epistemicstatus}{Epistemic status}

\newtheorem*{cthelp}{Avenues for exploration}

\newenvironment{tightcenter}{%
  \setlength\topsep{0pt}
  \setlength\parskip{0pt}
  \begin{center}
}{%
  \end{center}
}

\NewDocumentEnvironment{tikzcddiag}{O{}}{
  \tikzexternaldisable
  \begin{tikzcd}[#1]
  }{
  \end{tikzcd}
  \tikzexternalenable
}

\newcommand{\newdef}[1]{\textbf{#1}}

\newcommand{\Z}{\mathbb{Z}}
\newcommand{\R}{\mathbb{R}}
\newcommand{\N}{\mathbb{N}}

\newcommand{\dsh}{\textbf{-}}

\newcommand{\newcommandLink}[3]{%
  \expandafter\newcommand\csname #1\endcsname[1][true]{%
    \ifstrequal{##1}{true}{%
      \myref{#2}{#3}
    }{%
      #2
    }%
  }%
}

\newcommand{\NamedCat}[1]{\mathbf{#1}}
\newcommand{\NamedFunctor}[1]{\mathsf{#1}}

\newcommand{\Iso}{\NamedFunctor{Iso}}
\newcommand{\Epi}{\NamedFunctor{Epi}}
\newcommand{\Core}{\NamedFunctor{Core}}

\newcommand{\Gr}{\underset{\circlearrowright}{\NamedCat{Gr}}}
\newcommand{\GrC}{\underset{\circlearrowleft}{\NamedCat{Gr}}}

\newcommand{\MonCat}{\NamedCat{MonCat}}

\newcommand{\TwoCat}{\textbf{2}\NamedCat{Cat}}

\newcommand{\LxDbl}{\NamedCat{LxDbl}}

\newcommand{\CatInit}{\NamedCat{0}}
\newcommand{\CMon}{\NamedCat{CMon}}

\newcommand{\Learn}{\NamedCat{Learn}}
\newcommand{\Chart}{\NamedCat{Chart}}
\newcommand{\Cont}{\NamedCat{Cont}}
\newcommand{\DChart}{\NamedCat{DChart}}

\newcommand{\DChartA}{\NamedCat{DChart}_A}

\newcommand{\TwoOptic}{\NamedCat{2Optic}}

\newcommand{\Alg}{\NamedCat{Alg}}
\newcommand{\coAlg}{\NamedCat{Coalg}}

\newcommand{\Act}{\NamedCat{Act}}

\newcommandLink{Set}{\NamedCat{Set}}{def:set}
\newcommandLink{Cat}{\NamedCat{Cat}}{def:cat}
\newcommandLink{Smooth}{\NamedCat{Smooth}}{def:smooth} %
\newcommandLink{FVect}{\NamedCat{FVect}}{def:fvect}
\newcommandLink{FVectR}{\NamedCat{FVect}_{\R}}{def:fvect}
\newcommandLink{Poly}{\NamedCat{Poly}}{def:poly}
\newcommandLink{PolyR}{\NamedCat{Poly}_{\R}}{def:poly}
\newcommandLink{PolyZ}{\NamedCat{Poly}_{\Z_2}}{def:poly}
\newcommandLink{El}{\NamedCat{El}}{def:cat_of_elem}
\newcommandLink{tw}{\NamedCat{tw}}{ex:twisted_arrow}
\newcommandLink{CoKl}{\NamedCat{coKl}}{def:cokl}
\newcommandLink{Optic}{\NamedCat{Optic}}{def:weighted_optic}
\newcommandLink{CLACat}{\NamedCat{CLACat}}{def:clacat}
\newcommandLink{Para}{\NamedCat{Para}}{def:para}
\newcommandLink{Copara}{\NamedCat{coPara}}{def:copara}
\newcommandLink{Lens}{\NamedCat{Lens}}{eq:lens}
\newcommandLink{LensA}{\NamedCat{Lens}_A}{eq:lensa}
\newcommandLink{DLens}{\NamedCat{DLens}}{ex:slice_charts_lenses}
\newcommandLink{DLensA}{\NamedCat{DLens}_A}{ex:slice_charts_lenses}
\newcommandLink{Fam}{\NamedCat{Fam}}{ex:fam_functor}
\newcommandLink{CatTerm}{\NamedCat{1}}{def:edge_cases}
\newcommandLink{Mark}{\NamedCat{Mark}}{def:markov_kernel}
\newcommandLink{Conv}{\NamedCat{Conv}}{ex:mark_kernel_reparameterisation}
\newcommandLink{St}{\NamedCat{St}}{def:stc}
\newcommandLink{PolySpivak}{\NamedCat{Poly}}{not:poly}
\newcommandLink{PolySpivakA}{\NamedCat{Poly}_A}{ex:slice_charts_lenses}

\newcommandLink{laxsmoothiota}{\iota}{ex:iota_lax_monoidal}
\newcommandLink{iotaaccla}{\iota}{def:additively_closed_cla}
\newcommandLink{accla}{ACCLA}{def:additively_closed_cla}
\newcommandLink{discr}{\mathsf{discr}}{prop:adjoint_quadruple}
\newcommandLink{codiscr}{\mathsf{codiscr}}{prop:adjoint_quadruple}
\newcommandLink{conncomp}{\pi_0}{prop:adjoint_quadruple}
\newcommandLink{Ob}{\mathsf{Ob}}{prop:adjoint_quadruple}
\newcommandLink{graph}{\mathsf{graph}}{def:graph}
\newcommandLink{rev}{\mathsf{rev}}{def:reversed_monoidal_category}
\newcommandLink{Softargmax}{\mathsf{Softargmax}}{ex:activation_functions}
\newcommandLink{Attend}{\mathsf{Attend}}{def:attention}
\newcommandLink{eval}{\mathsf{eval}}{rem:eval}
\newcommandLink{op}{\mathsf{op}}{lemma:op_functor}
\newcommandLink{terminal}{!}{subsec:notation}
\newcommandLink{interchanger}{c}{def:interchanger}
\newcommandLink{interchangeract}{c}{def:monact_oplaxator}

\newcommand{\asc}[1]{\myref{\CMon(}{def:additive_maps}#1\myref{)}{def:additive_maps}} %
\newcommand{\enrichedbasechange}[1]{{#1}\myref{_*}{def:enriched_base_change}}

\newcommand{\CoparaWn}[1]{\NamedCat{coPara}^{#1}}
\newcommand{\CoparaW}[2]{\myref{\CoparaWn{#1}_{#2}}{def:coparaw}}

\newcommand{\compPara}[1]{\, {#1}^{\myref{\mathfrak{p}}{subsec:notation}}\,}
\newcommand{\compCopara}[1]{\, {#1}_{\myref{\mathfrak{p}}{subsec:notation}}\,}
\newcommand{\name}[1]{\myref{\lceil}{subsec:notation} #1 \myref{\rceil}{subsec:notation}}

\newcommand{\Hom}{\mathsf{Hom}}
\newcommand{\id}{\text{id}}
\newcommand{\comp}{\fatsemi}
\newcommand{\const}[1]{\mathsf{const}({#1})}
\newcommand{\List}[1]{\mathsf{List}_{#1}}
\newcommand{\BTree}[1]{\mathsf{BTree}_{#1}}

\newcommand{\RC}{\mathbf{R}_{\cC}}
\newcommand{\gda}{\mathsf{gda}}
\newcommand{\GAN}{\mathsf{GAN}}

\newcommand{\XOR}{\mathsf{XOR}}

\newcommand{\loss}{\mathsf{loss}}
\newcommand{\predicted}{y_p}
\newcommand{\groundtruth}{y_t}
\newcommand{\mil}{x} %
\newcommand{\mol}{y} %
\newcommand{\mi}{X} %
\newcommand{\mo}{Y} %
\newcommand{\lo}{L} %

\newcommand{\update}{\mathsf{update}}

\newcommand{\internalHom}[3]{\underbar{$#1$}(#2, #3)}

\DeclareDocumentCommand \internalBracket { o m m } {
  \IfNoValueTF {#1} {
    [#2, #3]
  }{
    [#2, #3]_{#1}
  }
}

\newcommand{\co}{\mathsf{co}}

\newcommand{\swap}{\mathsf{swap}}

\newcommand{\limm}{\mathsf{lim}} %
\newcommand{\colim}{\mathsf{colim}}
\newcommand{\B}{\mathbf{B}}

\newcommand{\pibc}{{\enrichedbasechange{\conncomp}}}

\newcommand{\act}{\bullet}
\newcommand{\actt}{\boxtimes}

\newcommand{\actttt}{\raisebox{0.1ex}{\scalebox{0.9}{$\scriptscriptstyle \blacksquare$}}}

\newcommand{\acttttt}{\circledast}

\newcommand{\actfw}{\act} %
\newcommand{\actbw}{\actttt} %
\newcommand{\actc}{\boxtimes} %
\newcommand{\actlast}{\acttttt}

\newcommand{\diset}[2]{\left({{#1} \atop {#2}}\right)}

\newcommand{\pr}[2]{\begin{matrix}{#1} \\ {#2} \end{matrix}} 
\newcommand{\OpticHom}[4]{\Bigg(\pr{#1}{#2} \, , \pr{#3}{#4}\Bigg)}

\newcommand{\Po}{\Para_{\actlast}(\Optic^W_{\diset{\actfw}{\actbw}})}
\newcommand{\Pon}{\Para(\Optic)}

\newcommand{\key}{\mathsf{key}}
\newcommand{\seq}{\mathsf{seq}}
\newcommand{\val}{\mathsf{val}}
\newcommand{\relu}{\mathsf{ReLU}}
\newcommand{\leakyrelu}{\mathsf{LeakyReLU}}
\newcommand{\gelu}{\mathsf{GELU}}

\newcommand{\TypeOne}{\textbf{Type 1 }}
\newcommand{\TypeTwo}{\textbf{Type 2 }}

\pretitle{%
  \begin{center}
    \huge
    \includegraphics[width=6cm]{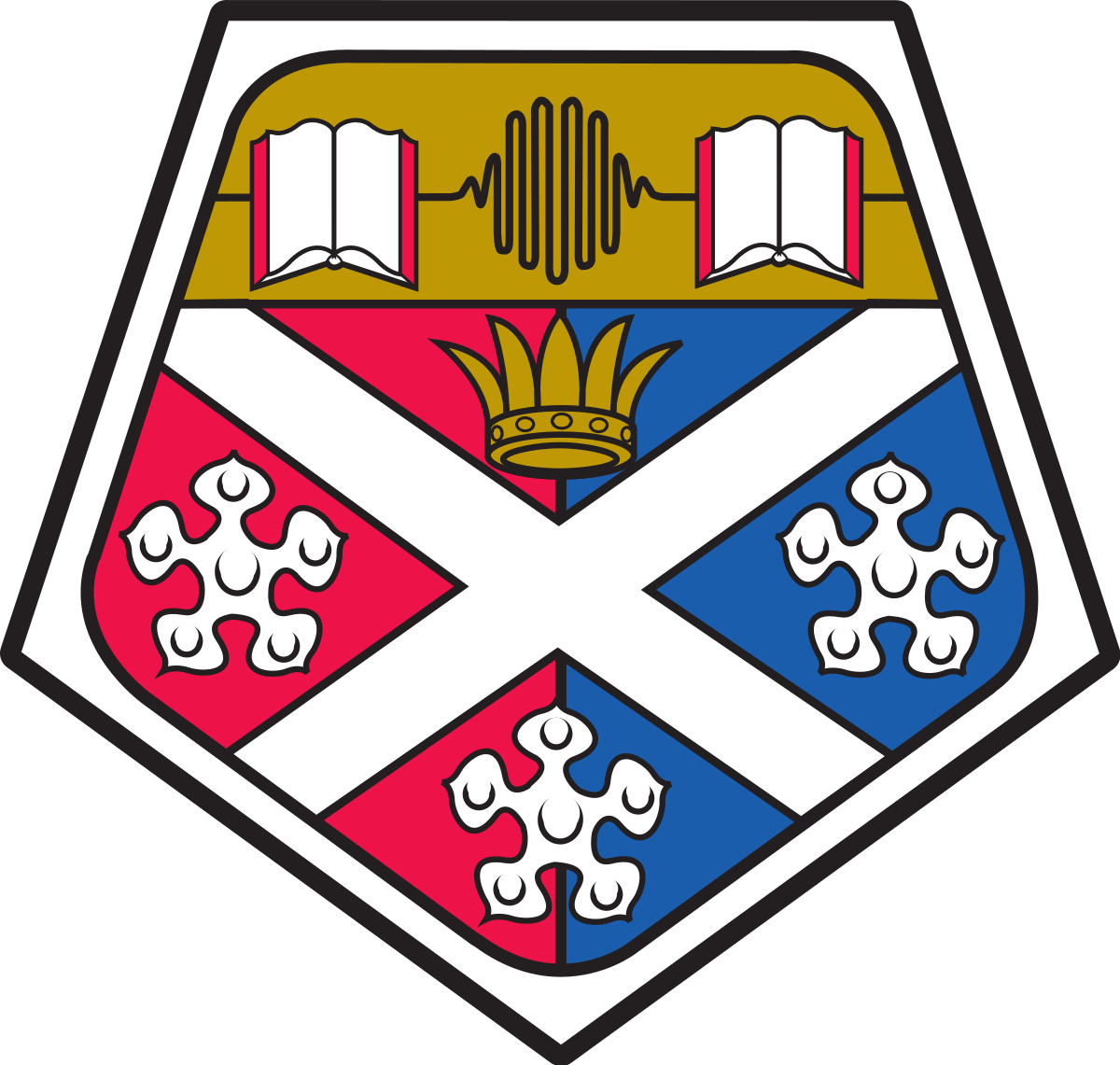}\\[\bigskipamount]
  }
\posttitle{\end{center}}

\title{Fundamental Components\\
  of Deep Learning\\
  \Large{A category-theoretic approach} \\ PhD Thesis}
\author{Bruno Gavranovi\'c
  \\ \small Mathematically Structured Programming Group\\[-0.8ex]
  \small Computer and Information Sciences\\[-0.8ex]
  \small University of Strathclyde, Glasgow
}

\begin{document}

  \frontmatter

  \maketitle
  
\topskip0pt
\vspace*{\fill}
\noindent
\begin{quote}
	\centering
	This thesis is the result of the author's original research. It has been composed by the author and has not been previously submitted for examination which has led to the award of a degree. \\[5pt]
	The copyright of this thesis belongs to the author under the terms of the United Kingdom Copyright Acts as qualified by University of Strathclyde Regulation 3.50. Due acknowledgement must always be made of the use of any material contained in, or derived from, this thesis. \\[5pt]
\end{quote}
\vspace*{\fill}

  \chapter*{Abstract}

  Deep learning, despite its remarkable achievements, is still a young field.
  Like the early stages of many scientific disciplines, it is marked by the discovery of new phenomena, ad-hoc design decisions, and the lack of a uniform and compositional mathematical foundation.
  From the intricacies of the implementation of backpropagation, through a growing zoo of neural network architectures, to the new and poorly understood phenomena such as double descent, scaling laws or in-context learning, there are few unifying principles in deep learning.

  This thesis develops a novel mathematical foundation for deep learning based on the language of category theory.
  We develop a new framework that is a) end-to-end, b) unform, and c) not merely descriptive, but prescriptive, meaning it is amenable to direct implementation in programming languages with sufficient features.
  We also systematise many existing approaches, placing many existing constructions and concepts from the literature under the same umbrella.

  In Part I, the theory, we identify and model two main properties of deep learning systems: they are \emph{parametric} and \emph{bidirectional}.
  We expand on the previously defined construction of actegories and $\Para[false]$ to study the former, and define \emph{weighted optics} to study the latter.
  Combining them yields \emph{parametric weighted optics}, a categorical model of artificial neural networks, and more: constructions in Part I have close ties to many other kinds of bidirectional processes such as bayesian updating, value iteraton, and game theory.
  
  Part II justifies the abstractions from Part I, applying them to model backpropagation, architectures, and supervised learning.
  We provide a lens-theoretic axiomatisation of differentiation, covering not just smooth spaces, but discrete settings of boolean circuits as well.
  We survey existing, and develop new categorical models of neural network architectures.
  We formalise the notion of optimisers and lastly, combine all the existing concepts together, providing a uniform and compositional framework for supervised learning.

  \tableofcontents

  \chapter*{Acknowledgements}

I'm incredibly grateful to people around me for the last four years.
Despite not being trained as a mathematician, I had the privilege of engaging with so many people who think in such crystal clear ways, ones which were not even on my radar of being possible.
And these people --- the Mathematically Structured Programming group (MSP) where I've done my PhD, and the Applied Category Theory commmunity in general --- are also such a welcoming and inclusive community, without whose help and encouragement I wouldn't have been able to be where I am today.

Specifically, I want to thank my advisor Neil Ghani whose guidance, lectures on category theory (and lectures on writing!), as well as the provided space to explore my own ideas were central in allowing me to learn so much.
Thanks to Jules Hedges who in many ways was not only my second, informal supervisor, but also a mentor, and more importantly, a friend whom I've had the pleasure to explore abstract nonsense with.
Thanks to J\'er\'emy Ledent for patience in our conversations in which he often helped me understand what was it that I actually \emph{meant} to say mathematically.
I thank Conor McBride for imparting priceless wisdom on how to \emph{think with types}. The conversations stayed with me in a way that is hard to summarise.
In general it is hard to overstate how positive my exprience was at MSP, and I can't thank enough all the people with whom I've had the pleasure to share a chat and a pint with --- Clemens, Bob, Fredrik, Dylan, Riu, Joe, Georgi, Andr\'e, Alasdair, Joe, Eigil, Guillaume, Malin, Zanzi and Ezra, to name just some.
Thanks to Joe for fun whiteboard chats and meanders through math which I've especially enjoyed!
There are countless other collaborators, mentors and friends and everything in between: Igor Bakovi\'c, Fabio Zanasi, Geoffrey Cruttwell, Bob Coecke, Dan Shiebler, David Jaz Myers, Vincent Wang-Ma\'scianica, Fabrizio Genovese, Brandon Shapiro, Owen Lynch, Toby Smithe, Tali Beynon, Nima Motamed, Mattia Villani.
I thoroughly appreciate all of the things I learned from you, and the fun times we've had.
Thanks especially to Brendan Fong for originally introducing me to Neil, and getting me started with category theory.
Thanks to people who gave me feedback on my thesis: Neil, Jérémy, Paul Lessard, and my thesis examiners Radu Mardare and Sam Staton.

I thank my family for support, and sacrifices they've made so I could get where I am. \emph{Hvala vam.} And lastly, Ieva. 
I've grown so much with you over the last four years.
Thank you for your love, support, and immeasurable kindness.
You've been my partner, my friend, and I'm proud to have you by my side.
You hold a special place for me, and I hope you know it.

  \mainmatter
  \chapter{Introduction}
\label{ch:introduction}

\epigraph{The task of the academic is not to scale great intellectual mountains but to flatten them.}{Conor McBride}

\newthought{Throughout the last few decades,} many behaviours thought to have been unique to humans have systematically been replicated by computer programs.

Computer programs have been exhibiting increasingly better performance on variety of benchmarks aimed at measuring high-level reasoning, perception, and planning. %
These include generation and classification of images (\cite{ramesh_hierarchical_2022,rombach_high-resolution_2022,dosovitskiy_image_2021}), text (\cite{openai_introducing_2022}), audio (\cite{chen_neural_2018, dong_speech-transformer_2018}), few-shot learning (\cite{brown_language_2020,finn_model-agnostic_2017,hospedales_meta-learning_2022}), game-playing --- notably Jeopardy (\cite{ferrucci_introduction_2012}), Atari games (\cite{mnih_human-level_2015}), and Starcraft II (\cite{vinyals_grandmaster_2019}), but also traditional games.
Today, the best chess and Go players in the world are not human (\cite{silver_general_2018}). The chess engine Stockfish is estimated to hold an ELO rating of 3542 (\cite{Team_ccrl_2023}), while the best human player, Magnus Carlsen, has an ELO of 2882.\footnote{ELO is a rating system for calculating the relative skill of players in zero-sum games such as chess.}
In Go, the system AlphaGo beat Lee Sedol --- the one of the best players in the world --- in 2016, overturning centuries of best Go practices, and making players around the world rethink how they approach the game (\cite{shin_superhuman_2023}).

Many of these achievements were once thought to be impossible, or even distinguishing factors between humans and machines.
In mid 20th century, it was not uncommon to believe that making computers play chess well would be a significant milestone towards understanding intelligence.
Likewise, until very recently it was not uncommon to believe that producing and responding to queries in natural language is something that's uniquely correlated with intelligent behaviour.

\textcolor{Black}{But these are not widespread beliefs anymore.}
With the defeat of the reigning world chess champion Gary Kasparov by a computer system called Deep Blue in 1997, it became clear that the kind of analytical thinking needed for chess is only one aspect of what we deem to be intelligence. 
And with the advent of large language models (LLMs) (\cite{openai_introducing_2022}) it became clear that natural language understanding is also only one aspect of what we might call intelligence, as LLMs often produce sensible sounding, but logically incorrect natural language phrases.

These and many other developments gave rise to the field of Artificial Intelligence (AI), which has since experienced an explosive rise, having permeated almost every realm of technology.
It is receving continued investment from both industry, venture capital and governments, increasing the rate at which these systems are being implemented.

Despite this, the consensus on what constitutes ``Artificial`` and ``Intelligence'' has not been well-established.
The goalposts for these definitions have been consistently shifted throughout history.
They have been shifted so many times that John McCarthy, one of the pioneers of AI, has given this phenomenon a name: the ``AI effect'' (\cite{haenlein_brief_2019}), describing the tendency to redefine the concept of Artificial Intelligence to exclude whatever has been recently accomplished, under the justification that this recent accomplisment never required intelligence to begin with.

The AI effect highlights the challenges with predicting such developments, and understanding their implications.
Indeed - where will this continued pace of developments lead us?
How far will we move the goalposts 30, or 50 years into the future?
Even when it comes to fundamental questions such as ``what distinguishes humans from machines'', history has shown us that our common sense has repeatedly been wrong about these matters. 
``Obvious'' answers such as \emph{only humans can play chess}, or \emph{only humans can communicate in natural language} seemed plausible before, but they are not anymore (\cref{fig:only_a_human}).
Even the revered Turing test (\cite{turing_icomputing_1950}) is not helpful here.
It leaves out the matter of defining artificial intelligence to the discretion of a panel of human judges without any guarantees that this panel is working with a good definition to begin with.

\begin{figure}[H]
  \centering
  \includegraphics[width=0.65\textwidth]{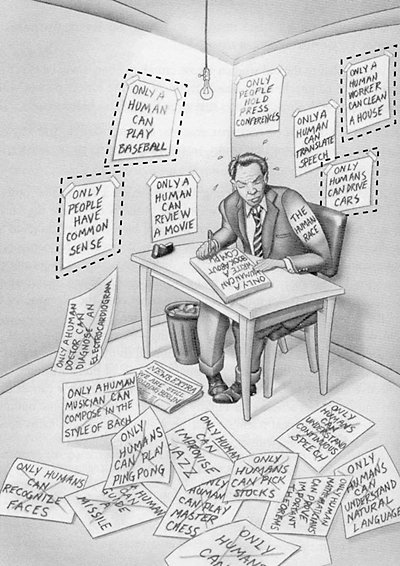}
  \caption{The number of behaviours that can still be exhibited \emph{only} by humans has steadily been shrinking. Graphic taken from \cite{kurzweil_age_1999}.}
  \label{fig:only_a_human}
\end{figure}

In this thesis, we take a step back, and search for a conceptual framework that will stand the test of time.
We study the machinery behind these systems --- artificial neural networks, and the underlying field of deep learning --- from a very specific vantage point.
We examine them through the lens of a new, and radically different way of understanding mathematics, and science as a whole: Category Theory.
We use this language as a means of flattening the mountain of almost insourmountable deep learning research (\cite{olah_research_2017}) created in recent years, and aim to lay the mathematical foundation for deep learning informed by 21st century mathematics.

\section{From alchemy to chemistry}

To understand where we are going, it is good to take a step back.

As a scientific field, where is deep learning currently?
It is not hard to see that it has achieved exceptional successes in recent years.
But it is still a young field.
Compared to millenia-old disciplines like mathematics or physics, deep learning is roughly $50$-year-old. 
It is notoriously ad-hoc, and only beginning to find its footing.
However, this initial \emph{ad-hoc} phase is not unique to deep learning --- this is how scientific fields start.

For example, people have been classifying animals and plants for thousands of years, but our conceptual understanding of the field of taxonomy changed drastically with the discovery of the concept of evolution.
In programming, researchers started tinkering with vacuum tubes, but most of what we call programming today --- pointers, recursion, map-reduce or monads --- is heavily abstracted from that.
People have been stacking rocks on top of each other since the dawn of time, but only in the last hundred years or so we have discovered materials science, finite-element analysis and general math and physics that enabled us to build skyscrapers once considered impossible.%

And many parallels can be drawn between deep learning and chemistry.
Throughout the centuries, there have been many practical advances in the field of chemistry.
These include the development of metallurgy, glass making processes, lab equipment, as well as many systematic studies of chemical substances and their transformations.
However, most of these advances were made by early researchers with no knowledge of the periodic table of elements, protons, neutrons, nor electrons: the basic building blocks of what they were studying.
Researchers were documenting many new and unexpected phenomena, and most of their theories interpreting the experiments did not survive the test of time.
Many researchers of that era are today called \emph{alchemists}.
Among them were ones who believed they could transmute base metals into gold, find the elixir of life, or create panacea --- a universal remedy able to cure all diseases.

But science progressed, and our understanding deepened.
We discovered the atom, and are now aware of the many phenomena of quantum physics.
Modern chemistry provided us with theories that allowed us to predict new elements and advance medicine, enabling us to build and scale today's sprawling healthcare systems.
Today, we have the ability to cure many diseases alchemists tried and failed to cure.
However, we are doing it based on completely different principles.

As argued by many, deep learning today is still in the alchemy stage (\cite{preserve_knowledge_nips_2018,hutson_ai_2018}).
Just like alchemists, deep learning researchers are running experiments and making practical advancements.
They are creating recipes for countless neural network architectures and optimisation procedures, while also documenting a number of new and poorly understood phenomena.
These include vanishing and exploding gradients (\cite{pascanu_difficulty_2013}), double descent and grokking (\cite{nakkiran_deep_2021, heckel_early_2021, liu_towards_2022}), lottery ticket hypothesis (\cite{frankle_lottery_2019, frankle_linear_2020}), scaling laws (\cite{kaplan_scaling_2020, bahri_explaining_2021}), and in-context learning (\cite{von_oswald_transformers_2023}), to name a few.\footnote{Unlike alchemists, we understand \emph{exactly} what happens at the low levels of abstraction.
We understand fundamentally everything about bits are moving through logic gates, or compositions of linear layers with activation functions.
What we don't have is a good understanding of the modular building blocks of neural networks \emph{at the intermediate scale}.
This is because \emph{more is different} (\cite{anderson_more_1972}), and knowledge of the low-level abstractions does not allow us to easily predict emergent behaviour.}

However, all of these phenomena, architectures and new practical advancements are made without any consensus on the overarching theory of deep learning.
Architectures and advancements are guided by heuristics of what works well in practice, often involve using ad-hoc design decisions and conflicting ways of thinking about the field as a whole.
Existing theories are usually highly specialised to a particular subfield of deep learning, while implementations are often brittle. 
There are examples of the performance machine learning models dropping drastically after internals of frameworks they used changed the default way of rounding numbers (\cite{preserve_knowledge_nips_2018}), or with different initial seeds (\cite{irpan_deep_2018}), a problem arising particularly in reinforcement learning.

Unlike with chemistry, we are not at a stage where we can confidently predict new phenomena, or easily map out a space of neural network architectures.

\begin{center}
  \emph{There is no periodic table of neural network architectures.}
\end{center}

Finally, just like alchemists were yearning for a remedy that could cure any disease, it is not uncommon for deep learning researchers to be motivated by the discovery of a general neural network architecture that can solve \emph{any} task.

\subsection*{From chemistry to nuclear physics}

Despite our lack of understanding of deep learning systems, we are deploying them at scale.
Deep learning systems drive cars, recommend products, amplify thoughts on social media, screen job applicants, and have generally permeated our daily lives.
The rate at which they are being implemented is increasing.
This is because of continued investment by major tech companies, venture capital firms and governments, with the biggest experiments in deep learning crossing the price point of \$100 million (\cite{knight_openais_2023}).

While undeniably producing immense value and benefit to our society, deep learning systems also exhibit numerous issues: lack of explainability (\cite{xu_explainable_2019}), amplification of biases (\cite{lloyd_bias_2018}), and adversarial attacks (\cite{vassilev_adversarial_2024}), to name a few.
With the increase of integration of these systems into our society there is a growing need for sensible regulatory paradigms and safety assurances. 
Unlike the regulation of nuclear weapons, where the dangers are clear and mitigation is generally a matter of execution, in the case of AI both the dangers and their mitigation are not clear.

As deep learning systems are generally \emph{black boxes} it is hard to prevent unwanted behaviour. 
When it comes to regulation, the prevalent practice is mostly \emph{retroactive}, meaning the focus is on explaining decisions a deep learning model has already made.
Instead, if we wanted to express the kind of reasoning we would \emph{proactively} uphold a model to, we'd need a formal language that is both general enough to cover a wide-array of possible use cases, and also precise enough to give rise to enforceable legislation.
But there is no consensus about what might such a language look like.

The deep learning community is aware of these issues.
Researchers have called for a deep learning ``Langlands programme'' (\cite{rieck_machine_2020}) and a deep learning ``Erlangen programme'' (\cite{bronstein_geometric_2021}), both originally inspired by far-reaching programs aimed at unifying disparate parts of mathematics.
Recently, the first PhD thesis on the safety of artificial general intelligence has been published (\cite{everitt_towards_2019}).

There is agreement in the deep learning community about the fact that we have a pressing need for a principled, rigorous and verifiable theory of deep learning.
Nonetheless, there does not seem to be widespread consensus about how exactly to reach that goal.
However, science will advance.
It is not unreasonable to expect that we will have solved many of these problems in the future.
The only question is, how will our understanding have changed?
Looking at deep learning in 30, or 50 years, how many theories will have stood the test of time?

\section{Category Theory}

While the mountain of deep learning research has steadily been growing, another field has gradually begun the arduous task of dismantling the ``alchemical'' theories of deep learning, and assembling a more rigorous, transparent and systematic foundation.
The field of \emph{Category Theory} has quietly been flourishing, unbeknownst to most deep learning researchers and the general scientific population.

Category theory (CT) is a language, a way of thinking and of structuring knowledge, markedly different from others.
Originating in abstract mathematics CT has since proliferated, and been used to express ideas from numerous different fields (\cref{box:ct_fields}) in a uniform way.
This consequently revealed commonalities between these fields which were previously unknown, resulting in a battle-tested system of interfaces which can reliably be applied across scientific fields.
It also resulted in a wide array of well-developed tools and theories (geometry, algebra, programming, combinatorics, and topology, to name a few), as well as a wide talent pool of researchers. 
The number of CT researchers is steadily growing, and it is relatively easy to onboard them to other fields, as they share the same conceptual knowledge base.

Category theory has evolved jointly with the consolidation of a community of researchers interested in \emph{Applied Category Theory} (ACT), a field aiming to implement ideas from category theory in other fields.
Unlike conferences in other disciplines, where researchers coalesce to discuss different techniques for thinking about the same problem, ACT conferences can be thought of as assembling researchers from different fields interested in applying the same techniques to different problems.\footnote{I have first seen CT described as a ``battle-tested system of interfaces'' with a deep talent pool of researchers in \cite{lynch_category_2023}, while the way of presenting ACT conferences is something I first heard from Jamie Vicary.}

So, what is category theory, precisely?
We delay a more detailed introduction to \cref{sec:categories_and_systems}, and instead here present the following quote from \cite{leinster_basic_2014}, accompanied by a graphic from \cite{rieck_machine_2020}:

\begin{center}
``Category theory takes a bird’s eye view of mathematics. From high in the sky, details become invisible, but we can spot patterns that were impossible to detect from ground level.''
\end{center}

\begin{figure}[H]
  \centering
  \includegraphics[width=0.71\textwidth]{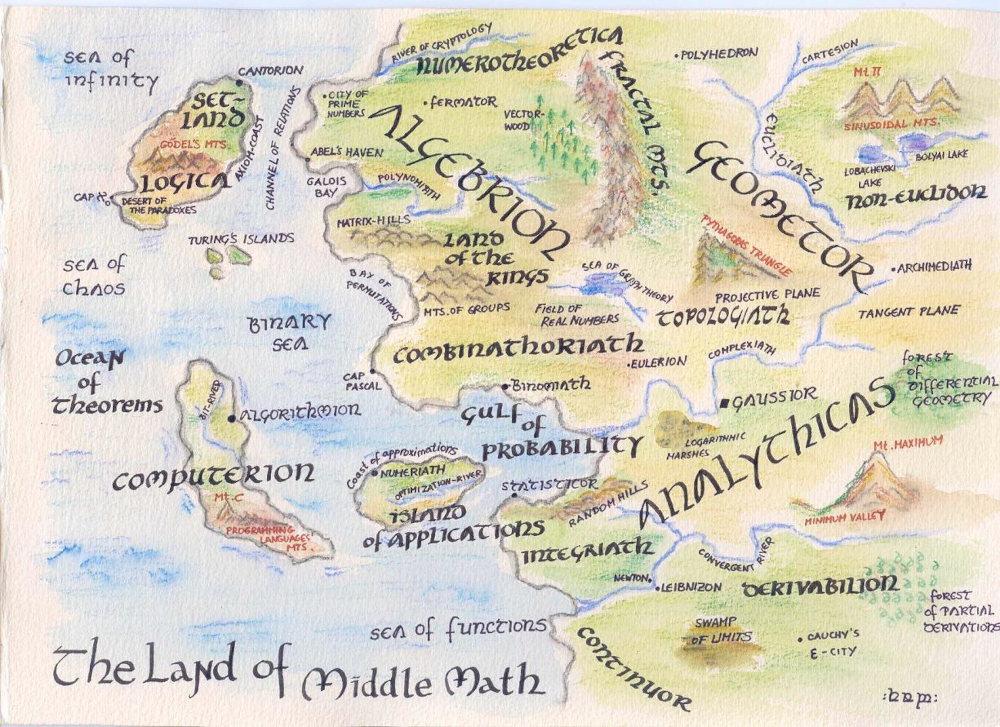}
  \label{fig:middle_math}
\end{figure}

Personally, I found it astounding the kinds of things that are possible to express in this language.
It changed the way I approach and think about problems, in a way that is hard to express.\footnote{I often say that category theory takes time to learn not because it's complex, but because it's so unintuitively simple.}
Hopefully a part of the spark that guides me towards category theory is transmitted throughout this thesis.

\begin{mybox}[label=box:ct_fields]{Cerulean}{Which fields has Category Theory permeated?}
  Other than modern mathematics which it thoroughly permeates, category theory has been used to describe:\footnote{We give no pretense of thoroughness, and note that many of these fields overlap. We list only one or two references but in most cases there are dozens, sometimes hundreds more.}
  \vspace{-1.5em}
  \begin{multicols}{2}
  \begin{itemize}
  \item Quantum physics (\cite{heunen_categories_2019, coecke_picturing_2017}) %
  \item Chemistry (\cite{baez_applied_2022, baez_compositional_2017, genovese_categorical_2020}) %
  \item Systems theory (\cite{capucci_towards_2022,  myers_categorical_2022}) %
  \item Functional programming (\cite{milewski_category_2019}) %
  \item Recursion schemes (\cite{hinze_unifying_2013, kabanov_recursion_2006})
  \item Database theory (\cite{spivak_functorial_2023, meyers_fast_2022})
  \item Game theory (\cite{ghani_compositional_2018, capucci_translating_2022} %
  \item Information theory (\cite{leinster_entropy_2021, bradley_entropy_2021})
  \item Control theory (\cite{hanks_compositional_2023, baez_categories_2015}) %
  \item Probability (\cite{sturtz_categorical_2015, heunen_convenient_2017, perrone_markov_2022})
  \item Cryptography (\cite{broadbent_categorical_2022, pavlovic_chasing_2014, barcau_composing_2023})
  \item Automata theory (\cite{boccali_bicategories_2023, arbib_categorists_1975})
  \item Trading protocols (\cite{genovese_escrows_2021})
  \item Ergodic theory (\cite{moss_category-theoretic_2023, delvenne_category_2019})
  \item Thermostatics \cite{baez_compositional_2023}
  \item Reinforcement learning (\cite{hedges_value_2023})
  \item Electrical circuits (\cite{baez_compositional_2018, boisseau_string_2022})
  \item Version control (\cite{mimram_categorical_2013})
  \end{itemize}
\end{multicols}

\end{mybox}

Lastly, to understand how category theory compares to other mathematical theories, we present the following remark.

\begin{remark}[Differences between category, set, and graph theory.]
  One might rightfully point out that set and graph theory, for instance, are all a part of the above mentioned fields.
  This raises the question: why is category theory singled out?
  Essentially, category theory tracks more information.
  While we can study everything with only sets and functions, we effectively start doing category theory as soon as we start tracking properties of these sets and functions in a consistent manner.
  If we're only interested in sets with some structure and functions that preserve that structure --- closed under composition --- we get a category. 
  When it comes to graph theory, the story is similar.
  Graphs and graph homomorphisms, when organised into collections, indeed form a category as well --- not just a graph.
  This is symptomatic of many constructions in various aspects of mathematics: they all assemble into categories.
  The extra data that category theory \emph{coherently} tracks ends up being an indespensible tool for describing numerous scientific and mathematical phenomena in a systematic manner that would not be possible without it. 
\end{remark}

\subsection{Category Theory $\cap$ Machine Learning}

In the recent years, a number of authors have began to study machine learning concepts through the lens of category theory.

\begin{figure}[H]
  \centering
  \includegraphics[width=1\textwidth]{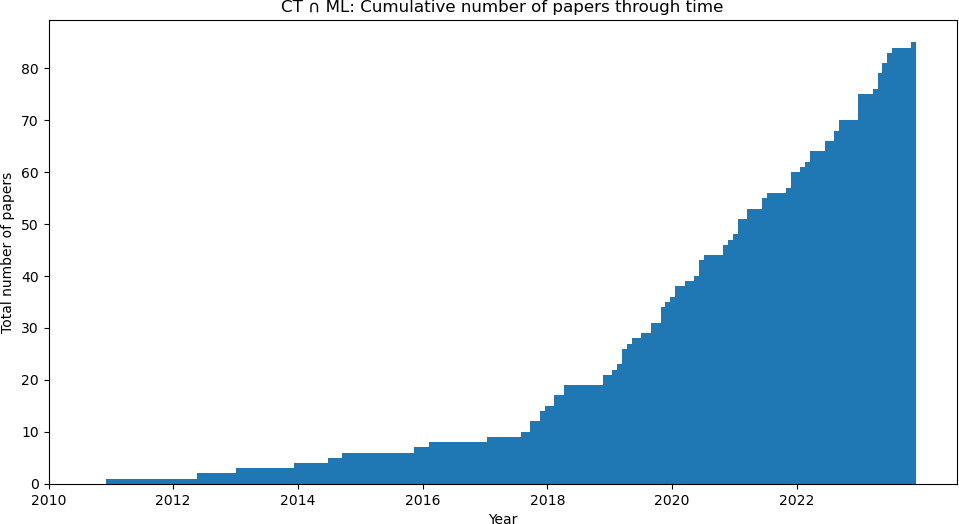}
  \caption{Category Theory $\cap$ Machine Learning: cumulative number of papers through time. Data and figure taken from \cite{gavranovic_category_2020}.}
  \label{fig:ct_ml_papers}
\end{figure}

\Cref{fig:ct_ml_papers} shows the total number of such papers so far, including papers in deep learning, bayesian and causal inference, general probability, differentiable programming, topological data analysis, and more.
While the total quantity of research is still relatively low, there is a steady rise.
The current stage of development of the machine learning in the context of category theory is best understood through the vantage point emphasized by the following prescient quote from \cite{baez_compositional_2018}\footnote{Unfortunately, this quote can only be found in version 1 of the paper on ArXiv.}:

\begin{displayquote}
  ``While diagrams of networks have been independently introduced in many disciplines, we do not expect formalizing these diagrams to immediately help the practitioners of these disciplines.
  At first the flow of information will mainly go in the other direction: by translating ideas from these disciplines into the language of modern mathematics, we can provide mathematicians with food for thought and interesting new problems to solve.
  We hope that in the long run mathematicians can return the favor by bringing new insights to the table.'' 
\end{displayquote}

This quote indirectly suggests partitioning applied category theory fields into two groups.
We call these \TypeOne and \TypeTwo fields of ACT.
In a \TypeOne field, the flow of information still only goes from the application to the theory, and all compositionality that the category theory describes already exists in the application domain.
A \TypeTwo field genuinely carves out new compositionality in the domain, bringing \emph{new insights} to the table that were not known before.

While merely a rough guideline, it helps give us a sense where Categorical Deep Learning is: currently a borderline Type 2 field.
It can describe \emph{a lot} of existing compositionality, but hasn't quite justified itself with many new insights.
In this thesis, we hope to provide a stable foundation that can enable these insights to happen.

\section*{Goals and Challenges}

The goal of this thesis is the provision of a mathematical foundation for artificial neural networks that is
\begin{itemize}
  \item \textbf{End-to-end.} We aim to deal with the entirety of neural network theory and practice --- from implementation details of backpropagation, over the algebra of parametric composition of neural networks, to specific architectures, optimisers, and a complete description of supervised learning. 
  \item \textbf{Uniform.} We take special care ensuring all the described pieces fit in a consistent way, through a uniform and consistent language.%
  \item \textbf{Prescriptive, not merely descriptive.} The aim is to provide a prescriptive framework for \emph{implementing} all of the above, one which can eventually become a part of future deep learning frameworks in programming languages with sufficiently expressive features.
\end{itemize}

To achieve this, we utilise the aforementioned language of category theory to its full extent.
We bring many constructs and insights from this abstract field of mathematics into deep learning, developing an entire foundation for deep learning in this language.

The thesis structure is as follows.

\paragraph*{Part I: Theory}

Part I covers two main components of deep learning: parameterisation and bidirectionality.
Its main contribution is a formal mathematical framework for studying the essential structure behind these concepts.
It is written with more than just deep learning in mind: the level of abstraction covers many other kinds of bidirectional processes found in machine learning such as bayesian updating and value iteration, and also having close ties to game theory.
As such, despite my best attempts, the level of abstraction here is intended to a seasoned category theorist.

In \cref{ch:para}, I build upon the $\Para$ construction in its full bicategorical nature, unpacking all the underlying structure in detail.
I define the conditions under which it can be strictified, equipped with a monoidal structure, and trivialisation conditions, as well as relate it to existing definitions in the literature.
In \cref{ch:bidirectionality}, I define the category of \emph{weighted optics}, a new foundation for bidirectional systems that allows modelling backpropagation not just denotationally correct, but also one which operationally captures distinctions relevant in practical implementations.
I define cartesian actegories and show how they are useful in studying particular kinds of weighted optics: lenses.
\Cref{ch:para_optic} combines the previous two chapters, defining a bicategory in which morphisms are \emph{parametric weighted optics}, the central construction of this thesis.
This is a versatile construction that models neural networks, loss functions, and whose 2-cells also model optimisers.
It also a category in which scalars (morphisms from the monoidal unit to the monoidal unit) model a parameter update step of an entire supervised learning system.

\paragraph*{Part II: Applications}

Part II justifies the abstractions from Part I, instantiating and using them to describe settings relevant for deep learning: backpropagation, neural network architectures, and supervised learning.
In \cref{ch:backpropagation} I provide a lens-theoretic axiomatisation of differentiation.
It gives an abstract view of differentiation covering more than just the standard smooth spaces, but also discrete settings of boolean circuits, while using the same general framework for bidirectionality.
Specifically, this includes the definition of a particular kind of an additively closed category, definition of $\LensA$ as an endofunctor and the definition of a generalised cartesian reverse derivative category as its coalgebra.
I also identify differences between checkpointed and non-checkpointed reverse-mode automatic differentiation solely in the language of category theory.
\Cref{ch:architectures} describes how the framework of parametric weighted optics can be used to provide a solid foundation for defining a general theory of neural network architectures.
I develop the a categorical formalism for weight tying, generative adversarial networks and graph convolutional neural networks, to name a few, as well as survey existing categorical work on recurrent and recursive neural networks.
Lastly, \cref{ch:supervised_learning} extends the previous frameworks of supervised learning in \cite{cruttwell_categorical_2022}, and puts all the pieces together.
I describe how learning rates and optimisers can be modelled and composed within this framework, eventually in conjunction with a model and a loss function forming a morphism that compositionally captures the notion of supervised learning of parameters.
This is instantiated in a variety of settings, even those that are not usually considered as supervised learning, such as generative adversarial networks, shedding light on their semantics as supervised learners.
I also describe how this framework captures supervised learning of \emph{inputs}, a method usually called \emph{deep dreaming} in the literature.

\paragraph*{Future work.}
While great care has been taken to separate out as many conceptual moving parts as possible, some things are still stuck together.
The last and the shortest chapter of this thesis lists a number of open research questions, branching out in multiple directions.
From establishing bridges between differential geometry, tangent categories and practical concerns of automatic differentiation (such as gradient checkpointing), over the design of new architectures and connection of Geometric Deep Learning (\cite{bronstein_geometric_2021}) to structural recursion, to ultimatively, \emph{taking dependent types seriously}.

\newpage
\subsection{Publications and preprints}

The ``\textbf{Contributions}'' environment at the end of the introduction of each chapter in Part I and II summarises the novel research done in that chapter, alongside with the relevant publications and preprints I (co)authored.
The complete list of these preprints and publications is below.

\vspace{1cm}

\begin{minipage}[t]{0.13\textwidth}
    \cite{cruttwell_categorical_2022}\\
    \\
    \cite{capucci_actegories_2023}\\
    \\
    \cite{shiebler_category_2021}\\
    \\
    \cite{capucci_actegories_2023}\\
    \\
    \cite{braithwaite_fibre_2021}\\
    \\
    \cite{gavranovic_space-time_2022}\\
    \\
    \cite{gavranovic_graph_2022}\\
    \\
    \cite{ghani_compositional_2018}\\
\end{minipage}
\begin{minipage}[t]{0.95\textwidth}
    \textbf{Categorical Foundations of Gradient-Based Learning}\\
    \\
    \textbf{Towards Foundations of Categorical Cybernetics}\\
    \\
    \textbf{Category theory in machine learning}\\
    \\
    \textbf{Actegories for the working amthematician}\\
    \\
    \textbf{Fibre optics}\\
    \\
    \textbf{Space-time tradeoffs of lenses and optics via higher category theory}\\
    \\
    \textbf{Graph Convolutional Neural Networks as Parametric CoKleisli morphisms}\\
    \\
    \textbf{Compositional game theory, compositionally}\\
\end{minipage}

  \part{Theoretical Foundations}
  \label{part:theory}
  \chapter{Setting the Stage}
\label{ch:settting_the_stage}

\epigraph{To deal with hyper-planes in a 14-dimensional space, visualize a 3-D space and say 'fourteen' to yourself very loudly. Everyone does it.}{Geoffrey Hinton}

\newthought{What is Deep Learning,} more closely? What about Category Theory? In this chapter we provide the necessary background material.

\section{Deep Learning: a bird's eye view}

Deep Learning is a subfield of the field of machine learning that studies the theory and practice of \emph{artificial neural networks}.
Loosely inspired by neural networks in the brains of humans and animals, the goal of deep learning is the construction of networks of interconnected artificial neurons that have the ability to learn from data by seeing lots of examples.

Consider one of the most common deep learning scenarios: supervised learning with a neural network.
This technique trains the model towards a certain task, i.e.\ the recognition of visual patterns in an image data set (\cref{fig:learning_sketch}).
There are several diferent ways of implementing this scenario.
Typically, at their core, there is a gradient update algorithm (often called the “optimiser”), depending on a given loss function, which updates in steps the parameters of the network, based on some learning rate controlling the scaling of the update.
All of these components can vary independently in a supervised learning algorithm and a number of choices is available for loss maps (quadratic error, Softmax
cross entropy, dot product, etc.) and optimisers (Adagrad \cite{duchi_adaptive_2011}, Momentum \cite{polyak_methods_1964}, and Adam \cite{kingma_adam_2015}, etc.).
Furthemore, the setting in which we perform the learning can vary as well: in addition to euclidean spaces, we can differentiate functions between hyperbolic spaces (\cite{peng_hyperbolic_2022}), between complex numbers (\cite{bassey_survey_2021}), and even between discrete settings of boolean circuits \cite{hubara_binarized_2016, qin_binary_2020, wilson_reverse_2021}.

\begin{figure*}[h]
  \centering
	\includegraphics[width=0.8\textwidth]{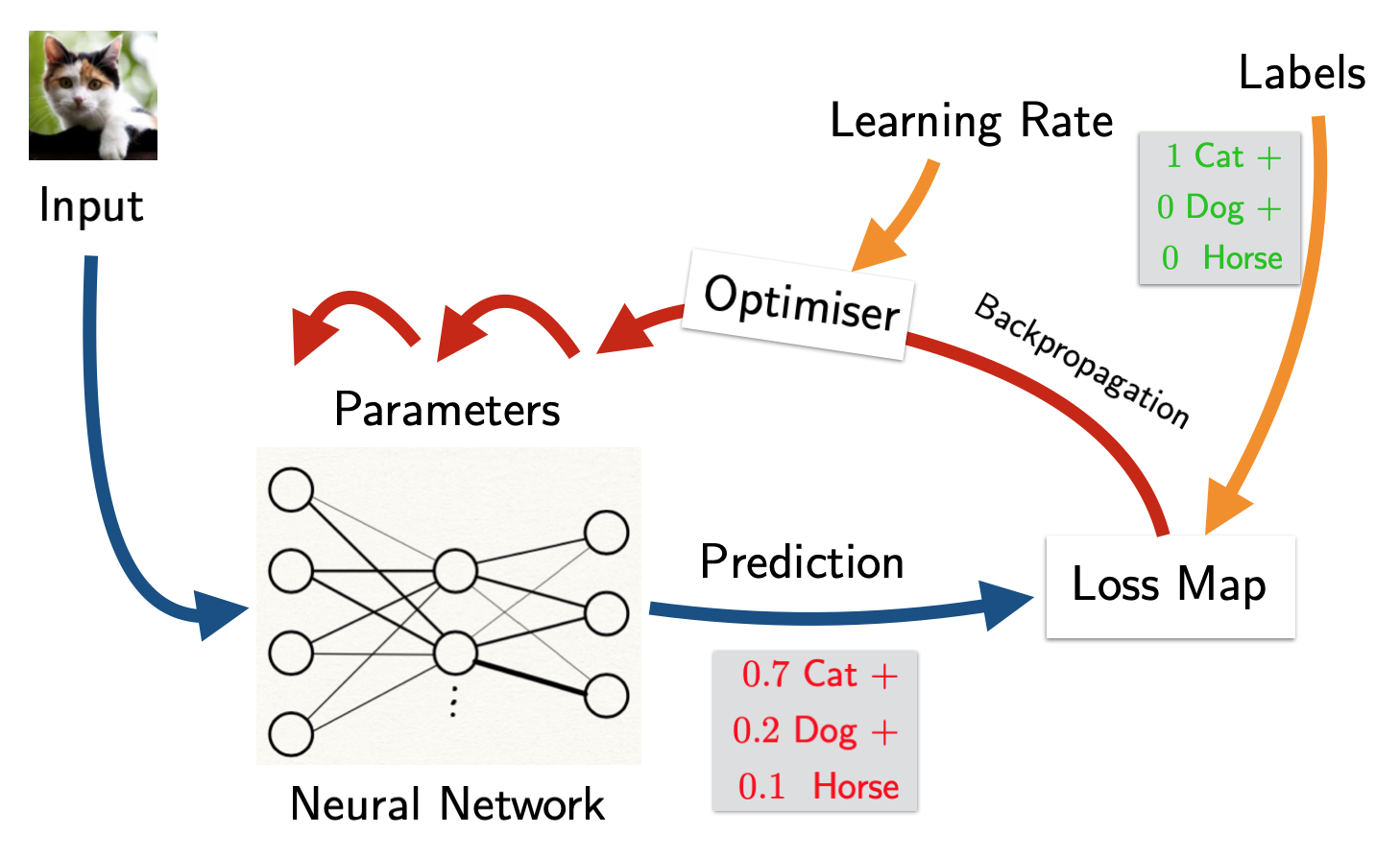} \caption{An
    informal illustration of gradient-based learning. This neural network is
    trained to distinguish different kinds of animals in the input image. Given
    an input $X$, the network predicts an output $Y$, which is compared by a
    \emph{loss function} with what would be the correct answer (\emph{label}). The loss function returns a real value expressing the error of the prediction; this information, together with the \emph{learning rate} (a weight controlling how much the model should be changed in response to error) is used by an \emph{optimiser}, which computes by gradient-descent the update of the parameters of the network, with the aim of improving its accuracy. The neural network, the loss map, the optimiser and the learning rate are all components of a supervised learning system, and can vary independently of one another.}
  \label{fig:learning_sketch}
\end{figure*}

Thinking of neural networks as processes, we can delineate their two main properties.
\begin{enumerate}
\item They are \textbf{parametric.} Each network comes equipped with a choice of ``weights'' (or parameters) on which the output depends on.
In the simplest case, we are given a function $f : X \times P \to Y$ and are tasked with finding a parameter $p : P$ such that $f(-, p) : X \to Y$ is the best function according to some criteria.
As we compose neural networks, the manner by which the weight spaces are composed too is non-trivial, and the algebra of such composition is often overlooked.
This is the topic of \cref{ch:para}.
\item They are \textbf{bidirectional.} These processes have a forward and a backward direction.
In the forward direction the computation turns inputs via a sequence of layers into predicted outputs, and then a loss value.
In the reverse direction it backpropagates changes through the layers, and then turns them into parameter updates.
This too has a non-trivial algebra of composition, with a) different modes of composition giving rise to different performance profiles of these bidirectional processes, and b) details and properties of both the forward and backward part being difficult to formally specify in a uniform and consistent manner.
This is the topic of \cref{ch:bidirectionality}.
\end{enumerate}

This gives us the high-level story of deep learning.
While deep learning is our main motivation, we keep in mind adjacent machine learning fields such as bayesian learning and reinforcement learning.
Throughout Part I of the thesis a lot of the abstract machinery for bidirectionality will inadvertently end up capturing structure that is present in Bayes' law, or value iteration.
As such, Part I is developed at a level of generality not specifically needed for neural networks.
But we will see that this level of generality gives us a fresh and useful perspective on neural networks themselves.

The main concern of deep learning is however, in many ways, much more specific.
Most of deep learning research studies the manner by which we structure the backward pass from the forward one (i.e.\ backpropagation), the manner by which we structure the parametric morphisms (i.e.\ the \emph{architecture} of the forward passes), and the manner by which all of these components fit into a coherent story of \emph{supervised learning}.
These three components define the contents of the three chapters in Part II of the thesis.

\begin{enumerate}
\item \textbf{Differentiation.} The forward and the backward pass of a bidirectional process are related: the backward pass is automatically constructed from the forward one by the process of differentiation. 
We are interested in providing formal models of each that capture all the relevant examples, but can also be integrated in a coherent whole.
This means that we should systematically be able to change the kind of space we're differentiating, or the kind of parameters we're considering, in and still be able to perform supervised learning.
Likewise, here we too want to be operationally aware, and distinguish between different performacne profiles of different modes of composition.
\item \textbf{Architectures.} In recent years the number of architectures of neural networks has proliferated. There exist countless architectures, combinations theoref, and entire theories describing how each class of architectures operates, and their expressive power. 
As such, there is a large need for a principled and formal specification language for these architectures, one which would obviate the need for informal diagrams that are used by most of today's deep learning literature.
\item \textbf{Supervised learning.} The most pervasive mode of deep learning is that by \emph{supervision}, where a dataset of input-output examples is used to specify an optimisation procedure which produces a function mapping each input to the corresponding output as closely as possible, while ensuring it generalises well.
This supervised learning system involves the interaction of a neural network, an optimiser, a loss function, and their corresponding backward passes, and forms a good litmus test for any end-to-end framework: how well can it explicitly describe the interaction of all of these components?
\end{enumerate}

Having given a high-level overview of deep learning, we now turn to category theory.

\section{Categories and Systems}
\label{sec:categories_and_systems}

\begin{center}
  \emph{You shall know a word by the company it keeps.}
\end{center}

This prescient quote from the linguist John Rupert Firth could be seen as a backbone of category theory.\footnote{Similar phrases have been uttered by various people. ``The meaning of words lies in their use.'' is one by Wittgenstein.}
It tells us that the string of characters a word is composed out of --- its syntactic content --- is in many ways meaningless.
Instead, what gives a word much of its meaning is the way we use it in relation to all the other words in the language.

Describing things \emph{exclusively} by how they can be related and interacted with from the outside, without looking at their insides, is one way to frame what category theory is about.
From its humble beginnings in 1940s, the field of category theory has experienced immense growth, now becoming what is called ``the mathematics of mathematics``.
It grew into an entire language in which people have described all sorts of scientific phenomena in, coming with its own markedly distinct way of reasoning and thinking.

One of the foundational principles that sets category theory apart is its rigorous approach to \emph{encapsulation}.
Just as in software we aim for a modular design of components, so in category theory we aim for a modular design of \emph{concepts}.
Ideas and concepts are abstracted away and separated from each other in a principed way, whose interaction is mediated by precise and mutually compatible interfaces.

Take the example of \emph{enriched} category theory.
In it we can turn a dial that chooses a category and, poetically, selects different branches of mathematics. Setting it to $\{ \text{false}, \text{true} \}$ selects order theory.
Turning it to $[0, \infty]$ selects metric geometry.
Turning it to $\Set$ selects category theory itself.
Turning it to $\Cat$ selects 2-category theory.
And so on.\footnote{Paraphrased from \cite{leinster_magnitude_2011}.}
In the study of bidirectional systems --- there is a dial that selects the kind of bidirectional systems we study.
One position selects deterministic updates, another one probabilistic, another on updates with computational effects, and so on.
Enriched category --- and category theory as a whole --- is extremely general, and most importantly --- reliable.

If something requires a category with some structure $X$ to work, we are assured that \emph{every} category with structure $X$ works; it will never be the case that the construction breaks because of an unforeseen consequence. 
This is because we are working through a well-specified interface, shielded from the specifics of a particular domain.
And we have to learn a particular interface only once, as opposed to repeatedly relearning ad-hoc interfaces for each new domain.
The mutual compatibility of these interfaces in turn allows us to scale up our systems more easily, and manage their difficulty of dealing with them as we add moving parts (\cref{fig:traditional_vs_structural_method}).

\begin{figure}[h]
  \centering
  \includegraphics[width=.7\textwidth]{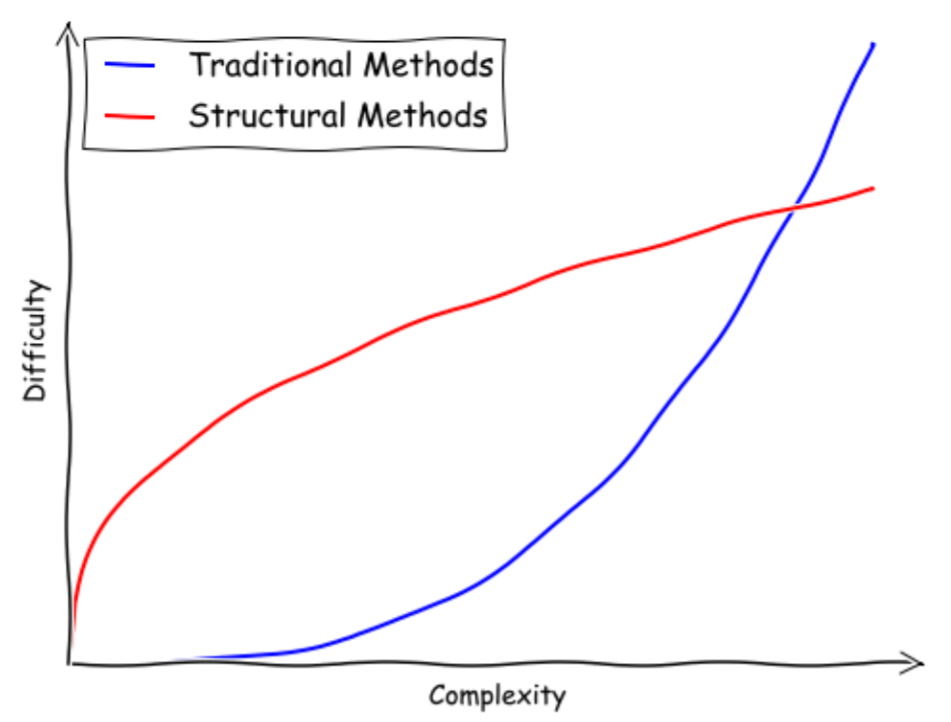}
  \caption{In traditional methods it is easy to start building a system. But as the system grows, it becomes difficult to understand how all the moving pieces inside interact.
    This slows down development, and introduces bugs and side effects.
    In structural methods, the situation is the opposite.
    It takes longer to start, but as special care is taken in accounting for all the moving pieces, it becomes easier to manage their complexity, and scale these systems up.
    (Figure taken from \cite{breiner_workshop_2018})}
  \label{fig:traditional_vs_structural_method}
\end{figure}

On the other hand, using category theory requires a large buy-in.
If you're coming from computer science, physics, or other adjacent fields it takes much more time than you're probably accustomed to pick up and apply these concepts.
The learning curve is steep, and from my humble experience category theory was just \emph{very surprising}; my meta-model of how my thinking would change by using it was completely off.
I did not expect \emph{this} level of change.
Category theory cannot easily be picked up in a month, and the sheer number of concepts and ideas in it takes a lifetime to understand.

For the reasons above I will not provide a comprehensive introduction to category theory in this thesis.
If you already know category theory, you will not need it, and if you don't, you will not learn it from here.
I will take \emph{some} time to give a sense of some of the main concepts appearing throughout the thesis with no pretense of thoroughness.
The thesis does quickly ramps up in complexity, and someone new to category theory might struggle with the terminology and the way of thinking in this thesis.
I refer the reader to the list of pedagogical resources for learning category theory in \cite{gavranovic_category_2023}.

Here are some of the main categories used throughout this thesis.

\begin{definition}[$\Set$]
  \label{def:set}
  One of the central categories in category theory is $\Set$, the category whose objects are sets and morphisms are functions.
\end{definition}

\begin{definition}[{$\Smooth$, \cite[Example 2.3]{cockett_reverse_2020}}]
  \label{def:smooth}
  We write $\Smooth$ for the category whose objects are real-valued vector spaces of the form $\R^n$ (i.e.\ objects are effectively a choice of a natural number $n : \N$) and morphisms are smooth functions $\R^n \to \R^m$.
\end{definition}

\begin{definition}[{$\Poly_R$, \cite[Example 2.2]{cockett_reverse_2020}}]
  \label{def:poly}
  Let $R$ be a commutative rig.\footnote{Also known as a commutative semiring.}
  We use $\Poly_R$ to denote the category of polynomials with coefficients in $R$.
  Its objects are natural numbers $n : \N$, and a morphism $p : n \to m$ is an $m$-tuple of polynomials $\langle p_1(x), p_2(x), \dots, p_m(x) \rangle$ where $p_i(x)$ is an element of $R[x_1, \dots, x_n]$, the polynomial ring over $R$ in $n$ variables.\footnote{Some readers of this thesis might be familiar with the category of polynomial functors (briefly mentioned in the literature review in \cref{sec:bidirectionality_literature}) which is often denoted by $\NamedCat{Poly}$ too. This is a different construction!}
\end{definition}

The above definition is dauntingly formal, but is best understood with an example: a morphism $2 \to 1$ in $\PolyR$ is the polynomial $p(x_1, x_2) = 5x_1^2 + 3x_1x_2 + 2x_2^2$.
This is a morphism with codomain $1$; if it was instead $n$, we'd have $n$ polynomials at our disposal.
Examples of commutative rigs relevant for this thesis are $\R$ and $\Z_2$.

The categories $\Smooth$ and $\PolyR$ will be the main examples of categories in which we model supervised learning (\cref{ch:supervised_learning}).
The category $\PolyZ$ will be used to model supervised learning on boolean circuits (\cref{ex:boolean_circuits_learning}).

\begin{definition}[{$\FVect_F$}]
  \label{def:fvect}
  Let $F$ be a field.
  The category of finite-dimensional vector spaces $\FVect_F$ has as objects vector spaces over $F$ and as morphisms linear maps.
\end{definition}

In this thesis we will focus most on the base field $\R$, i.e.\ $\FVectR$.
\begin{definition}[{$\Mark$ (compare \cite[Ex. 2.6]{fritz_synthetic_2020} and \cite[Ex. 3]{hedges_value_2023})}]
  \label{def:markov_kernel}
  Let $\Mark$ be a category whose objects are finite sets, and where a morphism $X \to Y$ is a \newdef{Markov kernel}, i.e.\ a $Y \times X$ matrix with non-negative real entries, whose columns sum to $1$.
  Composition is given by matrix multiplication, and identities are given by the identity matrices.
\end{definition}

Morphisms in $\Mark$ are best thought of as probabilistic functions: to each element of the domain, they do not assign \emph{one} element of the codomain, but instead a probability distribution over all elements.
Markov kernels will be used to model bidirectional process whose forward passes are probabilistic (\cref{ex:mark_ker_expectation,ex:mark_ker_expectation_closed})

\begin{definition}[Natural numbers]
  \label{def:natural_numbers}
  The set of natural numbers $\N$ can be thought of as a category where a morphism $n \to m$ exists iff $n \geq m$.\footnote{This is an example of a \emph{thin} category, i.e.\ one where there is at most one morphism between any two objects.}
\end{definition}

All the categories above have infinitely many objects and morphisms.
Not all categories do --- some some have just a few.

\begin{definition}[Edge cases]
  \label{def:edge_cases}
  There is a category with no objects or morphisms, this is often called \textbf{the empty category}, and denoted by $\CatInit$.
  There is a category with only one object, and only the identity morphism on it, often called \textbf{the terminal category}, and denoted by $\CatTerm$.
  Every set $A$ can canonically be turned into a \textbf{discrete category}, one whose objects are elements of $A$, and morphisms are only identities.
\end{definition}

\begin{definition}[Product of categories]
  \label{def:product_caegories}
  Let $\cC$ and $\cD$ be two categories.
  Then we can form the \newdef{product category} $\cC \times \cD$ whose objects are pairs of objects $(C, D)$ where $C : \cC$ and $D : \cD$, and a morphism $(C, D) \to (C', D')$ is a pair $(f, g)$ where $f : C \to C'$ in $\cC$ and $g : D \to D'$ in $\cD$.
\end{definition}

There are numerous more, and countless combinations thereof.
We mention one last example --- the category of categories.

\begin{definition}[$\Cat$]
  \label{def:cat}
  In $\Cat$ the objects are categories themselves, and functors are morphisms between categories.\footnote{We're ignoring size issues in this thesis --- the usual techniques dealing with them apply here.}
\end{definition}

These are only some examples of concepts appearing in category theory.
We do not mention monads, algebras, limits, fibrations, adjunctions nor bicategories here, instead referring the reader to the aforementioned list of pedagogical resources for learning those (\cite{gavranovic_category_2023}).

In the appendix we mention some propositions related to coends (\cref{prop:weight_colimit_as_coend}), profunctors \cref{lemma:profunctor_representable}, and the (co)Yoneda lemma (\cref{prop:yoneda,prop:coyoneda}) which will be used in this thesis.
We do not introduce these concepts here, instead referring the reader to the invaluable resource for these concepts, the book ``Coend calculus'' (\cite{loregian_coend_2021}) whose wisdom I've consulted time and time again.\footnote{This reference does assume a fair bit of category theory. For a more pedagogical introduction to profunctors see \cite[Def.\ 4.8]{fong_invitation_2019}.}

\subsection{Monoidal categories}
\label{subsec:monoidal_categories}

Monoidal categories are a powerful tool in the applied category theory toolbox.
Analogous to categories which capture the algebra of sequential composition of processes, monoidal categories additionally capture the algebra of their parallel composition.
They have been used in quantum theory (\cite{coecke_picturing_2017}), digital and electric circuit theory (\cite{ghica_diagrammatic_2017, boisseau_string_2022}), game theory (\cite{ghani_compositional_2018}), bayesian learning (\cite{braithwaite_compositional_2023}), control theory (\cite{baez_categories_2015, bonchi_calculus_2017}), linear algebra (\cite{paixao_high-level_2022}), and many more fields.

Their widespread use arose because they come equipped with a systematic 2-dimensional language of string diagrams (\cite{street_monoidal_2012, selinger_survey_2011}).
Unlike in many other fields where pictures are merely informal supplements to proofs, here instead pictures \emph{are} formal proofs.
This is made possible by strict rules about what counts as a picture, and how they can be manipulated, giving us a fully formal graphical method of reasoning.

A monoidal structure\footnote{For the full definition, see \cref{appendix:monoidal_categories}. For a pedagogical introduction, see \cite[Sec. 4.4.3]{fong_invitation_2019}, or \cite{wang-mascianica_distilling_2023}.} on a category $\cC$ consists of a functor $\otimes : \cC \times \cC \to \cC$ (which we can think of as multiplication), and an object $I : \cC$ (which we can think of as the neutral element of the said multiplication), as well as other structure morphisms: the associator $\alpha$, and the left and the right unitor $\lambda$ and $\rho$, together with properties they have to satisfy fully described in appendix (\cref{def:monoidal_category}.
Functoriality of $\otimes$ gives us one of the above mentioned strict rules: the \emph{interchange law}.

\begin{proposition}[Interchange law]
  \label{prop:interchange_law}
  Let $(\cC, \otimes, I)$ be a monoidal category.
  Let 
  \[
    (f, h) : (A, D) \to (B, E) \quad \text{and} \quad (g, i) : (B, E) \to (C, F)
  \]
  be morphisms in $\cC \times \cC$.
  The interchange law tells us that the following equation holds
  \begin{equation}
    \label{eq:interchange}
    (f \otimes h) \comp (g \otimes i) = (f \comp g) \otimes (h \comp i)
  \end{equation}
  describing that we get the same result if we first compose the morphisms in parallel, and then in sequence, or in sequence and then in parallel.
\end{proposition}

When dealing with string diagrams, we do not have to explicitly think about equations such as \cref{eq:interchange}.
These symbolic equations are built-in to the geometry of the plane, alleviating us from dealing with the bureaucracy of equational reasoning.
This can best be seen in the example below.
\begin{equation}
  \label{eq:morphism_example_monoidal}
  (X \otimes (Y \otimes I)) \otimes Z \xrightarrow{(X \otimes \rho) \otimes Z} (X \otimes Y) \otimes Z \xrightarrow{f \otimes Z} (X' \otimes Y') \otimes Z \xrightarrow{\alpha_{X', Y', Z}} X' \otimes (Y' \otimes Z) \xrightarrow{X' \otimes g} X' \otimes T \xrightarrow{h} W
\end{equation}

This is a composite morphism which has many components, many of which are mere bookkeeping.
These are mostly rebracketing of terms, and introduction/elimination of the monoidal unit $I$.
The string diagram representation of the composite morphism above is depicted in \cref{fig:diagram_example}, where this bookkeeping is completely invisible: the unit $I$, associator $\alpha$, and the unitor $\rho$ are absorbed into the geometry of the plane.
This graphical representation aids intuition, and helps us see the high-level structure.

\begin{figure}[H]
  \scaletikzfig[0.7]{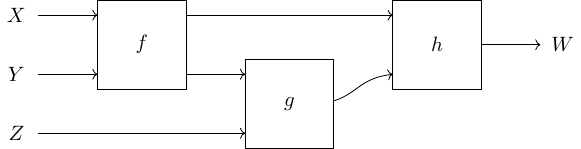}
  \caption{String diagram representation of the morphism \cref{eq:morphism_example_monoidal}.
    Objects of $\cC$ are denoted as wires, and morphisms as boxes. 
    Notably, unit $I$, the associator $\alpha$ and the unitor $\rho$ are completely invisible in the graphical representation.
 }
  \label{fig:diagram_example}
\end{figure}

All of the categories mentioned above ($\Set$, $\Smooth$, $\Poly_R$, $\FVect$, $\Mark$, $\Cat$, \dots) are monoidal categories, often in more than one way.
\begin{example}
The category $\Set$ is monoidal with its cartesian product $\times$ and the singleton set $1$, but also with the disjoint union $\sqcup$ and the empty set $\emptyset$.
This means that a morphism $X \otimes Y \to Z$ in $(\Set, \times, 1)$ describes a process which produces an element of $Z$ by consuming $X$ \emph{and} $Y$, while a morphism of the same type in $(\Set, \sqcup, \emptyset)$ describes a process which produces an element of $Z$ by consuming $X$ \emph{or} $Y$.
\end{example}

\begin{example}
  \label{ex:monoidal_cats_2}
The categories $\Smooth$ and $\Poly_R$ are monoidal with their cartesian product inherited from $\Set$, while $\Mark$ is monoidal with the cartesian product of sets on objects, and tensor product of matrices on morphisms (see \cite[Eq. 2.9]{fritz_synthetic_2020}).
The monoidal products of $\FVect$ are studied in detail in \cref{peculiarities_fvect}.
\end{example}

When reasoning about morphisms in monoidal categories, we often talk about \emph{open}  and \emph{closed} systems.
The morphism above is an example of an open system: roughly, it is one which exposes external ports other morphisms can be plugged in to.
A closed system does not.
Many systems can be partially closed in different ways, and we unpack the details of this in \cref{box:scs_monoidal}.
Such boxes will be recurrent characters in this thesis, and we will see how open and partially closed systems look like in as we vary the monoidal category.

\begin{mybox}[label=box:scs_monoidal]{Cerulean}{States, costates, and scalars in a monoidal category}
  The three kinds of (partially) closed systems in any monoidal category are called states, scalars, and costates.
  They are systems closed from respectively, the right, both sides, or left side.
  Graphically we represent them as below, drawing the partially closed side with a triangle, emphasising there are no exposed ports from that side. 
  Systems that are closed from both the left and the right are often called ``scalars'', following the intuition from monoidal categories like $(\FVectR, \otimes, \R)$ where maps of type $\R \multimap \R$ are in 1-1 correspondence with elements of $\R$.%
    \newline
    \begin{minipage}[t]{0.32\textwidth}
      \begin{tightcenter}
        States
      \end{tightcenter}
      \scaletikzfig[0.8]{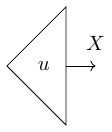}
      \begin{align*}
        \cC(I, X)
      \end{align*}
    \end{minipage}
    \hfill
    \begin{minipage}[t]{0.32\textwidth}
      \begin{tightcenter}
        Scalars
      \end{tightcenter}
      \scaletikzfig[0.8]{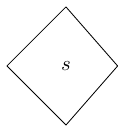}
      \begin{align*}
        \cC(I, I)
      \end{align*}
    \end{minipage}
    \hfill
    \begin{minipage}[t]{0.32\textwidth}
      \begin{tightcenter}
        Costates
      \end{tightcenter}
      \scaletikzfig[0.8]{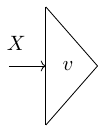}
      \begin{align*}
        \cC(X, I)
      \end{align*}
    \end{minipage}
  \end{mybox}

  Monoidal categories can come equipped with additional structure called \emph{braiding}.
  A braided monoidal category is one where wires can be crossed (\cref{fig:braiding}).
  \begin{figure}[h]
    \scaletikzfig[0.7]{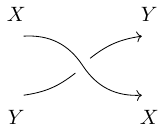}
    \caption{String diagram representation of a morphism $\beta_{X, Y} : X \otimes Y \to Y \otimes X$ in a braided monoidal category. Note how one wire is explicitly drawn on top of another. This is to emphasize that $\beta$ may not be equal to $\beta^{-1}$.}
    \label{fig:braiding}
  \end{figure}

  Formally this means that we additionally have a natural family of morphisms $\beta_{X, Y} : X \otimes Y \to Y \otimes X$ i.e.\ a natural isomorphism 
  \begin{tikzcddiag}
	{\cM \times \cM} & \cM
	\arrow[""{name=0, anchor=center, inner sep=0}, "\otimes", curve={height=-12pt}, from=1-1, to=1-2]
	\arrow[""{name=1, anchor=center, inner sep=0}, "{\otimes^{\rev}}"', curve={height=12pt}, from=1-1, to=1-2]
	\arrow["\beta", shorten <=3pt, shorten >=3pt, Rightarrow, 2tail reversed, from=0, to=1]
  \end{tikzcddiag}
  subject to axioms defined in the \cref{def:braided_monoidal_category}) in the appendix, where $\otimes^{\rev}$ is the reversed monoidal product on $\cM$ (\cref{def:reversed_monoidal_category}).
  Braiding often satisfies a particular property: one telling us braiding twice is the same as not doing it at all (\cref{fig:symmetric}).
  In this case, we call this category a \emph{symmetric} monoidal category.
  For a full formal definition \cref{def:symmetric_monoidal_category} in the appendix.
  
  \begin{figure}[h]
    \scaletikzfig[0.7]{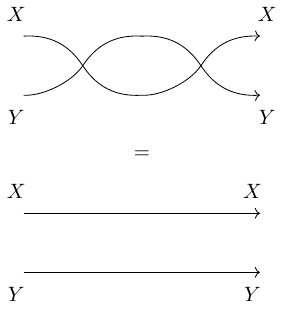}
    \caption{In a symmetric monoidal category braiding twice is \emph{equal} to the identity morphism. This means that $\beta = \beta^{-1}$, and in this case we do not distinguish between under- and over-crossings.}
    \label{fig:symmetric}
  \end{figure}

  Lastly, we mention that functors between monoidal categories can be monoidal in a few ways (lax, strong and strict) (\cref{def:lax_monoidal_functor}), and if the monoidal categories are braided (resp.\ symmetric), then the monoidal functor structure can additionally be braided (resp.\ symmetric) (\cref{def:braided_monoidal_functor}).

  \subsection{Cartesian monoidal categories}

  Cartesian monoidal categories are monoidal categories in which we think of processes as being \emph{deterministic}.
  They are often considered to live in the ``classical'' world (as opposed to quantum, or a probabilistic), as they allow unlimited copying and deleting of information.
  Cartesian monoidal categories are often referred to as ``cartesian categories'' where monoidality is implied.

  Formally, this means that the monoidal product has a particular universal property: it is \emph{a categorical product}.
  This means is that for any two objects $X, Y$ in $\cC$ there exist two maps out of $X \otimes Y$: $\pi_X : X \otimes Y \to X$ and $\pi_Y : X \otimes Y \to Y$ (called \emph{projections}), such that the triple $(X \otimes Y, \pi_X, \pi_Y)$ is the ``best cartesian product''.
  Being ``best'' means that any other candidate we might want to call a ``cartesian product'' (other such triples of objects $A :  \cC$ that come equipped with candidate projections $p_X : A \to X$ and $p_Y : A \to Y$) can be factored through $(X \otimes Y, \pi_X, \pi_Y)$.
  Formally, this means there exists a map $\langle p_X, p_Y \rangle : A \to X \otimes Y$ making the diagram (\cref{eq:product_univ_property}) commute.
  Moreover, this map has to be \emph{unique}, meaning there's no other maps satisfying this condition.
  
  \begin{equation}
    \label{eq:product_univ_property}
    \begin{tikzcddiag}[ampersand replacement=\&]
        \& A \\
        \\
        \& {X \otimes Y} \\
        X \&\& Y
        \arrow["{\pi_X}"', from=3-2, to=4-1]
        \arrow["{\pi_Y}", from=3-2, to=4-3]
        \arrow["{p_X}"', curve={height=12pt}, from=1-2, to=4-1]
        \arrow["{p_Y}", curve={height=-12pt}, from=1-2, to=4-3]
        \arrow["{\langle p_X, p_Y \rangle }"{description}, dotted, from=1-2, to=3-2]
      \end{tikzcddiag}
  \end{equation}

  In a cartesian category we often use $1$ to denote the monoidal unit (inspired by the monoidal unit of $\Set$ which is a one-element set), and $\times$ to denote their monoidal product (inspired by the cartesian product of sets).

  There are two equivalent approaches of showing a category is cartesian: a) by exhibiting the following (natural)  isomorphism
  \begin{equation}
    \label{eq:product}
      \cC(A, X \times Y) \cong \cC(A, X) \times \cC(A, Y)
  \end{equation}
  or b) by equipping every object $A : \cC$ with a unique comonoid structure $(A, \delta_A, \epsilon_A)$ and showing that all morphisms preserve comonoids (\cite{fox_coalgebras_1976}).

  The latter approach provides us with a formal way of augmenting the graphical language of the underlying monoidal structure of this cartesian category.
  That is, the comonoid maps $\delta_X$  and $\epsilon_X$ have special interpretations as the \emph{copy} $\scaletikzfignocenter[0.45]{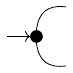}$ and $\emph{delete} \scaletikzfignocenter[0.5]{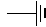}$ maps, drawn in a suggestive manner.
  The preservation of comonoid structure then unpacks to the following  equations:
  
\begin{figure}[h]
  \begin{subfigure}[c]{0.42\textwidth}
    \scaletikzfig[0.7]{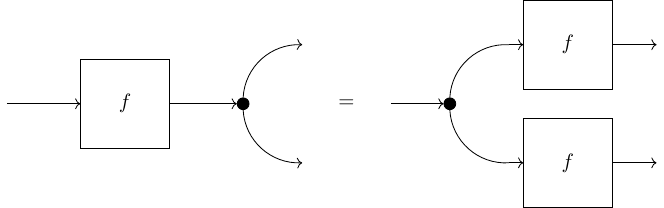}
    \label{fig:deterministic_morphism_copy}
  \end{subfigure}
  \hfill
  \begin{subfigure}[c]{0.42\textwidth}
    \scaletikzfig[0.7]{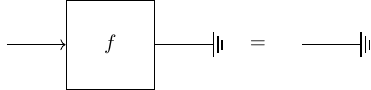}
    \label{fig:deterministic_morphism_delete}
  \end{subfigure}
  \caption{A morphism $f$ preserves comonoids if the above graphical equations are satisfied.}
  \label{fig:deterministic_morphism}
\end{figure}
describing (on the left) that applying the map $f$ and copying the result is the same as copying the result and applying $f$ individually to each copy.
This is not the case in probabilistic settings: rolling a dice and copying the output is not the same as rolling two die.
Likewise, the right equation describes that applying a map to an input, and then discarding the output is the same as just discarding the input.
In other words, there should be only one way to delete something.

  \begin{mybox}[label=box:scs_cartesian]{Cerulean}{States, costates, and scalars in a cartesian monoidal category}
    States, costates and scalars for cartesian monoidal categories follow those for monoidal categories.
    Additionally, maps into $1$ now trivialise.
    
    \begin{minipage}[t]{0.32\textwidth}
      \begin{tightcenter}
        States
      \end{tightcenter}
      \scaletikzfig[0.8]{state}
      \begin{align*}
             &\cC(1, X)\\
      \end{align*}
    \end{minipage}
    \hfill
    \begin{minipage}[t]{0.32\textwidth}
      \begin{tightcenter}
        Scalars
      \end{tightcenter}
      \scaletikzfig[0.8]{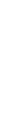}
      \begin{align*}
        &\cC(1, 1)\\
   \cong&1
      \end{align*}
    \end{minipage}
    \hfill
    \begin{minipage}[t]{0.32\textwidth}
      \begin{tightcenter}
        Costates
      \end{tightcenter}
      \scaletikzfig[0.8]{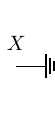}
      \begin{align*}
        &\cC(X, 1)\\
   \cong&1
      \end{align*}
    \end{minipage}
    \newline
    States do not reduce in general. Though, if $\cC$ is $\Set$, then $\Set(1, X) \cong X$.
    Scalars are drawn as an empty diagram, because there is only one possible map of type $1 \to 1$: the identity map.
    Costates also trivialise, because all morphisms have to be comonoid homomorphisms.
    And the counit preservation law of such a homomorphism implies that here can only be one morphism of type $X \to 1$: usually interpeted as the ``delete'' map described above.
  \end{mybox}

Most of the above mentioned monoidal categories are cartesian.
The notable exception is $(\Mark, \otimes, 1)$, which is not a cartesian category as its morphisms are probabilistic maps.

Given any cartesian category $\cC$ and any object $X$ therein we will always be able to form a particular category called $\CoKl(- \times X)$ (see \cref{subsec:local_vs_global}) that will allow us to describe properties of morphism that hold \emph{in one variable} (see \cref{def:additive_second_variable}).

\subsection{Monoids on every object = ``left-additive''}
\label{subsec:monoids_left_additive}

In addition to a category having the cartesian structure, we will often find that it comes equipped with a particular "additive" structure on each object.
For instance, in $\Smooth$ given any object $\R^n$ we can always define the operation $+ : \R^n \times \R^n \to \R^n$ pointwise via the addition of $\R$, and its netural element as a map $0 : \R^0 \to \R^n$ picking out zeroes.
Categorically this is captured by every object $A : \cC$ being equipped with a monoid structure in a manner that interacts coherently with the existing cartesian structure.
Such categories are often called \emph{cartesian left-additive} categories.

\begin{definition}[{\cite[Proposition 1.2.2.]{blute_cartesian_2009}}]
  \label{def:clac}
  A cartesian category $\cC$ is said to be cartesian \newdef{left-additive} (CLA) if every object $A : \cC$ is equipped with a commutative monoid $(+_A, 0_A)$ compatible\footnote{In particular, this compatibility means that the monoid maps $+_A : A \times A \to A$ and $0_A : 1 \to A$ are comonoid homomorphisms, giving a bimonoid structure to each object.} with the cartesian structure, i.e.\ such that for all objects $A, B : \cC$ the following diagrams commute:

  \begin{minipage}{.65\textwidth}
    \centering
    \[\begin{tikzcddiag}
    	{(A \times B) \times (A \times B)} && {A \times B} \\
    	\\
    	{(A \times A) \times (B \times B)}
    	\arrow["{+_{A \times B}}", from=1-1, to=1-3]
    	\arrow["{\interchanger_{A, B, A, B}}"', from=1-1, to=3-1]
    	\arrow["{+_A \times +_B}"', from=3-1, to=1-3]
    \end{tikzcddiag}\]
  \end{minipage}%
  \begin{minipage}{0.3\textwidth}
    \centering
    \[\begin{tikzcddiag}[ampersand replacement=\&]
        1 \&\& {A \times B} \\
        \\
        {1 \times 1}
        \arrow["{0_{A \times B}}", from=1-1, to=1-3]
        \arrow["{0_B \times 0_B}"', from=3-1, to=1-3]
        \arrow["\cong"', from=1-1, to=3-1]
      \end{tikzcddiag}\]
  \end{minipage}
\end{definition}

where here $\interchanger_{A, B, A, B}$ is the interchanger (\cref{def:interchanger}).

The categories $\Smooth$ and $\Poly_R$ are examples of cartesian left-additive categories.
Specifically for $\R$, the category $\PolyR$ is a cartesian left-additive \emph{subcategory} of $\Smooth$.
The categories $\Set$, $\Cat$ and $\N$ are not examples of cartesian left-additive categories.

\begin{remark}[Why ``left-additive''?]
  \label{rem:why_left_additive}
  A cartesian left-additive category is called so because for every $A, B : \cC$ we can always define a monoid struture on the hom-set $\cC(A, B)$.\footnote{Sometimes this monoid structure is also called \emph{convolution} (\cite[Sec. 4.1]{genovese_escrows_2021}).} 
  That is, we can define a function $+ : \cC(A, B) \times \cC(A, B) \to \cC(A, B)$ mapping $(f, g)$ to $f + g$ where 
  \[
    f + g \coloneqq \boxed{A \xrightarrow{\Delta_A} A \times A \xrightarrow{f \times g} B \times B\xrightarrow{+_B} B}
  \]
  and a function $0 : 1 \to \cC(A, B)$ mapping $\bullet : 1$ to $\boxed{A \xrightarrow{\terminal_A} 1 \xrightarrow{0_B} B}$.
  The adjective \emph{left}-additive is there because this monoid is not ``fully'' natural.
  Only composition from the left (pre-composition) fully preserves the additive structure (see \cite[Def.\ 1]{cockett_reverse_2020}).
  In what follows we will see that additionally preserving the hom-set monoid by composition from the right yields an \emph{additive category}.
\end{remark}

Categories that have monoids on every object come equipped with a graphical way of interpreting these monoids, analogously to the graphical language of categories with commutative \emph{comonoids} on every object (such as cartesian categories defined above).
That is, the monoid maps $+_X$  and $0_X$ have special interpretations as the \emph{sum} $\scaletikzfignocenter[0.45]{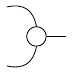}$ and $\emph{zero} \scaletikzfignocenter[0.5]{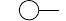}$ maps, drawn in a suggestive manner with string diagrams.

An arbitrary morphism in a cartesian left-additive category $\cC$ is not required to preserve its monoid structure.
If it does, we call it an \emph{additive} morphism.

\begin{definition}[{Additive morphism, \cite[Definition 1.1.1]{blute_cartesian_2009}}]
  \label{def:additive_morphism}
A morphism $f : A \to B$  in a cartesian left-additive category is \newdef{additive} if it
preserves the monoid structure, i.e.\ the following diagrams commute:

\begin{minipage}{.5\textwidth}
  \centering
  \[\begin{tikzcddiag}[ampersand replacement=\&]
      {A \times A} \&\&\& A \\
      \\
      {B \times B} \&\&\& B
      \arrow["{+_A}", from=1-1, to=1-4]
      \arrow["{f \times f}"', from=1-1, to=3-1]
      \arrow["f", from=1-4, to=3-4]
      \arrow["{+_B}"', from=3-1, to=3-4]
    \end{tikzcddiag}\]
\end{minipage}%
\begin{minipage}{0.5\textwidth}
  \centering
  \[\begin{tikzcddiag}[ampersand replacement=\&]
      1 \&\&\& A \\
      \\
      \&\&\& B
      \arrow["{0_A}", from=1-1, to=1-4]
      \arrow["f", from=1-4, to=3-4]
      \arrow["{0_B}"', from=1-1, to=3-4]
    \end{tikzcddiag}\]
\end{minipage}

\end{definition}

In equations, these diagrams unpack to ${f(a +_A a') = f(a) +_B f(a')}$ and $0_B = f(0_A)$.
Additive morphisms form a subcategory of $\cC$.

\begin{definition}
  \label{def:additive_maps}
  Let $\cC$ be a cartesian left-additive category.
  We use $\asc{\cC}$ to denote its subcategory of additive morphisms i.e.\ the category of commutative monoids and commutative monoid homomorphisms in $\cC$.
\end{definition}

Analogously to how a category where all maps preserve the underlying comonoid structure is cartesian (i.e.\ has products), one where they preserve the underlying monoid strucure is \emph{cocartesian} (i.e.\ has coproducts).
In this case, products and coproducts coincide, making this is biproduct category.

\begin{proposition}[$\asc{\cC}$ is a biproduct category]
  \label{prop:linsc_biproduct}
See \cite[Example 2.1]{cockett_reverse_2020}.
\end{proposition}

The categories $\PolyR$ and $\Smooth$ share the subcategory of additive maps, i.e.\ $\asc{\PolyR} = \asc{\Smooth}$, and this subcategory is equivalent to the category $\FVectR$ of linear maps between finite-dimensional vector spaces (\cite[Ex. 2.7]{ikonicoff_cartesian_2023}).

So far we have talked about additivity of a morphism with respect to the entirety of its domain.
If we want to reason about additivity (or any other property) holding in \emph{one variable} only, we can can make use of the previously mentioned category $\CoKl(- \times X)$.

\begin{definition}[{Additive in second component, (compare \cite[Lemma 1.2.3]{blute_cartesian_2009})}]
  \label{def:additive_second_variable}
A morphism $f : X \times A \to B$  is additive in the variable $A$ if it is an additive morphism of type $A \to B$ in the cartesian left-additive category $\CoKl(X \times -)$.\footnote{See \cref{subsec:local_vs_global} for the definition of $\CoKl(X \times -)$.}
\end{definition}

It can routinely be shown by unpacking the condition of an additive morphism that the condition of additivity here ends up pertaining only to the component $A$.

\subsubsection{Peculiarities of the category $\FVectR$}
\label{peculiarities_fvect}

The category $\FVectR$ is full of intricate structure.
The one most relevant to us are its multple monoidal structures.
Products in $\FVectR$ (the ones defined as in $\Smooth$) here also become coproducts, making $\FVectR$ a biproduct category.
Often this biproduct is denoted with $(\oplus, \R^0)$ since it can computed by taking the coproduct of the basis sets.
Taking the \emph{product} of the basis sets gets us the other important monoidal structure on $\FVectR$: the tensor product $(\otimes, \R)$.
This is a monoidal structure that is neither cartesian nor cocartesian, but has a separate universal property.%

This universal property states that $U \otimes V$ --- an object of $\FVectR$ --- comes equipped with a bilinear map $U \times V \to U \otimes V$ (this map is often also denoted by $\otimes$ since it takes two vectors and computes their tensor product) such that for any other object $X$ and any other bilinear map $U \times V \to X$ there is a unique \emph{linear} map $U \otimes V \to X$ making the diagram commute.\footnote{This diagram, perhaps surprisingly, does not live in $\FVectR$, but rather in $\Set$, or more precisely in the multicategory of finite dimensional vector spaces and multilinear maps, which we only mention for completeness.}
\begin{equation}
  \label{eq:universal_prop_tensor_product}
  \begin{tikzcddiag}[ampersand replacement=\&]
    {U \times V} \&\& {U \otimes V} \\
    \&\& X
    \arrow["{\text{bilinear}}", from=1-1, to=1-3]
    \arrow["{\text{bilinear}}"', from=1-1, to=2-3]
    \arrow["{\text{linear}}", dashed, from=1-3, to=2-3]
  \end{tikzcddiag}
\end{equation}
The category $\FVectR$ is closed with respect to this monoidal structure, where $\internalHom{\FVectR}{\R^n}{\R^m} = \R^{n \times m}$.
There is an identity-on-objects faithful functor $\iota : \FVectR \to \Smooth$, making $\FVectR$ a subcategory of $\Smooth$.
As $\Smooth$ is cartesian, and $\FVectR$ has two monoidal structures, it becomes natural to ask whether $\laxsmoothiota$ is a monoidal functor.
This is indeed the case: this functor becomes a monoidal functor in two different ways depending on the monoidal structure of the domain.
With $(\oplus, \R^0)$ this monoidal functor becomes a product-preserving functor, meaning it is strong monoidal.
Interestingly, with the tensor product on the domain $\iota$ is only lax monoidal.
\begin{lemma}
  \label{ex:iota_lax_monoidal}
  The functor $\laxsmoothiota : (\FVectR, \otimes, \R) \to (\Smooth, \times, \R^0)$ is lax monoidal, and its data is defined as follows.
  \begin{itemize}
  \item The laxator $\phi_{\R^n, \R^m}$, as $\laxsmoothiota$ is identity-on-objects, unpacks to a map of type $\R^n \times \R^m \to \R^n \otimes \R^m$. It is given by the bilinear structure map in \cref{eq:universal_prop_tensor_product} defining the tensor product of vector spaces.
  That is, the laxator $\phi_{\R^n, \R^m}(u, v) = uv^\top$ computes the outer product of $u$ and $v$.
  \item The unitor unpacks to a map of type $\R^0 \to \R$, and is given by the constant map at the multiplicative identity $1 : \R$.
  \end{itemize}
\end{lemma}

\subsection{Notation}
\label{subsec:notation}

We use uppercase letters $A, B, \dots$ to denote objects of categories and lowercase letters $f, g, \dots$ to denote their morphisms.
The identity morphism on an object $A$ is denoted by $\id_A$ or simply just $A$, trusting that is clear from the context that $A$ refers to a morphism, and not the object.
When defining composite morphisms, we often use the blueprint
\[
f \coloneqq \boxed{A \xrightarrow{f_1} B \xrightarrow{f_2} C \xrightarrow{f_3} D}
\]
to define $f : A \to D$ as the composite of $f_1$, $f_2$ and $f_3$.
Otherwise, we use the diagrammatic notation, i.e.\ $f = f_1 \comp f_2 \comp f_3$.

We use the boldface font for named categories ($\Set, \FVect_F, \Cat, \dots$), otherwise we use caligraphic letters ($\cC, \cD, \dots$).
In a 2-category, we use use $\to$ to denote 1-morphisms, and $\Rightarrow$ to denote 2-morphisms.
In $\FVect_F$ we often use $\multimap$ instead of $\to$ as arrows to emphasize the property of morphisms being linear.

In categories with products we label projections using $\pi$, subscripted with either indices, or names of objects.
For instance, the projection $A \times B \times C \to C$ might be written as $\pi_3$ or $\pi_C$.
For a composite $g \comp \pi_C$ we sometimes write $f_3$ to save space.
For categories with coproducts we label injections using $i$ analogously, i.e.\ the injection $B \to A + B + C$ might be labelled either as $i_{2}$ or $i_B$.
We sometimes write pairs as $\diset{A}{A'}$ to emphasize their usage in a bidirectional setting: the top object $A$ is the forward object, while the bottom object $A'$ is the backward one.
We always additionally mark the object of the backward pass with a superscript $'$.
For composition of parametric morphisms we often use the superscript $\compPara{}$, while for coparametric ones use the subscript $\compCopara{}$.
This applies to both sequential ($\compPara{\comp}, \compCopara{\comp}$) and parallel ($\compPara{\boxtimes}, \compCopara{\boxtimes}$) composition.\footnote{A good way to remember this is via their graphical languages: reparameterisations in $\Para$ are composed from the top, while reparameterisations in $\Copara$ from the bottom.}

We call the structure morphisms of both lax monoidal functors and lax functors ``unitors'' (when it comes to identities) and ``laxators'' (when it comes to composition).
We use $\cC / A$ to denote the slice category over $A : \cC$. %
When referring to internal hom objects of some category $\cC$, we write $\internalHom{\cC}{A}{B}$.
We write $1$ for the set $\{ \bullet \}$ with one element, $\CatTerm$ for the category with one object and one morphism, and, in a category with a terminal object, $\terminal_A$ for the unique morphism from some object $A$ to the terminal one.
Given an element of a set $a : A$, we use $\name{a} : 1 \to A$ for the ``name'' of $a$: the function picking out that element in $A$.

	\chapter{Parameterisation}
\label{ch:para}

\epigraph{The purpose of abstraction is not to be vague, but to create a new semantic level in which one can be absolutely precise.}{Edsger W. Dijkstra}

\newthought{Thought of as processes,} the forward passes of neural networks have a few distinguishing characteristics.

Firstly, the space of neural network weights i.e.\ their parameter space is an additional type of inputs to this process, playing a special role.
For instance, no matter whether we compose neural networks in sequence, or in parallel, the parameter spaces are always composed in parallel.
Likewise, it is also possible to reparameterise neural networks, for instance by tying weights together, or by reindexing them to a more convenient space.
All of these compositions and reparameterisations satisfy particular properties. For instance, reparameterising two morphisms individually and then composing the results is equal to composing the original morphisms and applying the product of original reparameterisations.
Or, reparameterising a morphism twice with different morphisms is equal to reparameterising it once with their composite.
But are there others?
Is there a compact way to describe this algebraic structure, and all the underlying invariants it possesses?

Second, the way neural networks are specified in the literature is often needlesly specific, or unnecessarily vague.
For example, the description of forward passes often relies on euclidean spaces, despite neural networks being well-defined in complex-valued spaces (\cite{bassey_survey_2021}), graphs (\cite{wu_comprehensive_2021}), simplicial and cellular complexes (\cite{papillon_architectures_2023}), and even boolean circuits (\cite{wilson_reverse_2021}).
What minimal structure do we need to talk about proceses with parameters, and their algebra of composition?

And lastly, in many different fields there are analogous concepts of ``hidden'' inputs:

\begin{itemize}
\item They manifest as weights in neural networks (\cite{fong_backprop_2021, cruttwell_categorical_2022});
\item They represent strategies of players in game theory (\cite{capucci_towards_2022, capucci_diegetic_2023});
\item They form closures of programs (\cite{new_closure_2017});
\item They serve as sources of randomness in probabilistic programming (\cite{heunen_convenient_2017});
\end{itemize}

In all of them we can reason about notions of reparameterisation, where for instance weight tying constraints two players to play the same strategy, or in probabilistic programming we can bind two random variables together.

Is there a uniform formalism that captures all of these?
Is there a formalism that is stable across changes of architectures, and one that will later permit dealing with the flow of gradients going backwards?
Likewise, analogous to the graphical language of string diagrams for processes in monoidal categories, is there an analogous graphical language for parametric processes?

In this thesis we give a positive answer to all of these, using the construction of \emph{actegories} (\cref{sec:actegories}) and the $\Para$ construction (\cref{sec:para}).

\begin{remark}
  I've made the pedagogical decision not to throw all the actegory and $\Para$ constructions and the definitions at the reader in this chapter.
  Instead, they appear throughout the next chapter as well, thoroughly motivated.
\end{remark}

\begin{contributions}
  Large fragments of this chapter contain novel research.
  While the definition of actegories and the definition of $\Para$ are not novel, a lot of the surrounding research around it is.
  This includes notions of reparameterisation of actegories (\cref{not:reparameterisation}), strictification and quotients of $\Para$ (\cref{box:frombicatto2cat,box:from2cattocat,box:frommon2cattomoncat}), trivialisation condition (\cref{thm:para_trivialises_when_unit_initial}), correspondence between local and global contexts (\cref{subsec:local_vs_global}), monoidal structure of actegories and $\Para$ (\cref{sec:monodial_actegories,sec:monoidal_para}).
  Parts of this chapter previously appeared in my publications \cite{capucci_actegories_2023,capucci_towards_2022}.
\end{contributions}

\begin{epistemicstatus}
  The fact that the formulation of $\Para$ yields a relatively elegant construction makes me confident that the research here is on the right track, especially when it comes to the emphasis on its 2-dimensional stucture as something that should be explicitly tracked, instead of quotiented out.
  On the other hand, I expect its formulation via actegories to be cleaned up.
  In  \cref{sec:act_para_literature} I describe two independent generalisations of actegories: locally graded categories and dependent actegories, both of which I believe remove many artifical identifications that are made in the actegorical framework. 
  Likewise, while the 2-dimensional graphical language of $\Para$ is freely used, no formal rules for this particular kind of a graphical language have been established in the literature, nor has its connection to the graphical language of double categories been studied.
  There is therefore some uncertainty about the actual boundaries of this language, despite which its use in the literature has only increased.
\end{epistemicstatus}

\section{Actegories and the $\Para[false]$ construction}
\label{sec:actegories}

The goal of this section is to establish the categorical structure necessary to reason about parametric processes.

In essence, we aim to construct a category where a morphism $A \rightarrow B$ encapsulates both a parameter space $P$ and the ``implementation'' map of type $f : A \otimes P \rightarrow B$.
Since the domain is a product, this necessitates the provision of a underlying monoidal category $C$.
For instance, this would allow us to reason about parametric maps $f : A \otimes (P \otimes Q \otimes R) \to B$ where $A$ and $B$ are classical inputs and outputs, respectively, while $P \otimes Q \otimes R$ is the parameter space.
However, this does not capture the full extent of what we can do with parameters: it falls short in scenarios where the types of parameters and the types of standard inputs differ. 

This calls for an additional layer of refinement --- one provided by ``actegories'', a term we use to describe actions of monoidal categories (\cite{capucci_actegories_2023,janelidze_note_2001}).
Actegories are not a new construction. They have previously appeared in numerous incarnations in the literature (see \cref{sec:act_para_literature}).
Actegories will serve as our formal framework to distinguish the types of parameters from the types of usual inputs.
We will see how properties of actegories translate to properties of various parametric morphisms, and many of the theorems here will also shed light on our subsequent discussions on bidirectionality (\cref{ch:bidirectionality}).
We give their description below, after which we describe how they can be used to talk about parametric morphisms.

\begin{definition}[{Actegory (\cite{capucci_actegories_2023, janelidze_note_2001})}]
  \label{def:actegory}
  Let $(\cM, \otimes, I)$ be a monoidal category.
  A $(\cM, \otimes, I)$-actegory $\cC$ consists of a category $\cC$ equipped with a functor
\[
  \act : \cC \times \cM \to \cC
\]
and two natural isomorphisms called \newdef{the unitor} and \newdef{the multiplicator}

\begin{equation}
  \label{eq:actegory_coherence}
\begin{tikzcddiag}[ampersand replacement=\&,sep=scriptsize]
	\cC \&\&\&\& {\cC \times \cM \times \cM} \&\& {\cC \times \cM} \\
	\\
	{\cC \times \cM} \&\& \cC \&\& {\cC \times \cM} \&\& \cC
	\arrow["\act", from=1-7, to=3-7]
	\arrow["\act"', from=3-5, to=3-7]
	\arrow["{\cC \times \otimes}"', from=1-5, to=3-5]
	\arrow["{\act \times \cM}", from=1-5, to=1-7]
	\arrow["\mu"', shorten <=15pt, shorten >=15pt, Rightarrow, 2tail reversed, from=1-7, to=3-5]
	\arrow[""{name=0, anchor=center, inner sep=0}, Rightarrow, no head, from=1-1, to=3-3]
	\arrow["\act"', from=3-1, to=3-3]
	\arrow["{\langle \cC , I \rangle}"', from=1-1, to=3-1]
	\arrow["\eta"', shorten <=6pt, shorten >=6pt, Rightarrow, 2tail reversed, from=0, to=3-1]
\end{tikzcddiag}
\end{equation}
whose components at each $C : \cC$ and $M, N : \cM$ are explicitly the maps
\begin{equation}
  \label{eq:actegory_coherence_components}
  \eta_c : C \xrightarrow{\cong} C \act I  \quad\quad \text{and} \quad\quad \mu_{C, M, N} : C \act (M \otimes N) \xrightarrow{\cong} (C \act M) \act N
\end{equation}

satisfying the coherence laws defined in \cref{def:act_coherence}.
\end{definition}

\begin{notation}
  We often refer to the above actegory as an ``$\cM$-actegory $(\cC, \act)$'', leaving $\otimes$, $I$, $\eta$ and $\mu$ implicit.
  When multiple actegories and their natural isomorphisms are in scope, we might refer to the natural isomorphisms as $\eta^\act$ and $\mu^\act$, superscripting the relevant action functor.
\end{notation}

This is often called a \emph{right} $\cM$-actegory, because we think of $\cM$ as acting on $\cC$ from the right side.
That is, acting with $M \otimes N$ on $C$ from the right side multiplies $C$ first by $M$, and then by $N$ (giving us $(C \act M) \act N$).
One can analogously define a \emph{left} $\cM$-actegory $\cC$, where the action happens from the left side, i.e.\ where acting with $M \otimes N$ on $C$ would multiply $C$ first by $N$, and then by $M$ (giving us $M \act (N \act C)$).
Categorically, the side from which ``an action happens'' is defined by the orientation of the monoidal product on $\cM$: in addition to $(\cM, \otimes, I)$ we can always form \emph{the reversed monoidal product} (\cref{def:reversed_monoidal_category}).
Then a right actegory defined on the monoidal category $(\cM, \otimes^{\rev}, I)$ is in the literature defined as a left $(\cM, \otimes, I)$-actegory. (\cite[Remark 3.1.3]{capucci_actegories_2023})

Even though we don't need any extra structure to \emph{define} right actegories in terms of left ones, in order to \emph{turn} a specific right actegory into a left one we will additionally need $\cM$ to be braided monoidal.

\begin{restatable}{lemma}{BraidedRightToLeft}
  \label{lemma:braided_right_to_left}
  If $(\cM, \otimes, I)$ is a braided monoidal category, then any right $(\cM, \otimes, I)$-actegory can be turned into a left $(\cM, \otimes^{\rev}, I)$-actegory, and vice-versa.
\end{restatable}

\begin{proof}
  We leave the full proof to the appendix (\cref{appendix:actegories}), here only mentioning that the braiding is required only when defining the multiplicator of this left-actegory, as it is the only component of the actegory interacting with the monoidal structure of $\cM$.
\end{proof}

In braided settings we will often omit any notational distinction between the two actions.

\begin{remark}
  \label{rem:act_left_vs_right}
  Even though actegories can be defined without braiding on the base, we will see that in order to state most theorems about them braiding will be necessary, and --- as it turns out --- existent in all our applications.
  Despite the equivalence of left and right actions under braiding, we will often make a notational distinction between them in the context of $\Para$ and $\Copara$ constructions (\cref{sec:para} and \cref{sec:copara}).
  This is inspired by \cref{sec:act_para_literature} where we will briefly descibe a generalisation of actegories called \emph{lax actegories} where such equivalence is nonexistent, and also by the general principle of being cognisant of the arrow of time: things which appear first in time will appear first in notation, reading left to right.
\end{remark}

\begin{example}
  \label{ex:self_action}
  If $\cC$ is a monoidal category, then we can consider it as acting on itself from either left or right.
  From the right side, the action is defined by $\otimes$, the unitor $\rho^{-1}$ and the multiplicator $\alpha^{-1}$.
  From the left side, the action is defined by $\otimes^{\rev}$, the unitor $\lambda^{-1}$ annd the multiplicator $\alpha$.
\end{example}

Examples abound.
For us the most relevant will be $\Set$, $\FVect$, and $\Smooth$, both of which have at least one monoidal structure.
Actegories allow us to additionally restrict the type of parameters involved in the action, something that is not possible within the framework of just monoidal categories.

\begin{example}
  \label{ex:monoidal_subcategory}
  Let $\cC$ be a monoidal category, and $\cB$ a monoidal category which comes equipped with a strong monoidal functor $E : \cB \to \cC$.
  Then we can form a $\cB$-actegory $(\cC, \actfw)$ with $\actfw$ defined as the composite $\cC \times \cB \xrightarrow{\cC \times E} \cC \times \cC \xrightarrow{\otimes} \cC$, where $\otimes$ is the monoidal product of $\cC$.
\end{example}
The above construction is a good example of a \emph{reparameterisation of actegories} which works for general action, and not just a self-action.

\begin{definition}[{Reparameterisation of actegories, compare \cite[Prop. 3.6.1]{capucci_actegories_2023}}]
  \label{def:reparameterisation}
  Let $(\cC, \actfw)$ be a $\cM$-actegory, and $E : \cN \to \cM$ a strong monoidal functor.
  Then we can define a $\cN$-actegory $(\cC, \actfw^E)$ where
  \[
    \actfw^E \coloneqq \boxed{\cC \times \cN \xrightarrow{\cC \times E} \cC \times \cM \xrightarrow{\actfw} \cC}
  \]
  If the functor $E$ was merely lax monoidal, the actegory would only be lax (\cref{sec:act_para_literature}) too.
\end{definition}

\begin{notation}
  \label{not:reparameterisation}
  We highlight an important piece of notation: reparameterisations will be written with a superscript, i.e.\ as $\actfw^E = \cM(-, E(-))$.
  This notation will make continued appearance in other kinds of reparameterisations that we will encounter, especially as 2-morphisms in forthcoming bicategories like $\Para$ (\cref{def:para}).
\end{notation}

\begin{example}[Markov kernels and expectations]
  \label{ex:mark_kernel_reparameterisation}
  Let $\Mark$ be the monoidal category of Markov kernels (\cref{def:markov_kernel,ex:monoidal_cats_2}), and $\Conv$ be the monoidal category of convex sets\footnote{A convex set is a set $X$ with an abstract expectation operator, and a morphism of convex sets is given by an algebra morphism of the finite support probability monad. For more details see \cite[Ex. 4]{hedges_value_2023}}, where $\Delta : \Mark \to \Conv$ is a strong monoidal functor with respect to these monoidal structures (\cite[Ex. 4]{hedges_value_2023}).
  Here $\Delta$ is an example of a reparameterisation of actegories, turning the self-action of $\Conv$ (labelled with $\otimes'$) into an action of $\Mark$ defined as ${\otimes'}^\Delta : \Conv \times \Mark \to \Conv$.
\end{example}

Lastly, we mention the trivial action that can be defined for any category, and note that many more examples can be found in \cite[Section 3.2]{capucci_actegories_2023}.

\begin{example}[{Trivial action, \cite[Ex. 3.2.1]{capucci_actegories_2023}}]
  \label{ex:trivial_action}
 Every category $\cC$ has a trivial action of the terminal category $\CatTerm$ on it, and this action $\actfw : \cC \times \CatTerm \to \cC$ is an isomorphism.
\end{example}

In the next chapter we will see that monoidal, cartesian and closed structures can be all generalised and stated in the actegorical setting.

\subsection{The $\Para[false]$ construction}
\label{sec:para}

Having described actegories, we will now see how to use them to construct a category whose morphisms are parametric in the way that the actegory describes.
This will be done with the $\Para$ construction which takes in an actegory and produces a \emph{bicategory}, where the 2-morphisms will model reparameteristaions, and whose bicategory laws completely capture all the necessary invariants we might ask of these parametric maps.
Para is also not a new construction, and has appeared in numerous incarnations in the literature (see \cref{sec:act_para_literature}).

\begin{definition}[{Parametric maps, compare \cite[Def. 2]{capucci_towards_2022}}]
  \label{def:para}
  Let $(\cM, \otimes, I)$ be a monoidal category, and let $(\cC, \act)$ be a $\cM$-actegory.
  We define the bicategory $\Para_{\act}(\cC)$ with the following data:
\begin{itemize}
\item \textbf{Objects} are those of $\cC$;
\item \textbf{Morphisms} are often referred to as \emph{parametric morphisms}. A parametric morphism of type $A \to B$ is a pair $(P, f)$ where we think of $P : \cM$ as its \emph{parameter space} and $f : A \act P \to B$ in $\cC$ as its \emph{implementation}.
  Every parametric morphism has a horizontal, but also a vertical component, emphasised by its string diagram representation (\cref{fig:para_morphism}).
  \forlater{we believe we know how to give a formal language}
  \begin{figure}[H]
    \scaletikzfig[0.8]{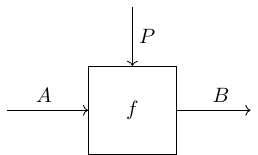}
    \caption{String diagram representation of a parametric morphism. We often draw the parameter wire on the \emph{vertical} axis to signify that the parameter object is part of the data of the morphism.}
    \label{fig:para_morphism}
  \end{figure}
\item \textbf{2-morphisms} are called \emph{reparameterisations}.
  A reparameterisation from $(P, f) \Rightarrow (P', f')$ is a morphism $r : P'
  \to P$ in $\cM$ such that the diagram in \cref{eq:para_triangle_reparam} commutes\footnote{Observe that the $r$ goes in the opposite direction of the 2-cell. See the variance of $\cM$ in \cref{prop:para_local_cat_elements}.} in $\cC$.
  Following \cref{not:reparameterisation} we will often write $f^r$ for the reparameterisation of $f$ with $r$.
  
  \begin{minipage}{0.55\textwidth}
    \begin{figure}[H]
      \scaletikzfig[0.8]{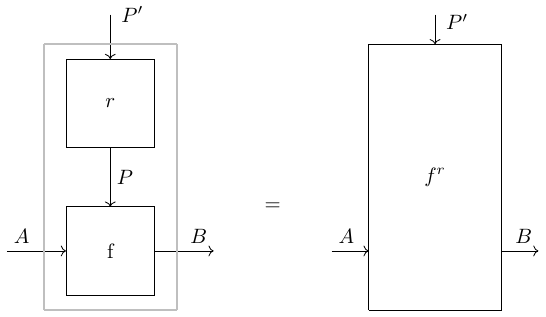}
      \caption{String diagram of reparameterisation. The reparameterisation map $r$ is drawn vertically.}
      \label{fig:para_reparam}
    \end{figure}
  \end{minipage}
  \begin{minipage}{0.45\textwidth}
    \begin{equation}
      \label{eq:para_triangle_reparam}
      \begin{tikzcddiag}[ampersand replacement=\&]
        {A \act P'} \& B \\
        {A \act P}
        \arrow["{A \act r}"', from=1-1, to=2-1]
        \arrow["f"', from=2-1, to=1-2]
        \arrow["{f'}", from=1-1, to=1-2]
      \end{tikzcddiag}
    \end{equation}
  \end{minipage}
  
\item \textbf{Identity morphism.} For every object $A$ there is an identity parametric map $(I, {\eta_A}^{-1})$ where $I : \cM$ and $\eta_A : A \act I \to A$ is the unitor of the underlying actegory.
\item \textbf{Morphism composition.} The composition of parametric morphisms
  $$
  A \xrightarrow{\quad (P, f)\quad} B \xrightarrow{\quad (Q, g)\quad} C
  $$
  i.e.\ of
  \begin{equation*}
\begin{aligned}
    &P : \cM\\
    &A \act P \xrightarrow{f} B \quad \text{in} \quad \cC
\end{aligned}
\qquad\qquad\text{and}\qquad\qquad
\begin{aligned}
    &Q : \cM\\
    &B \act Q \xrightarrow{g} C \quad \text{in} \quad \cC
\end{aligned}  
\end{equation*}
is the pair $(P \otimes Q, f \compPara{\comp} g)$ where
\begin{align*}
    &P \otimes Q : \cM&\\
    &f \compPara{\comp} g \coloneqq \boxed{A \act (P \otimes Q) \xrightarrow{\mu_{A, P, Q}} (A \act P) \act Q \xrightarrow{f \act Q} B \act Q \xrightarrow{g} C} &\text{in} \quad \cC
\end{align*}
\begin{figure}[H]
  \scaletikzfig[0.8]{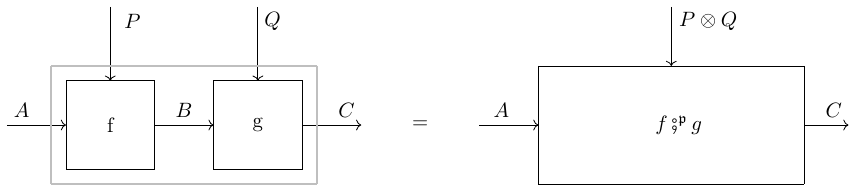}
  \caption{String diagram representation of the composition of parametric morphisms. By treating parameters on a separate axis we obtain an elegant graphical depiction of their composition.}
  \label{fig:para_composition}
\end{figure}

\item \textbf{Horizontal and vertical composition of 2-morphisms} is given by parallel and sequential product of $\cM$, respectively.
\end{itemize}

It is routine to verify that the unitor and associator data --- defined in the only possible way (and remarked upon in \cref{box:frombicatto2cat}) --- satisfy the unity and the pentagon axioms (\cite[Def.\ 2.1.3]{johnson_2-dimensional_2021}).
This concludes the definition of the bicategory $\Para_{\act}(\cC)$.
\end{definition}

\begin{notation}
  For a self-action of a monoidal category $\cC$ we'll sometimes omit the subscript of the action, writing just $\Para(\cC)$ and trusting the monoidal structure is clear from the context.
\end{notation}

\begin{example}[$\Para(\cC)$ for monoidal $\cC$]
  Any monoidal category $\cC$ gives rise to $\Para(\cC)$.
  In this case a parametric morphism $A \to B$ is a pair $(P, f)$ where $f : P \otimes A \to B$.
  If $\cC$ is cartesian, then $f : P \times A \to B$ consumes both the input and the parameter.
  If $\cC$ is cocartesian, then $f : P + A \to B$ consumes \emph{either} the input or the parameter.
\end{example}

Monoidal categories abound, such as $\Set$, $\Smooth$, $\FVect$. The construction $\Para$ can be instantiated in any of those with any of their monoidal products.
In \cref{ch:architectures} (specifically \cref{def:linear_layer,def:bias_layer,subsec:layers_no_parameters}), we will show how to model standard neural network layers as morphisms in $\Para(\Smooth)$.

\begin{remark}
The graphical language we used for depicting parametric morphims is slightly different than the graphical language of classical string diagrams.
Consider a parametric map $f : A \otimes P \to B$.
With classical string diagrams we might draw them as on the left.
However this notation does not emphasise the special role played by the parameter $P$: it is part of the data of \emph{the morphism}.
By separating them on two different axes (as on the right) we obtain the graphical language which more closely mirrors their semantics.

\begin{minipage}{0.55\textwidth}
   \scaletikzfig[0.7]{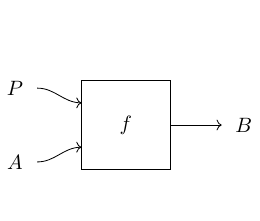}
\end{minipage}
\begin{minipage}{0.45\textwidth}
\scaletikzfig[0.7]{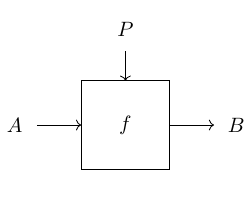}
\end{minipage}
  
\end{remark}

An important class of parametric morphisms arises out of subcategory actions (\cref{ex:monoidal_subcategory}).

\begin{example}
  \label{ex:para_monoidal_subcategory}
  Let $(\cC, \act)$  be a $\cM$-actegory and let $E : \cN \to \cM$ be a strong monoidal functor.
  Then we can form $\Para_{\actfw^E}(\cC)$.
  Here a parametric morphism $A \to B$ is a pair $(P, f)$ where $P$ lives in $\cN$ (not $\cM$!) and the implementation $f$ is of type $E(P) \otimes A \to B$.
  Reparameterisations, like parameters, live strictly in $\cN$.
\end{example}

\begin{mybox}[label=box:frombicatto2cat]{Gray}{From bicategories to 2-categories}
  As defined, $\Para_{\actfw}(\cC)$ is a bicategory, i.e.\ a weak 2-category.
  This is easy to see: take any parametric morphism $(P, f) : \Para_{\act}(\cC)(A, B)$ and postcompose it with the identity on $B$, i.e.\ $(I, \eta^{-1}_B) : \Para_{\actfw}(\cC)(B, B)$.
  The result of this is a map with the parameter object $P \otimes I$.
  This is different from $P$, as the unitor law of a 2-category would require.\footnote{Similar story arises with associativity; composing three parametric morphisms yields either $(P \otimes Q) \otimes R$ or $P \otimes (Q \otimes R)$ as the parameter object, depending on the order of composition.}
  This is because an arbitrary monoidal category isn't necessarily \emph{strict}, i.e.\ objects $P \otimes I$ and $P$ are not necessarily equal, but only isomorphic.
  This leads us to ponder --- is strictness of $\cM$ all we need to turn $\Para_{\actfw}(\cC)$ into a 2-category?
  Or for example, do the unitor $\eta$ and the multiplicator $\mu$ of the actegory also need to be strict?
  As it turns out, strictness of $\cM$ is all we need.

\begin{definition}[{Strict actegory (compare \cite[Subsection 3.4.]{capucci_actegories_2023})}]
  \label{def:strict_actegory}
  A $\cM$-actegory $(\cC, \act)$ is called \newdef{strict} if $\cM$ is strict monoidal.
\end{definition}

\begin{example}[Strict self-action]
  \label{ex:strict_self_action}
  Any strict monoidal category give rise to a strict self-action from both the left and the right side.
\end{example}

Strict actegories allow us to state the following proposition (proved in \cref{app:para}) reducing a bicategory to a 2-category.

\begin{restatable}{proposition}{ParaBiCatToParaTwoCat}
  \label{prop:act_strc_para_2cat}
  If the $\cM$-actegory $(\cC, \actfw)$ is strict, then $\Para_{\act}(\cC)$ is a 2-category.
\end{restatable}

\end{mybox}

On a similar note, if $\cM$ is a discrete monoidal category (i.e.\ a monoid), then $\Para_{\act}(\cC)$ is a category.
This follows because 2-morphisms in $\Para$ only exist when morphisms in $\cM$ do, meaning that, if $\cM$  has only identity morphisms, then $\Para_{\act}(\cC)$ has only identity 2-morphisms.

A good way to understand this brings us to another important property of $\Para_{\act}(\cC)$: its hom-categories are categories of elements.
This proposition will be instrumental in \cref{ch:bidirectionality}.

\begin{proposition}
  \label{prop:para_local_cat_elements}
  For given objects $A$ and $B$, the hom-category $\Para_{\act}(\cC)(A, B)$ can be computed as a category of elements (\cref{def:cat_of_elem}) of the functor $\cC(A \act -, B) : \cM^{\op} \to \Set$.
  That is,
  \[
    \Para_{\act}(\cC)(A, B) = \El(\cC(A \act - , B)) = \sum_{P : \cM^{\op}}\cC(A \act P, B)
  \]
\end{proposition}

We can see that, if $\cM$ is discrete, then any category of elements formed over it must be discrete too.
Though, not all categories that go under the name $\Para$ in the literature arise from a discrete action.

\begin{mybox}[label=box:from2cattocat]{Gray}{From 2-categories to categories}
  Throughout its brief history, $\Para_{\act}(\cC)$ wasn't always thought of as a 2-category.\footnote{Here we assume a strict actegory (\cref{box:frombicatto2cat}).} 
  Starting humbly as a category, it was defined a number of times, each with its own distinct flavour of morphisms.
  We list all of them below, and show how they arise as a shadow of the 2-categorical perspective, i.e.\ as an image of the enriched base change $\enrichedbasechange{F} : \TwoCat \to \Cat$ for a given lax monoidal functor $F : (\Cat, \times, 1) \to (\Set, \times, 1)$. Each choice of such a functor describes how we turn the hom-category of parametric morphisms $A \to B$ into a hom-\emph{set}.

  \begin{example}[Discretisation, $\Ob : \Cat \to \Set$]
    Simply ignores the 2-cells.
    This results in a set of parametric morphisms that are only related by strict equality in $\Para_{\act}(\cC)$.
  \end{example}
  This prohibits us from modelling the parameter update step of neural networks with gradient descent, or other gradient updates (\cref{sec:optimisers}). 
  \begin{example}[Connected components, $\conncomp : \Cat \to \Set$]
    Identifies any two morphisms connected by a (potentially non-invertible) reparameterisation.

    \end{example}

This is a strong condition.
Via \cref{prop:colim_conn_comp_of_el} this is equivalent to computing the local colimit of the functor $\cC(A \act - , B)$, instead of its category of elements (\cref{prop:para_local_cat_elements}).
It's used in \cite{dalrymple_dioptics_2019} where the colimit is expressed as the coend mute in one variable.
In \cref{thm:para_trivialises_when_unit_initial} we describe conditions under which this quotient trivialises the parametric category.

  \begin{example}[Isomorphisms, $\Cat \xrightarrow{\Core} \Cat \xrightarrow{\conncomp} \Set$]
    Identifies two morphisms if they are connected by an invertible reparameterisation.
  \end{example}
  
  Used in \cite{fong_backprop_2021, heunen_tensor-restriction_2021, gavranovic_compositional_2019}.
  This is one of the better behaved quotients.
  
\begin{example}[Epimorphisms, $\Cat \xrightarrow{\Epi} \Cat \xrightarrow{\conncomp} \Set$]
  Identifies two morphisms if they are connected by a reparameterisation which is an epimorphism.
\end{example}
Used in \cite{spivak_learners_2022} without justification.
\end{mybox}

The examples above suggest that attempts to simplify the 2-categorical structure of $\Para$ into a category generally ends up losing important information relevant to modelling deep learning.
$\Para_{\act}(\cC)$ is truly a higher-dimensional beast!

Information is tangibly lost under the connected components quotient if the monoidal unit of $\cM$ is initial, which happens for self-actions of cocartesian categories.
Quotienting $\Para_{\act}(\cC)$ here trivialises the entire structure.

\begin{theorem}
  \label{thm:para_trivialises_when_unit_initial}
  Let $(\cC, \act)$ be a $\cM$-actegory, where the monoidal unit of $\cM$ is initial.
  Then
  \[
    \pibc(\Para_{\act}(\cC)) \cong \cC
  \]
\end{theorem}

\begin{proof}
  As they both share objects, all we have to prove is that there is a one-to-one correspondence between their hom-sets that respects composition and identities.
  The hom-set isomorphism follows from the fact that $\cM$ has an initial object (which is terminal in $\cM^{\op}$ allowing us to reduce the colimit) and the fact that this initial object is the monoidal unit (allowing us to further simplify the term using the actegory unitor).\footnote{The converse doesn't necessarily hold, as the converse of \cref{prop:colim_terminal} doesn't necessarily hold.}
  \begin{alignat*}{2}
    & && \pibc(\Para_{\act}(\cC)(A, B))\\
    & \text{(Def.)} &=& \conncomp(\El(\cC(A \act -, B)))\\
    & \text{(\cref{prop:colim_conn_comp_of_el})} &\cong& \colim(\cC(A \act -, B))\\
    & \text{(\cref{prop:colim_terminal})} &\cong& \cC(A \act 0, B)\\
    & \text{(\cref{eq:actegory_coherence_components}, left)} &\cong& \cC(A, B)
  \end{alignat*}
  It is routine to show that this respects composition and identities.
\end{proof}

Intuitively, this means that any parametric morphism $f : A \act M \to B$ in $\Para_{\act}(\cC)$ can be precomposed with the initial object $0_M : 0 \to M$ yielding $f^{0_M} : A \act 0 \to B$.
If $0$ is the unit of the monoidal product, then $f^{0_M}$ and $f$ are equivalent in $\pibc(\Para_{\act}(\cC))$.
This theorem has an important corollary in terms of self-actions of a cocartesian category.%

\begin{corollary}
  \label{corr:cocart_trivialises}
Let $\cC$ be a cocartesian category. Then $\pibc(\Para(\cC)) \cong \cC$.
\end{corollary}

Note that this doesn't necessarily hold for a cartesian category.
Consider the category $\N$ (\cref{def:natural_numbers}).
This category is cartesian where $m \times n = \max(m, n)$ and where the terminal object is $0$.
A parametric morphism $a \to b$ in $\Para(\N)$ consists of a number $p : \N$ and a morphism $f : \max(a, p) \to b$, i.e.\ an inequality $\max(a, p) \geq b$.
Therefore, even though $\N(0, 1) = \emptyset$ , $\conncomp(\Para(\N)(0, 1)) \neq \emptyset$ as it is inhabited, for example by $42 : \N$ and an inequality $\max(0, 42) \geq 1$.

\subsection{Local vs. global contexts}
\label{subsec:local_vs_global}

An instructive way to understand $\Para_{\act}(\cC)$ is as a model of processes which have access to a \emph{local} context.
Effectively, this means that every morphism can choose its own parameter space, and that composition of morphisms tensors the parameters together.
This is in contrast to categories which model a \emph{global} context, which we now describe, and compare to $\Para$.

A family of such categories can be formed by considering the coKleisli categories over comonads formed on a base cartesian category $\cC$.
For each $X : \cC$ there is a one such comonad in the literature called the coreader comonad.\footnote{The coreader comonad is also sometimes referred to as the ``the writer comonad'' (though this is often confused with the writer \emph{monad} which additionally requires $X$ to have a monoid structure), ``the reader comonad'' (because of the adjunction $- \times X \dashv \internalHom{\cC}{X}{-}$ when $\cC$ is cartesian closed), ``product comonad'' or ``environment comonad''.}
It is given by the functor $- \times X : \cC \to \cC$ and corresponding counit and comultiplication natural transformations defined in \cite[Def.\ 4]{gavranovic_graph_2022}.
Its coKleisli category has the following form.

\begin{definition}[Category $\CoKl(- \times X)$]
  \label{def:cokl}
  The category $\CoKl(- \times X)$ has the following data.
  \begin{itemize}
  \item \textbf{Objects} are those of $\cC$;
  \item \textbf{Morphisms} are $X$-parametric morphisms; a map $A \to B$ in $\CoKl(- \times X)$ consists of a map $f : A \times X \to B$;
  \item \textbf{Morphism composition.} The composition of morphisms
    $$
    A \xrightarrow{\quad f \quad} B \xrightarrow{\quad g \quad} C
    $$
    i.e.\ of
    \begin{equation*}
      \begin{aligned}
        &A \times X \xrightarrow{f} B \quad \text{in} \quad \cC
      \end{aligned}
      \qquad\text{and}\qquad
      \begin{aligned}
        &B \times X \xrightarrow{g} C \quad \text{in} \quad \cC
      \end{aligned}  
    \end{equation*}
    is map $(f \compPara{\comp} g)^{\Delta_X}$ defined using the notion of reparameterisation (\cref{fig:para_reparam}), i.e.\
    \begin{align*}
      &(f \compPara{\comp} g)^{\Delta_X} = \boxed{A \times X \xrightarrow{A \times \Delta_X} A \times (X \times X) \xrightarrow{\cong} (A \times X) \times X \xrightarrow{f \times X} B \times X \xrightarrow{g} C} &\text{in} \quad \cC
    \end{align*}
  \item \textbf{Identity.} Identity on $A$ is the projection $\pi_A : A \times X \to A$.
  \end{itemize}
\end{definition}

We see that every morphism in $\CoKl(- \times X)$ has access to an additional \emph{global} parameter $X$, and \emph{only} this parameter.
A composition of multiple $X$-parameterised morphisms does not take the product of individual $X$ objects, instead it is still parameterised by only one $X$.
What this composition then does is relay the input $X$ to its constituents by copying (\cref{fig:cokl_composition}), which is where the requirement that the base category $\cC$ is cartesian arises.
Compare this to $\Para$ where each morphism can choose its type of parameter.
And when the morphisms are composed, the total parameter space gets bigger.
For more details on this comparison, we invite the reader to \cite[Sec. 3.1.1]{gavranovic_graph_2022}.
The coKleisli construction will be relevant in studying weight tying (\cref{def:weight_tying}) and graph convolutional neural networks (\cref{subsec:graph_neural_networks}).

\begin{figure}[h]
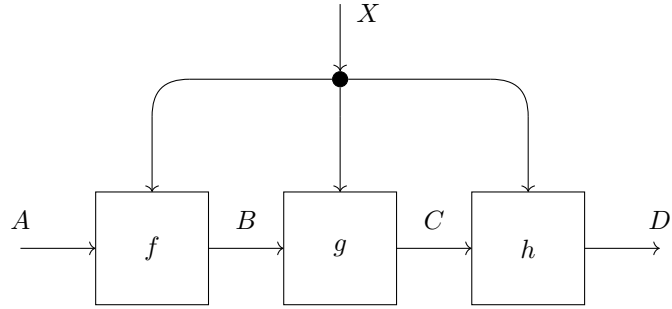

  \scaletikzfig{cokl_composition}
  \caption{Unlike $\Para$ composition which allows each morphism to be
    parameterised by an arbitrary object, the $\CoKl(- \times X)$ construction
    fixes a specific object $X$ as the parameter for all of them.}
  \label{fig:cokl_composition}
\end{figure}

Since $\CoKl(- \times X)$ models parametric processes, we would expect that this category is related to $\Para(\cC)$ in some way.
That is indeed the case: $\CoKl(- \times X)$ can be embedded in $\Para(\cC)$, but as there is this extra reparameterisation step we need to perform in composition and identities, this embedding is only \emph{oplax}.

\begin{lemma}
  \label{lemma:cokl_to_para}
  Let $\cC$ be a cartesian category.
  Then for every object $X : \cC$ there is an identity-on-objects oplax functor
  \[
    \CoKl(- \times X) \to \Para(\cC)
  \]
  mapping every morphism $f : \CoKl(- \times X)(A, B)$ to $(X, f) : \Para(\cC)(A, B)$.
  The opunitor and oplaxator are respectively the counit $\epsilon_X : X \to 1$ and the comultiplication $\Delta_X : X \to X \times X$ of the unique comonoid on $\cC$ arising from the cartesian structure.
\end{lemma}

This shows that $\Para(\cC)$ can also model processes with a global context, and that special care has to be taken when composing them.

For completeness, we mention the variant with \emph{no} context: this is simply an ordinary category $\cC$.
Unlike the previous case which gives us an oplax functor into $\Para(\cC)$, this one gives us a pseudofunctor.
We state it in a form not specific to just self-actions, but arbitrary ones.

\begin{lemma}[$\cC$ embeds into $\Para_{\act}(\cC)$]
  \label{lemma:c_embeds_into_para}
  Let $(\cC, \act)$ be a $\cM$ actegory.
  Then there is an identity-on-objects pseudofunctor
  \[
    \cC \to \Para_{\act}(\cC)
  \]
  mapping a morphism $f : A \to B$ to the composite $A \act I \xrightarrow{\rho_A} A \xrightarrow{f} B$.\footnote{If the actegory is strict, then this becomes a mere 1-functor.}
  Since all the parameter object is always $I$, it is easy to check that this is indeed a pseudofunctor.
\end{lemma}

Furthermore, if $\cM$ is strict, then this reduces to a 2-functor.

\begin{figure}[h]
  \centering
  \includegraphics[width=0.7\textwidth]{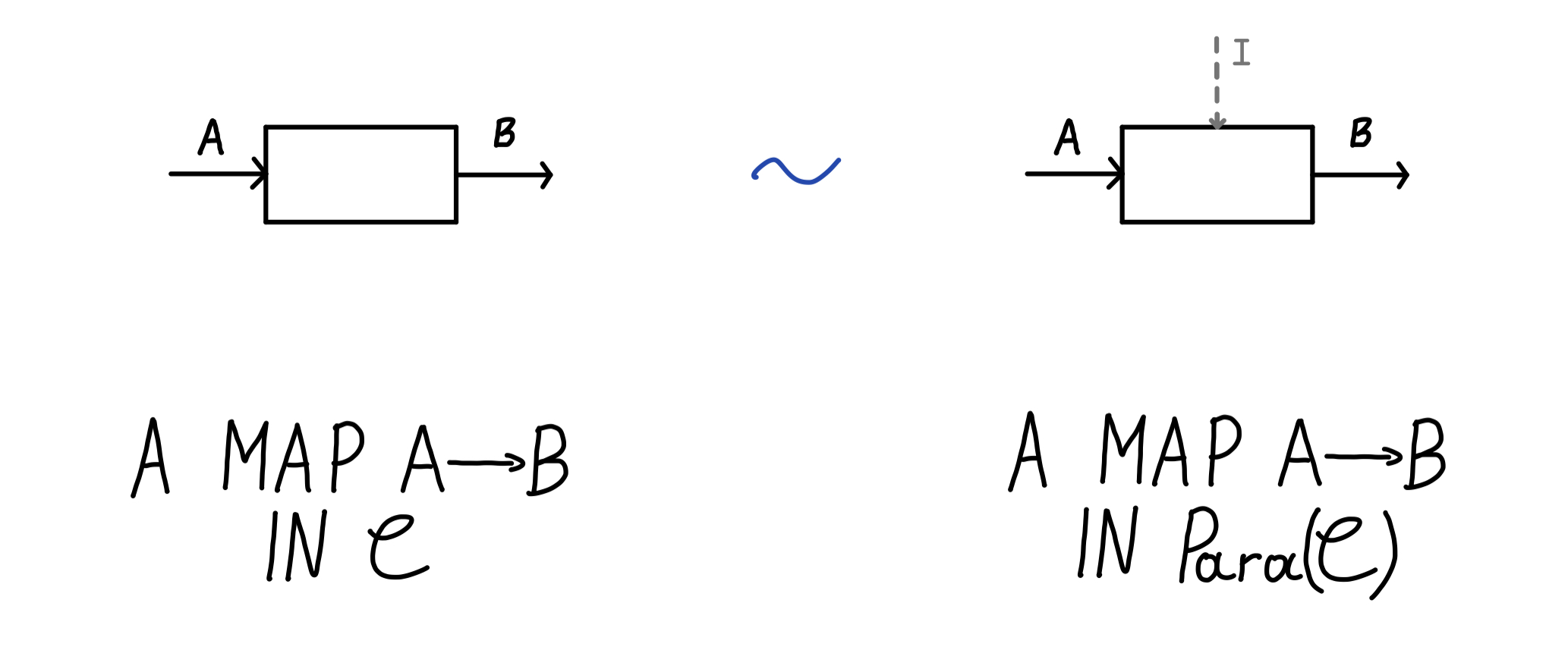}
  \caption{Every morphism in $\cC$ is an $I$-parametric morphism in $\Para_{\act}(\cC)$. }
  \label{fig:para_unit}
\end{figure}

\section{Monoidal actegories and monoidal $\Para[false]$}
\label{sec:monodial_actegories}

Categories describe processes that can be composed sequentially, and parametric categories describe parametric morphisms that can be composed sequentially.
Just like making categories monoidal permits us to compose processes in parallel, we can ponder: are there \emph{monoidal actegories} that allow us to compose parametric systems in parallel?
In this section, we will see that the answer is yes.
We begin by defining monoidal actegories and note that the acting monoidal category $\cM$ is now required to be braided monoidal.\footnote{Despite symmetry of monoidal categories being widespread in our applications, we do not assume it for two reasons: the theory does not require it, and the notable exception of a monoidal category of used to form \emph{affine traversals} does not possess it (\cref{rem:affine_traversals}).}

\begin{definition}[{Monoidal actegory, compare \cite[Def.\ 5.1.1.]{capucci_actegories_2023}}]
  \label{def:monoidal_actegory}
  Let $(\cM, \otimes, I)$ be a braided monoidal category.
  A $\cM$-actegory $(\cC, \act)$ is called a \newdef{monoidal $\cM$-actegory} if the
  underlying category $\cC$ has a monoidal structure $(\boxtimes, J)$ and the action respects that structure.
  This means that
  \begin{itemize}
  \item The underlying functor $\act : \cC \times \cM \to \cC$ is
    strong monoidal;
  \item The underlying natural transformations $\eta$ and $\mu$ are monoidal.
  \end{itemize}
  If $\cM$ is a strict monoidal category, then we call $(\cC, \act)$ a \newdef{strict} monoidal $\cM$-actegory.
  If the functor is only (op)lax monoidal, then we call the actegory an \newdef{(op)lax} monoidal actegory.
\end{definition}

Explicitly, the oplaxator of $\act$ is a map
\begin{equation}
  \label{def:monact_oplaxator}
  \interchangeract_{A, A', M, M'} : (A \boxtimes A') \act (M \otimes M') \cong (A \act M) \boxtimes (A' \act M')
\end{equation}
which we call \emph{the mixed interchanger}.\footnote{Given an actegory, it is the only additional piece of data that needs to be naturally defined in order to make it into a monoidal actegory. See \cite[Remark 5.1.2.]{capucci_actegories_2023}}

\begin{remark}[{Why does $\cM$ need to be braided?}]
  \label{rem:m_brided_in_definition}
  To state that $\mu$ is  monoidal natural transformation, both the functors it is spanned between ($(\act \times \cC) \comp \act$ and $(\cC \times \otimes) \comp \act$) need to be monoidal.
  The latter is monoidal only if $\otimes : \cM \times \cM \to \cM$ itself is a monoidal functor, which happens only when $\cM$ is braided.
  See \cref{appendix:monoidal_categories} for more details.
\end{remark}

\begin{example}[Self-action of a braided monoidal category]
  \label{ex:self_action_braided_monoidal}
  Any braided monoidal category $\cC$ is also a monoidal $\cC$-actegory with the interchanger defined using braiding.
  In this case the mixed interchanger reduces to the braided monoidal category interchanger (\cref{def:interchanger}), effectively swapping the order of middle two components.
\end{example}

Examples of this are $\Set$, $\Smooth$ and $\FVectR$, when instantiated with their cartesian product as the self-action.
A large class of examples of monoidal actegories arises by reparameterisations induced by a braided monoidal functor.

\begin{definition}[Reparameterisation of monoidal actegories]
  \label{def:reparam_monoidality}
  In \cref{def:reparameterisation} we have described how to reparameterise a $\cM$-actegory $(\cC, \actfw)$ with a strong monoidal functor $E : \cN \to \cM$.
  If the starting actegory is monoidal --- implying $\cM$ is braided --- then the resulting actegory will be too if $E$ is additionally braided (\cref{def:braided_monoidal_functor}).
\end{definition}

\begin{example}[Markov kernels and expectation]
  \cref{ex:mark_kernel_reparameterisation} yields a monoidal actegory since $\Mark$ is monoidal (\cref{ex:monoidal_cats_2}), and the reparameterisation $\Delta : \Conv \to \Mark$ is a strong braided monoidal functor.
\end{example}

\subsection{Monoidal $\Para[false]$}
\label{sec:monoidal_para}

Having discussed the monoidal structure of actegories, we are now ready to discuss the monoidal structure of the bicategory $\Para_{\act}(\cC)$.

\begin{restatable}[{Monoidal structure of $\Para$ (compare \cite[Section 2.1.]{capucci_towards_2022})}]{proposition}{ParaMonoidal}
  \label{prop:para_monoidal}
  Let $(\cC, \act)$ be monoidal $\cM$-actegory with the underlying product $(\boxtimes, J)$.
  Then $\Para_{\act}(\cC)$ becomes a monoidal bicategory with the following data.
  \begin{itemize}
  \item The monoidal product of two objects $X, Y$ is given by the monoidal product $X \boxtimes Y$ of $(\cC, \boxtimes, J)$;
  \item The monoidal unit is the monoidal unit $J$ of $(\cC, \boxtimes, J)$; 
  \item The monoidal product of two parametric morphisms
    \[
      (P, f : A \act P \to B) \quad \text{and} \quad (Q, g : C \act Q \to D)
    \]
    is given by $(P \otimes Q, f \compPara{\boxtimes} g)$ where
    \[
      f \compPara{\boxtimes} g \coloneqq \boxed{(A \boxtimes C) \act (P \otimes Q) \xrightarrow{\interchangeract_{A, C, P, Q}} (A \act P) \boxtimes (B \act Q) \xrightarrow{f \boxtimes g} B \boxtimes D}
    \]
    \begin{figure}[H]
      \scaletikzfig[0.8]{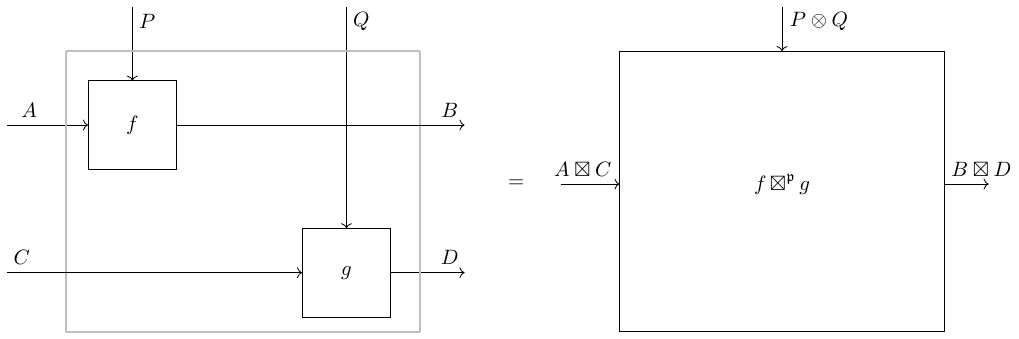}
      \caption{String diagram of the monoidal product of two parametric morphisms}
      \label{fig:para_tensor}
    \end{figure}
  \item The associators and unitors are those of $(\cC, \boxtimes, J)$.
  \end{itemize}
\end{restatable}

Examples of the monoidal bicategory $\Para$ arise out of any action that is monoidal, such as cartesian and cocartesian self-actions, but also their subactions (\cref{ex:monoidal_subcategory}), and any reparameterisations along braided monoidal functors.

When reasoning about this monoidal structure, we rely on a series of simplifications.

\begin{mybox}[label=box:frommonbicattomon2cat]{Gray}{From monoidal bicategories to monoidal 2-categories}
  In \cref{box:frombicatto2cat} we have seen that the bicategory $\Para_{\act}(\cC)$ becomes a 2-category under an additional condition of the underlying actegory: strictness (of the acting category $\cM$).
  But if $\Para_{\act}(\cC)$ is a \emph{monoidal bicategory}, strictness condition is not sufficient for turning it into a \emph{monoidal 2-category}.
  We additionally need commutativity.
  \begin{restatable}{proposition}{ParaMonoidalTwoCat}
    \label{prop:para_monoidal_twocat}
    Let $(\cC, \act)$ be a strict monoidal $\cM$-actegory.
    Then $\Para_{\act}(\cC)$ is a monoidal 2-category if and only if $\cM$ is commutative monoidal.
  \end{restatable}

  \begin{proof}
    \cref{app:para}.
  \end{proof}
\end{mybox}

In the previous section, assuming $\cM$ is strict monoidal
was a mild assumption.
Many such categories are found in nature, and moreover, every monoidal category is monoidally equivalent to a strict one.
On the other hand, assuming a category is \emph{commutative} monoidal is extremely prohibitive.
Commutative monoidal categories are seldom found in nature, and they are not monoidally equivalent to strict monoidal ones.
Assuming only strictness of $\cM$, can we somehow simplify $\Para_{\act}(\cC)$ and its monoidal structure?
In the following box we describe how to apply a quotient (as we did in \cref{box:from2cattocat}), and obtain a monoidal \emph{category}.

\begin{mybox}[label=box:frommon2cattomoncat]{Gray}{From monoidal bicategories to monoidal categories}
  In \cref{box:from2cattocat} we described how change of base functors $F : \Cat \to \Set$ allow us to obtain a category $\enrichedbasechange{F}(\Para_{\act}(\cC))$ from a 2-category $\Para_{\act}(\cC)$.
  If $\Para_{\act}(\cC)$ additionally has monoidal structure as a bicategory, which of these functors transfer this monoidal structure to $\enrichedbasechange{F}(\Para_{\act}(\cC))$?

  We provide an answer for functors in \cref{box:from2cattocat}.
  Its proof can be found in \cref{app:para}.

  \begin{restatable}{theorem}{ParaMonCat}
    \label{thm:para_moncat}
    Let $(\cC, \act)$ be a monoidal $\cM$-actegory.
    Then the following three categories are monoidal.
    \[
      \pibc(\Para_{\act}(\cC)) \quad, \quad \enrichedbasechange{\Iso}(\Para_{\act}(\cC)) \quad \text{and} \quad \Epi_*(\Para_{\act}(\cC))
    \]
    On the other hand, $\enrichedbasechange{\Ob}(\Para_{\act}(\cC))$ is not a monoidal category.
  \end{restatable}

  If in addition to $\cM$ the category $\cC$ is braided (resp.\ symmetric), then the three monoidal categories become braided monoidal (resp.\ symmetric).
  For example, the symmetric monoidal category $\Para$ as defined in \cite[Def.\ III.1]{fong_backprop_2021} is the symmetric monoidal category $\enrichedbasechange{\Iso}(\Para_{\times}(\Smooth))$ in this thesis.
\end{mybox}

Lastly, as we have equipped $\Para(\cC)$ with a monoidal structure we can now reason about states, costates and scalars in this setting.

\begin{mybox}[label=box:scs_para]{Cerulean}{States, costates, and scalars in $\Para[false](\cC)$}
  For a self-action of a monoidal category $\cC$ the states, scalars, and costates of $\Para(\cC)$ have an interesting form.
  \newline
  
  \begin{minipage}[t]{0.32\textwidth}
    \begin{tightcenter}
      States
    \end{tightcenter}
    \scaletikzfig[0.7]{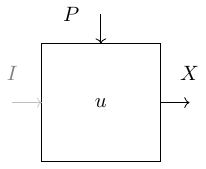}
    \begin{align*}
      &\Para(\cC)(I, X)\\
      \cong & {(\cC/X)}^{\op}
    \end{align*}
  \end{minipage}
  \hfill
  \begin{minipage}[t]{0.32\textwidth}
    \begin{tightcenter}
      Scalars
    \end{tightcenter}
    \scaletikzfig[0.7]{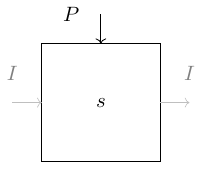}
    \begin{align*}
      &\Para(\cC)(I, I)\\
      \cong &\sum_{P : \cC}\cC(P , I)
    \end{align*}
  \end{minipage}
  \hfill
  \begin{minipage}[t]{0.32\textwidth}
    \begin{tightcenter}
      Costates
    \end{tightcenter}
    \scaletikzfig[0.7]{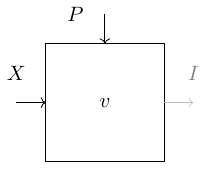}
    \begin{align*}
      &\Para(\cC)(X, I)\\
      \cong &\sum_{P : \cC}\cC(X \otimes P , I)
    \end{align*}
  \end{minipage}

  Here we see that states of $\Para(\cC)$ correspond to morphisms in $\cC$, and scalars correspond to costates of $\cC$.
  This is going to be especially relevant when $\cC$ is a more complex category, such as the category of optics (\cref{ch:para_optic}).
  Just like for monoidal categories, if the category is cartesian, then scalars and costates trivialise.
\end{mybox}

\subsection{Braided, symmetric monoidal actegories?}
\label{subsec:braided_symmetric_monoidal}

Monoidal categories can additionally have braided structure which might or might not satisfy the property of symmetry (\cref{subsec:monoidal_categories}).
The same story is true for monoidal \emph{actegories}.
We do not explicitly restate their definitions here, and instead point them out in the literature: braided monoidal actegories are defined in \cite[Def.\ 5.4.1]{capucci_actegories_2023}, and symmetric monoidal actegories in \cite[Remark 5.4.3]{capucci_actegories_2023}.
In both cases they require the acting category $\cM$ to be symmetric.

These actegories are important because they provide an avenue for equipping $\Para_{\act}(\cC)$ with braided and symmetric monoidal structure.
As not just proving, but merely defining braided and symmetric monoidal categories is an extremely laborious task (see \cite[Sec. 12.1]{johnson_2-dimensional_2021}), in this subsection we only conjecture the following.

\begin{conjecture}[Braiding and symmetry of $\Para_{\act}(\cC)$]
  \label{conj:braided_actegory}
  If $(\cC, \act)$ is a braided (resp.\ symmetric) monoidal actegory, then $\Para_{\act}(\cC)$ is a braided (resp.\ symmetric) monoidal bicategory.\footnote{Following the periodic table of elements higher categories (\cite[Sec. 2.1]{baez_lectures_2010}), we conjecture that a braided monoidal actegory doesn't make $\Para$ merely braided, but additionally \emph{sylleptic} monoidal.}
\end{conjecture}

\section{Actegories and $\Para[false]$ in the literature}
\label{sec:act_para_literature}

Actegories are not a new construction.
They've appeared throughout the literature in many forms (see \cite[Def.\ 5.1]{fujii_2-categorical_2019} and \cite{janelidze_note_2001} for two examples), and we refer the reader to their only existing survey \cite{capucci_actegories_2023} for more information.
The $\Para$ construction is also not new.
Its earliest appearance known to us is in \cite[Sec. 2.2]{hermida_monoidal_2012} where it appeared under the name of a \emph{monoidal category with indeterminates}, as a 1-category.\footnote{Section 2.1 of the aforementioned paper also remarks about the difference between local and global contexts (\cref{subsec:local_vs_global}) in the language of terms and indeterminates.}
Under the name $\Para$ it was originally introduced in \cite{fong_backprop_2021}, albeit in a slightly different form.
We leave a nuanced analysis to \cref{sec:supervised_learning_in_the_literature}).

We review the work related to actegories and $\Para$ through the concept of graded monads, as that gives an enlightening view of $\Para$.

\begin{definition}[{Graded monad, compare \cite[Def.\ 2.1]{fujii_2-categorical_2019}}]
  Let $(\cM, \otimes, I)$ be a strict monoidal category.
  Let \cC be a category.
  A $\cM$-graded monad on $\cC$ is a lax monoidal functor

  \[
    (T, \eta, \mu) : (\cM, \otimes, I) \to (\internalHom{\Cat}{\cC}{\cC}, \circ, 1_{\cC})
  \]
\end{definition}

Unpacking the definition, it can easily be seen that a graded monad (resp. comonad) is similar to a strict actegory, except the natural isomorphisms $\eta$ and $\mu$ are weakened into mere natural transformations.
Such actegories are in \cite[Def.\ 5.2]{fujii_2-categorical_2019} called \emph{lax} (resp. \emph{oplax}) actegories, and in fact, they're equivalent to graded (co)monads (see \cite[Sec. 5.1]{fujii_2-categorical_2019}).
This correspondence allows us to then restate various kinds of results for graded (co)monads to those of (op)lax actegories.
For instance, it allows us to see ordinary ($\mathbf{1}$-graded) (co)monads as (op)lax $\CatTerm$-actegories.
It also allows us to study the analogue of the (co)Kleisli construction for (co)monads.
As it turns out --- the coKleisli construction of a graded comonad is precisely $\Para$!
Following the intuitions about equivalences between coKleisli objects, oplax colimits and the Grothendieck construction, this allows us to restate $\Para$ as a kind of a bicategorical Grothendieck construction (\cite{bakovic_grothendieck_2009}).

\begin{equation}
  \label{eq:triple_equivalence}
  \let\scriptstyle\textstyle
  \begin{tikzcddiag}[ampersand replacement=\&]
    \&\& {\substack{\text{coKleisli} \\ \text{object}}} \\
    \\
    {\substack{\text{Oplax} \\ \text{colimit}}} \&\&\&\& {\substack{\text{Grothendieck} \\ \text{construction}}}
    \arrow["\simeq"', color={rgb,255:red,52;green,153;blue,254}, curve={height=12pt}, tail reversed, from=1-3, to=3-1]
    \arrow["\simeq", color={rgb,255:red,52;green,153;blue,254}, curve={height=-12pt}, tail reversed, from=1-3, to=3-5]
    \arrow["\simeq"', color={rgb,255:red,52;green,153;blue,254}, curve={height=12pt}, tail reversed, from=3-1, to=3-5]
  \end{tikzcddiag}
\end{equation}

\begin{proposition}
  \label{prop:para_bicategorical_grothendieck}
  Let $(\cC, \act)$ be a $\cM$-actegory.
  Then the bicategorical Grothendieck construction of the composite pseudofunctor
  \[
    \B\cM \xrightarrow{\B\act} \B\internalHom{\Cat}{\cC}{\cC} \hookrightarrow \Cat
  \]
  is equivalent to the bicategory $\Para_{\act}(\cC)$, where $\B$ denotes the delooping of monoidal categories and monoidal functors.
\end{proposition}

This is especially relevant in light of (\cite{orchard_unifying_2020}) which serves as a good reference for these constructions, outlining a number of further appearances in theory and practice.

Another generalisation of actegories are \emph{dependent actegories} \cite{myers_para_2022}, a topic of ongoing research.
It stems from the observation that the domain of the actegory functor $\act : \cC \times \cM \to \cC$ is a product of categories, i.e.\ a special version of the Grothendieck construction for a constant functor $\cC^{\op} \to \Cat$ at $\cM$.
Dependent actegories allow an arbitrary functor of this type here, and model the notion that the type of a parameter can depend on the input type.

We do not study this in detail, but merely give some more intuition in the case of $\Set$.
In the non-dependent case the domain of a parametric morphism $A \to B$ for the self-action of the cartesian product on $\Set$ is a product i.e.\ the set $A \times P$.
This is the constant version of the \emph{dependent pair type}.
If we want to make this dependent, we can replace the constant self-action of $\Set$ by the (non-constant) representable functor $\internalHom{\Cat}{-}{\Set} : \Set^{\op} \to \Cat$ mapping $A$ to the functor category $\internalHom{\Cat}{A}{\Set}.$\footnote{We are ignoring universe levels, and omitting the inclusion $\Set \to \Cat$.}
Appropriately defining dependent $\Para$ here (see \cite{myers_para_2022} for details) we obtain a bicategory where a morphism $A \to B$ consists of a indexed parameter function $P : A \to \Set$ and the implementation function $f : \sum_{a : A}P(a) \to B$.
The former describes for each $a : A$ the set of possible parameters, and the latter consumes an $a : A$, a $p : P(a)$ and produces an element of $B$.

Yet another way to generalise $\Para$ stems from the observation that bicategories are a special kind of double categories (\cite[Sec. 12.3]{johnson_2-dimensional_2021}) with only \emph{loose} 1-morphisms.
That is, the only kinds of 1-cells here are the parametric ones.
As it turns out, $\Para_{\act}(\cC)$ can be turned into a double category into a coherent way, where the non-parametric morphisms play the role of \emph{tight} 1-cells.
The constructions in \cite{myers_para_2022} defining dependent $\Para$ automatically takes care of this too, producing a double category as just described.

Lastly, during the writing of this thesis I learned\footnote{From Dylan Braithwaite.} that there is a promising generalisation of actegories to \emph{locally graded categories} (\cref{sec:locally_graded_categories}): categories enriched in $\internalHom{\Cat}{\cM^{\op}}{\Set}$.
Locally graded categories unify parameterisation that is \emph{internal} (the one defined in this chapter), allowing us to see actegories as locally graded categories with copowers by representables (see \cref{prop:copower_representable}), but also parameterisation that is \emph{external} --- something we will explore in \cref{sec:optics_backward_closed_actegory}.
Locally graded categories also have a coherent interpretation with respect to global parameterisation (\cref{subsec:local_vs_global}), explored in (\cref{sec:locally_indexed_categories}).
See \cref{sec:locally_graded_categories} for some preliminary work exploring the relation of locally graded categores to parameterisation.

	\chapter{Bidirectionality}
\label{ch:bidirectionality}

\epigraph{You're stealing all the terminology from physics.}{Ieva Čepaitė}

\newthought{What is there in common} between neural networks, bayesian learners, reinforcement learners, and game-theoretic agents?

At their core, all of these systems are bidirectional.
They involve two separate processes, happening sequentially one after the other, connected by a shared internal state.
Neural networks propagate values forward, and valuations on gradients backward. Bayesian learners propagate evidence forward, and belief updates backward.
In both value iteration and game theory we propagate values forward, while we compute policy improvements and payoffs, respectively, going backward.

Each of these has its own distinct character, involving functions that are differentiable or probabilistic, for example.
Each of these cases is an entire mathematical field in itself with its own notation, conventions, theorems and ideas.
However, these fields are mostly pursued and developed independently, and there is no formal and unifying mathematical framework that would allow us to systematically translate ideas, techniques, and theorems from any one of these fields to the other.
In this chapter we set out to provide such a framework.

\begin{contributions}
  Large fragments of this chapter are a novel contribution.
  Most notably, the definition of weighted optics (\cref{def:weighted_optic}), the isomorphism of their hom-objects to weighted colimits (\cref{prop:copara_weighted_colimit}), the definition of a cartesian actegory (\cref{ex:cartesian_cat_act}), the definition of a cartesian optic (\cref{def:cartesian_optic}), factorisation theorem for $\Copara$ (\cref{thm:copara_cart_reduction}).
  Lastly, novel contribution is the provision of formal desiderata for dependent optics, survey of existing work done on this topic (\cref{subsec:dependent_optics}).
\end{contributions}

\begin{epistemicstatus}
  I have found that the prevalent perspective on categorical models of bidirectionality in the literature is that of dependently-typed \emph{denotational} models, such as those given by dependent lenses or tangent categories.
  This is why in this thesis --- motivated by the discovery of the operational interpretation of optic composition --- I instead focus on the categorical semantics which can model the \emph{operational} aspects of bidirectionality, which I felt was an unjustifiably underexplored area.
  This direction took me to 2-category theory, lax functors, emergent effects, and many more interesting concepts, but also, at the moment, unfortunately a relatively complex theory.
  As such I believe many constructions in this chapter  will eventually be subsumed by advances in locally graded categories and dependent optics, leading to a cleaner formulation.
  Nonetheless, many constructions --- such as weighted $\Copara$, and new variants of weighted mixed optics --- are those I find to be useful abstractions nonetheless, deserving of further exploration.
\end{epistemicstatus}

\begin{remark}[A note on terminology.]
  Originating in functional programming, bidirectional processes of a particular kind emerged under the name of \emph{lenses}.
  Later on these were generalised to other settings where they were called \emph{prisms}, \emph{grates}, and \emph{glasses}, for instance.
  A unifying framework was thereafter proposed unifying them called under the name of \emph{optics} and later generalised to \emph{mixed} optics).
  Despite the fact that this terminology that has no semblance whatsoever to types of cut glass, we continue using it.
\end{remark}

\section{High-level intuition}

Following the intuition of describing processes using monoidal categories from \cref{sec:categories_and_systems}, let's first describe informally what we mean by a process being \emph{bidirectional}.
Fundamentally it means that we have structured information about the underlying category it is built upon.
If the domain of a process (i.e.\ an object $\mi$ of some category $\cC$) specifies the type of inputs consumed by said process, then in the bidirectional setting each object specifies a type of both an input and an output consumed by a bidirectional process --- an input going forward, and an output going backward.
We will write these objects as $\diset{\mi}{\mi '}$ where the top symbol denotes the type of inputs consumed by the forward part of the bidirectional process, and the bottom symbol $\mi '$ denotes the type of outputs produced by its backward part.
Sometimes we will have to refer to these types outside of the stacked representation $\diset{\mi}{\mi '}$ in which case we will always rely on the fact that forward parts do not come with a superscript $'$, but backward parts do.
We will draw such processes as in \cref{fig:process_vs_bidirectional}.\footnote{We will see later that this is formal graphical notation for bidirectional processes called \emph{optics} (\cite{boisseau_string_2020}).}

\begin{figure}[H]
  \scaletikzfig[0.8]{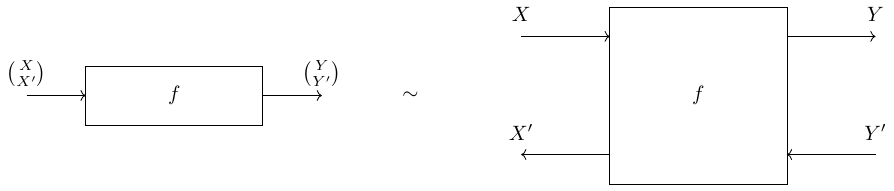}
  \caption{A bidirectional process drawn as a string diagram.}
  \label{fig:process_vs_bidirectional}
\end{figure}

This describes the objects.
What about the morphisms?
Here the situation is more complex.
Looking at the above box of type $\diset{X}{X'} \to \diset{Y}{Y'}$, in many ways we want to abstract away what happens \emph{inside} of it, and instead provide a general formulation of these processes by only describing what they do from the outside.
One thing we might want to do is say that given any map of type $Y \to Y'$ (which we think of as the environment producing a $Y$ and responding with $Y'$), the bidirectional process $\diset{X}{X'} \to \diset{Y}{Y'}$ needs to turn that into a map of type $X \to X'$.
This will prove useful later in the \cref{sec:weighted_optics}.

In other ways, we might want to be more concrete.
What are the types of these forward and backward processes?
A naive approach might attempt to mirror the structure of objects, and take bidirectional processes to be morphisms in $\cC \times \cC^{\op}$, giving their forward and backward maps types $X \to Y$ and $Y' \to X'$, respectively.
But the problem with this is that the backward part needs to be able to depend on the forward one.
For instance, when we perform backpropagation, the type of the gradient we send backwards depends on the original input we received.

We could simply add this as an argument.
That is, a bidirectional process could be considered to have types $X \to Y$ on the forward part and $X \times Y' \to X'$ on the backward one.
This is a construction that is often called \emph{a lens}, and we will thoroughly unpack it in \cref{sec:optics_for_cartesian_act_fw}.
But this has problems as well.
Specifying the backward type as $X \times Y' \to X'$ prevents us from a) modelling cases where, say, depending on the input $X$ we conditionally might choose not to interact with the environment at all (see ``prisms'' in \cref{subsec:prisms}), and b) reusing (in the backward pass) useful values of a type different than $X$ computed in the forward pass (see \cref{subsec:operational}).

But we can alleviate these problems as well.
We can allow the bidirectional process to choose the type of information it communicates from the forward to the backward pass.
This gives it something like an existential quantifier in the type, and at the same time allows us to abstract away from the categorical product $\times$ into a more general monoidal one:
\[
  \exists M . (X \to M \otimes Y) \times (M \otimes Y' \to X')
\]
Roughly, this is a construction that is called \emph{an optic} (\cref{fig:optic_residual}), and its particular generalisation will be the topic of this chapter.

\begin{figure}[h]
  \centering
  \includegraphics[width=.6\textwidth]{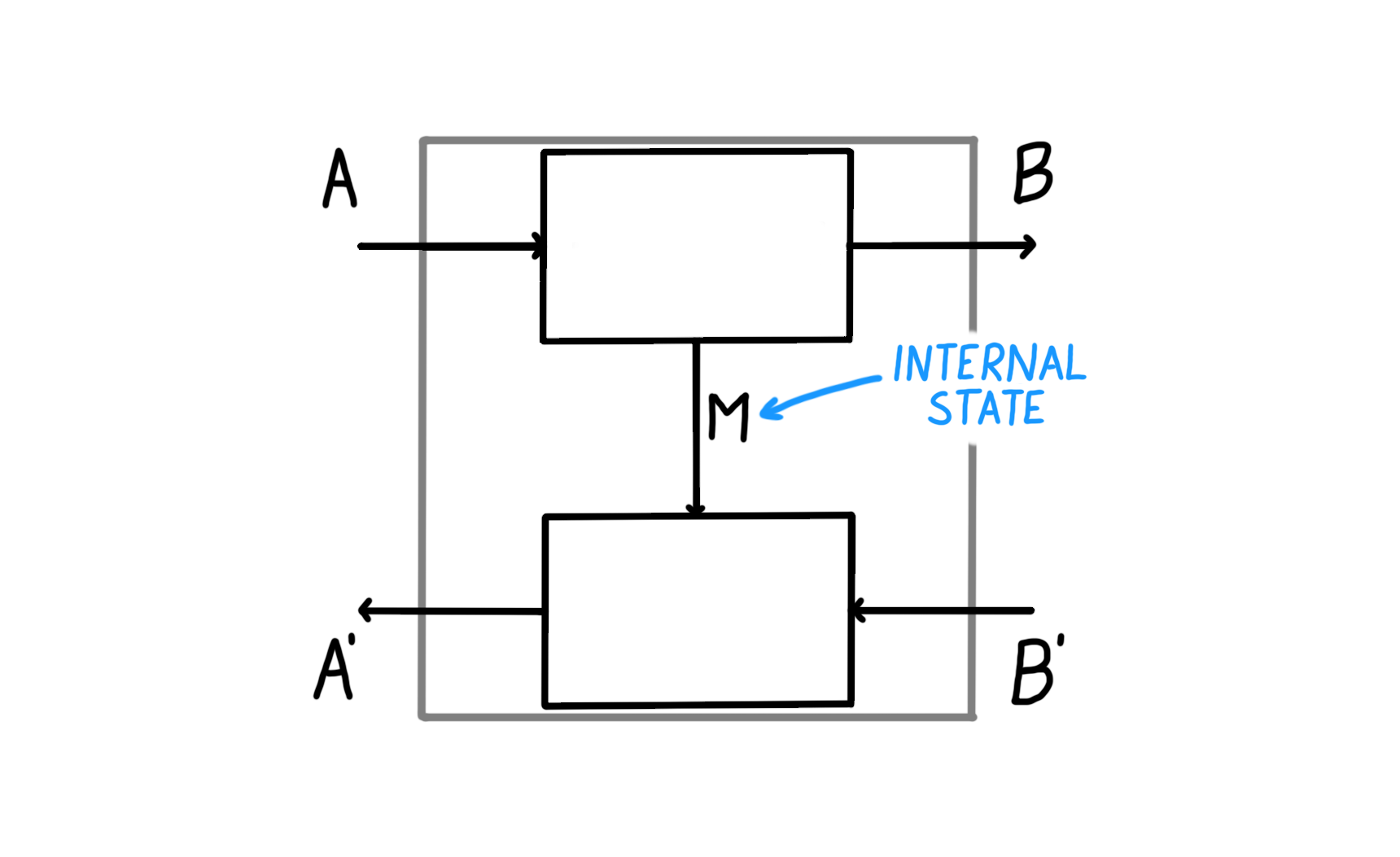}
  \caption{Schematic of a bidirectional process with an internal state.}
  \label{fig:optic_residual}
\end{figure}

But this story has problems as well.
What if the type of the backward pass is different than the one of the forward one?
In backpropagation, for instance, while the forward pass can in principle be any differentiable map, its reverse derivative at some point $x : X$ will always be a \emph{linear map}.
The above framework does not allow us to accurately capture these types, as the forward and the backward maps live in the same category.

But this can be solved as well.
Using actegories defined in the previous chapter (\cref{def:actegory}), we can separate out the type of the forward category $\cC$, the backward category $\cD$ and the monoidal category $\cM$ that serves as the type of intermediate data.
Its type (now more precisely specified) is

\[
  \exists (M : \cM) . \cC(X, M \actfw Y) \times \cD(M \actbw Y', X')
\]

These are called \emph{mixed optics} (see \cref{box:mixed_optics_are_weighted_optics}), and represent the penultimate step of our abstraction.
There is one last problem.
This particular construction has the same monoidal category $\cM$ acting on both the forward and the backward part.
In many of the cases of interest, this is not the case.
In backpropagation, for instance, while the action of the forward part uses the categorical product of a suitable category of differentiable maps (such as $(\Smooth, \times, 1)$), this categorical product is not something we want for the backwards category of linear maps.
In $\FVectR$ it would necessitate that the reverse derivative is linear with respect to the product of the type of the point $X$ we're differentiating at, and the type of the gradient $Y'$.
But this is not true, because the derivative is linear only with respect to the latter.
The monoidal product we want here is the \emph{tensor product}, which is not a categorical product.

But this kind of a complex gadget has not appeared in the literature so far.
It models an intricate interaction between the forward and backward categories, as well as the forward and backward \emph{monoidal categories} of intermediate information.
As a matter of fact, we will see that the data needed to form these kind of optics arises as an \emph{action} of, roughly, the product of monoidal categories of intermediate information \emph{on} the product of forward and backward categories.
We will also see that the category of these optics --- that we will name \emph{weighted} optics --- arises as the dual of the aforementioned $\Para$ construction on this actegory.
We proceed in making this intuition precise and mathematically formal --- leading us to the definition of the construction $\Copara$.

\section{The $\Copara[false]$ construction}
\label{sec:copara}

The $\Copara$ construction arises as a systematic dualisation of the bicategory $\Para$.
Because of its central role in describing bidirectional processes, we take special care in unpacking its details.

\begin{definition}[{Coparametric maps (compare \cite[Rem. 4]{capucci_towards_2022})}]
  \label{def:copara}
  Let $(\cM, \otimes, I)$ be a monoidal category, and $(\cC, \act)$ a left $\cM$-actegory.
  We define the bicategory $\Copara_{\act}(\cC)$ with the following data:
\begin{itemize}
\item \textbf{Objects} are those of $\cC$;
\item \textbf{Morphisms} are often referred to as \emph{coparametric morphisms} (\cref{fig:copara_morphism}). A coparametric morphism $A \to B$ is a pair $(P, f)$ where
  $P : \cM$ and $f : A \to P \act B$ is a morphism in $\cC$.
  As parametric morphisms, coparametric ones also have a horizontal and a vertical component (\cref{fig:copara_morphism}).
  \begin{figure}[H]
    \scaletikzfig[0.8]{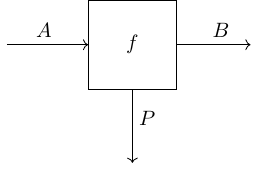}
    \caption{String diagram representation of a coparametric morphism. As with parametric maps (\cref{fig:para_morphism}), here we also draw the coparameter wire vertically, but \emph{below}.}
    \label{fig:copara_morphism}
  \end{figure}
\item \textbf{2-morphisms} are called \emph{reparameterisations}.
  A 2-morphism from $(P, f) \Rightarrow (P', f')$ is a morphism $r : P \to P'$ in $\cM$ such that the following diagram commutes in $\cC$.
  As with parametric morphisms, we will often write $f_r$ for the reparameterisation of $f$ with $r$, where $r$ now appears as a subscript instead of a superscript.
  
  \begin{minipage}{0.55\textwidth}
    \begin{figure}[H]
      \scaletikzfig[0.8]{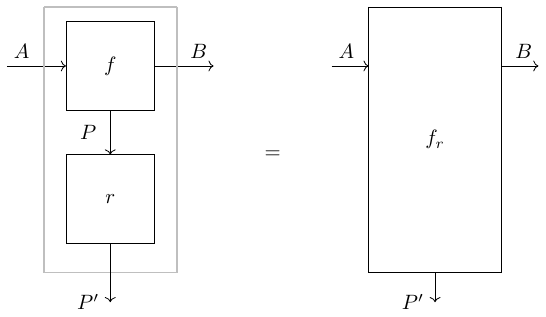}
      \caption{String diagram of reparameterisation. As with \cref{fig:para_reparam}, we exploit the second dimension. Note that here 2-cells point in a different direction than with $\Para$!}
      \label{fig:copara_reparam}
    \end{figure}
  \end{minipage}
  \begin{minipage}{0.45\textwidth}
    \begin{equation}
      \label{eq:copara_triangle_reparam}
      \begin{tikzcddiag}[ampersand replacement=\&]
        A \& {P \act B} \\
        \& {P' \act B}
        \arrow["{r \act B}", from=1-2, to=2-2]
        \arrow["{f'}"', from=1-1, to=2-2]
        \arrow["f", from=1-1, to=1-2]
      \end{tikzcddiag}
    \end{equation}
  \end{minipage}
\item \textbf{Identity morphism.} For every object $A$ there is an identity coparametric morphism $(I, \eta_A)$ where $I : \cM$ and $\eta_A : I \act A \to A$ is the unitor of the underlying actegory;
\item \textbf{Morphism composition.} The composition of coparametric morphisms
  $$
  A \xrightarrow{\quad (P, f)\quad} B \xrightarrow{\quad (Q, g)\quad} C
  $$
  i.e.\ of
  \begin{equation*}
\begin{aligned}
    &P : \cM\\
    &A \xrightarrow{f} P \act B
\end{aligned}
\qquad\qquad\text{and}\qquad\qquad
\begin{aligned}
    &Q : \cM\\
    &B \xrightarrow{g} Q \act C
\end{aligned}  
\end{equation*}
is given by $(P \otimes Q, f \compCopara{\comp} g)$ where
\begin{align*}
  &P \otimes Q : \cM&\\
  &f \compCopara{\comp} g \coloneqq \boxed{A \xrightarrow{f} P \act B \xrightarrow{P \act g} P \act (Q \act C) \xrightarrow{\mu^{-1}_{P, Q, C}}  (P \otimes Q) \act C} &\text{in} \quad \cC
\end{align*}
\begin{figure}[H]
  \scaletikzfig[0.8]{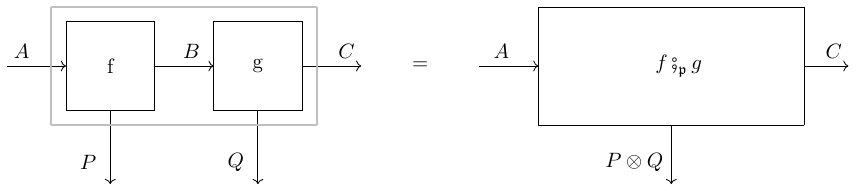}
  \caption{String diagram representation of the composition of coparametric morphisms, following the same organisational principles as \cref{fig:para_composition}.}
  \label{fig:copara_composition}
\end{figure}
\end{itemize}
This concludes the definition of the bicategory $\Copara_{\act}(\cC)$.
\end{definition}

Unlike with $\Para$, here we require $(\cC, \actfw)$ to be a \emph{left} actegory, inspired by the idea that, despite the equivalence of left and right actegories when $\cM$ is braided (\cref{lemma:braided_right_to_left}) (which is the  assumption peremeating this thesis) it is good \emph{mental hygiene} to distinguish between the two (\cref{rem:act_left_vs_right}).
In the same manner that left and right actegories become equivalent under braiding, so do $\Para$ and $\Copara$ constructions thereof.

\begin{proposition}[Relation between $\Para$ and $\Copara$]
  \label{prop:relation_para_copara}
  Let $(\cC, \act)$ be a left $\cM$-actegory where $\cM$ is braided.
  Then
  \[
    \Para_{\act}(\cC^{\op}) \cong \Copara_{\act}(\cC)^{\co\op} \quad \text{and}\quad \Copara_{\act}(\cC^{\op}) \cong \Para_{\act}(\cC)^{\co\op}
  \]
  where we're overloading $\act$ to also include the action from the right given by (\cref{lemma:braided_right_to_left}).
\end{proposition}

\begin{proof}
  We prove both claims side by side.
  
  \begin{minipage}[t]{0.45\linewidth}
    \begin{equation*}
      \begin{aligned}
      & \Para_{\act}(\cC^{\op})(A, B)\\
      =& \sum_{P : \cM}\cC^{\op}(A \act P, B)\\
      =& \sum_{P : \cM}\cC(B, A \act P)\\
      =& \Copara_{\act}(\cC)^\co(B, A)\\
      =& \Copara_{\act}(\cC)^{\co\op}(A, B)
      \end{aligned}
    \end{equation*}
  \end{minipage}
  \hfill
  \begin{minipage}[t]{0.45\linewidth}
    \begin{equation*}
      \begin{aligned}
      & \Copara_{\act}(\cC^{\op})(A, B)\\
      =& \sum_{M : \cM}\cC^{\op}(A, M \act B)\\
      =& \sum_{P : \cM}\cC(M \act B, A)\\
      =& \Para_{\act}(\cC)^\co(B, A)\\
      =& \Para_{\act}(\cC)^{\co\op}(A, B)
      \end{aligned}
    \end{equation*}
  \end{minipage}
  \\

  The reason $^{\co}$ appears is because 2-cells in $\Para$ and $\Copara$ point in different directions.
  Where braiding comes in play is in morphism composition -- when we need to change the order in which the composite $M \otimes N$ acts on an object of $\cC$.
\end{proof}

The above proposition allows us to translate all the theorems and propositions about $\Para$ to theorems about $\Copara$.
This includes (\cref{prop:para_local_cat_elements}) which, translated to $\Copara$ in \cref{prop:copara_local_cat_elements}, tells us that locally the hom-category of $\Copara$ can also be computed as a category of elements, albeit now of a functor whose domain is $\cM$ and not $\cM^{\op}$.

\begin{proposition}
  \label{prop:copara_local_cat_elements}
  Let $(\cC, \act)$ be a $\cM$-actegory.
  Let $A, B : \cC$. Then
  \[
    \Copara_{\act}(\cC)(A, B) = \El(\cC(A, B \act - )) = \sum_{P : \cM}\cC(A, B \act P)
  \]
\end{proposition}

This means that the trivialisation theorem (\cref{thm:para_trivialises_when_unit_initial}) now holds for the case of $\cM$ with a monoidal unit which is \emph{terminal}, instead of initial.\footnote{Intuitively, notice that any coparametric map $f : A \to M \act B$ can be reparameterised with $\terminal_M : M \to 1$ yielding $f \comp (\terminal_M \act B)$ which under connected components becomes equivalent to $f$.}
This yields the following theorem, and corollary.

\begin{theorem}
  \label{thm:copara_trivialises_when_unit_initial}
  Let $(\cC, \act)$ be a $\cM$-actegory where the monoidal unit of $\cM$ is terminal.
  Then
  \[
    \pibc(\Copara_{\act}(\cC)) \cong \cC
  \]
\end{theorem}

\begin{corollary}
  \label{corr:cart_trivialises}
  Let $\cC$ be a cartesian category. Then $\pibc(\Copara(\cC)) \cong \cC$.
\end{corollary}

As we will see in \cref{rem:cart_act_dont_trivialise}, $\Copara(\cC)$ will not trivialise for a generalisation of a cartesian category to the setting of actegories.

\subsection{Forget, then act: a piece of notation}
\label{subsec:a_piece_of_notation}

We are almost ready to define the category of bidirectional processes.
We just need to introduce one important piece of notation.
This notation will capture the idea that we are mostly interested in the $\Copara$ construction on actegories of a particular form.\footnote{The entire story holds for $\Para$ as well, but is less relevant for us.}

Consider a $\cE$-actegory $(\cC, \act)$ for some monoidal category $\cE$.
Think about the role of an object $E$ in a coparametric morphism $(E : \cE, f : X \to E \act Y$).
Even though it appears in the codomain of $f$, it really is a \emph{part of the data of the coparametric morphism}.
And it is a part of the data which determines \emph{the type} of the implementation map $f$.
This determination can happen in many different ways, and the type of implementation maps may end up being the same for different coparameters $E$ and $E'$.
For instance, this is the case when the actions of $E$ and $E'$ on $B$ are equal (i.e.\ $E \act Y = E' \act Y$), in which case we have forgotten something about $E$ and $E'$ in the action.
This suggests that often, our action should really be of the form $\pi(E) \act Y$, where $\pi : \cE \to \cB$ is some kind of a forgetful functor, and $\bullet$ is of type $\cB \times \cC \to \cC$.
When is this the case?

This happens under two conditions.
The first one is that the category $\cE$ arises as the category of elements $\El(W)$ (\cref{def:cat_of_elem}) of some functor $W : \cB^{\op} \to \Set$, where $\El(W)$ is how we will refer to $\cE$ from now on.
This substantiates the idea that an object $E : \El(W)$ consists of two parts --- a part $B : \cB$ which acts on an object of $\cC$, and a part $B_W : W(B)$ which only serves as an inert piece of additional data.
In this case we have a canonical functor $\El(W) \xrightarrow{\pi_W} \cB^{\op}$ which forgets these additional pieces of data\footnote{This functor is a \emph{fibration}, a concept we mention for completeness but do not touch upon in this thesis.}, allowing us to form the functor $\act^{\pi_W^{\op}}$ below (where we're using the notation for reparameterisation of actegories from \cref{def:reparameterisation}).

\begin{equation}
    \label{eq:actegory_weight_projection}
    \act^{\pi_W^{\op}} = \boxed{\El(W)^{\op} \times \cC \xrightarrow{\pi^{\op}_W \times \cC} \cB \times \cC \xrightarrow{\act} \cC}
\end{equation}

This brings us to the second condition.
For this to be a valid reparameterisation of actegories (\cref{def:reparameterisation}), $\pi_W^{\op}$ needs to be a strong monoidal functor.
And that indeed happens if $W$ is lax monoidal (\cref{prop:monoidal_structure_category_of_elements}).

With these two conditions we get a a $\El(W)^{\op}$-actegory $\cC$ which explicitly acknowledges these two components in the acting category.
More specifically, by instantiating $\Copara_{\act^{\pi_W^{\op}}}(\cC)$ we can see that a coparametric morphism $X \to Y$ now consists of a parameter object $(B, B_W)$ (where $B : \cB$ and $B_W : W(B)$), and an implementation map $f : X \to B \actfw X$ which is independent of the choice of $B_W$.
Additionally, we will often be interested in the coparametric \emph{category}, and one that arises out of the connected components quotient.
This is why we introduce the following piece of notation.

\begin{equation}
  \label{def:coparaw}
  \CoparaW{W}{\act}(\cC) \coloneqq \pibc(\Copara_{\act^{\pi_W^{\op}}}(\cC))
\end{equation}

This piece of notation will be central in our definition of optics.
It describes a \emph{weighted} coparametric category of an action $\act$ and weight $W$, where in the presence of the superscript $W$ on $\Copara$ the subscript action $\act$ is meant to be reparameterised by $\pi_W$.
The reason we use the adjective ``weighted'' is because, analogously to \cref{prop:para_local_cat_elements} which tells us that the connected components quotient locally computes a colimit of the functor $\cC(X, - \act Y)$, here we will see that the same quotient ends up computing a colimit of $\cC(X, - \act Y)$ \emph{weighted} by $W$ (\cref{prop:copara_weighted_colimit}).

\begin{proposition}
  \label{prop:copara_weighted_colimit}
  Let $X, Y : \CoparaW{W}{\act}(\cC)$.
  Then
  \[
    \CoparaW{W}{\act}(\cC)(X, Y) \cong \colim^W\cC(X, - \act Y)
  \]
\end{proposition}

\begin{proof}
  The proof below relies crucially on identification the connected components of the category of the elements with the colimit, and identification of $\Set$-colimits of a particular form with weighted colimits (\cref{prop:from_set_weighted_colimits_to_colimits}).
  \begin{alignat*}{2}
    & && \CoparaW{W}{\act}(\cC)(X, Y)\\
    &\text{(Def.)}   &=& \conncomp(\Copara_{\act^{\pi_W}}(\cC)(X, Y))\\
    &\text{(\cref{prop:copara_local_cat_elements})}   &\cong& \pibc(\El(\El(W)^{\op} \xrightarrow{\pi^{\op}_W} \cB \xrightarrow{\cC(X, - \act Y)} \Set) \\
    &\text{(\cref{prop:colim_conn_comp_of_el})} &\cong& \colim(\El(W)^{\op} \xrightarrow{\pi^{\op}_W} \cB \xrightarrow{\cC(X, - \act Y)} \Set) \\
    &\text{(\cref{prop:from_set_weighted_colimits_to_colimits})}   &\cong& \colim^W \cC(X, - \act Y) \tag*{\qedhere}
  \end{alignat*}
\end{proof}

\begin{remark}
  We've defined $\CoparaW{W}{\act}$ using the projection from the category of elements of $W$.
  This was only possible because the weight is a functor into $\Set$, i.e.\ because we are dealing with $\Set$-enriched categories.
  The fact that in $\Set$ this construction ends up being isomorphic to a colimit weighted by $W$ (\cref{prop:copara_weighted_colimit}) suggests an avenue for generalisation to a setting enriched in an arbitrary monoidal category $\cV$: by taking the aforementioned proposition as a definition.
\end{remark}

\section{Weighted optics}
\label{sec:weighted_optics}

We are now in position to understand one of the central definitions of this thesis: weighted optics.
The first step is to fix a ``forward'' left $\cM$-actegory $(\cC, \actfw)$ and a ``backward'' left $\cN$-actegory $(\cD, \actbw)$, and form their product.

\begin{proposition}[Product of actegories, {\cite[Prop. 4.2.1.]{capucci_actegories_2023}}]
  \label{prop:product_of_actegories}
  Fix the following.
  \begin{center}
    A $\cM$-actegory $(\cC, \actfw)$\\
    A $\cM'$-actegory $(\cD, \actbw)$
  \end{center}
  Their cartesian product is a $\cM \times \cM'$-actegory $(\cC \times \cD, \diset{\actfw}{\actbw})$ where the action is defined as

  \begin{equation}
    \label{eq:parallel_action}
    \diset{\actfw}{\actbw} \coloneqq \boxed{\cM \times \cM' \times \cC \times \cD \xrightarrow{\cM \times \swap \times \cD} \cM \times \cC \times \cM' \times \cD \xrightarrow{\actfw \times \actbw} \cC \times \cD}
  \end{equation}
\end{proposition}

A sensible attempt at defining these bidirectional processes is to compute the $\Copara$ construction of this product actegory.\footnote{Recall that a $\cM$-actegory $\act : \cM \times \cD \to \cD$ gives rise to a $\cM^{\op}$ actegory $\cD^{\op}$ given by $\cM^{\op} \times \cD^{\op} \to \cD^{\op}$.}
That is, we can dualise the second actegory first, and then take the product --- forming a $\cM \times \cM'^{\op}$-actegory $(\cC \times \cD^{\op}, \diset{\actfw}{\actbw^{\op}})$.
However, simply taking $\Copara$ of this actegory will not suffice.
Let's see why that is the case.

A coparametric morphism in $\Copara_{\diset{\actfw}{\actbw^{\op}}}(\cC \times \cD^{\op})$ consists of two choices of coparameters $M : \cM$ and $M' : \cM'$ and two maps $f : \mi \to M \actfw \mo$ and $f' : \mo ' \actbw M' \to \mi '$, depicted in \cref{fig:missing_internal_state}.

\begin{figure}[h]
  \scaletikzfig[0.8]{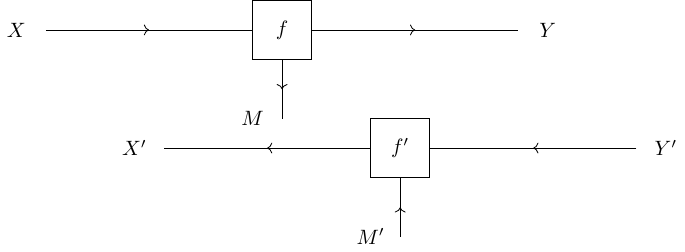}
  \caption{Graphic of this coparametric category showing its missing internal state.}
  \label{fig:missing_internal_state}
\end{figure}

But this is problematic --- as there is nothing connecting $M$ and $M'$.
That is, given any way of turning a $Y$ into a $Y'$, using $f$ and $f'$ we can't turn $X$ into an $X'$ because we are missing this piece of connecting tissue.
This extra piece of data, we can also see, would not impact what the action $\diset{\actfw}{\actbw}$ does --- it would simply be a part of the data of the morphism that would enable us to make sense of bidirectionality this way.
How do we formally describe this?
As hinted by \cref{subsec:a_piece_of_notation}: we can do that with weighted Copara.

\begin{figure}[h]
  \scaletikzfig[0.8]{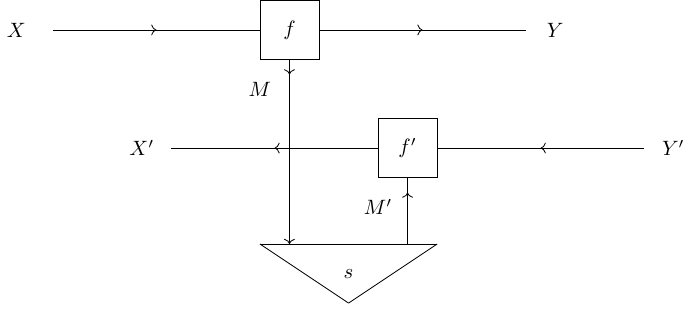}
  \caption{Graphic of the putative correct coparametric category (optic, not a category) augmented with the data of a ``weight''.}
  \label{fig:putative_correct_internal_state}
\end{figure}

Since we need this for any possible residual, this can be included as an additional piece of data in the definition of a bidirectional process: a functor $W : \cM \times \cM' \to \Set$ which, given any two objects $(M, M')$ gives us a set of ``connecting tissues'' between $M$ and $M'$, where $s$ in \cref{fig:putative_correct_internal_state} above is an element of this set.
This results in a description of this bidirectional process which is ``closed off'' in the vertical dimension by $W$ which we think of as the costate.
As we will see below, in order for everything to be well-defined we will need $W$ to be a lax monoidal functor with respect to the induced monoidal structure on the domain, and the product monoidal structure of $\Set$ on the codomain.

This brings us to the central definition of the paper.

\begin{definition}[Category of weighted optics]
  \label{def:weighted_optic}
  Fix the following data.
  \begin{center}
    A $\cM$-actegory $(\cC, \actfw)$\\
    A $\cM'$-actegory $(\cD, \actbw)$\\
    A lax monoidal functor $(W, \phi, \epsilon) : \cM^{\op} \times \cM' \to \Set$
  \end{center}
  We define the \textbf{category of weighted optics} $\Optic^W_{\diset{\actfw}{\actbw}}$ as a weighted coparametric category
  \[
    \CoparaW{W}{\diset{\actfw}{\actbw^{\op}}}(\cC \times \cD^{\op})
  \]
\end{definition}

This construction packs a lot of punch, and we will now take the time to unpack its contents.
An object in $\Optic^W_{\diset{\actfw}{\actbw}}$ is a pair $\diset{A}{A'}$, where $A : \cC$ and $A' : \cD^{\op}$.
Following \cref{prop:copara_weighted_colimit}, the set of morphisms $\diset{A}{A'} \to \diset{B}{B'}$ is given by the weighted colimit, where we'll often use the following explicit coend representation.
\begin{align}
  \label{eq:weighted_optic_coend_form}
  \Optic^W_{\diset{\actfw}{\actbw}}\OpticHom{A}{A'}{B}{B'} =& \colim^{W}_{\diset{M}{M'}}\cC(A, M \act B) \times \cD(M' \actbw B', A') \nonumber \\
  \cong & \text{(Weighted colimit as coend (\cref{prop:weight_colimit_as_coend}))}\\
        & \int^{M : \cM, M' : \cM'} \cC(A, M \act B) \times W(M, M') \times \cD(M' \actbw B', A') \nonumber
\end{align}

Weighted optics are elements of this set, and they consist of the following three pieces of data:
\begin{itemize}
\item A coparameter $\diset{M}{M'}$ consisting of a forward and a backward residual, where $M : \cM$ and $M' : \cM'$;
\item A map $s : W(M, M')$ which we think of as the connecting tissue between the forward and the backward residuals;
\item A coparametric map $\diset{f}{f'} : (\cC \times \cD^{\op})\OpticHom{A}{A'}{M \actfw B}{M' \actbw B'}$, i.e.\ a pair of maps $f : A \to M \actfw B$ in $\cC$ and $f' : M' \actbw B' \to A'$ in $\cD$ (\cref{fig:putative_correct_internal_state}).
\end{itemize}
  
  This triple is quotiented by an equivalence relation which identifies it with any other $(\diset{N}{N'}, t, \diset{g}{g'})$ if there exists a pair
  of maps $r : M \to N$ in $\cM$ and $r' : N' \to M'$ in $\cM'$ such that
  $W(r, r')(t) = s$ and the following diagrams commute:
  \begin{equation}
    \label{eq:copara_optic_diagram}
    \begin{tikzcddiag}[ampersand replacement=\&]
        A \&\& {M \actfw B} \&\& {M' \actbw B'} \&\& {A'} \\
        \\
        \&\& {N \actfw B} \&\& {N' \actbw B'}
        \arrow["f", from=1-1, to=1-3]
        \arrow["{r \actfw B}", from=1-3, to=3-3]
        \arrow["g"', from=1-1, to=3-3]
        \arrow["{f'}", from=1-5, to=1-7]
        \arrow["{g'}"', from=3-5, to=1-7]
        \arrow["{r' \actbw B'}", from=3-5, to=1-5]
      \end{tikzcddiag}
    \end{equation}
  Identity and composition of weighted optics follow by those of $\Copara$.
  
  \textbf{Identity.} The identity optic $\diset{A}{A'} \to \diset{A}{A'}$ is the one whose coparameters are $\diset{I}{I'}$, the map connecting them is one given by the opunitor $\epsilon(\bullet) : W(I, I')$ of $W$ on the unique element $\bullet : 1$, and implementation $\diset{\eta^{\actfw}_A}{{\eta^{\actbw}}^{-1}_{A'}}$ is given by the unitors of the two actegories (\cref{fig:optic_identity}).
  \begin{figure}[H]
    \scaletikzfig[0.65]{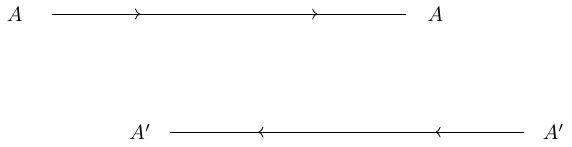}
    \caption{The identity weighted optic.}
    \label{fig:optic_identity}
  \end{figure}
  
  \textbf{Composition.} Composition of
  \[
    \diset{A}{A'} \xrightarrow{\quad (\diset{M}{M'}, s, \diset{f}{f'})\quad} \diset{B}{B'} \xrightarrow{\quad (\diset{N}{N'}, t, \diset{g}{g'})\quad} \diset{C}{C'}
  \]
  i.e.\ of
  \begin{equation*}
    \begin{aligned}
      &\diset{M}{M'} : \diset{\cM}{\cM'} \\
      &s : W(M, M') \\
      &\diset{A}{A'} \xrightarrow{\diset{f}{f'}} \diset{M \actfw B}{M' \actbw B'}
    \end{aligned}
    \qquad\text{and}\qquad
    \begin{aligned}
      &\diset{N}{N'} : \diset{\cM}{\cM'} \\
      &t : W(N, N') \\
      &\diset{B}{B'} \xrightarrow{\diset{g}{g'}} \diset{N \actfw C}{N' \actbw C'}
    \end{aligned}  
  \end{equation*}
  is a triple $(\diset{M \otimes N}{M' \otimes N'}, \phi^W((s, t)), \diset{f \compCopara{\comp} g}{g' \compPara{\comp} f'})$ where
  \begin{align*}
    &\diset{M \otimes N}{M' \otimes N'} : \diset{\cM}{\cM'} \\
    &\phi^W((s, t)) : W(M \otimes N, M' \otimes' N')
  \end{align*}
  and $f \compCopara{\comp} g : A \to (M \otimes N) \actfw C$ and $g' \compPara{\comp} f' : (M' \otimes' N') \actbw C \to A$ are compositions of coparametric and parametric maps, respectively (\cref{fig:optic_composition}).
  \begin{figure}[H]
    \scaletikzfig[0.65][1.1]{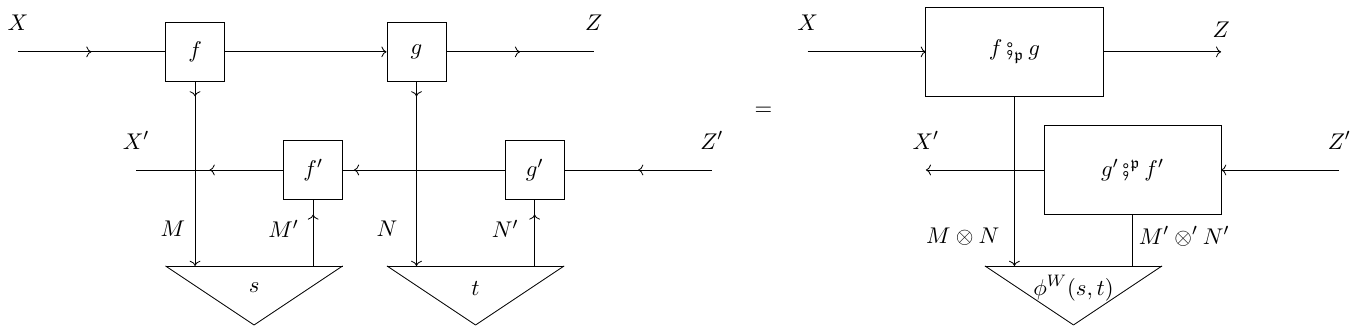}
    \caption{Composition of weighted optics.}
    \label{fig:optic_composition}
  \end{figure}
  \begin{notation}
    \label{not:optic}
    Like we did with $\Para$ and $\Copara$, we will simplify the notation for weighted optics in special cases.
    We will omit the superscript $W$ if the weight is a hom-functor (see \cref{box:mixed_optics_are_weighted_optics}).%
    If optics are given by a self-action of some monoidal category $\cC$ we will write $\Optic(\cC)$ (see \cref{ex:optics_for_products,ex:optics_for_coproducts}), trusting that it is clear from the context what the underlying monoidal structure is.
    These reductions agree with the notation in \cite{clarke_profunctor_2022, riley_categories_2018}.
  \end{notation}

  While this weighted colimit definition might seem daunting, there are a number of benefits of thinking about bidirectional processes this way.
  First of all, it's conceptual clarity.
  Weighted optics allow us to separate the formal structure of the acting categories (describing how residuals compose) from the formal structure of the forward and backward categories (describing the kinds of processes on the forward and the backward pass), from the intermediate process happening in between the forward and the backward pass.
  This conceptual clarity becomes useful in implementation of these systems on computers (\cref{ch:backpropagation}), where the posession of mathematical tools to talk about their practical aspects such as the memory footprint of their composition improves our ability to implement them correctly and efficiently.
  Second, the reduction of the definition of weighted optics to that of the $\Copara$ construction provides us with general means of proving many theorems about optics (\cref{sec:optic_monoidal,sec:optics_for_cartesian_act_fw,sec:optics_backward_closed_actegory}).
  Third, this definition suggests a generalisation to \emph{dependent optics}, an important unification of dependent lenses and optics (\cref{subsec:dependent_optics}).

  We now proceed to reduce this abstract definition into many of its concrete forms.
  Many special cases of optics were only originally defined as these reductions, each of them with their own composition rule.

  \begin{proposition}
    \label{prop:weighted_to_mixed_representable}
    Consider an optic $\Optic^W_{\diset{\actfw}{\actbw}}$  with a weight represented by a strong monoidal functor $i : \cM \to \cM'$ (resp.\ corepresented by a strong monoidal functor $j : \cM' \to \cM$).
    Then there is the following isomorphism of categories:
    \[
      \Optic^W_{\diset{\actfw}{\actbw}} \cong \Optic_{\diset{\actfw}{\actbw^i}} \quad \text{(resp.)} \quad \Optic^W_{\diset{\actfw}{\actbw}} \cong \Optic_{\diset{\actfw^j}{\actbw}}
    \]
  \end{proposition}
  \begin{proof}
    Note that $i$ (resp.\ $j$) needs to be strong in order for $\actbw^i$ (resp.\ $\actfw^j$) to be a well-defined reparameterisation of actegories.
    We prove only the left isomorphism by exhibiting an isomorphism of hom-sets and noting that naturality with respect to identities and composition follows by routine.
    \pagebreak
    \begin{alignat*}{2}
      & && \Optic^W_{\diset{\actfw}{\actbw}}\OpticHom{A}{A'}{B}{B'}\\
      &\text{(Def.)}                           &=& \int^{M, M'} \cC(A, M \act B) \times W(M, M') \times \cD(M' \actbw B', A')\\
      & \text{($W$ is representable)}      &\cong& \int^{M, M'} \cC(A, M \act B) \times \cM'(i(M), M') \times \cD(M' \actbw B', A') \\
      & \text{(coYoneda)}      &\cong& \int^{M} \cC(A, M \act B) \times \cD(i(M) \actbw B', A')\\
      &\text{(Def.)}  &=& \Optic_{\diset{\actfw}{\actbw^i}}\OpticHom{A}{A'}{B}{B'} \tag*{\qedhere}
    \end{alignat*}
  \end{proof}
  This tells us the following.
  \begin{mybox}[label=box:mixed_optics_are_weighted_optics]{teal}{Mixed optics are weighted optics}
    A variant of optics called \emph{mixed optics} (\cite{clarke_profunctor_2022}) is a special case of this definition.
    It arises when $\cM = \cM'$, and when the weight is given by the Hom-functor of $\cM$.\footnote{This functor needs to be lax monoidal in order to be a well-defined weight, but that is indeed always the case if $\cM$ itself is monoidal. Intrestingly it is only lax monoidal, and not strong monoidal.}
    \begin{align*}
      \Optic^{\Hom}_{\diset{\actfw}{\actbw}}\OpticHom{A}{A'}{B}{B'} =& \colim^{\Hom}_{\diset{M}{M'}}\cC(A, M \act B) \times \cD(M' \actbw B', A')\\
      \cong & \text{(Colimit weighted by hom is coend)}\\
            & \int^{M} \cC(A, M \act B) \times \cD(M \actbw B', A')
       \end{align*}

    This yields the string diagrams of optics as found in \cite{capucci_towards_2022, boisseau_string_2020}.
    \scaletikzfig[0.72]{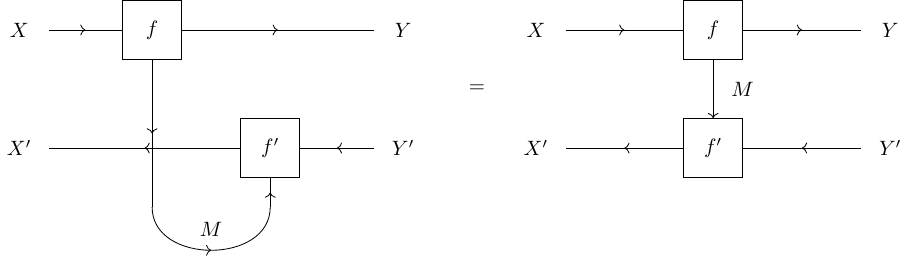}
    \captionof{figure}{Optics weighted by the hom-functor can absorb the weight into into either the forward or the backward map.}
  \end{mybox}

  This gives us a straightforward way to see that there are classes of weighted optics not captured by mixed ones: they're the ones where the representing functor is merely lax monoidal instead of strong.
  We now proceed to study examples of weighted optics, and then give general classes of theorems allowing us to state when optics are monoidal, cartesian or closed.

  \subsection{Examples of weighted optics}
  \label{subsec:examples_weighted_optics}
  We now aim to give the flavour of the kind of generality this definition possesses.
  They capture a disparate array of phenomena, each of a radically different nature: deterministic, conditional, and probabilistic, for example.
  The simplest kind of bidirectional systems are those whose forward and backward processes do not communicate.
  \begin{example}[{Adapter, \cite[Def.\ 3.38]{clarke_profunctor_2022}}]
    \label{ex:adapters}
    Consider the trivial action $\actfw$ on some category $\cC$ (\cref{ex:trivial_action}).
    Then a morphism $\diset{A}{A'} \to \diset{B}{B'}$ in the category $\Optic_{\diset{\actfw}{\actfw}}$ is often called \newdef{an adapter} and it consists of two separate maps in $\cC$ of type $A \to B$ and $B' \to A'$.
  \end{example}
  The category in the above example is isomorphic to the product category $\cC \times \cC^{\op}$, something which clearly is not the case for all optics.
  The two main and non-trivial examples of these are optics for a products (often called \emph{lenses}) and optics for a coproducts (often called \emph{prisms}).
  \begin{example}[Optics for products]
    \label{ex:optics_for_products}
    Consider any cartesian category $\cC$.
    Then morphisms in the category $\Optic(\cC)$ for self-action of $\cC$ (following the reduction from \cref{box:mixed_optics_are_weighted_optics}) of type $\diset{A}{A'} \to \diset{B}{B'}$ involve a choice of a residual $M : \cC$, a forward map $A \to B \times M$ and a backward map $M \times B' \to A'$.
    We interpret this optic as a system which irrespective of the input $A$, always performs some internal computation (via $M$) \emph{and} also interacts with the environment (via the query which produces a $B$ and consumes a $B'$).\footnote{For more detail on this operational perspective see \cite{gavranovic_lenses_2022}.}
    Often called \newdef{lenses}; for full disambiguaton see \cref{sec:optics_for_cartesian_act_fw}.
  \end{example}
  Examples of this are the usual suspects of $\Set$, $\Smooth$, and so on.
  But we can also describe ``conditional`` bidirectional systems.
  \begin{example}[Optics for coproducts]
    \label{ex:optics_for_coproducts}
    Consider any cocartesian category $\cC$.
    A morphism in $\Optic(\cC)$ of type $\diset{A}{A'} \to \diset{B}{B'}$ involves a choice of a residual $M : \cC$, a forward map $A \to B + M$ and a backward map $M + B' \to A'$.
    We interpret this optic as a system that consumes an input $A$ and provides the response $A'$ by \emph{either} a) querying the environment (i.e.\ producing a $B$ while saving no intermediate computation, and turning the $B'$ received by the environment into an $A'$) or b) not querying the environment, instead turning that $A$ directly into some $A'$ locally.
    Often called \newdef{prisms}; for full disambiguation see \cref{subsec:prisms}.
  \end{example}

  So far, the forward and the backward categories equal, and acting categories either trivial, or equal to forward/backward categories.
  We now study examples where that's not the case.
 
  \begin{example}[Markov kernels and expectations]
    \label{ex:mark_ker_expectation}
    Recall the setting of Markov kernels and expectations from \cref{ex:mark_kernel_reparameterisation} where we defined an action
    \[
      {\otimes'}^\Delta : \Mark \times \Conv \to \Conv
    \]
    By using this as the backward action, and the self-action of $\Mark$ as the forward action, we can form the category of mixed optics $\Optic_{\diset{\otimes}{{\otimes'}^\Delta}}$ of Markov kernels and expectations (\cite{hedges_value_2023}).
    Unlike previous examples, the forward part of an optic $\diset{A}{A'} \to \diset{B}{B'}$ here consists of a forward map $f : \Mark(A, M \times B)$ which in probabilistic manner interacts with its internal state, not producing a residual $m : M$, but instead a joint distribution over the space of residuals and outputs.
  \end{example}

  While we don't directly study probability theory in this thesis, the above example --- useful in reinforcement learning \cite{sutton_reinforcement_2018}, where value iteration can be expressed in terms of optic composition \cite{hedges_value_2023} --- illustrates the breadth of examples covered by the definition of weighted optics.

  This example too does not require the full breadth of weighted optics --- this is merely a mixed optic.
  But we note that weighted optics here provide an equivalent definition which has a more principled compartmentalisation of its components.
  Via \cref{prop:weighted_to_mixed_representable} we can show that
  the above category is isomorphic to
  \[
    \Optic_{\diset{\otimes}{{\otimes'}^\Delta}} \cong \Optic^{\Conv(\Delta(-), -)}_{\diset{\otimes}{\otimes'}}
  \]
  where on the right side we have pulled out the weight from the action, leaving only the monoidal product $\otimes'$ in the backward pass.
  This gives us an interpretation of mixed optics used in value iteration as corepresentably weighted optics with self-actions of two monoidal categories.

  The reason why we were able to do that is because $\Delta$ is a \emph{strong} monoidal functor.
  While the previous example is captured in \cite{hedges_value_2023} with merely mixed optics, the following example is a construction that requires the full generality of weighted optics because it contains an analogue of $\Delta$ which is not strong, but only lax.

  \begin{example}[Smooth maps and linear maps]
    \label{ex:smooth_and_linear_weighted_optics}
    Recall the setting of the cartesian category $\Smooth$ (\cref{def:smooth}), the monoidal category $\FVectR$ (\cref{def:fvect}), and their relationship via the lax monoidal functor $\laxsmoothiota$ (\cref{ex:iota_lax_monoidal}).
    Then by setting $\Smooth$ as forward self-action (denoted by $\times$), and $\FVectR$ as the backward self-action (denoted by $\otimes'$) we can define the category
    \[
      \Optic^{\Smooth(-, \laxsmoothiota(-))}_{\diset{\times}{\otimes}}
    \]
    whose forward passes are smooth functions, and backward maps are bilinear maps.
    Explicitly, a morphism $\diset{A}{A'} \to \diset{B}{B'}$ here consists of the following three pieces of data, quotiented out by the equivalence relation in \cref{eq:copara_optic_diagram}:
    \begin{itemize}
    \item A coparameter $\diset{M}{M'}$, where $M : \Smooth$ and $M' : \FVectR$;
    \item A map $s : \Smooth(M, \laxsmoothiota(M'))$;
    \item A coparametric map $\diset{f}{f'} : \Smooth \times \FVectR^{\op}\OpticHom{A}{A'}{M \times B}{M' \otimes B'}$, i.e.\ a pair of maps $f : \Smooth(A, M \times B)$ and $f' : \FVectR(M' \otimes B', A')$.
    \end{itemize}
    This will be instrumental in thinking about neural networks.
    Intuitively, the map $f$ is the forward map of our neural network.
    In addition to computing the output $B$ it also explicitly models the intermediate result $M$ that will be needed for the backward pass.
    The object $M$ is most often $A$. If we set $M' = [B', A']$ as the space of linear maps  $B' \multimap A'$, then the $s$ could be thought of as the Jacobian of $A \xrightarrow{f} M \times B \xrightarrow{\pi_B} B$, and $f'$ as the evaluation map $\eval_{B', A'}$.
    We describe all of this in detail in \cref{ch:backpropagation}.
  \end{example}
  Because $\laxsmoothiota$ here is merely lax, and not strong monoidal, we cannot reduce this to mixed optics.\footnote{If we tried reducing it, we'd find that the forward actegory isn't strong, but lax. On the other hand, theere is a natural place for a lax component in the definition of weighted optics: in the weight.}
  
\subsection{When are optics monoidal?}
\label{sec:optic_monoidal}

As the category of optics is defined as a $\Copara$ construction, we might wonder whether existing theorems for monoidality of $\Para$ (\cref{prop:para_monoidal} and \cref{thm:para_moncat}) can be appropriated for this coparametric setting.
That is indeed the case.
To turn $\Copara$ monoidal the requirements on the base actegory are equivalent requirements to those for turning $\Para$ monoidal.\footnote{A difference between these arises when we study weakenings of actegories. Requirements for $\Copara$ can be weakened to a lax monoidal actegory, and requirements for $\Para$ to an \emph{oplax} monoidal actegory.}

\begin{proposition}[{Monoidal structure of $\Copara$, \cite[p. 238]{capucci_towards_2022}}]
  \label{prop:copara_monoidal}
  Let $(\cC, \act)$ be a monoidal $\cM$-actegory with the underlying product $(\boxtimes, J)$.
  Then $\Copara_{\act}(\cC)$ becomes a monoidal bicategory, defined analogous to \cref{prop:para_monoidal}, where the monoidal product of two coparametric morphisms
  \[
    (M, f : A \to M \act B) \quad \text{and} \quad  (N, g : C \to N \act D)
  \]
  is the pair $(M \otimes N, f \compCopara{\boxtimes} g)$ where
  \[
    f \compCopara{\boxtimes} g \coloneqq \boxed{A \boxtimes C \xrightarrow{f \boxtimes g} (M \act B) \boxtimes (N \act D) \xrightarrow{\interchanger_{M, B, N, D}^{-1}} (M \otimes N) \act (B \boxtimes D)}
  \]
\end{proposition}

Let's unpack what this means for optics.
To form optics as a coparametric category, we need a product of two actegories, and reparameterisation thereof.
It is straightforward to see that the product of two actegories is monoidal if each individually is.
\begin{proposition}[Product of monoidal actegories is a monoidal actegory]
  \label{prop:product_monoidal_act}
  Fix a monoidal $\cM$-actegory $(\cC, \act)$ and a monoidal $\cN$-actegory $(\cD, \actbw)$.
  Then their product (as defined in \cref{prop:product_of_actegories}) is a monoidal $\cM \times \cN$-actegory.
\end{proposition}

When it comes to the reparameterisation, we have to ensure that the projection from the category of elements of the weight (\cref{eq:actegory_weight_projection}) is a braided monoidal functor (\cref{def:reparam_monoidality}), which is the case when the underlying weight is braided monoidal (\cref{prop:monoidal_cat_hom_functor}).
If the weight is the hom-functor of a braided monoidal category, then this always holds (\cref{corr:hom_functor_projection_braided}).

\begin{theorem}[{Monoidal optics (compare \cite[Thm. 2.0.12.]{riley_categories_2018})}]
  \label{thm:optic_monoidal}
  Consider the category of weighted optics $\Optic^W_{\diset{\actfw}{\actbw}}$.
  This category is monoidal if the underlying actegories are, and the weight $W$ is a braided monoidal functor (\cref{def:braided_monoidal_functor}).
\end{theorem}

\begin{proof}
Direct consequence of \cref{prop:copara_monoidal}, and \cref{prop:product_monoidal_act}.
\end{proof}

Let's see what this means more concretely.
Consider the category of weighted optics $\Optic^W_{\diset{\actfw}{\actbw}}$ where the monoidal product of the forward actegory is $(\cC, \boxtimes, J)$ and that of the backward actegory is $(\cD, \boxtimes', J')$.
Then the monoidal product of two weighted optics
\[
  \diset{A}{A'} \xrightarrow{(\diset{M}{M'}, s, \diset{f}{f'})} \diset{B}{B'} \quad \text{and} \quad \diset{C}{C'} \xrightarrow{(\diset{N}{N'}, s, \diset{g}{g'})} \diset{D}{D'}
\]
is a weighted optic of type $
  \diset{A \boxtimes C}{A' \boxtimes' C'} \xrightarrow{} \diset{B \boxtimes D}{B' \boxtimes' D'} $
whose residual is the pair $\diset{M \otimes N}{M' \otimes' N'}$, the map connecting them is $\phi^W_{M, N}(s, t)$, and the implementation maps are $\diset{f \compCopara{\boxtimes} g}{g' \compPara{\boxtimes} f'}$.

\begin{example}[Self actions]
  \label{ex:self_actions_monoidal}
  Given two braided monoidal categories $(\cM, \otimes, I)$ and $(\cM', \otimes', I')$ and a braided monoidal functor $W : \cM^{\op} \times \cM' \to \Set$ the category $\Optic^W_{\diset{\otimes}{\otimes'}}$ is monoidal.
  If $\cM = \cM'$ then this reduces to the monoidal category of non-mixed optics $\Optic(\cM)$.
\end{example}
We can go further, and say when the category of optics is braided, or symmetric monoidal.
Following arguments in \cref{subsec:braided_symmetric_monoidal}, both of these require the acting category $\cM$ to be symmetric.
\begin{proposition}
  \label{prop:optics_braid_symm}
  If the monoidal $\cM$-actegories $(\cC, \actfw)$ and $(\cD, \actbw)$ are braided (resp. symmetric), then $\Optic_{\diset{\actfw}{\actbw}}$ is braided (resp. symmetric) monoidal.
\end{proposition}
This is only a sufficient condition, because optics involve a notion of a weight coupling the actegories together, allowing them to interact in non-trivial ways.
We will explore this in \cref{ch:backpropagation}.
  \begin{mybox}[label=box:scs_optic]{Cerulean}{States, costates, and scalars in $\Optic[false](\cC)$}
    We unpacked states, scalars, and costates in a general monoidal category $\cC$ (\cref{box:scs_monoidal}) and in $\Para(\cC)$ (\cref{box:scs_para}).
    Now we unpack them for $\Optic(\cC)$.

    \begin{minipage}[t]{0.29\textwidth}
      \begin{tightcenter}
        States
      \end{tightcenter}
      \scaletikzfig[0.6]{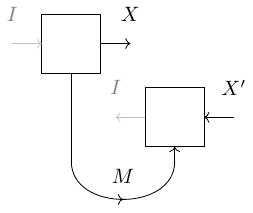}
      \begin{align*}
        &\Optic(\cC)\OpticHom{I}{I}{X}{X'}\\
        \cong & \int^M \cC(I, M \otimes X) \times \cC(M \otimes X', I)
      \end{align*}
    \end{minipage}
    \hfill
    \begin{minipage}[t]{0.28\textwidth}
      \begin{tightcenter}
        Scalars
      \end{tightcenter}
      \scaletikzfig[0.6]{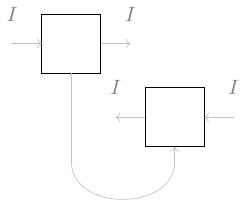}
      \begin{align*}
        &\Optic(\cC)\OpticHom{I}{I}{I}{I}\\
        \cong &\cC(I, I)
      \end{align*}
    \end{minipage}
    \hfill
    \begin{minipage}[t]{0.28\textwidth}
      \begin{tightcenter}
        Costates
      \end{tightcenter}
      \scaletikzfig[0.6]{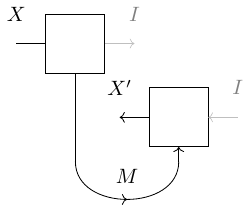}
      \begin{align*}
        & \enskip \Optic(\cC)\OpticHom{X}{X'}{I}{I}\\
        \cong & \enskip \cC(X, X')
      \end{align*}
    \end{minipage}
    
    Scalars and costates (but not states) have a particularly interesting form.
    Costates become morphisms in the base category and scalars become endomorphisms of the monoidal unit.
    Additionally, if $\cC$ is cartesian scalars become trivial, as we will see in \cref{box:scs_lens}.
  \end{mybox}

\section{Lenses, prisms, and closed optics}
\label{sec:lenses_prisms_closed_optics}

We now proceed to study weighted optics defined on actegories with particular structure.

\subsection{Lenses; from cartesian actegories}
\label{sec:optics_for_cartesian_act_fw}

We turn our focus to an important class of optics.
These are optics with a forward actegory that's of a particular kind to be defined in this section: a cartesian actegory, and their particular reduced form --- \emph{lenses} --- is originally what sparked a lot of research on bidirectionality.
As we will see in \cref{subsec:operational}, lenses give a denotationally equivalent, but an operationally distinct perspective on such processes.
They are immensely useful in practice and, in addition to modelling deep learning, they have been used to model bayesian learning (\cite{braithwaite_compositional_2023, smithe_mathematical_2022}), database theory (\cite{spivak_functorial_2023}), dynamical systems (\cite{myers_categorical_2022}), game theory \cite{ghani_compositional_2018, capucci_diegetic_2023}, data accessors (\cite{pickering_profunctor_2017}), trading protocols (\cite{genovese_escrows_2021}), server operations (\cite{videla_lenses_2022}) and more (\cite{gavranovic_theory_2023}).
We start by giving their formal definition in the most abstract form, based on the concept of a \emph{cartesian actegory}.\footnote{Idea for this concept arose in a conversation with Matteo Capucci.}

\begin{definition}[Cartesian actegory]
  \label{def:cartesian_actegory}
  A monoidal $\cM$-actegory $\cC$ is a \newdef{cartesian actegory} if the monoidal product of the base $\cC$ is given by a cartesian one.
  Dually, when the monoidal product is cocartesian, the resulting actegory is too.
\end{definition}

\begin{remark}
Note that there is no requirement of $\cM$ being cartesian, nor any requirement about the interaction of the cartesian structure of $\cC$ and the monoidal one of $(\cM, \otimes, I)$ and $\act$, other than the one already imposed by the definition of a monoidal actegory.
\end{remark}

If the actegory is cartesian, we can formulate the actegorical analogue of the classical isomorphism in a category with products.
It is a property reminiscent of the \emph{universal product} in a cartesian category, but one which only admits the isomorphism representation (\cref{eq:product}), and not one defined by the universal morphism (\cref{eq:product_univ_property}).

\begin{proposition}
  \label{prop:act_univ}
  Let $(\cC, \actfw)$ be a cartesian $\cM$-actegory. Then we have the natural isomorphism:\footnote{While the codomain of the right element could be further reduced to $B$, the same can't be said for the left one --- there is no way to reduce $\cC(X, M \act 1)$ to $\cC(X, M)$ in general since $M$ isn't an object of $\cC$.}
  \[
  \cC(X, M \act B) \cong \cC(X, M \act 1) \times \cC(X, I \act B)
  \]
\end{proposition}

\begin{proof}
  Exhibited below.
  \begin{align*}
          &\cC(X, M \act B) &\\
    \cong &\cC(X, (M \otimes I) \act (1 \times B)) & \text{(Interchanger)}\\
    \cong &\cC(X, (M \act 1) \times (I \act B)) & \text{(Universal property of the product)}\\
    \cong &\cC(X, M \act 1) \times \cC(X, I \act B) \tag*{\qedhere}
  \end{align*}
\end{proof}

\begin{example}
  \label{ex:cartesian_cat_act}
  Every cartesian category $\cC$ canonically becomes a cartesian $\cC$-actegory by virtue of being a monoidal $\cC$-actegory (\cref{ex:self_action_braided_monoidal}).
\end{example}

\begin{remark}
  \label{rem:cart_act_dont_trivialise}
  Given a self-action of a cartesian category we have established that connected components of the $\Copara$ construction on it trivialise (\cref{corr:cart_trivialises}).
  Cartesian actegories allow us to circumvent this problem, as the necessary condition for trivialisation --- the monoidal unit of $\cM$ being terminal --- is not something that holds for a general cartesian actegory.
\end{remark}

The example of the above follows.
In it, despite the lack of cartesian structure of $\cB$, the braided monoidal functor produces a cartesian actegory via reparameteristaion.

\begin{example}
  \label{ex:monoidal_subcategory_of_cartesian_one}
  Let $(\cC, \times, 1)$ be a cartesian category, $\cB$ a braided monoidal category, and let $E : \cB \to \cC$ be a strong braided monoidal functor (\cref{def:braided_monoidal_functor}).
  Then $\times^E : \cC \times \cB \to \cC$ is a cartesian $\cB$-actegory.
\end{example}

We are now ready to study cartesian optics.

\begin{definition}[Cartesian optic]
  \label{def:cartesian_optic}
  Consider optics $\Optic^W_{\diset{\actfw}{\actbw}}$ where the forward actegory is a cartesian actegory.
  We call morphisms in this category \newdef{cartesian optics}\footnote{Not to be confused with the property of optics being a cartesian category, which holds in an even more specialised case.}.
\end{definition}

As mentioned, these have been in the literature defined under the name of ``lens'', though the actual definition often varies.
There seems to be a general ``folkloric'' definition of what a lens is, but many different mathematical formulations, and no clear hierarchy between them.
In this thesis, we claim that a lens $\diset{A}{A'} \to \diset{B}{B'}$ is a kind of a bidirectional process that allows us to ``extract'' a morphism of type $A \to B$, and deterministically separate it from the backward one.
To the best of our knowledge the above definition of a cartesian optic captures this phenomenon in the most general setting.
Formally, this means the following.

\begin{proposition}[{Compare \cite[Proposition 2.0.4.]{riley_categories_2018}}]
  \label{prop:optic_cartesian}
  In any cartesian weighted optic we can factor out the forward pass:
  \[
    \Optic^W_{\diset{\actfw}{\actbw}}\OpticHom{A}{A'}{B}{B'} \cong \cC(A, B) \times \Optic^W_{\diset{\actfw}{\actbw}}\OpticHom{A}{A'}{1}{B'}
  \]
\end{proposition}
\begin{proof}
  \label{proof:optic_adjoint_lens}
  \begin{alignat*}{2}
    & && \Optic^W_{\diset{\actfw}{\actbw}}\OpticHom{A}{A'}{B}{B'}\\
    &\text{(Def.)}           &=& \int^{M, M'} \cC(A, M \actfw B) \times W(M, M') \times \cD(M' \actbw B', A')\\
    & \text{(\cref{prop:act_univ})}  &\cong& \int^{M} \cC(A, M \actfw 1) \times \cC(A, I \actfw B) \times W(M, M') \times \cD(M' \actbw B', A')\\
    &   &\cong& \cC(A, B) \times \int^{M, M'} \cC(A, M \act 1) \times W(M, M') \times \cD(M' \actbw B', A')\\
    &\text{(Def.)}   &=& \cC(A, B) \times \Optic^W_{\diset{\actfw}{\actbw}}\OpticHom{A}{A'}{1}{B'} \tag*{\qedhere}
  \end{alignat*}
\end{proof}

This is how lenses historically originated, by being described as gadgets with a separate forward and backward passes.
This kind of formalism that separates out the backward and forward passes necessitates a different composition rule, which as we'll see is denotationally equivalent, but operationally quite different than the one given directly by optics (\cref{def:weighted_optic}).

All of the definitions of lenses in the literature are special cases of \emph{mixed optics} (\cref{box:mixed_optics_are_weighted_optics}), as weighted optics are only introduced in this thesis.
In \cite{clarke_profunctor_2022} lenses are defined as mixed optics whose forward actegory is given by the self-action of a cartesian category $\cC$ (thus necessitating that the acting category of both the forward and the backward pass is $\cC$).
In such a setting \cref{prop:optic_cartesian} further reduces to
\[
  \Optic_{\diset{\times}{\actbw}}\OpticHom{A}{A'}{B}{B'} \cong \cC(A, B) \times  \cD(A \actbw B', A')
\]

In this setting, we take the liberty of providing a definition of composition and identity.
\begin{definition}[Lens composition and identity]
  \label{def:lens_composition_identity}
  Given
  \[
    \diset{A}{A'} \xrightarrow{\diset{f}{f'}} \diset{B}{B'} \quad \text{and} \quad \diset{B}{B'} \xrightarrow{\diset{g}{g'}} \diset{C}{C'}
  \]
  where $f : \cC(A, B)$ and $f' : \cD(A \actbw B', A')$ (analogously for $\diset{g}{g'}$) their composite is defined as a lens whose forward part is $f \comp g$ and the backward part is $ A \actbw C' \xrightarrow{{(g' \compPara{\comp} f')}^{\graph(f)}} A'$ (where $\graph(f)$ is defined in \cref{def:graph}).
  To specify the backward pass here we have used parametric compositon and reparameterisation thereof (\cref{def:para}).
  Identity $\diset{A}{A'} \to \diset{A}{A'}$ is defined as $\id_A : A \to A$ on the forward pass and reparameterisation of $\eta_A^{\actbw} : 1 \actbw A' \to A'$ with $\terminal_A : A \to 1$ on the backward one.
  It is routine to check that composition and identity are strictly preserved going between the abstract and the concrete representation.
\end{definition}

Many others (\cite{riley_categories_2018}) specialise lenses further.
They define them in a non-mixed setting, where the forward and the backward category are the same, and are given by the cartesian product of $\cC$.
In this setting we can further reduce the above to

\begin{equation}
  \label{eq:lens_representation}
  \Optic(\cC)\OpticHom{A}{A'}{B}{B'} \cong \cC(A, B) \times  \cC(A \times B', A')
\end{equation}

which yields the more familiar representation of a lens in the literature.
We will often denote this category as in \cref{eq:lens}.
Its string diagram representation is in \cref{fig:lens}.

\begin{equation}
  \label{eq:lens}
  \Lens(\cC) \coloneqq \Optic(\cC)
\end{equation}

\begin{figure}[H]
  \scaletikzfig[0.8]{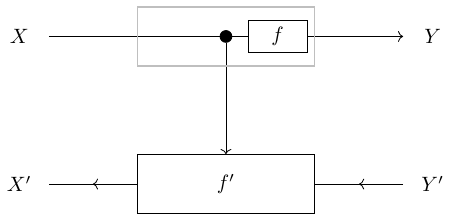}
  \caption{String diagram representation of a lens. The area in the gray box $(A, \graph(f)$ is the forward part: a morphism in $\Para(\cC)(X, Y)$. Notably, the input $X$ is explicitly copied.}
  \label{fig:lens}
\end{figure}

These sometimes go under the name of \emph{bimorphic} (\cite{hedges_limits_2019}) lenses, as they allow different choices of objects on the forward and backward pass.
Some authors specialise this construction even more and require that objects are those given by the diagonal $\diset{A}{A}$ in which case they are referred to as \emph{monomorphic}, or \emph{simple} lenses.

\begin{mybox}[label=box:scs_lens]{Cerulean}{States, costates, and scalars in $\Lens[false](\cC)$}
  The descriptions of states, costates and scalars from \cref{box:scs_optic,box:scs_cartesian} can be specialised to lenses.
  
  \begin{minipage}[t]{0.32\textwidth}
    \begin{tightcenter}
      States
    \end{tightcenter}
    \scaletikzfig[0.7]{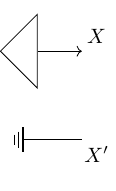}
    \begin{align*}
      & \Lens(\cC)\OpticHom{1}{1}{X}{X'}\\
      \cong & \cC(1, X)
    \end{align*}
  \end{minipage}
  \hfill
  \begin{minipage}[t]{0.32\textwidth}
    \begin{tightcenter}
      Scalars
    \end{tightcenter}
    \scaletikzfig[0.7]{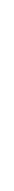}
    \begin{align*}
      &\Lens(\cC)\OpticHom{1}{1}{1}{1}\\
      \cong & 1
    \end{align*}
  \end{minipage}
  \hfill
  \begin{minipage}[t]{0.32\textwidth}
    \begin{tightcenter}
      Costates
    \end{tightcenter}
    \scaletikzfig[0.7]{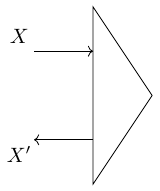}
    \begin{align*}
      &\Lens(\cC)\OpticHom{X}{X'}{1}{1}\\
      \cong &\cC(X, X')
    \end{align*}
  \end{minipage}
\end{mybox}

\subsection{Prisms; from cocartesian actegories}
\label{subsec:prisms}

Dual to the setting of cartesian actegory on the forward pass is the setting of a cocartesian actegory on the backward pass. Analously to the setting above, this is in the literature known under the name of \emph{prism} (\cite[Def.\ 3.16]{clarke_profunctor_2022}).
These yield a generalisation of the definition of optics for coproducts (\cref{ex:optics_for_coproducts}) where we allow the forward category to be potentially different than the backward one.

\begin{definition}[Cocartesian optics]
  \label{def:prism}
  Consider the category of optics $\Optic^W_{\diset{\actfw}{\actbw}}$ for a backwards cocartesian category.
  We call this category the category of the category of \newdef{cocartesian optics}, and often refer to them as prisms.
\end{definition}

In such a case, we can exhibit a dual reduction which extracts out the backward pass:

\begin{proposition}[{Compare \cite[Sec. 4.2]{riley_categories_2018}}]
  \label{prop:optic_cocartesian}
  In any cocartesian weighted optic we can factor out the backward pass:
  \[
    \Optic^W_{\diset{\actfw}{\actbw}}\OpticHom{A}{A'}{B}{B'} \cong  \Optic^W_{\diset{\actfw}{\actbw}}\OpticHom{A}{A'}{B}{0} \times \cD(B', A')
  \]
\end{proposition}

The proof follows analogously to the one of \cref{prop:optic_cartesian}.

While we do not directly study prisms in this thesis, we mention them for completeness, and to bring to attention an important disambiguation thereof.

\begin{remark}[Two kinds of prisms]
  \label{rem:two_kinds_of_prisms}
In the literature there are two non-isomorphic constructions referred to as prisms: one arising out of coproducts or generally cocartesian actegories (\cite{riley_categories_2018, clarke_profunctor_2022}) and another one arising as lenses in the opposite of a cartesian category (\cite[Ex. 2.4]{spivak_generalized_2022}).
Only the former has 1) a clear operational interpretation and 2) a form that is compositional i.e.\ allowing us to compose prisms with lenses, obtaining affine traversals (\cite[Sec. 3.2.1]{clarke_profunctor_2022}).
For more detail see \cite{gavranovic_two_2023}.
\end{remark}

\subsection{A lens and a prism at the same time}

Cartesian actegories allow us to make a particularly interesting reduction for $\Copara$.

\begin{theorem}
  \label{thm:copara_cart_reduction}
  Let $(\cC, \actfw)$ be a cartesian $\cM$-actegory.
  Then we have the following factorisation:
  \[
    \Copara_{\act}(\cC)(X, Y) \cong \cC(X, Y) \times \Copara_{\act}(\cC)(X, 1)
  \]
\end{theorem}

\begin{proof}
  \begin{alignat*}{2}
    & && \Copara_{\act}(\cC)(X, Y)\\
    &\text{(Def.)} &=& \sum_{M : \cM}\cC(X, M \actfw Y)\\
    &\text{(\cref{prop:act_univ})}     &\cong& \sum_{M : \cM}\cC(X, M \actfw 1) \times \cC(X, I \actfw Y)\\
    &  &\cong& \cC(X, Y) \times \sum_{M : \cM}\cC(X, M \actfw 1)\\
    &\text{(Def.)}  &=& \cC(X, Y) \times \Copara_{\act}(\cC)(X, 1) \tag*{\qedhere}
  \end{alignat*}
\end{proof}

Intuitivelly, this tells us that any coparametric map $f : X \to M \act Y$ of a cartesian actegory can be decomposed into the horizontal part with a trivial vertical component (a map $X \to Y$) and a vertical part with a trivial horizontal component (a map $X \to M \act 1$).
In a way, we have seen this happening already in \cref{prop:optic_cartesian}.
We can ponder what this means for optics.
Recall that they are defined as an action on the product of $\cC$ and $\cD^{\op}$

 \begin{proposition}[Product of cartesian actegories is a cartesian actegory]
   Consider a cartesian $\cM$-actegory $(\cC, \act)$ and a cartesian $\cN$-actegory $(\cD, \actbw)$.
   Then their product (as defined in Prop. \cref{prop:product_of_actegories}) is a cartesian $\cM \times \cN$-actegory.
 \end{proposition}

 As our backward actegory is always dualised --- this means that for optics we need a \emph{cocartesian} $\cM'$-actegory $(\cD, \actbw)$.
 In such a case, following \cref{thm:copara_cart_reduction} we get the following isomorphism:
\begin{alignat*}{2}
  &\Optic_{\diset{\actfw}{\actbw}}\OpticHom{A}{A'}{B}{B'} &&\cong \cC(A, B) \times \cD(B', A') \times \Optic_{\diset{\actfw}{\actbw}}\OpticHom{A}{A'}{1}{0}
\end{alignat*}

This is a reduction that is both lens-like, and prism-like at the same time!
It can be seen as a generalisation of optics constructed on a biproduct category.

\subsection{Closed optics; from closed actegories}
\label{sec:optics_backward_closed_actegory}

So far we've described actegories which parameterise systems \emph{internally}.
This means that any parametric map $f : P \act A \to B$ is a morphism of the base category.
For instance, if our base category is $\Smooth$, then $f$ is a smooth map.
However, sometimes we're interested in a map $f$ which is smooth only in one variable, and, say, linear in the other.
In this case we're interested in \emph{external} parameterisation.
This is the main idea behind \emph{closed} actegories, which generalise the setting of a monoidal closed category.
Recall that monoidal closure (\cref{def:monoidal_closed}) means that for every $D : \cD$ the functor $- \otimes D$ has a right adjoint denoted by $\internalHom{\cD}{D}{-}$ which forms the internal hom functor $\internalHom{\cD}{-}{-} : \cD^{\op} \times \cD \to \cD$ such that there is an isomorphism
\[
  \cD(X \otimes Y, Z) \cong \cD(X, \internalHom{\cD}{Y}{Z})
\]
natural in all bound variables.
A closed actegory generalises this setting into that with two categories, while preserving the essence of the idea.

\begin{definition}[{Closed actegory (compare \cite[Example 3.2.6.]{capucci_actegories_2023})}]
  Let $(\cD, \actbw)$ be a $\cM$-actegory.
  If for every $D : \cD$ the functor $- \actbw D : \cM \to \cD$ has a right adjoint, we call $(\cD, \actbw)$ a \newdef{closed} actegory.\footnote{If $\cM$ is also monoidal closed, then this is equivalent to a $\cM$-enriched category $\cD$ with copowers. (\cite[Example 3.2.6]{capucci_actegories_2023})}
\end{definition}

We often denote the right adjoint as $\internalHom{\cD}{D}{-}$.
In this setting we can generalise the standard tensor-hom isomorphism to
\begin{equation}
  \label{eq:closed_actegory}
  \cD(M \actbw X, Y) \cong \cM(M, \internalHom{\cD}{X}{Y})
\end{equation}

leading us to the following proposition.

\begin{proposition}[{External parameterisation (compare \cite[Def.\ 3.2.8.]{smithe_mathematical_2022})}]
  When the actegory is closed, we have the following isomorphism $\Para_{\actbw}(\cD)(X, Y) \cong \cM/\internalHom{\cD}{X}{Y}$.\footnote{Note that one can in principle define external parameterisation when $\cD$ is merely enriched in $\cM$, without any actions. Giving a general formulation of parameterisation that encompasses both of these settings is something we hope to explore in future work based on locally graded categories (see end of \cref{sec:act_para_literature} and \cref{sec:locally_graded_categories}).}
\end{proposition}

\begin{example}[Monoidal closed category $\cC$]
  \label{ex:monoidal_closed_category}
  Any monoidal closed category $\cC$ is a closed actegory whose action is given by the self-action of the monoidal product.
\end{example}

An especially relevant example is the category $\FVectR$ with monoidal structure $(\otimes, \R)$, which can be used to model linearity of the backward pass in backpropagation.
We can state what this means abstractly for weighted optics whose backward actegory is closed, and weight is corepresentable with a lax monoidal functor.

\begin{definition}[Closed optics]
  \label{def:closed_optic}
  Consider the category $\Optic^W_{\diset{\actfw}\actbw}$ whose backward actegory is closed and weight is corepresented by a lax monoidal functor $j : \cM' \to \cM$.
  We call this the category of \newdef{closed}\footnote{Sometimes also called ``linear'' optics.} \newdef{optics}.
\end{definition}

\begin{remark}
  Unlike in \cref{prop:weighted_to_mixed_representable} we do not require $j$ to be strong monoidal because it does not need to be a part of a reparameterisation of actegories.
\end{remark}

Closed optics are important is because they allow us to make the following reduction to \emph{closed lenses}, also often called \emph{linear lenses} in the literature (\cite[Sec. 4.8]{riley_categories_2018}, \cite[Def.\ 3.14]{clarke_profunctor_2022}).
Despite the string ``lens'' in their name, they do not require the forward actegory to be cartesian, though in (\cref{corr:optics_forward_cart_backward_closed_concrete}) we will study what happens in that case.

\begin{proposition}
  \label{thm:optic_backward_act_closed}
  Consider the category of closed optics $\Optic^W_{\diset{\actfw}{\actbw}}$.
  Then we can exhibit the following isomorphism
  \[
    \Optic^W_{\diset{\actfw}{\actbw}}\OpticHom{A}{A'}{B}{B'} \cong \cC(A, j(\cD(B', A')) \act B)
  \]
\end{proposition}

\begin{proof}
  \begin{alignat*}{2}
    & && \Optic^W_{\diset{\actfw}{\actbw}}\OpticHom{A}{A'}{B}{B'}\\
    &\text{(Def.)} &&=\int^{M, M'} \cC(A, M \actfw B) \times W(M, M') \times \cD(M' \actfw B', A')\\
    &\text{(Tensor-hom (\cref{eq:closed_actegory}))} &&\cong \int^{M, M'} \cC(A, M \actfw B) \times W(M, M') \times \cM(M', \cD(B', A'))\\
    &\text{(coYoneda)} &&\cong \int^{M} \cC(A, M \actfw B) \times W(M, \cD(B', A')) \\
    &\text{(\cref{prop:weighted_to_mixed_representable})} &&\cong \int^{M} \cC(A, M \actfw B) \times \cM(M, j(\cD(B', A')))\\
    &\text{(coYoneda)} &&\cong \cC(A, j(\cD(B', A')) \act B) \tag*{\qedhere}
  \end{alignat*}
\end{proof}

Analogously to cartesian optics, we can deal with concrete closed lenses directly, and provide a composition rule and identity in terms of their explicit representation.

\begin{definition}[Closed lens composition and identity]
  \label{def:linear_lens_composition_identity}
Given
\[
  \diset{A}{A'} \xrightarrow{f} \diset{B}{B'} \quad \text{and} \quad \diset{B}{B'} \xrightarrow{g} \diset{C}{C'}
\]
where
\[
  f : A \to j(\cD(B', A')) \actfw B \quad \text{and} \quad g : B \to j(\cD(C', B')) \actfw C
\]
the composite closed lens is a map of type $A \to j(\cD(C', A')) \actfw C$ that can compactly be written as $(f \compCopara{\comp} g)^{\mu^j \comp j(\comp)}$.
Intuitively, here we first composed $f$ and $g$ as coparametric maps, yielding
\[
  f \compCopara{\comp} g : A \to (j(\cD(B', A')) \otimes j(\cD(C', B'))) \actfw C
\]
and then reparameterised the result along
\[
  j(\cD(B', A')) \otimes j(\cD(C', B')) \xrightarrow{\mu^j} j(\cD(B', A') \otimes' \cD(C', B')) \xrightarrow{j(\comp)} j(\cD(C', A'))
\]

where $\mu^j$ is the laxator of $j$ and $j(\comp)$ composition inside the category $\cD$ under $j$.

The identity closed lens $\diset{A}{A'} \to \diset{A}{A'}$ is given by the unitor $\eta_A^\actfw : A \to I \actfw A$ reparameterised by $I \xrightarrow{\epsilon^j} j(I') \xrightarrow{j(\id_{A'})} j(\cD(A', A'))$.
It is routine to show that composition and identity is strictly preserved going between the abstract and the concrete closed optic representation.
\end{definition}

\begin{example}[Markov kernels and expectation]
  \label{ex:mark_ker_expectation_closed}
  The category of mixed optics of Markov kernels and expectations (\cref{ex:mark_ker_expectation}) is a category of closed optics, as $\Conv$ is monoidal closed (\cite[Ex. 4]{hedges_value_2023}).
\end{example}

\begin{example}[Smooth maps and linear maps]
  \label{ex:smooth_linear_weighted_optic}
  The category of weighted optics of smooth and linear maps (\cref{ex:smooth_and_linear_weighted_optics}) is a category of closed optics, as $\FVectR$ is monoidal closed.
\end{example}

Specifically, the composition of linear lenses $f : A \to B \times (B' \multimap A')$ and $g : B \to C \times (C' \multimap B')$ in $\Smooth$ is the following composite in $\cC$:
\begin{equation}
  \label{eq:linear_lens_smooth}
  A \xrightarrow{f} B \times (B' \multimap A') \xrightarrow{g \times (B' \multimap A')} C \times (C' \multimap B') \times (B' \multimap A') \xrightarrow{C \times \multimap_{\comp}} C \times (C' \multimap A')
\end{equation}

And lastly, if the forward actegory is additionally cartesian then, like before, we can separate out the forward pass from the backward one.

\begin{corollary}[{Optics with cartesian on the forward and closed actegories on the backward pass}]
  \label{corr:optics_forward_cart_backward_closed_concrete}
  Let the category $\Optic^W_{\diset{\actfw}{\actbw}}$ be the category of optics where the actegory $\actfw$ on the forward pass is cartesian, and the actegory $\actbw$ on the backward pass is a closed.
  Then we have
  \[
    \Optic^W_{\diset{\actfw}{\actbw}}\OpticHom{A}{A'}{B}{B'} \cong \cC(A, B) \times \cC(A, j(\cD(B', A')) \actfw 1)
  \]
  If the forward actegory is given by the self-action of a cartesian actegory, then we can reduce the above further to
  \[
    \Optic^W_{\diset{\times}{\actbw}}\OpticHom{A}{A'}{B}{B'} \cong \cC(A, B) \times \cC(A, j(\cD(B', A')))
  \]
\end{corollary}

Out of the two examples above, only smooth and linear maps are cartesian, as $\Mark$ does not have products.

\section{Bidirectionality in the literature}
\label{sec:bidirectionality_literature}

A lot of research on optics took place because they arose as a generalisation of the category of lenses in a cartesian category, and lenses are a widely studied construction (\cref{sec:bidirectionality_literature}).
But optics aren't the \emph{only} generalisation of lenses.
The category of \emph{dependent lenses} (\cite{braithwaite_fibre_2021}) is another generalisation that was rediscovered numerous times under the names such as ``the category of polynomial functors'' (\cite{gambino_polynomial_2013, spivak_generalized_2022, spivak_poly_2020}) (often denoted\footnote{Not to be confused with the category $\Poly_R$ defined in \cite{cockett_reverse_2020,blute_cartesian_2009} and this thesis.} by $\PolySpivak$), or ``the category of containers'' (\cite{abbott_categories_2003, goos_derivatives_2003, hutchison_higher-order_2010}) (often denoted by $\Cont$).\footnote{Despite being the same category, the first category is named after its morphisms, and the second one after its objects.}

What is the category of dependent lenses?
It is one where the type of the backward part --- both on object and morphism level --- depends on the value of the forward one.
Instead of having independent choices of $A$ and $A'$ in $\cC$ forming $\diset{A}{A'}$, now the possible choices for the backward object depend on the forward one.
In $\Set$ a possible notation for these kinds of objects would be $\diset{a : A}{A'(a)}$ where $A'$ is now a function of type $A \to \Set$.
Similar idea applies to morphisms.
The forward part $f : A \to B$ stays the same, and the backward part $f'$ is now of type $\sum_{a : A}B'(f(a)) \to A'(a)$.
It differs from the non-dependent backward type $A \times B' \to A'$ in two ways: 1) the pair type $A \times B'$ is replaced with the \emph{dependent pair} type $\sum_{a : A}B'(f(a))$, and 2) the function type $A \times B' \to A'$ is replaced with the \emph{dependent function type}.
The definition of dependent lenses works in any base category $\cC$ with pullbacks, where it is denoted by $\DLens(\cC)$ (analogously, $\PolySpivak(\cC)$ or $\Cont(\cC)$) and defined as the Grothendieck construction of the pointwise opposite of the contravariant slice functor $\cC / - : \cC^{\op} \to \Cat$ (\cref{ex:slice_charts_lenses}).

When the base category is $\Set$, this category is often referred to as simply $\DLens$, $\Poly$, and $\Cont$, respectively.
It is an important construction that has found uses in type theory (\cite{von_glehn_polynomials_2015,abbott_categories_2003,goos_derivatives_2003}), mathematical theory of interaction (\cite{niu_polynomial_2023}), dynamical systems (\cite{myers_categorical_2022}), and more.
It is rich in structure, admitting all small limits and colimits, exponentials, initial algebras, final coalgebras, and more than a dozen interesting monoidal products (\cite{spivak_reference_2023,niu_polynomial_2023}).

\begin{mybox}[label=box:why_not_poly]{RedOrange}{Why not $\PolySpivak[false]$?}
  As $\PolySpivak$ is a versatile construction rich in structure, one might wonder why we are not using it as a foundation of bidirectionality in this thesis.
  There are a few reasons.
  \begin{itemize}
    \item \textbf{Baked-in determinism.} As formation of dependent lenses over $\cC$ requires $\cC$ to have products, this necessitates that the interaction between the forward and the backward pass is \emph{deterministic}.
    That is, to make composition in $\DLens(\cC)$ is associative we need to assume that copying the input $a: A$ and applying $f : A \to B$ to each instance is the same as applying $f$ and copying $A$ (see \cref{fig:deterministic_morphism} or \cite{gavranovic_space-time_2022}).
    This is not an assumption that holds in bayesian settings (\cite{braithwaite_compositional_2023,bolt_bayesian_2023,smithe_mathematical_2022}) or general effectful settings.
    \item \textbf{Blindness to performance concerns.} We believe studying how systems behave operationally can give us new insights about their denotational models.
    This is because observing them only from the outside squashes out a lot of information which is useful in building models of their internals.
    And the category $\PolySpivak$ is a denotational, or extensional model of interaction, describing interaction as observed from the \emph{outside} of these systems.
    It is not suited for operational, or intensional software oriented approaches where we are not merely observing these systems from the outside, but building them with their internal setups in mind.
    This is especially important in deep learning, where different implementations of dependent lenses yield radically different performance profiles.
    See "Memoisation" in \cref{ch:backpropagation}, \cref{subsec:operational}, and \cite{gavranovic_space-time_2022}.
    \item \textbf{Inadequate enforcement of typing discipline.} The forward pass, the residual, and the backward pass of bidirectional systems often have different semantics, and live in different categories.
    Such typing discipline is not possible to enforce in $\PolySpivak$: both the forward and the backward pass are drawn from the same base category $\cC$, while the notion of residual is not reified at all.

    \item \textbf{Additivity fundamentally changes the category.} When it comes to deep learning specifically, the assumption of additivity of the second component of the backward pass yields the category $\PolySpivakA$, a fundamentally different category than $\PolySpivak$.
    For instance, the monoidal product on $\PolySpivak$ given by products both on the forward and the backward types\footnote{In literature called the \emph{Dirichlet}, or the \emph{Hancock} tensor product.} becomes a categorical product in on $\PolySpivakA$.
    And the embedding $\PolySpivakA \to \PolySpivak$ with these products is not a monoidal functor in any sense: it's not strict, strong, lax nor oplax.
    This in turn impacts the study of closure with respect to this product, for instance, which give us higher-order functions, an important feature of modern automatic differentiation frameworks (\cref{ch:backpropagation}).
  \end{itemize}

\end{mybox}

All of the above questions are something we deal with in the next subsection on \emph{dependent optics}, where we tackle the question: how to generalise the optics to a dependently typed setting so we get the best of both worlds --- of optics, but also of $\PolySpivak$?

Before this, we mention that lenses and optics are not a new construction: bidirectional tranformations in general have a long history.
They include many cousins of optics in lenses that we do not tackle in this thesis: Delta lenses \cite{clarke_delta_2021, clarke_internal_2020, diskin_multiple_2019, diskin_general_2020} and Dialectica categories (\cite{gray_dialectica_1989}), for example.
For the history of lenses we point to \cite{hedges_lenses_2018}, and for a general literature list related to this topic to \cite{gavranovic_theory_2023}.

Most closely related to weighted optics are the definitions of non-mixed optics (\cite{riley_categories_2018}), and those of mixed optics (\cite{clarke_profunctor_2022}).
We unpack the definitions of their hom-sets below.

\begin{definition}
  \label{def:homset_optic}
  Let $\cC$ be a monoidal category.
  The set of \newdef{(non-mixed) optics} $\diset{A}{A'} \to \diset{B}{B'}$ is defined as the following coend
  \[
    \Optic(\cC)\OpticHom{A}{A'}{B}{B'} \coloneqq \int^{M : \cC} \cC(A, M \otimes B)
    \times \cC(M \otimes B', A')
  \]

  Its elements are equivalence classes of triples $(M, f, f')$, where $M
  : \cC$, $f : A \to M \otimes B$ and $f' : M \otimes B' \to A'$. They're
  quotiented out by the equivalence relation where $(M, f, f') \sim (N, g, g')$
  if there is a residual morphism $r : M \to N$ in $\cC$ such that the following diagrams commute:
\begin{equation}
  \label{eq:optic_equiv_diagrams}
  \begin{tikzcddiag}[ampersand replacement=\&, row sep=3ex, column sep=4ex]
      A \&\& {M \otimes B} \&\& {M \otimes B'} \&\& {A'} \\
      \\
      \&\& {N \otimes B} \&\& {N \otimes B'}
      \arrow["f", from=1-1, to=1-3]
      \arrow["{r \otimes B}", from=1-3, to=3-3]
      \arrow["g"', from=1-1, to=3-3]
      \arrow["{f'}", from=1-5, to=1-7]
      \arrow["{g'}"', from=3-5, to=1-7]
      \arrow["{r \otimes B'}"', from=1-5, to=3-5]
    \end{tikzcddiag}
\end{equation}
\end{definition}

Note the difference between \cref{eq:optic_equiv_diagrams} and \cref{eq:copara_optic_diagram}.
It is a good exercise to check that these are indeed different ways of stating the same thing.\footnote{Hint: there is a canonical way to factorise any morphism in $\tw(\cC)$.}
Mixed optics, on the other hand, conceptually separate the forward category, the acting category, and the backward category.

\begin{definition}
  \label{def:homset_mixed_optic}
  Let $(\cC, \actfw)$ and $(\cD, \actbw)$ be $\cM$-actegories.
  The set of \newdef{mixed optics} $\diset{A}{A'} \to \diset{B}{B'}$ is defined as the following coend, quotiented out in the same way as \cref{eq:optic_equiv_diagrams}.
  \[
    \Optic_{\diset{\actfw}{\actbw}}\OpticHom{A}{A'}{B}{B'} \coloneqq \int^{M : \cM} \cC(A, M \actfw B) \times \cD(M \actbw B', A')
  \]
\end{definition}

\subsection{Dependent optics}
  \label{subsec:dependent_optics}
    As dependent lenses are a distinct and important generalisation of lenses on a category with products to the one with \emph{pullbacks}\footnote{A category with all pullbacks doesn't necessarily have all products: it only does if it also has a terminal object. This fact seems to be the main source of tension of what's to come.} an important question surfaced:
    \begin{center}
      \emph{Can optics be made dependent in the same way?}
    \end{center}
    This is an infamous question that has attracted a lot of attention and spawned a myriad of proposed solutions (\cite{braithwaite_fibre_2021, vertechi_dependent_2023, winskel_making_2023, milewski_dependent_2021, capucci_seeing_2022}) in the span of approximately two years (which were in the middle of my PhD).
    Additionally, during those two years different generalisations of optics were defined in parallel (\cite{milewski_compound_2022, milewski_polylens_2021}) which lead to (at the time of the writing of this thesis) a situation with many competing generalisations of optics, and their unclear hierarchy.\footnote{The names of these generalisations appear to be obtainable by prefixing the noun ``optic'' any of the adjectives such as ``indexed'', ``fibre'', ``dependent'', or ``presheaf''.}
    The main reason for this diverse jungle of solutions is that there are many different things the phrase above ``\emph{in the same way}'' could mean, and that there is a lack of agreement by the applied category theory community of what the above phrase \emph{should} mean.

    There are at least four different characterisations of dependent optics.
    We first provide a short description, and then elaborate each of them.
    \begin{enumerate}
      \item \textbf{Actions.} The category $\Lens(\cC)$ is an example of an optic for the self-action of a cartesian category (\cref{eq:lens_representation}). But the category $\DLens(\cC)$ is not, because optics do not allow the type of the backward pass to be dependent on the values of the forward one. Is there a generalisation of optics such that instantiating it with appropriate actions yields the category of dependent lenses?
      \item \textbf{Pushout.} There exist canonical embeddings of non-dependent lenses into both optics and dependent lenses.
      This suggests a question: is there a generalisation of optics arising as the pushout of these embeddings (\cref{eq:dependent_optics_pushout})?
    \begin{equation}
      \label{eq:dependent_optics_pushout}
      \begin{tikzcddiag}[ampersand replacement=\&]
          {\Lens(\cC)} \&\& {\Optic(\cC)} \\
          \\
          {\DLens(\cC)} \&\& {?}
          \arrow[from=1-1, to=3-1]
          \arrow[from=1-1, to=1-3]
          \arrow[from=1-3, to=3-3]
          \arrow[from=3-1, to=3-3]
          \arrow["\lrcorner"{anchor=center, pos=0.125, rotate=180}, draw=none, from=3-3, to=1-1]
      \end{tikzcddiag}
  \end{equation}
      \item \textbf{Non-determinism.} The requirement that morphisms in our base are deterministic (\cref{fig:deterministic_morphism}) is a crucial component in ensuring lenses form a category.
      It arises in defining lens composition, where we need comonoids, and in proving this composition is associative, where we need maps to preserve these comonoids.
      Operationally, this means that we need to be able to copy the input to the lens, and slide forward maps through it without any emergent effects.
      As elaborated in \cref{ch:backpropagation} ("Memoisation" bullet point), \cref{subsec:operational}, and \cite{gavranovic_space-time_2022}, this is one way of composing bidirectional systems, via \emph{checkpointing}.

      The framework of optics explicitly shows us something well-known in the deep learning community: this is not the only way of composing bidirectional systems.
      We can instead store intermediate results of the forward pass in memory, relinquishing us from the need to recompute them again.
      This in turn removes the requirement that maps preserve comonoids, i.e.\ that they are deterministic, allowing us to model bidirectional systems where the forward map is probabilistic (\cref{ex:mark_ker_expectation}).
      Can the same kind of reasoning be applied to the setting of dependent lenses, and what does it mean for something to be dependent on the output of a non-deterministic map?

            \item \textbf{Coproducts.} Even if the base category $\cC$ has coproducts\footnote{In a way that interacts well with existing pullbacks, making it an \emph{extensive} category.}, $\Lens(\cC)$ does not have all coproducts. It has \emph{some} coproducts, for instance those of the form $\diset{X}{Z} + \diset{X}{Z}$ which share the backward type.
      On the other hand, the category $\DLens(\cC)$ does have all coproducts.
      If $\DLens(\cC)$ arises as the coproduct completion\footnote{Not to be confused with \emph{the free coproduct completion} which does not preserve any existing coproducts.} of $\Lens(\cC)$, can we define dependent optics as the coproduct completion of $\Optic(\cC)$?
    \end{enumerate}

    Let's unpack all of them.
    The characterisation with actions is the most common one --- appearing as a proposed solution in \cite{braithwaite_fibre_2021, vertechi_dependent_2023,milewski_dependent_2021, capucci_seeing_2022}.
    For instance, the work of \cite{vertechi_dependent_2023}, given a bicategory $\cB$ and two $\cB$-indexed bicategories $L, R$ defines dependent optics as a particular kind of a coend construction.
    They show that the category of dependent lenses and the category of optics both arise as particular (but different!)  instantiations of $\cB, L, R$.
    Therefore, this satisfies the first characterisation of dependent optics as ones arising from actions.
    But it doesn't satisfy the pushout characterisation.
    This is simply because this work does not provide a definition of a \emph{single} category which both dependent lenses and optics include into, but rather a mechanism for producing them separately.
    Likewise, the resulting categories of optics and dependent lenses have shared components that do not align with their semantics: the forward pass of dependent lenses is constructed in a different manner than the forward pass of optics.%

    This brings us to the pushout characterisation, first posed in \cite{braithwaite_fibre_2021}.
    Here, $\cC$ is a category with all finite limits\footnote{i.e.\ a category with all equalisers and all binary products i.e.\ a category with all pullbacks and a terminal object.}, $\Optic(\cC)$ is the category of (non-mixed) optics defined by the cartesian monoidal structure of the products in $\cC$, and $\DLens(\cC)$ is the category of dependent lenses over $\cC$.
    The embedding of lenses into optics is given by \cref{eq:lens}, while the embedding of lenses into dependent lenses is the embedding defined in \cite[Sec. 2.3.3]{niu_polynomial_2023}.
    This characterisation ensures that dependent optics have a shared subcategory of non-dependent lenses, consistent with embeddings arising either through dependent lenses or optics.

    While this a well-defined characterisation, the simplicity of its solution suggests that a more refined approach is needed.
    Namely, the embedding from lenses into optics is an isomorphism, and pushouts along isomorphisms are isomorphisms.
    Via this characterisation, then, dependent optics are simply dependent lenses!
    This is unsatisfactory for two reasons.
    Firstly, it does not tell us how to generalise dependent optics.
    In the non-dependent setting by moving from lenses to optics we lost some of our chains: we discovered that we can forget about the cartesian structure, because we need only the monoidal one to form optics.\footnote{Of course, we've seen that we can generalise far beyond monoidal too!}
    But here it is not clear that there's any structure we can remove, and it's not clear how to generalise this solution.

    The second problem is that this definition does not inform implementation.
    Like any other definition by universal properties, this definition is a specification only up to isomorphism.
    And the particular isomorphism here is blind to the performance profile of the resulting implementation.
    But performance in the context of deep learning is of utmost importance.
    Ignoring this issue we'd be forced to go outside of our compositional framework to achieve a performant implementation, when in fact it should be the other way around: performance should come because of compositionality, not at its expense.

    This leads us to search for a pushout in a more refined ambient setting than just $\Cat$.
    As noted in \cite{gavranovic_space-time_2022}, one can study the weakening of $\Optic(\cC)$ to a bicategory $\TwoOptic(\cC)$, in which case the embedding $\Lens(\cC) \to \Optic(\cC)$ becomes a \emph{lax} functor,\footnote{With $\kappa$ still being a strict 1-functor.} which now detects the different composition rule of lenses and optics.
    2-categories and lax functors form a category, and this category is the simplest setting that has the needed level of fidelity to distinguish operational issues outlined above, and prevent the pushout from being an isomorphism.
    Alternatively, the pushout might still need to be weakened in order to exist, though it's not clear to what extent.
    This is because this lax functor is not a part of an isomorphism anymore, but rather an adjunction.
    The simplest setting where this adjunction can live is the 2-category of 2-categories, lax functors and \emph{icons} (\cite{lack_icons_2010,gavranovic_space-time_2022}).
    Another option is $\LxDbl$, the 2-category of double categories, lax double functors, and lax transformations (\cite[p. 148]{grandis_higher_2019}).
    In such cases we might want to ask for a 2-pushout, a lax pushout or a quasi-pushout.
    It's not clear which one, and in the context of optics, none of these options have been acknowledged in literature.

    A different path to dependent optics was explored through coproduct completions.
    Even if the base category $\cC$ has coproducts\footnote{In a way that interacts well with existing pullbacks, making it an \emph{extensive} category.}, $\Lens(\cC)$ does not have all coproducts.
    It has \emph{some} coproducts, for instance those of the form $\diset{X}{Z} + \diset{X}{Z}$ which share the backward type.
    On the other hand $\DLens(\cC)$ has all coproducts, suggesting that a coproduct completion\footnote{Not to be confused with \emph{the free coproduct completion} which does not preserve any existing coproducts.} of $\Optic(\cC)$ could yield dependent optics.
    This was the idea behind solutions in \cite{braithwaite_fibre_2021} and \cite{vertechi_dependent_2023} which define categories called \emph{indexed} and \emph{dependent} optics, respectively, and show that they have all coproducts.
    It is not clear to me what the relationship between these definitions is, nor their relation to the pushout definition (\cref{eq:dependent_optics_pushout}).
    As cherry on top, the situation is made more complicated by the fact that during the writing of this thesis it was discovered that while $\DLens(\Set)$ has all coproducts, the embedding from $\Lens(\Set)$ does not preserve the ones that already exist!\footnote{I learned this from Jules Hedges and Eigil Rischel.}
    For example, the coproduct of $\diset{1}{0}$ and $\diset{2}{0}$ exists in $\Lens(\Set)$, but its image under the embedding into $\DLens(\Set)$ is not a coproduct of the embeddings of $\diset{1}{0}$ and $\diset{2}{0}$.%

    Lastly, it is not abundantly clear whether these four characterisations are even seeing the whole picture.
    Optics for coproducts (\cref{ex:optics_for_coproducts}) form prisms, and comparatively little research was devoted to the question of existence of dependent prisms, actions they might be formed by, the pushout of the analogous diagram, or their operational interpretation.
    In \cite{vertechi_dependent_2023} dependent prisms are defined, albeit for \emph{one} of the two different definitions of prisms in the literature (see \cref{subsec:prisms} for disambiguation): the one that does not have the needed compositional properties.
    Do dependent prisms for the other definition exist, do they satisfy the pushout diagram analogous to \cref{eq:dependent_optics_pushout}, and what can be said about their relation to dependent lenses?
    It is not clear.

    This concludes the problem statement of dependent optics --- it's rather lengthy!
    Hopefully this will help us understand the direction we need to move in order to solve it better, and we believe the definition of optics using weighted $\Copara$ provides a good avenue for doing so.\footnote{Our work-in-progress Idris 2 code repository of dependent optics is at \url{https://github.com/bgavran/DependentOpticsIdris2}.}

	\chapter{Putting the pieces together}
\label{ch:para_optic}

\epigraph{Computer science is no more about computers than astronomy is about telescopes.}{Edsger W. Dijkstra}

\newthought{Having described} parameterisation and bidirectionality separately, in this short chapter we will take advantage of the compositional nature of category theory and put these two concepts together.
This will yield \emph{parametric weighted optics}.

\begin{figure}[H]
	\scaletikzfig[0.75]{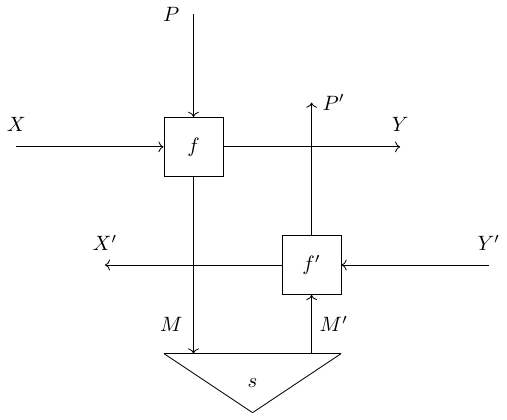}
  \caption{A parametric weighted optic, with its ``three-legged'' shape.
    This is a versatile construction can be used to model both neural networks in deep learning and players in game theory.
  }
  \label{fig:para_woptic}
\end{figure}

The diagram to have in mind is the one above (\cref{fig:para_woptic}).
Here we see the three-legged shape of parametric weighted optics\footnote{We'll often refer to them as just ``parametric optics''.}, where each leg has both inputs flowing in and outputs flowing out.
Horizontal direction (one containing $\diset{X}{X'}$ and $\diset{Y}{Y'}$) is the one where information flows that is in many ways \emph{publicly available} --- other layers of neural networks or players in game theory can plug in this interface.
On the other hand, vertical direction (one containing $\diset{X}{X'}$ and $\diset{Y}{Y'}$) where information flows that is \emph{private}.
These are the parameters of a neural networks or strategies of players in game theory (on the top), or, intermediate results used for backprop or intermediate results used to calculate payoffs in game theory (on the bottom).

We think of $X$ as the input to the neural network, for instance an encoding of an image to be classified (i.e.\ $\R^{w \times h \times 3}$ where $w$ is the width of the image, $h$ is its height, and $3$ denotes that the image is encoded in RGB).
The vertical input $P$ here denotes the space of possible weights for the network, and $Y$ denotes the output of this neural network, which could be yet another hidden layer to which other networks will connect, or something that is ready to be consumed by the loss function (for instance, an element of $\R$ denoting the classification of this image as a cat vs. dog).
The part of this diagram that has primes $'$ in it denotes everything that is related to gradients and derivatives.
We think of $Y'$ as feedback the network received as a response to its produced output.
Say, if the network produced an output $y : Y$, then elements of $Y'$ denote the gradient of the loss with respect to that output $y$.
In conjunction with the information from the forward pass, this is turned into a gradient with respect to a) the parameter $p : P$ and b) the original input $x : X$.

\begin{figure}[H]
	\scaletikzfig[0.65]{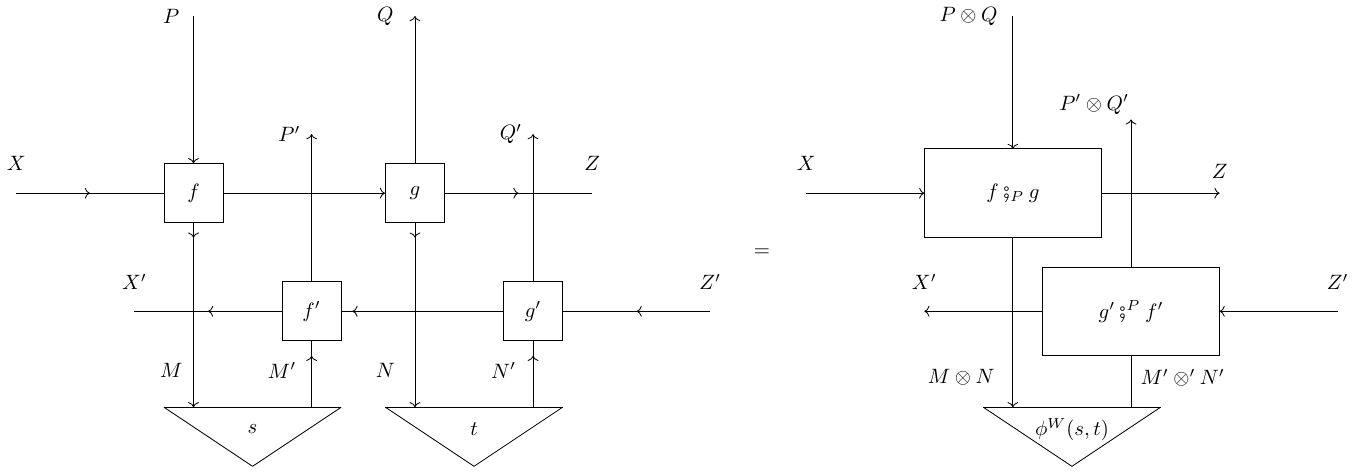}
  \caption{Composition of parametric optics}
  \label{fig:para_optic_composition}
\end{figure}

Composition is again natural in this formulation.
Given another box with input-output wires $\diset{Y}{Y'}$ on the right and another one with input-output wires $\diset{Y}{Y'}$ on the left, we can plug them together as in \cref{fig:para_optic_composition}.
The algebra of composition of $\Para$ and $\Optic$ automatically takes care of backpropagating the relevant information to correct ports, both on parameter level (vertical) and on the input level (horizontal).

Reparameterisation in this setting also takes on a useful form.
Depicted in \cref{fig:para_optic_reparameterisation}, it allows us to model neural network optimisers, described thoroughly in \cref{sec:optimisers}.
They modify the parameters before passing it on to actual layers of the forward pass of a neural network (a central feature of Nesterov momentum (\ref{ex:nesterov_momentum})), but also perform the gradient update step on the backward one.

\begin{figure}[H]
	\scaletikzfig[0.8]{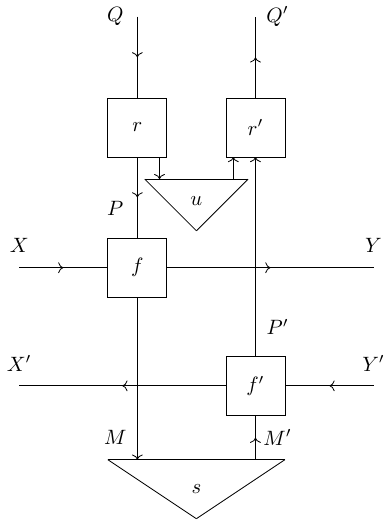}
  \caption{Reparameterisation of parametric optics models optimisers of neural networks.}
  \label{fig:para_optic_reparameterisation}
\end{figure}

Lastly, we also see that states, costates, and scalars of parametric optics all provide useful semantics for supervised learning (\cref{box:scs_para_optic}).
We continue to put all of these informal stories on a concrete mathematical footing.

\begin{contributions}
  The entire chapter is novel contribution.
  This includes formulation of parametric weighted optics, morphisms of actegories, and the base change of $\Para$.
  Fragments of this work appeared in (\cite{cruttwell_categorical_2022, capucci_towards_2022, capucci_actegories_2023}).
\end{contributions}

\begin{epistemicstatus}
  The epistemic status of this chapter mirrors that of the two previous ones, as it mainly involves putting these chapters together.
  I believe the general simplification of two previous ones compounds here, yielding a definition of parametric optics that might have an even more pragmatic compartmentalisation of all the moving pieces.
\end{epistemicstatus}

\section{Parametric optics and how to construct them}
\label{sec:parametric_optics}

Recall the data we need to form weighted optics (\cref{def:weighted_optic}).
It includes two monoidal categories $(\cM, \otimes, I)$ and $(\cM', \otimes', I')$, the weight $W : \cM^{\op} \times \cM' \to \Set$, and actions of $\cM$ and $\cM'$.
To form parametric weighted optics, we first need to define an action of some monoidal category on $\Optic^W_{\diset{\actfw}{\actbw}}$.
Which monoidal category is it?

This is the category of weighted optics $\Optic^W_{\diset{\otimes}{\otimes'}}$ arising from self actions of $(\cM, \otimes, I)$ and $(\cM', \otimes', I')$.
If $\cM$ and $\cM'$ are additionally braided, then the category $\Optic^W_{\diset{\otimes}{\otimes'}}$ is monoidal, and braided (following \cref{ex:self_actions_monoidal,prop:optics_braid_symm}), meaning we can act with it on $\Optic^W_{\diset{\actfw}{\actbw}}$.

\begin{proposition}[{\cite[Prop. 10]{capucci_towards_2022}}]
  \label{prop:optics_acted_on}
  Fix the data required to form weighted optics:
  \begin{center}
    A $\cM$-actegory $(\cC, \actfw)$\\
    A $\cM'$-actegory $(\cD, \actbw)$\\
    A lax monoidal functor $(W, \phi, \epsilon) : \cM^{\op} \times \cM' \to \Set$
  \end{center}
  where here the monoidal categories $(\cM, \otimes, I)$ and $(\cM', \otimes', I')$ are additionally braided.
  Then we can form an action of $\Optic^W_{\diset{\otimes}{\otimes'}}$ on $\Optic^W_{\diset{\actfw}{\actbw}}$ defined on objects as:
  \[
    \diset{M}{M'} \actlast \diset{X}{X'} \coloneqq \diset{M \actfw X}{M' \actbw X'}
  \]
\end{proposition}

This allows us to form \emph{parametric weighted optics} by taking its $\Para$ construction.

\begin{definition}[{Parametric weighted optics (compare \cite[Sec. 4]{capucci_towards_2022})}]
  Fix the data required to form weighted optics:
  \begin{center}
    A $\cM$-actegory $(\cC, \actfw)$\\
    A $\cM'$-actegory $(\cD, \actbw)$\\
    A lax monoidal functor $(W, \phi, \epsilon) : \cM^{\op} \times \cM' \to \Set$
  \end{center}
  where here the monoidal categories $(\cM, \otimes, I)$ and $(\cM', \otimes', I')$ are additionally braided.
  Then we call $\Po$ the bicategory of parametric weighted optics.
\end{definition}

The definition of parametric optics packs a lot of punch, since it is composed out of already complex definitions of  $\Para$ (\cref{def:para}) and $\Optic^W_{\diset{\actfw}{\actbw}}$ (\cref{def:weighted_optic}).
In its full generality a parametric optic $\diset{A}{A'} \to \diset{B}{B'}$ (depicted in \cref{fig:para_woptic}) consists of a choice of parameters $\diset{P}{P'}$ where $P : \cM$ and $P' : \cM'$; and a weighted optic $\diset{A \actfw P}{A' \actbw P'} \to \diset{B}{B'}$ which itself consists of a choice of residuals $\diset{M}{M'}$, a map $s : W(M, M')$, and a pair of maps$\diset{f}{f'} : \diset{A \actfw P}{A' \actbw P'} \to \diset{M \actfw B}{M' \actbw B'}$.
Despite this, in many cases, especially those that pertain to neural networks, parametric optics will take on a much simpler form.
We proceed to describe their states, costates and scalars, and then a general recipe for constructing parametric optics from only their forward passes.

  \begin{mybox}[label=box:scs_para_optic]{Cerulean}{States, costates, and scalars in $\Para[false](\Optic[false](\cC))$}
    Here we focus on parametric optics for the self-action of a monoidal category $\cC$.

    \begin{minipage}[t]{0.29\textwidth}
      \begin{tightcenter}
        States
      \end{tightcenter}
      \scaletikzfig[0.6]{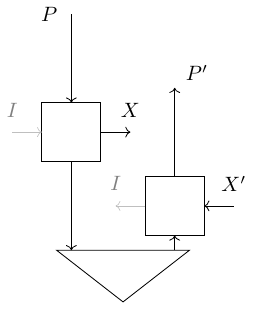}
      \begin{align*}
        &\Pon\OpticHom{I}{I}{X}{X'}\\
        \cong & \sum_{\diset{P}{P'} : \Optic}\Optic\OpticHom{P}{P'}{X}{X'}
      \end{align*}
    \end{minipage}
    \hfill
    \begin{minipage}[t]{0.28\textwidth}
      \begin{tightcenter}
        Scalars
      \end{tightcenter}
      \scaletikzfig[0.6]{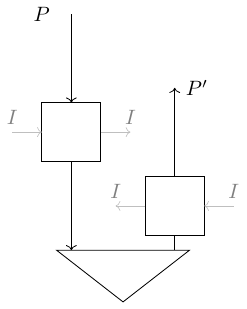}
      \begin{align*}
        &\Pon\OpticHom{I}{I}{I}{I}\\
        \cong & \sum_{\diset{P}{P'} : \Optic}\Optic\OpticHom{P}{P'}{I}{I}\\
        \cong & \sum_{\diset{P}{P'} : \Optic}\cC(P, P')
      \end{align*}
    \end{minipage}
    \hfill
    \begin{minipage}[t]{0.28\textwidth}
      \begin{tightcenter}
        Costates
      \end{tightcenter}
      \scaletikzfig[0.6]{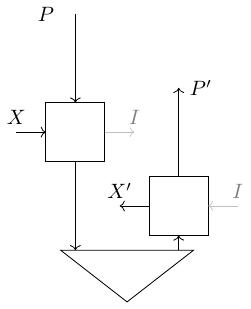}
      \begin{align*}
        & \Pon\OpticHom{X}{X'}{I}{I}\\
        \cong & \sum_{\diset{P}{P'} : \Optic}\cC(P \otimes X, P' \otimes X')
      \end{align*}
    \end{minipage}

    We see that states correspond simply to optics $\diset{P}{P'} \to \diset{X}{X'}$ for some $\diset{P}{P'}$.
    Scalars correspond to maps $P \to P'$ in $\cC$, and costates to maps $P \otimes X \to P' \otimes X'$ in $\cC$.
    Scalars will especially be important in \cref{sec:supervised_learning} where they will model a step of supervised learning. 
  \end{mybox}

\subsection{Constructing parametric optics}

In most cases we are intersted in constructing parametric weighted optics from the data of only its forward pass --- which is parametric as well.
This is often done by first constructing \emph{weighted optics} from their forward pass, and then showing the construction coherently translates to the parametric setting.

\subsection{Para is functorial in the base}

In \cref{ch:para} we have shown how to construct the bicategory $\Para_{\act}(\cC)$ given a base $\cM$-actegory $(\cC, \act)$. 
We now study what happens if we have a morphism of actegories, and show that $\Para$ then induces a morphism of the corresponding bicategories.
We start by describing a special kind of a morphism of actegories: one which does not change the acting category.

\begin{definition}[{$\cM$-linear morphism, compare \cite[Def.\ 3.3.2]{capucci_actegories_2023}}]
  \label{def:m-linear_morphism}
  Let $(\cM, \otimes, I)$ be a monoidal category, and $(\cC, \act)$ and $(\cD, \actt)$ be $\cM$-actegories.
  Then a \newdef{$\cM$-linear functor} $(\cC, \act) \to (\cD, \actt)$ is a functor $R^{\sharp} : \cC \to \cD$ together with a natural isomorphism
  \begin{equation}
    \label{eq:m-linear_morphism}
    \begin{tikzcddiag}
    	{\cC \times \cM} && {\cD \times \cM} \\
    	\\
    	\cC && \cD
    	\arrow["\act"', from=1-1, to=3-1]
    	\arrow["\actt", from=1-3, to=3-3]
    	\arrow["{R^{\sharp}}"', from=3-1, to=3-3]
    	\arrow["{R^\sharp \times \cM}", from=1-1, to=1-3]
    	\arrow["\ell"', shorten <=14pt, shorten >=14pt, Rightarrow, from=1-3, to=3-1]
    \end{tikzcddiag}
  \end{equation}
  whose component at each $C : \cC$ and $M : \cM$ is explicitly the map $\ell_{C, M} : R^\sharp(C) \actt M \to R^\sharp(C \act M)$ making the following diagrams commute for all $C : \cC$ and $M, N : \cM$:
  \begin{equation*}
\begin{tikzcddiag}[ampersand replacement=\&]
	{R^{\sharp}(C) \actt I} \&\&\& {(R^{\sharp}(C) \actt M) \actt N} \& {R^{\sharp}(C \act N) \actt N} \& {R^{\sharp}((C \act M) \act N)} \\
	\\
	{R^{\sharp}(C \act I)} \&\& {R^{\sharp}(C)} \& {R^{\sharp}(C) \actt (M \otimes N)} \&\& {R^{\sharp}(C \act (M \otimes N))}
	\arrow["{\ell_{C, M} \actt N}", from=1-4, to=1-5]
	\arrow["{\ell_{C \act M, N}}", from=1-5, to=1-6]
	\arrow["{{{\mu^\actt}^{-1}}_{R^{\sharp}(C), M, N}}"', from=1-4, to=3-4]
	\arrow["{\ell_{C, M \otimes N}}"', from=3-4, to=3-6]
	\arrow["{R^{\sharp}({\mu^{\act}}^{-1}_{C, M, N})}", from=1-6, to=3-6]
	\arrow["{\eta^{-1}_{R^{\sharp}(C)}}", from=1-1, to=3-3]
	\arrow["{R^{\sharp}(\eta^{-1}_C)}"', from=3-1, to=3-3]
	\arrow["{\ell_{C, I}}"', from=1-1, to=3-1]
\end{tikzcddiag}
  \end{equation*}

  Without any other qualifications, we also refer to $(R^{\sharp}, \ell)$ as a \newdef{lax} $\cM$-linear functor.
  If $\ell$ is invertible, we call it a \newdef{strong}, and when it is identity a \newdef{strict} $\cM$-linear functor.
\end{definition}

This definition is a generalisation of the notion of a tensorial strengh (\cref{def:strong_endofunctor}).

\begin{example}[{compare \cite[Ex. 3.3.5]{capucci_actegories_2023}}]
  \label{ex:strong_endofunctor_morphism_of_actegories}
  Every strong endofunctor (\cref{def:strong_endofunctor}) is a morphism of canonical right self-actions. 
\end{example}

As we will see, a general morphism of actegories will also capture strong monoidal functors.

\begin{definition}[{Morphism of actegories, compare \cite[Prop. 3.6.5]{capucci_actegories_2023}}]
  \label{def:morphism_of_actegories}
  Let $(\cC, \act)$ be a $\cM$-actegory, and $(\cD, \actt)$ be a $\cN$-actegory.
  A morphism of actegories $(\cC, \cM, \act) \to (\cD, \cN, \actt)$ is a pair of a strong monoidal functor $R : \cM \to \cN$ and a $\cM$-linear functor $(R^{\sharp}, \ell) : (\cC, \act, \cM) \to (\cD, \cM, (\cD \times R) \comp \actt)$.
  It can be depicted by the data of the following 2-cell:
    \begin{equation}
      \label{eq:actegory_morphism}
      \begin{tikzcddiag}[ampersand replacement=\&]
        {\cC \times \cM} \&\& {\cD \times \cN} \\
        \\
        \cC \&\& \cD
        \arrow["\act"', from=1-1, to=3-1]
        \arrow["\actt", from=1-3, to=3-3]
        \arrow["{R^\sharp}"', from=3-1, to=3-3]
        \arrow["{R^\sharp \times R}", from=1-1, to=1-3]
        \arrow["\ell"', shorten <=14pt, shorten >=14pt, Rightarrow, from=1-3, to=3-1]
      \end{tikzcddiag}
    \end{equation}
    which on components defines a map $\ell_{C, M} : R^\sharp(C) \actc R(M) \to R^\sharp(C \act M)$ satisfying conditions of a lax $\cM$-linear functor (\cref{def:m-linear_morphism}).
    The naming conventions of laxness, strength and strictness from \cref{def:m-linear_morphism} apply here too.

    \end{definition}

\begin{example}
Every strong monoidal functor (\cref{def:lax_monoidal_functor}) is a morphism of canonical right self-actions. 
\end{example}

In the next chapter we will see an example of such a functor -- the functor $\cC \to \LensA(\cC)$ that performs differentiation on a cartesian left-additive category.
Interestingly, a lax monoidal functor is \emph{not} an example of a morphism of actegories.\footnote{Though, it is an example of a morphism of \emph{lax} actegories whose detailed unpacking we leave for future work.}
More examples of morphisms of actegories can be found in \cite[Sec. 3.3]{capucci_actegories_2023}.

With the definition of morphism of actegories in hand, we are now ready to see how $\Para$ acts on this, yielding a pseudofunctor between the corresponding $\Para$ bicategories.

\begin{figure}[h]
  \scaletikzfig[0.75]{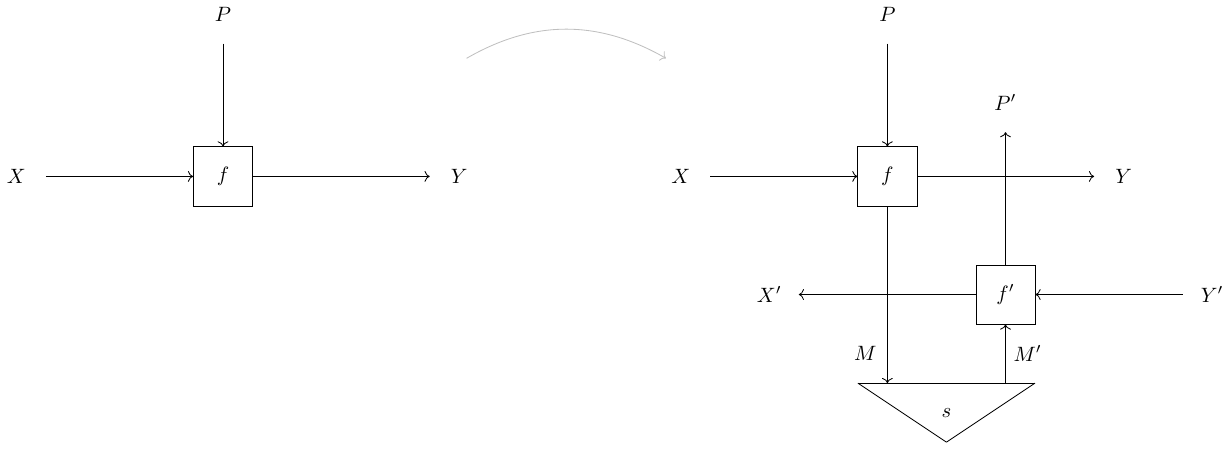}
  \caption{From a parametric map to a parametric weighted optic}
  \label{fig:para_to_parawoptic}
\end{figure}

\begin{restatable}[A morphism of actegories induces a morphism of parametric bicategories]{proposition}{ParaFunctorial}
  \label{prop:para_morphism}
  Let $(\cC, \cM, \act)$ and $(\cD, \cN, \actt)$ be two actegories.
  Let $(R, R^\sharp, \ell)$ be a morphism of actegories $(\cC, \cM, \act) \to  (\cD, \cN,
  \actt)$.

  Then $\Para$ respects that structure, inducing a pseudofunctor
  \[
    \Para((R, R^\sharp, \ell)) : \Para_{\act}(\cC) \to \Para_{\actt}(\cD)
  \]
  with the following data.
  \begin{itemize}
  \item On \textbf{objects} it borrows the action of $R^\sharp$;
  \item On \textbf{morphisms} at acts as below, and graphically depicted in \cref{fig:para_to_parawoptic}:
\[\begin{tikzcddiag}[ampersand replacement=\&]
	{\Para_{\act}(\cC)} \&\& {\Para_{\actc}(\cD)} \\
	A \&\& {R^\sharp(A)} \\
	\\
	B \&\& {R^\sharp(B)}
	\arrow[from=1-1, to=1-3]
	\arrow[""{name=0, anchor=center, inner sep=0}, "{(P, f)}"', from=2-1, to=4-1]
	\arrow[""{name=1, anchor=center, inner sep=0}, "{(R(P), R_{\ell}^\sharp(f))}", from=2-3, to=4-3]
	\arrow[maps to, from=2-1, to=2-3]
	\arrow[maps to, from=4-1, to=4-3]
	\arrow[shorten <=20pt, shorten >=20pt, maps to, from=0, to=1]
\end{tikzcddiag}\]
   where $R_{\ell}^\sharp(f)$ is shorthand for the composite
  \[
    R^\sharp(A) \actc R(P) \xrightarrow{\ell_{A, P}} R^\sharp(A \act P) \xrightarrow{R^\sharp(f)} R^\sharp(B)
  \]
\item On \textbf{2-morphisms} it borrows the action of $R$.
\end{itemize}
\end{restatable}

\begin{proof}
  \Cref{app:para}.
\end{proof}

\begin{proposition}[When is $\Para((R, R^\sharp, \ell))$ a 2-functor?]
  \label{ex:paraf_2functor}
  If the morphism of actegories $(R, R^\sharp, \ell)$ is a \emph{strict} morphism, then so is the induced pseudofunctor $\Para((R, R^\sharp, \ell))$, i.e.\ it becomes a 2-functor.
\end{proposition}

This induced functor will be a central component of neural networks -- as elaborated in \cite[Sec. 3.1]{cruttwell_categorical_2022}, and the second part of this thesis.
Its strict 2-functoriality will tell us that starting with two composable neural networks, we get the same result if we a) augment each neural network separately with its backward pass and then compose the result, or b) compose the networks and then augment the result with its backwards pass (\cref{fig:model_loss_commuting}).

\begin{figure}[H]
  \scaletikzfig[0.8]{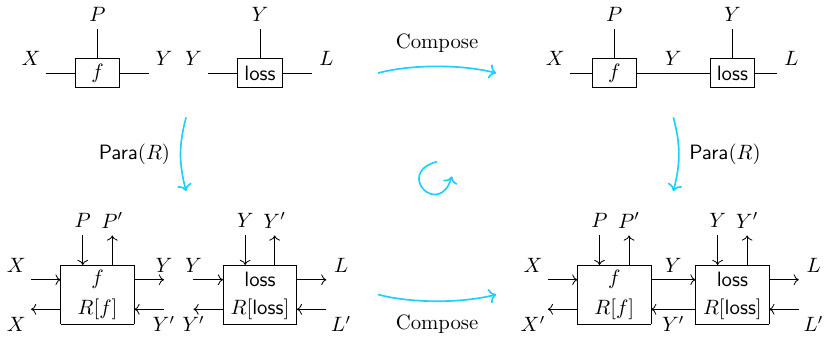}
  \caption{``Parametric chain rule'': starting with a neural network and a loss function, we can either differentiate them individually and compose the results, or first compose them and differentiate the result. These yield the same result, enabling us to break down very large composites into manageable pieces.}
  \label{fig:model_loss_commuting}
\end{figure}

Furthemore, the generalisation of this pseudofunctor to \emph{a lax functor} is explored in \cite{capucci_diegetic_2023} where applications to settings such as that of game theory are studied.\footnote{Here laxness is crucial, providing an argument that a good abstraction for these settings is that of a \emph{lax} actegory.}
Building on top of this, we could spend a lot of time unpacking further categorical structure --- morphisms of monoidal actegories, and consequently morphisms of monoidal paras; morphisms of cartesian actegories; morphisms of cartesian paras, and so on.
We don't do that in this thesis, and instead move on to applications: backpropagation, neural network architectures, and supervised learning.

We will see how optics can be specialised to a particular category $\LensA(\cC)$ of lenses additive in the backward pass, and how $\Para$ can be applied to a symmetric monoidal functor $\cC \to \LensA(\cC)$ giving us a way to augment parametric maps with their derivative.

  \part{Applications}
  \label{part:applications}
	\chapter{Backpropagation}
\label{ch:backpropagation}

\epigraph{Automatic differentiation is symbolic differentiation performed by a compiler.}{Conal Elliott \cite{elliott_video_2018}}

\newthought{When we think about derivatives,} we often think about differentiable functions between euclidean spaces.
This is a setting in which derivatives are easy to formalise in, and one which been thoroughly studied.
Unfortunately, this setting does not capture numerous things done with derivatives in practice, and in theory.

Firstly, not all differentiable functions are functions between \emph{euclidean} spaces.
Deep learning has been studied on complex numbers (\cite{bassey_survey_2021}), hyperbolic spaces (\cite{peng_hyperbolic_2022}), finite-state automata (\cite{hannun_differentiable_2020}), graphs (\cite{wu_comprehensive_2021}), hypergraphs, simplicial complexes, cellular complexes (\cite{papillon_architectures_2023}), and even boolean circuits (\cite{wilson_reverse_2021}) which --- surprisingly --- also admit a coherent notion of a derivative.
Is there a precise notion of derivative that coherently captures all of these?

Secondly, euclidean spaces are not cartesian closed, don't have all coproducts, nor do they host dependent types.
This means modelling derivatives in this category prevents us from reasoning about derivatives of higher-order functions\footnote{Not to be confused with higher-order derivatives of functions.}, differentiation with branching, and differentiation on manifolds.
This is especially important with the advent of numerous frameworks heavily used in practice which often provide branching or higher-order differentiation such as PyTorch (\cite{paszke_pytorch_2019}) or JAX (\cite{bradbury_jax_2018}), but whose categorical semantics are unknown.
Indeed: what are their categorical semantics? 
Can we perform branching in hyperbolic spaces, or boolean circuits?
What about derivatives of higher-order functions? %

Lastly, we are interested in not just studying neural networks, but \emph{implementing} them.
The process of computing derivatives in a programming language is often referred to as \emph{automatic differentiation} (AD) (\cite{baydin_automatic_2017}) where --- in addition to correctness --- this is done with an additional constraint of \emph{efficiency}.
There are two main efficiency concerns.
\begin{itemize}
  \item \textbf{Order of multiplication of Jacobian matrices.} Given many Jacobian matrices to be multiplied together, different orders of their composition induce a different computational cost.
  All left-associated composition is often referred to as ``forward-mode AD'' while all-right one is called ``reverse-mode AD''.
  If $n$ and $m$ are respectively the dimensionality of the domain and the codomain of the composite function we are differentiating, forward-mode AD tends to be more efficient $n \ll m$, and reverse-mode AD tends to be more efficient when $m \ll n$.
  \item \textbf{Memoisation.} An important efficiency aspect is whether to save intermediate results of derivative computation or to recompute them from scratch.
  The former is often the standard in automatic differentiation, while the latter is called \emph{gradient checkpointing} (\cite{chen_training_2016, griewank_algorithm_2000, dauvergne_data-flow_2006}).
  Gradient checkpointing, i.e.\ always recomputing the derivatives from scratch is useful when dealing with large models where saving intermediate results would quickly consume the available GPU memory.
  On the other hand, if we do have enough memory, saving intermediate results is faster.
\end{itemize}

When it comes to the multiplication of Jacobians, implementing forward-mode AD in a language with complex features is well understood, admitting an elegant proof of correctness (\cite{huot_correctness_2020}).
But this is not the case for reverse-mode AD.
Until very recently, reverse-mode AD was only well-understood only as a compile-time source-code transformation on programming languages with limited features, such as only first order functions (\cite{elliott_simple_2018}).
Implementations of reverse-mode AD in more expressive languages often rely on their interpreted features, where a \emph{computational graph} is built during runtime in a manner where higher-order functions and branching are evaluated, leaving only a simple first-order program.
This approach suffers the main drawback that it presents a separate language (happening at runtime) outside of the reach of the compiler and any existing compiler optimisations.
As such, proofs of correctness have often introduced mutable state (\cite{pearlmutter_reverse-mode_2008}), increasing the difficulty of understanding and verifying their correctness.
Recent work tackled these problems (\cite{vakar_chad_2022,vakar_reverse_2021,alvarez-picallo_functorial_2021}).
The first two provide structured translation for specific languages that host higher-order types.
They also provide explicit categorical semantics in terms of structure preserving functors, though explicit relationship to existing categorical frameworks such as tangent categories are not established (\cite{cockett_differential_2014,cruttwell_reverse_2023,garner_embedding_2018}).\footnote{Especially with respect to the axioms of the \emph{vertical lift} or the \emph{canonical flip}.}
The latter work provides a purely functional and formal interpretation of  (\cite{pearlmutter_reverse-mode_2008}), for the first time proving its soundness.

On the other hand, when it comes to memoisation, the distinction between gradient-checkpointed AD and non-checkpointed AD is mostly ignored.
Indeed, none of the above mentioned literature acknowledge it.
The only exception is \cite{gavranovic_space-time_2022} which provides a 2-categorical framework distinguishing between reverse mode AD that performs gradient checkpointing versus one that does not.

The points above motivate a need for a characterisation of backpropagation that has a number of properties.
It needs to a) cover exotic spaces; b) host sufficient structure for higher-order functions and branching, and c) be operationally aware, providing a well defined procedure of computing these derivatives efficiently in terms of order of multiplication of Jacobians and memoisation.
Is there general notion of derivative encompassing all these?
Can we have a framework that is both descriptive --- denotationally describing all the relevant aspects --- but also \emph{prescriptive} --- equipped with a recipe for implementing it that is cognizant of the various efficiency concerns?
Moreover, motivated by our constructions in \cref{ch:bidirectionality} we ask: does this framework arise as a special case of a general construction that covers a variety of other kinds of bidirectional processes, such as bayesian learning and value iteration?

In this chapter, we provide a resounding ``yes'' to the last question, and a partial answer to the previous ones.
We define \emph{additively closed cartesian left-additive categories} (\cref{def:additively_closed_cla}), serving as our foundation for differentiation.
In \cref{eq:concrete_additive_optic2} we describe how weighted optics defined on this base model backpropagation, and how they can be reduced to their more concrete formulation $\LensA(\cC)$ of lenses whose backward passes are additive in the second component.
In \cref{sec:backprop_as_functor} we formally study $\LensA$, defining it as a functor, and consequently introducing \emph{generalised reverse derivative categories} as coalgebras of this functor.
This gives us a framework for categorically modelling differentiation that we use in the rest of this thesis, capturing semantics of differentiation solely in in the language of lenses.
This framework covers euclidean spaces and boolean circuits (among others), but not hyperbolic spaces, or simplicial complexes, for instance.
It is closely related to \emph{reverse derivative categories} (\cite{cockett_reverse_2020}) for whose axioms it serves as a justification.
Likewise, while we don't directly tackle the question of higher-order differentiation\footnote{Meaning derivatives of higher-order functions as opposed to higher-order derivatives of functions.}, we believe the general categorical constructions in this thesis can provide a way forward in understanding those.

\begin{contributions}
  In this chapter, the novel contribution is the lens- and optic-theoretic axiomatisation of differentiation.
  Specifically, this includes the definition of an additively closed cartesian left-additive category (\cref{def:additively_closed_cla}), definition of $\LensA$ as an endofunctor (\cref{thm:lens_functor_clac}) and a generalised cartesian reverse differential category as its coalgebra (\cref{thm:coalgebra_rdc}).
  What is also novel is the identification of checkpointed and non-checkpointed automatic differentiation as lens and optic composition, respectively (\cref{subsec:operational}), originally appearing in my preprint \cite{gavranovic_space-time_2022}.
  Seeds of the idea of axiomatisation of differential structure as lenses appeared in \cite{cruttwell_categorical_2022}, though this work still used the framework of reverse derivative categories.
\end{contributions}

\begin{epistemicstatus}
  Follows that of the chapter on bidirectionality.
  Operational aspects were chosen in favour of dependent types, with the hope that in the long run the operational and intensional view will shed more light on these systems than a mere extensional and denotational one could.
  The state of the research leaves much to be desired: I want to establish formal connections to (reverse) tangent categories, provide a deeper characterisation of operational aspects, and simplify the definition of ACCLA by making it clearer what is the interaction between the (non-cartesian) monoidal closed structure of the backward pass and the cartesian structure of the forward pass.
  Likewise, I do not expect left-additive structure to be to be necessary in the long-run either, as we can study the category of additive maps in any category that's cartesian, and even merely just monoidal.
\end{epistemicstatus}

\section{What is differentiation?}
\label{sec:what_is_differentiation}

We often interpret the derivative of a function as its slope at a particular point, measuring how fast the function changes at that point.
For instance, starting with a function $f(x) = x^2$ on real numbers, the derivative of $f$ is usually written as $\frac{df}{dx} = 2x$ and evaluated to a real number representing the coefficient of the slope.
At the point $x=3$ the derivative of $f$ is $6$, telling us that for each step we move to the right in the graph of the function, we move $6$ steps up (\cref{fig:derivative_graph}).

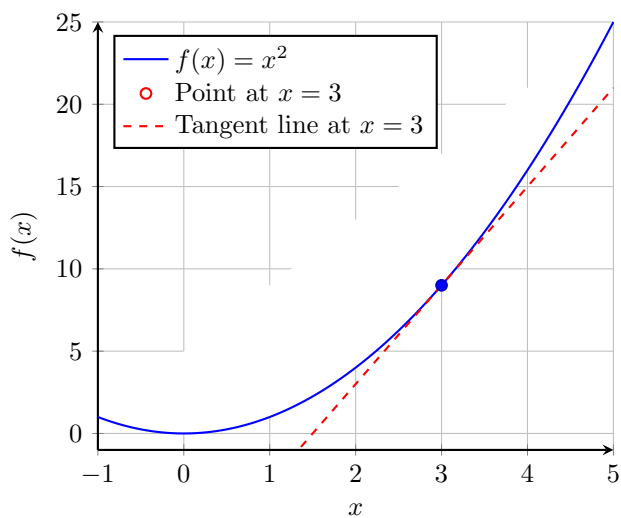
\begin{figure}[h]
  \centering
  \begin{tikzpicture}
    \begin{axis}[
      axis lines=left,
      xlabel=$x$, ylabel=$f(x)$,
      xmin=-1, xmax=5, ymin=-1, ymax=25,
      legend pos=north west,
      legend cell align={left},
      legend style={draw},
      samples=100,
      domain=-1:5,
      thick,
      grid=major,
      ]
      \addplot[blue, domain=-1:5] {x^2}; \addlegendentry{$f(x) = x^2$}
      \addplot[scatter, only marks, mark size=2pt, scatter src=y, red] coordinates {(3, 9)}; \addlegendentry{Point at $x=3$}
      \addplot[dashed, red, domain=1:5] {2*3*(x-3) + 9}; \addlegendentry{Tangent line at $x=3$}
    \end{axis}
  \end{tikzpicture}
  \caption{Function $f(x)=x^2$ and its tangent at the point $x=3$.}
  \label{fig:derivative_graph}
\end{figure}

Of course, we see that the further out we go, the less accurate the derivative is.
This is because the derivative is only a linear approximation, local at a point.
And the fact that we are able to represent it as a number is merely a consequence of the isomorphism between the vector space of real numbers $\R$ and the space of linear maps $\R \multimap \R$ (\cref{box:scs_monoidal}).
In a multivariate case we will have to represent it as a matrix instead.

Take the example of a smooth function $f : \R^3 \to \R^2$ defined as
\[
  f(x_1, x_2, x_3) = (x_1^2  + x_1x_2x_3, x_3^3 + 3x_1)
\]
Its derivative is a linear function of the same type $\R^3 \multimap \R^2$ that best approximates $f$ at a particular point.
It can be represented as a $2 \times 3$ matrix whose elements contain partial derivatives of $f$ with respect to each component of the input.
That is, at the point $(x_1, x_2, x_3) : \R^3$ the derivative of $f$ is

\begin{equation*}
  \begin{bmatrix}
    2x_1 + x_2x_3 & x_1x_3 & x_1x_2 \\
    3 & 0 & 3x_3^2 
  \end{bmatrix}
\end{equation*}

and it describes how each of the two outputs (rows) change as we change any one of the inputs (columns).
This leads us to the well-known definition of \emph{the Jacobian matrix} in the category $\Smooth$ (\cref{def:smooth}) which we recall has a monoidal closed subcategory $\FVectR$ (\cref{def:fvect}) of linear maps.

\begin{proposition}[compare \cite{cockett_reverse_2020}, Example 5.3, Example 14.2]
  Let $f : \R^n \to \R^m$ be a map in $\Smooth$.
  Then the \textbf{Jacobian} of $f$ is a map of type
  $$
  J[f] : \R^n \to (\R^n \multimap \R^m)
  $$
  defined as a matrix of partial derivatives at every point $x : \R^n$:
  \begin{equation*}
    J[f](x) =
    \begin{bmatrix}
      \frac{\partial f_1}{\partial x_1} & \dots & \frac{\partial f_m}{\partial x_1}\\
      \vdots & \ddots & \vdots\\
      \frac{\partial f_1}{\partial x_n} & \dots & \frac{\partial f_m}{\partial x_n}
    \end{bmatrix}
  \end{equation*}
\end{proposition}

The Jacobian tells us how at every point $x : \R^n$ a change in the input causes a change in the output.
However, often we're interested in percolating this information backwards --- computing how a change in the output affects the input.
This can be done by transposing the Jacobian at $x$, obtaining $J[f](x)^\dagger : (\R^m \multimap \R^n)$.
Not to be confused with computing the \emph{inverse} of the Jacobian at $x$, which is something related to Netwon's method, and computationally more expensive.\footnote{This too forms a lens, and was discovered together with Mario Alvarez-Picallo.}
The transposed Jacobian is often called ``the reverse derivative'', while the non-transposed version is called ``the forward derivative'', matching closely the intuition behind reverse-mode and forward-mode AD.

\begin{notation}
  \label{rem:not_forward_reverse_deriv}
  Every Jacobian $J[f] : \mi \to (\mi \multimap \mo)$ and its transposed version --- which we sometimes denote by $J[f]^\dagger : \mi \to (\mo \multimap \mi)$ --- can be additionally be uncurried, and treated as maps linear in their second argument.
  We introduce the following notation for this.\footnote{We also note that $R[\dots]$ notation is also used in \cref{def:poly} for polynomial rings, but it will always be clear from the context and variable names which one is being used.}
  \begin{align*}
   D[f] &: \mi \times \mi \to \mo \quad  \quad &R[f] &: \mi \times \mo \to \mi\\
   D[f]&(x, v) = J[f](x)(v) \quad \quad &R[f]&(x, v) = J[f]^{\dagger}(x)(v)
  \end{align*}
\end{notation}

\section{Jacobians and weighted optics}
\label{sec:jacobians_weighted_optics}

Jacobians are an important concept in differential calculus, and will be the central component of this chapter.
In fact, the direction where we are heading might have already become apparent once the type of $J[f]^\dagger$ and $R[f]$ were specified.
These types precisely match the type of the backward part of lenses, with $f$ matching their forward one!

It might seem like a spurious coincidence, but it turns out that the similarities run deep.
For instance, let's take a look at an important property of Jacobians that relates them to sequential composition of maps in $\Smooth$: the chain rule.
\begin{proposition}[Chain Rule]
  \label{prop:chain_rule}
  Let $f : A \to B$ and $g : B \to C$ be two maps in $\Smooth$.
  Then \newdef{the chain rule} is a property of Jacobians: it describes how Jacobians of $f$ and $g$ are related to the Jacobian of their sequential composite.
  It can be expressed as the commutativity of the diagram below:
  \[\begin{tikzcddiag}[ampersand replacement=\&]
      A \&\&\& {A \multimap C} \\
      \\
      {A \times B} \&\&\& {(A \multimap B) \times (B \multimap C)}
      \arrow["{J[f \comp g]}", from=1-1, to=1-4]
      \arrow["{\multimap_{\comp}}"', from=3-4, to=1-4]
      \arrow["{\graph(f)}"', from=1-1, to=3-1]
      \arrow["{J[f] \times J[g]}"', from=3-1, to=3-4]
    \end{tikzcddiag}\]
  Equationally, it tells us that $J[f \comp g](x) = J[f](x) \multimap_{\comp} J[g](f(x))$.\footnote{ If we consider functions of type $\R \to \R$, then the composition of linear maps reduces to multiplication of their coefficients.
  This yields the familiar representation of chain rule $(f \comp g)'(x) = g'(f(x)) \cdot f'(x)$, where the superscript $'$ here denotes the Jacobian, i.e.\ $f' = J[f]$.}
\end{proposition}
The chain rule bears a striking similarity to the composition rule for closed lenses \cref{def:linear_lens_composition_identity}.
As a matter of the fact, they are exactly the same!

Another similarity arises when we look at the way we assign Jacobians to smooth maps.
Note that the Jacobian is an assignment that is made to \emph{every} morphism in $\Smooth$.
This suggests that it might be a functor.
Which functor is it, and what is its codomain?
The problem is that the type of the Jacobian does not have a straightforward interpretation as a morphism in a category.
But this is only a problem because we assumed the original $f$ has to be forgotten.
Instead, if we keep the forward map $f$ and bundle it together with the Jacobian as the pair $(f, J[f]^\dagger)$, we can see that $f : A \to B$ and $J[f]^\dagger : A \to (B \multimap A)$ resemble the types of forward and backward maps in a closed lens.
This is precisely what is done in automatic differentiation: the forward map $f$ is never forgotten --- instead it is paired with the Jacobian since it is needed to compute the input of $J[g]^\dagger : B \to (C \multimap B)$ when expressing the chain rule.

And lastly, we observe that in most implementations of automatic differentiation\footnote{At least, the ones aiming to minimise redundant computation.} the above formulation of chain rule \emph{is not} the one explicitly used.
This is because the definition of $J[f \comp g]$ explicitly recomputes something that we can memoise: the output of $f$.
Such memoisation will require not just bundling $f$ and $J[f]$ together into a pair of maps, but treating them as a singular unit\footnote{This is something we can do because their domains agree.}: one which given an input $a : A$ it produces two things: the output $f(b) : B$ and also the linear map $J[f]^\dagger(a) : B \multimap A$.
A careful reader might have noticed that this resembles the definition of a closed lens (\cref{def:closed_optic}).
And as we will see, a more efficient formulation of the chain rule is precisely the one that arises from the definition of closed lens composition (\cref{eq:linear_lens_smooth}).\footnote{I've first seen this observed in \cite[Sec. 3.1]{elliott_simple_2018}, though the terminology of lenses was not used.}

The above points motivate the usage of weighted optics as a framework for studying differentiation, as they give us not just a way of modelling differentiation denotationally, but also operationally.
More precisely, a coherent way of assigning Jacobians to every morphism will be modelled a part of the data of the functor
\[
  \cC \to \Optic^{\cC(-, \iotaaccla(-))}_{\diset{\times}{\otimes}}
\]
for a particular kind of a category $\cC$: an \emph{additively closed cartesian left-additive category}, a novel construction we define in this thesis.

\begin{definition}[Additively closed cartesian left-additive category]
  \label{def:additively_closed_cla}
  Let $\cC$ be a cartesian left-additive category.
  We call it an \newdef{additively closed cartesian left-additive category (ACCLA)} if
  \begin{itemize}
  \item Its subcategory of additive maps $\asc{\cC}$ has a monoidal product $(\asc{\cC}, \otimes, I)$ with respect to which it is closed;
  \item The embedding $\iotaaccla : \asc{\cC} \to \cC$ is lax monoidal with respect to the monoidal structure $(\otimes, I)$ of $\asc{\cC}$ and $(\times, 1)$ of $\cC$.
  \end{itemize}
\end{definition}

This is a construction that takes advantage of the fact that in all the settings of interest to us ($\Smooth$ and $\Poly_R$) additive and linear maps coincide.\footnote{In general, linear maps are special cases of additive maps. For an example where they do not coincide, see \cite[p. 541]{blute_cartesian_2009}. It is not clear to me what the implications of this counterexample are.}
This is why it can also reasonably be called a \textbf{linearly closed cartesian left-linear category}.\footnote{Despite the unfortunate fact that LCCLA is a less punchy acronym.}
It enables us to compartmentalise many moving pieces of derivatives, as well as provide us with an operational perspective on their composition.

\begin{remark}
  \label{rem:iota_strong_monoidal}
  The above definition implies a few things. 
  The category $\asc{\cC}$ is always a cartesian left-additive subcategory of $\cC$.
  The existence of lax monoidal structure of $\iotaaccla$ and self-enrichment of $\asc{\cC}$ via the enriched base change (\cref{def:enriched_base_change}) implies enrichment of $\asc{\cC}$ in $\cC$.
  Additionally, the functor $\iotaaccla$ is --- in addition to being lax monoidal with respect to $\otimes$ --- always strong monoidal with respect to the product structure of $\asc{\cC}$.
\end{remark}

In this kind of a category, we can exhibit a generalised kind of a tensor-hom adjunction which captures the idea of a property of a morphism being valid in only one variable (see \cref{sec:optics_backward_closed_actegory} for more details and intuition).

\begin{restatable}{proposition}{GeneralisedTensorHom}
  \label{prop:generalised_tensor_hom}
  Let $\cC$ be an $\accla$.
  Then for every $X, Y', X' : \cC$ the following isomorphism holds:
  \[
    \cC(X, \iotaaccla(\internalHom{\asc{\cC}}{Y'}{X'})) \cong \asc{\CoKl(X \times -)}(Y', X')
  \]
\end{restatable}

This is a proof that follows by routine, but formally tedious calculation.

\begin{example}
The category $\Smooth$ is an additively closed cartesian left-additive category: its subcategory of additive maps $\asc{\Smooth}$ coincides with $\FVectR$ which is monoidal closed, and the embedding is a lax monoidal functor (\cref{ex:iota_lax_monoidal}).
\end{example}

\begin{example}
The category $\Poly_R$ is additively closed, as its subcategory of additive maps also coincides with $\FVect_R$ which is monoidal closed.
\end{example}

Now let's unpack the contents of the category of weighted optics $\Optic^{\cC(-, \iotaaccla(-))}_{\diset{\times}{\otimes}}$.
Its forward action is the self-action of the cartesian structure of $\cC$, backward action is the self-action of the monoidal closed structure of $\asc{\cC}$, and the weight is the composite
\[
  \cC^{\op} \times \asc{\cC} \xrightarrow{\cC^{\op} \times \iotaaccla} \cC^{\op} \times \cC \xrightarrow{\Hom_{\cC}} \Set
\]
Concretely, a weighted optic $\diset{\mi}{\mi '} \to \diset{\mo}{\mo '}$ here consists of

\begin{itemize}
\item A coparameter $\diset{M}{M}$ where $M : \cC$ and $M' : \asc{\cC}$;
\item A map $s : \cC(M, \iotaaccla(M'))$;
\item A forward map $f : \mi \to M \times \mo$ in $\cC$;
\item A backward map $f' : M' \otimes \mo ' \to \mi '$ in $\asc{\cC}$
\end{itemize}
quotiented out by diagrams in \cref{eq:copara_optic_diagram}.
Despite its modest presentation, this construction captures many of the things we care about in differentiation, including its operational characteristics.

The reductions described in \cref{sec:optics_for_cartesian_act_fw} can be applied here, and we can unpack them explicitly by chooosing appropriate types of residuals.
Set $M  = \mi$ and $M' = \internalHom{\asc{\cC}}{\mo '}{\mi '}$.
This reduces the type of the underlying maps to $\cC(\mi, \mi \times \mo)$, $\cC(\mi, \iotaaccla(\internalHom{\asc{\cC}}{\mo '}{\mi '}))$ and $\asc{\cC}(\internalHom{\asc{\cC}}{\mo '}{\mi '} \otimes \mo ', \mi')$.
These are exactly the types of $\graph(f)$, $J[f]$ and $\eval$ (\cref{fig:optic_jacobian})! 

\begin{figure}[h]
  \scaletikzfig[0.8]{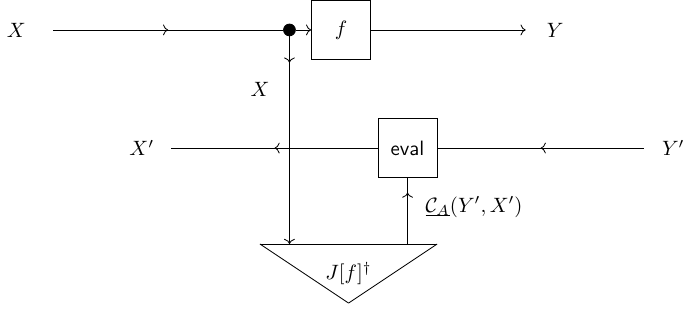}
  \caption{An optic implementing reverse mode AD.}
  \label{fig:optic_jacobian}
\end{figure}

What do these maps do?
The forward pass computes the output of $f$ while at the same time saving the input $\mi$ to be used on the backward pass.
Via the Jacobian this $\mi$ is transformed into a linear map $\mo ' \multimap \mi '$ which is in conjuction with a $\mo '$ evaluated into an $\mi '$.
This looks like an implementation of reverse mode AD, and we will prove in \cref{sec:backprop_as_functor} that this is exactly the case.

Alternatively, we can abstract away the $\eval$ and look at a different concrete representative given by \cref{corr:optics_forward_cart_backward_closed_concrete}:
\begin{equation}
  \label{eq:concrete_additive_optic}
  \Optic^{\cC(-, \iotaaccla(-))}_{\diset{\times}{\otimes}}\OpticHom{\mi}{\mi '}{\mo}{\mo '} \cong \cC(\mi, \mo) \times \cC(\mi, \iotaaccla(\internalHom{\asc{\cC}}{\mo '}{\mi '}))
\end{equation}

yielding the types of $f$ and $J[f]^\dagger$.
By applying \cref{prop:generalised_tensor_hom} we can provide another representation

\begin{equation}
  \label{eq:concrete_additive_optic2}
  \Optic^{\cC(-, \iotaaccla(-))}_{\diset{\times}{\otimes}}\OpticHom{\mi}{\mi '}{\mo}{\mo '} \cong \cC(\mi, \mo) \times \asc{\CoKl(\mi \times -)}(\mo ', \mi ')
\end{equation}

which has concrete morphisms of the form $f : \cC(X, Y)$ and $f' : \cC(X \times Y', X')$ with the additional property that $f'$ is additive in the second component (\cref{def:additive_second_variable}).
This is precisely the uncurrying of $J[f]^\dagger$ described in \cref{rem:not_forward_reverse_deriv}.
Similarities do not stop here.
While the lens representation defines the implementation of backpropagation that performs gradient checkpointing\footnote{Gradient checkpointing being a method of computing derivatives mentioned in the introduction of this chapter.}, it can be shown that their composition in the optic representation computes it \emph{without} gradient checkpointing.
This suggest that optics do not merely describe reverse mode AD, but give us a prescriptive recipe for implementing it.\footnote{See \cite{gavranovic_space-time_2022} for more information.}

In general, additivity as we have defined it is crucial, as it enables optics to have products.

\begin{proposition}
  \label{prop:lens_a_cartesian}
  The category $\Optic^{\cC(-, \iotaaccla(-))}_{\diset{\times}{\otimes}}$ has products defined as
  \[
    \diset{A}{A'} \times \diset{B}{B'} = \diset{A \times B}{A' \oplus B'}
  \]
  where $\oplus$ is here suggestive notation for the biproduct structure of $\asc{\cC}$ (\cref{prop:linsc_biproduct}).
\end{proposition}

\begin{proof}
  We prove the universal property of the product (\cref{eq:product}).
  \begin{align*}
    &\Optic^{\cC(-, \iotaaccla(-))}_{\diset{\times}{\otimes}}\OpticHom{X}{X'}{A}{A'} \times \Optic^{\cC(-, \iotaaccla(-))}_{\diset{\times}{\otimes}}\OpticHom{X}{X'}{B}{B'}\\
    \text{(Def.)} \enskip =& \cC(X, A) \times \cC(X, \iotaaccla(\internalHom{\asc{\cC}}{A'}{X'})) \times \cC(X, B) \times \cC(X, \iotaaccla(\internalHom{\asc{\cC}}{B'}{X'}))\\
    \text{(Univ. prop. of product)} \enskip \cong & \cC(X, A \times B) \times \cC(X, \iotaaccla(\internalHom{\asc{\cC}}{A'}{X'}) \times \iotaaccla(\internalHom{\asc{\cC}}{B'}{X'})) \\
    \text{($\iotaaccla$ is strong monoidal)} \enskip \cong & \cC(X, A \times B) \times \cC(X, \iotaaccla(\internalHom{\asc{\cC}}{A'}{X'} \oplus \internalHom{\asc{\cC}}{B'}{X'}))\\
    \text{(Univ. prop. of coproduct)} \enskip \cong & \cC(X, A \times B) \times \cC(X, \iotaaccla(\internalHom{\asc{\cC}}{A' \oplus B'}{X'}))\\
    \text{(Def.)}  \enskip =&\Optic^{\cC(-, \iotaaccla(-))}_{\diset{\times}{\otimes}}\OpticHom{X}{X'}{A \times B}{A' \oplus B'} \tag*{\qedhere}
  \end{align*}
\end{proof}

These similarities between derivatives and optics lead us to stating the following conjecture:
\begin{center}
  \textbf{Weighted optics provide good denotational and operational semantics for differentiation.}  
\end{center}

We say ``conjecture'' because this is not something we fully prove in this thesis.
This is because we deal exlusively with a specific representation of optics that prohibits modelling non-checkpointed automatic differentiation.
In other words, to \emph{formally prove}, given the time constraints in writing this thesis, a correspondence between differentiation and optics, we use the concrete representation of optics in \cref{eq:concrete_additive_optic2} as lenses with backward pass \emph{additive in the second component}.
This means that, instead of using optic composition and identities, we restricted to use the defined lens composition and identities in \cref{def:lens_composition_identity,def:linear_lens_composition_identity}.
The reason we do this is because this allows us to reduce the underlying structure to just a cartesian left-additive category (note that the right-hand side of \cref{eq:concrete_additive_optic2} does not use any closed structure of $\asc{\cC}$), and thus greatly simplify the proofs in the rest of this chapter.\footnote{This is a peculiar thing, because it implies that non-checkpointed reverse mode AD necessitates the equipment of $\cC$ with an additively closed subcategory. This is something we hope to explore in future work.}

Despite its reduced expresiveness, this is still a fully formal statement which can be made more general in a systematic way.
We introduce the shorthand notation for this construction:

\begin{equation}
  \label{eq:lensa}
  \LensA(\cC) \coloneqq \Optic^{\cC(-, \iotaaccla(-))}_{\diset{\times}{\otimes}}
\end{equation}

In light of this, the formal statement of the main theorems of this chapter is:

\begin{center}
  The construction $\LensA$ is a functor (\cref{thm:lens_functor_clac}), and generalised reverse differential structure is its coalgebra (\cref{def:gcrdc}).
\end{center}

As we will see in \cref{thm:coalgebra_rdc}, this compact definition of a coalgebra will specify concepts such as the chain rule, additivity of the backward pass, behaviour with respect to products and summing, and so on.

\section{Backpropagation as a $\LensA$-coalgebra}
\label{sec:backprop_as_functor}

We first aim to prove the following.

\begin{theorem}
  \label{thm:lens_functor_clac}
  Lenses with backward passes additive in the second component form a functor
  \[
    \LensA : \CLACat \to \CLACat
  \]
\end{theorem}

In order to do that, we need to precisely specify the category of cartesian left-additive categores. 

\begin{definition}
  \label{def:clacat}
  We call $\CLACat$ the category whose objects are cartesian left-additive categories and whose morphisms are cartesian left-additive functors (functors which preserve products and commutative monoid structure on objects (\cite[Definition 1.3.1]{blute_cartesian_2009})).
\end{definition}

Now we can prove $\LensA$ is a functor.
Recall that it takes in a cartesian left-additive category $\cC$ and produces the category $\LensA(\cC)$.
We have shown it is cartesian, and we now proceed to show it is left-additive as well.

\begin{proposition}
  \label{prop:lens_cart_left_additive}
  The category $\LensA(\cC)$ is cartesian left-additive.
\end{proposition}

\begin{proof}
  We need to equip $\LensA(\cC)$ with a commutative monoid on every object in a way that's compatible with the cartesian structure.\footnote{We don't need to show that this monoid is unique, just that it exists and can be canonically defined.}
  That is, for every object $\diset{A}{A'}$ we need to provide two morphisms:
\begin{itemize}
\item \textbf{Unit } $0_{\diset{A}{A'}} : \diset{1}{1} \to \diset{A}{A'}$. This is a lens whose forward map we set as the zero map $0_A$ and the backward map as the delete $!_{1 \times A'}$.
\item \textbf{Multiplication} $+_{\diset{A}{A'}} : \diset{A \times A}{A' \times A'} \to \diset{A}{A'}$. This is a lens whose forward map we set as sum $+_A$ and the backward map as copy, i.e.\ $\pi_2 \comp \Delta_{A'}$.
\end{itemize}
Additionally, these morphisms need to obey the monoid laws.
This can be verified by routine.
\end{proof}

This defines the action of $\LensA$ on objects of $\CLACat$.
Action on morphisms is defined below.

\begin{restatable}{proposition}{LensFunctorCartLeftAdd}
  \label{prop:lens_functor_cart_left_add}
  Let $F : \cC \to \cD$ be a cartesian left-additive functor.
  This induces a cartesian left-additive functor $\LensA(F)$ between the corresponding categories of lenses: 
  
  \begin{equation}
    \label{eq:lensa_action_on_morphisms}
    \begin{tikzcddiag}[ampersand replacement=\&]
        {\LensA(\cC)} \&\&\& {\LensA(\cD)} \\
        {\diset{A}{A'}} \&\&\& {\diset{F(A)}{F(A')}} \\
        \\
        {\diset{B}{B'}} \&\&\& {\diset{F(B)}{F(B')}}
        \arrow["{ \LensA(F)}", from=1-1, to=1-4]
        \arrow["{\diset{f}{f'}}"', from=2-1, to=4-1]
        \arrow["{\diset{F(f)}{\overline{f'}}}", from=2-4, to=4-4]
        \arrow[maps to, from=2-1, to=2-4]
        \arrow[maps to, from=4-1, to=4-4]
      \end{tikzcddiag}
  \end{equation}
  where $\overline{f'} \coloneqq F(A') \times F(B') \cong F(A' \times B') \xrightarrow{F(f')} F(A')$. 
\end{restatable}

\begin{proof}
  The above proposition defines the action on objects and morphisms.
  What remains to prove is that identities and composition is preserved, as well as that this functor is cartesian left-additive.
  We defer this to the Appendix (\cref{appendix:cart_left_additive}).
\end{proof}

What remains to show is that $\LensA$ preserves identities and composition, which follows routinely, concluding the proof of \cref{thm:lens_functor_clac}.

This functor has additional structure --- it is copointed.\footnote{Despite this the functor $\LensA$ does not have the comonad structure, for similar reasons that tangent categories do not.}

\begin{proposition}[{Copointed structure of $\LensA$}]
  \label{prop:lensa_copointed}
  There is a natural transformation
  \[
    \begin{tikzcddiag}
        \CLACat && \CLACat
        \arrow[""{name=0, anchor=center, inner sep=0}, "{\LensA}", curve={height=-12pt}, from=1-1, to=1-3]
        \arrow[""{name=1, anchor=center, inner sep=0}, "\id"', curve={height=12pt}, from=1-1, to=1-3]
        \arrow["\epsilon", shorten <=3pt, shorten >=3pt, Rightarrow, from=0, to=1]
      \end{tikzcddiag}
  \]
    assigning to cartesian left-additive category $\cC$ a cartesian left-additive functor $\epsilon_{\cC} : \LensA(\cC) \to \cC$ which forgets the backward part of the lens.
\end{proposition}

This brings us to the main definition.
Its name is inspired by the existing categorical framework for differentiation called \emph{reverse derivative categories} (RDCs) (\cite{cockett_reverse_2020}).
We remark more on it in \cref{thm:coalgebra_rdc}, \cref{rem:crdc_independent} and the literature review (\cref{sec:backprop_literature}).

\begin{definition}
  \label{def:gcrdc}
  A \newdef{generalised cartesian reverse derivative category} is a coalgebra of the copointed $\LensA$ endofunctor.
\end{definition}

This definition packs a lot of punch.
Explicitly, it consists of a choice of a cartesian left-additive category $\cC$ and a cartesian left-additive functor $\RC : \cC \to \LensA(\cC)$ such that the following diagram commutes

\begin{equation}
  \label{eq:projection}
  \begin{tikzcddiag}
	{\LensA(\cC)} && \cC \\
	  \\
	  \cC
	  \arrow["{\epsilon_{\cC}}", from=1-1, to=1-3]
	  \arrow[Rightarrow, no head, from=3-1, to=1-3]
	  \arrow["\RC", from=3-1, to=1-1]
  \end{tikzcddiag}
\end{equation}

A good way to understand it is through the aforementioned framework of reverse derivative categories, for whose axioms it serves as a justification.

\begin{theorem}
  \label{thm:coalgebra_rdc}
  Every cartesian reverse derivative category (\cref{def:crdc}) is a \emph{generalised} cartesian reverse derivative category.
\end{theorem}

\begin{proof}
  As we're already starting with a cartesian left-additive category $\cC$, what remains to prove is that there is reverse derivative combinator (\cite[Definition 13]{cockett_reverse_2020}) for which the first five axioms of RDC's are satisfied.
  The reverse derivative combinator is precisely the data of the functor $\RC : \cC \to \LensA(\cC)$, where \cref{eq:projection} necessitates that the only choice involved in this functor is one of the backward pass. 
  When it comes to the axioms, we show how they arise one by one:
\begin{enumerate}
\item \textbf{Additivity of reverse differentiation.} This is recovered by $\RC$ preserving left-additive structure.
\item \textbf{Additivity of reverse derivative in the second component.} This is recovered by definition of $\LensA$ --- the backward maps are additive in the 2\textsuperscript{nd} component.
\item \textbf{Coherence with identities and projections.} Coherence with
  identities is recovered by preservation of identities of the functor $\RC$, where for every $X : \cC$, $\RC(\id_X) = \id_{\RC(X)} = (\id_X, \pi_2 : X \times X \to X)$. Coherence with projections is recovered by $\RC$ preserving cartesian structure.
\item \textbf{Coherence with pairings.} Recovered by $\RC$ preserving cartesian structure.
\item \textbf{Reverse chain rule.} This is recovered by functoriality of $\RC$.
\end{enumerate}
\end{proof}

\begin{remark}
  \label{rem:crdc_independent}
  A cartesian reverse derivative category has two additional axioms --- the 6th and 7th axiom --- which are independent from the others, and not captured with our coalgebraic construction \cite[Sec. 2.3]{cockett_differential_2014}.
  As we will see in the rest of this thesis, we will not need these two axioms to compositionally model supervised learning.\footnote{Though they might become important in future work where we hope to prove properties of these learning systems.}
\end{remark}

\begin{example}[{$\Smooth$, compare \cite[Ex. 14.2]{cockett_reverse_2020}}]
  \label{ex:smooth_grdc}
The cartesian left-additive category $\Smooth$ is an example of a generalised cartesian reverse derivative category, with $\RC(f) \coloneqq \diset{f}{R[f]}$.
\end{example}

\begin{example}[{$\Poly_R$, compare \cite[Ex. 14.1]{cockett_reverse_2020}}]
  \label{ex:polyr_grdc}
  The cartesian left-additive category $\Poly_R$ is an example of a generalised cartesian reverse derivative category, with $\RC(f) \coloneqq \diset{f}{R[f]}$.
\end{example}

\begin{mybox}[label=box:backpropagate_bits]{ForestGreen}{What does it mean to backpropagate bits?}
  In categories like $\Smooth$ or $\PolyR$, given a map $f : X \to Y$ the input $y' : Y$ to its reverse derivative $R[f](x) : Y' \to X'$ at a point $x : X$ has a straightforward interpretation: we think of it as a vector which denotes the direction of steepest descent at $f(x)$.\footnote{Even more preciely, it consumes a \emph{covector}. See \cref{sec:backprop_literature}.} 
  The map $R[f](x)$ then backpropagates this information to $X'$, computing the direction of steepest descent at $x$.
  But how does this work in $\PolyZ$?
  That is, given some polynomial $f : \Z^n_2 \to \Z^m_2$, what is the ``steepest-descent'' interpretation of the input of $R[f][x] : \Z^m_Z \to \Z^n_2$ for some input $x$? The answer is as follows.

  Here the gradient vector $y' : \Z^m_2$ is \emph{bit-valued}, and the information it communicates is whether flipping each one of bits $f(x)_i$ would have yielded a better outcome.
  The map $R[f](x)$ backpropagates this information to $\Z^n_2$, computing whether each of the bits $x_i$ should be flipped.
\end{mybox}

\begin{example}[{coKleisli category (compare \cite[Theorem 2]{gavranovic_graph_2022})}]
  \label{ex:cokleisli_of_grdc}
  If $\cC$ is a generalised cartesian reverse derivative category, then $\CoKl(A \times -)$ is too, for every $A : \cC$.
\end{example}

This concludes the main aspect of this section --- showing us how to categorically model differentiation using only the framework of lenses.
We will see later how this functor will prove important for differentiating parametric morphisms in a compositional way. 

\subsection{Operational concerns}
\label{subsec:operational}

In this short subsection we mention the distinction between denotational and operational semantics of optics.
Denotational ones we have already described --- they are the ones arising our the 1-categorical framework.
In this setting cartesian optics are isomorphic to lenses (\cref{eq:lens_representation}).
This allows us to move between their representatives as we wish, as they all compute the same output.
However, implementing these in code necessitates a choice of representation.
And different choices induce programs with different performance profiles.

This is one of the issues of implementing the chain rule naively: thinking of it as a dynamic programming task, it is one with \emph{overlapping subproblems}.
In other words, with sufficient memory, computing the derivative of a composite of functions can take advantage of already computed solutions.
There is a systematic way to interpret optics as a bicategory where instead of computing a quotient and obtaining hom-sets we instead think of the hom-objects in their natural form as categories.
In such a setting lenses embody a \emph{different} composition rule than optics, one with a different \emph{space-time tradeoff}.
This tradeoff precisely captures distinction between code that is gradient checkpointed versus one that is not.
We hypothesise that such a 2-categorical framework can provide formal grounds for informing various rewrite rules and compiler optimisations.
See \cite{gavranovic_space-time_2022} for more details.

\section{Backpropagation via category theory in the literature}
\label{sec:backprop_literature}

Differentiation comes in many shapes and forms, and so do its categorical frameworks.

One of the biggest clusters is centered around differential categories (CDCs) (\cite{blute_cartesian_2009}), a construction admitting numerous generalisations and extensions (\cite{cockett_reverse_2020,cockett_faa_2011,cruttwell_monoidal_2022,cockett_differential_2014}).
We note that the second reference from this list formulates CDCs as a coalgebra of a particular functor, much like we do in this thesis.
Unlike us, they capture all 7 axioms of CDC! Contrary to us, they do not establish any connections reverse derivatives, nor the optical literature, something we deem is crucial if one wants to obtain a refined operational specification of these constructions, and generalise them to a wide array of bidirectional processes, from bayesian learning to value iteration.
Reverse derivative categories have also been generalised to the purely monoidal setting in \cite{cruttwell_monoidal_2022}.
We conjecture this construction fits into the general story with monoidal weighted optics, as optics do not assume any cartesianness.
We also believe the optical framework in general serves as a powerful justification the axioms of reverse derivative categories.
It removes the need for a monolithic list of axioms, instead compartmentalising them into more conceptual pieces (see the proof of \cref{thm:coalgebra_rdc}).

Differential categories are examples of tangent categories (\cite{cockett_differential_2014,garner_embedding_2018}), a framework that additionally captures manifolds and various other dependently-typed constructs, outside of reach of differential categories which are non-dependent.
Tangent categories spawned a lot of interesting research, especially related to exponentials and curve objects (\cite{cockett_differential_2021}), concepts useful for studying differential equations in tangent categories.
We conjecture an analogous characterisation of tangent categories can be done in the language of dependent lenses, where $\LensA$ would be replaced by $\DLensA$, the category of dependent lenses additive in the second component of their backward pass.
Understanding operational aspects of these dependent lenses through their optical characterisation is part of ongoing research on dependent optics (\cref{subsec:dependent_optics}).

On the other hand, there is a lot of work in many ways disjoint from cartesian differential and tangent categories that studies more practical aspects of differentiation.
As briefly mentioned in the introduction, the work of (\cite{alvarez-picallo_functorial_2021}) proves the soundness of the original description of AD (\cite{pearlmutter_reverse-mode_2008}) using the semantic category of reverse derivative categories.
It's worthy to mention that they also use a graphical language, but not the one arising from the graphical language of lenses and optics themselves.
Unifying these two graphical languages is something we hope to explore in the future.
The work \cite{elliott_simple_2018} presents an elegant, purely functional formulation of both forward and reverse mode automatic differentiation that is homomorphic with respect to a collection of standard categorical abstractions --- like those described in \cref{sec:backprop_as_functor}, but applicable to a language with only limited features, such as first-order functions.
Combinatory Homomorphic Automatic Differentiation (CHAD) (\cite{vakar_chad_2022}, and also its predecessor \cite{vakar_reverse_2021}) takes this work further and introduced a compositional, type-respecting method that performs source code translation and generates purely functional code, extending it to languages with higher-order functions.
We believe this bears a close resemblance to the framework presented here, especially with respect to usage of the linear subcategories of smooth maps.
We note that like us, they also do not distinguish between linear and additive maps.
While they provide \emph{a} categorical framework, they do not axiomatise it, or show how it is related to existing axiomatic frameworks such as tangent categories.

A more precise way to characterise backpropagation of gradient vectors is by framing it as the pullback of gradient \emph{covectors}.
Covectors of some vector space $V$ are vectors in the dual vector space of linear functionals, i.e.\ \emph{valuations} on vectors, represented as linear maps into the base field.
This appears amenable for capturing with the putative folkoric idea of \emph{reverse categories}, yet to be defined.\footnote{They have been defined after this thesis was submitted for examination; see \cite{cruttwell_reverse_2023}.}
Fragments of this idea can already be seen in \cite{elliott_simple_2018} where it is noted that backpropagation can be studied as the image of the representable functor $\FVectR(-, \R) : \FVectR^{\op} \to \FVectR$ which flips the direction from forward-mode AD to backward-mode AD.
Despite this, the underlying covectors as still treated as mere vectors as a result of the isomorphism $V \cong \internalHom{\FVectR}{V}{\R}$ in the finite-dimensional case. 
Covectors are more explicitly acknowledged in \cite{dalrymple_dioptics_2019}.
There the author formulates reverse-mode AD as a functor from trivialisable diffeological spaces (\cite{laubinger_diffeological_2006}) into optics whose objects are of the form $\diset{X}{\internalHom{\FVectR}{X'}{\R}}$.
This is a functor that is symmetric monoidal, and the author conjectures it can be made lax, additionally capturing counterfactual reasoning in game theory.
This is subsequently built upon by \cite{capucci_diegetic_2023} where it is shown that gradient-based learning can be thought of as "infinitesimal counterfactual reasoning" in a way that almost completely overlaps with feedback propagation in open games \cite{ghani_compositional_2018}.\footnote{On the other hand, this assumption completely breaks down compositionality of players! See \cite[Remark 2.1]{capucci_diegetic_2023}. Understanding the implications of this fact is something we hope to explore in the future.}
This suggests that a unified framework of a general notion of differentiation can capture a variety of settings, including even those of bayesian reasoners.

Interesting work has been done related to formalising differentiation in a setting of incremental computation under the name of \emph{change actions} \cite{alvarez-picallo_change_2020}.
For completeness we also mention that automatic differentiation has been characterised also as a fold \cite{nguyen_folding_2022}.
And lastly, the $\LensA$-coalgebra framework can be seen as an example of a dynamical systems doctrine (\cite{myers_categorical_2022}), establishing connections to general systems theories.

	\chapter{Backprop through Structure}
\label{ch:architectures}

\epigraph{When the limestone of imperative programming has worn away, the granite of functional programming will be revealed underneath.}{Simon Peyton Jones}

\newthought{Throughout the last few decades} the number of neural network architectures --- structured ways of constructing their forward passes --- has proliferated.
Today neural networks can be ``feedforward'', ``recurrent'' (\cite{schmidt_recurrent_2019}), ``convolutional'' (\cite{bengio_convolutional_1995}), ``residual'' (\cite{he_deep_2016}),  ``graph'' (\cite{wu_comprehensive_2021}), ``topological'' (\cite{papillon_architectures_2023}), ``generative-adversarial'' (\cite{goodfellow_generative_2014}), ``autoregressive'' (\cite{gregor_deep_2014}), ``autoencoding'' (\cite{bank_autoencoders_2021}), and even hybrid, i.e.\ composed out of more than one type of architecture at the same time (\cite{rombach_high-resolution_2022,ramesh_hierarchical_2022,openai_introducing_2022}).
The number of architectures is increasing, and many of them are becoming respective subfields of deep learning with sizeable communities of researchers.

Despite this explosive growth, this zoo of architectures is tied together by very few unifying principles.
Most of these architectures come with their own theoretical underpinning, best practices, evaluation measures and often different ways of thinking about the field as a whole.
They require different types of underlying datasets, possess different expressive power, and exhibit a wide range of inductive biases.
Although many of these architectures are special cases of each other, the current fast-paced nature of deep learning research made it hard to translate results and theorems across them.
This is problematic for both beginners trying to learn the sheer volume of the underlying concepts, but also for experts trying to get a sense of the research landscape.
Many ideas repeat across subfields, and it is becoming easier than ever to reinvent concepts.
Even if some individual components of neural network architectures are well-understood, their combinations and composition are not.

Lastly, all neural networks architectures are eventually implemented as programs.
But despite this, the research on construction and design of architectures has been largely disjoint from the research on construction and design of programs. %
The advances from type theory and functional programming such as dependent (\cite{mckinna_why_2006}) and quantitative (\cite{atkey_syntax_2018}) types, constructive reasoning and type-driven development (\cite{brady_type-driven_2016}) have generally not percolated to mainstream deep learning.
While we think about and denote programs using a rich and expressive typed language, when it comes to neural networks --- despite the advent of architectures which deal with structure --- the main mode of reasoning still involves thinking about arrays of numbers.
One consequence of this is that there is no type theory for neural network architectures.
Architectures are instead specified using a mixture of mathematical notation and natural language, often using ambiguous notation (\cite{chiang_named_2021}), and no clear specification of underlying types.
Unlike with typed programs --- which can be verified to be well-typed before running them --- there are no such guarantees when it comes to architectures.
There is no formal way to verify an implementation of architecture is correct --- because there are no formal specifications of most architectures.

The deep learning community is aware of many of these issues.
Some researchers have begun to organise workshops focused on compositionality of deep learning components (\cite{turek_context_2019, marcus_challenge_2022}).
Others have called for a deep learning ``Langlands programme'' (\cite{rieck_machine_2020}) or a deep learning ``Erlangen programme'' (\cite{bronstein_geometric_2021}), both inspired by far-reaching mathematical programs aimed at unifying seemingly disparate parts of mathematics, 
A completely different approach was suggested by \cite{olah_neural_2015}, emphasising the relationship between many kinds of neural network architectures and concepts in functional programming such as folds, unfolds, and recursion schemes. %

This resulted in some progress, most notably \emph{Geometric Deep Learning} (\cite{bronstein_geometric_2021}), an approach focusing on geometric structures of neural networks and symmetries within, using graph neural networks as a foundation.
Another one is \emph{Topological Deep Learning} (\cite{papillon_architectures_2023}), an approach focused on generalising the pairwise relational structures of graphs to those of higher arity.
Geometric and topological deep learning made tremendous progress in unifying neural networks based on graphs, and represent the state of the art in this regard.
But they answer only some of the concerns above.
Their thorough reliance on the algebraic structure of \emph{groups} comes built in with an assumption of invertibility which is an assumption that does not hold in general for programs.

When it comes to unification between the way we think about neural networks and the way we think about programs --- in terms of types, data structures, and often recursion --- the search for the underlying granite foundation of deep learning architectures is still in its infancy.

Acknowledged in \cite{bronstein_geometric_2021} and \cite{olah_neural_2015}, category theory is a natural continuation of these efforts.
As seen throughout this thesis, it is a \emph{language of structure} that has already described various phenomena throughout the sciences.
As we will see in the rest of this chapter, in recent years category theory has started to be used to describe neural network architectures.\footnote{Despite the increase in research, when compared for the size of the field of deep learning the total volume is negligible.}

Category theory is also foundation for many structural components of programming languages we use to \emph{implement} the neural networks we are trying to understand.
Functional features such as \emph{reduce} in Python is an example of a concept of a \emph{fold}, i.e.\ of a catamorphism in category theory, while others, such as \emph{map} can often be seen as examples of actions of functors on morphisms. 
This aligns with our goal to be \emph{operationally aware}, and gives credence to hypotheses in \cite{olah_neural_2015}, suggesting implementation of many architectures in practice can potentially be conceptually simplified.

Nonetheless, in this chapter we provide a comprehensive survey of existing developments, and provide a useful vantage point: we recast all of them through the unified framework of parametric lenses.

 document the existing developments, define a number of architectures, and provide a vantage point that we believe will be useful in the long-term.
In such a fast paced field, aim this to be a foundation that will not be outdated in a few years, but will continue to provide a stable base upon which we can constrwe hope to write content that will not be outdated 

As opposed to doing a literature review at the end of this chapter, we opt for discussing avenues for explorations at the end of each section.

\begin{contributions}
  The novel research in this chapter consists of the formalisation of a) weight tying (\cref{sec:networks_with_shared_context}), b) graph convolutional neural networks (\cref{def:prim_gcnn_layer}), c) encoding and decoding neural networks (\cref{fig:para_encoding_decoding}), and d) generative adversarial networks (\cref{def:gan}), e) and e) Transformers (\cref{sec:transformers}).
  Loss functions (\cref{sec:loss_functions}) are also formalised, though the novelty here boils down to an assignment of types to their inputs and outputs.
  Other than concrete definitions, the novel contribution of this chapter is a survey and a unified framing of numerous category-theoretic papers on neural network architectures scattered around the literature.
\end{contributions}

\begin{epistemicstatus}
  This chapter is the one I am the most confident in when it comes to utility of category theory and its constructions.
  This is because there currently exists a schism between the theory of architectures of deep learning and the classical theory of computer science, and almost no work combining algebraic constructs in traditional computer science theory such as automata, (co)recursion or dynamic programming to the fuzzy, differentiable landscape of deep learning.
  Many of the contents of this chapter (and especially \cref{subsec:recursive_encoding_decoding}) will be useful for defining a generalised theory of equivariance that will allow us to provide formal guarantees of the behaviour of neural networks with respect to arbitrary algebraic operations, and in general many algorithms found throughout computer science.
  I envision this in turn having a downstream effect on regulation (see \cref{ch:to_boldly_go}), and therefore issues of AI safety and bias.
\end{epistemicstatus}

\section{Layers and how to compose them}

Before describing some of the well-known neural network layers, we start humbly with some of the simplest ones: those without any parameters.

\subsection{Layers with no parameters}
\label{subsec:layers_no_parameters}

The layers without parameters described in this section are pervasively used throughout practice and applications.
At the same time, they mostly go unnoticed and unacknowledged in any formal descriptions.
These are various gadgets that copy, sum, delete or create information.
We will see throughout the rest of this chapter how they are central components of weight tying (\cref{def:weight_tying}), or generative adversarial networks (\cref{def:gan}), for instance.
We start by acknowledging that most of these have one additional property: they have trivial residuals, i.e.\ they are adapters (\ref{fig:residual_trivial}).

\begin{figure}[H]
  \scaletikzfig[0.8]{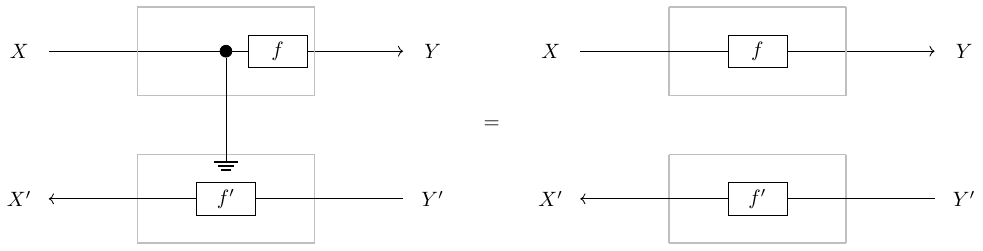}
  \caption{The backward part (the gray box at the bottom) of many lenses can be factored through the delete map. This means that each of these has an equivalent optic representation whose residual is the terminal object $1$.
  In turn, this means that these lenses are adapters (\cref{ex:adapters}).}
  \label{fig:residual_trivial}
\end{figure}

\begin{minipage}{0.4\textwidth}
  \begin{example}[Identity]
    \label{ex:identity_derivative}
    There is a trivial neural network that doesn't change the input: it is the identity map $\id_X : X \to X$.
    Its reverse derivative is the projection $R[\id_X](x, \alpha) = \alpha$.
  \end{example}
\end{minipage}
\begin{minipage}{0.6\textwidth}
  \begin{figure}[H]
    \scaletikzfig[0.8]{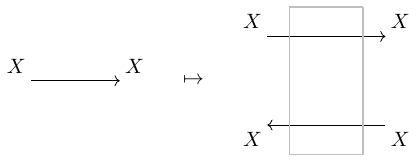}
    \caption{Identity and its derivative.}
    \label{fig:identity_derivative}
  \end{figure}
\end{minipage}

\begin{minipage}{0.4\textwidth}
  \begin{example}[Copy]
    \label{ex:copy_derivative}
    We can copy information via the comonoid $\Delta_X : X \to X \times X$.
    Its reverse derivative is $R[\Delta_X](x, (\alpha, \beta)) = \alpha + \beta$.
  \end{example}
\end{minipage}
\begin{minipage}{0.6\textwidth}
  \begin{figure}[H]
    \scaletikzfig[0.8]{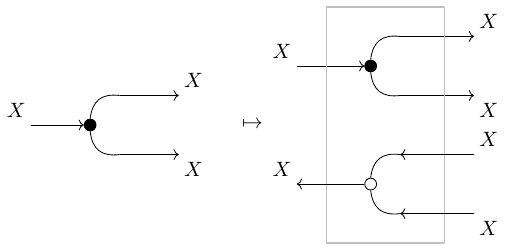}
    \caption{Copy and its derivative.}
    \label{fig:copy_derivative}
  \end{figure}
\end{minipage}

\begin{minipage}{0.4\textwidth}
  \begin{example}[Sum]
    \label{ex:sum_derivative}
    We can sum information using the monoid structure $+_X : X \times X \to X$.
    Its reverse derivative is copy $R[+_X](x, \alpha) = (\alpha, \alpha)$.
  \end{example}
\end{minipage}
\begin{minipage}{0.6\textwidth}
  \begin{figure}[H]
    \scaletikzfig[0.8]{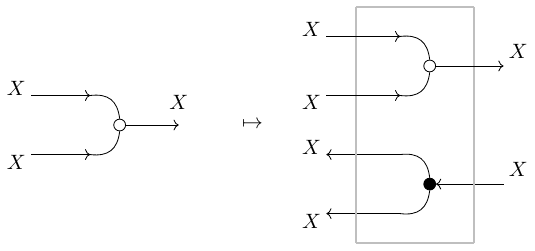}
    \caption{Sum and its derivative.}
    \label{fig:sum_derivative}
  \end{figure}
\end{minipage}

\begin{minipage}{0.4\textwidth}
  \begin{example}[Delete]
    \label{ex:delete_derivative}
    The morphism $\terminal_X : X \to 1$ deletes all information.
    Its reverse derivative is the zero gradient map, i.e.\ $R[\terminal_X](x, \alpha) = 0$.
  \end{example}
\end{minipage}
\begin{minipage}{0.6\textwidth}
  \begin{figure}[H]
    \scaletikzfig[0.8]{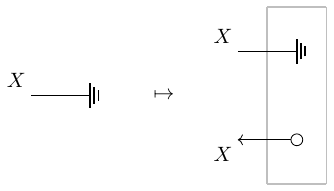}
    \caption{Delete and its derivative.}
    \label{fig:delete_derivative}
  \end{figure}
\end{minipage}

\begin{minipage}{0.4\textwidth}
  \begin{example}[State]
    \label{ex:state_derivative}
    Given any object $X$, all morphisms $s : 1 \to X$ have the same reverse derivative map $R[s] = \terminal_{1 \times X}$.
  \end{example}
\end{minipage}
\begin{minipage}{0.6\textwidth}
  \begin{figure}[H]
    \scaletikzfig[0.8]{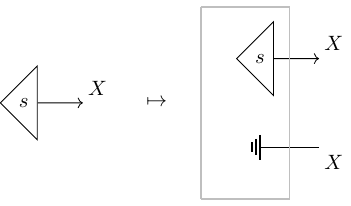}
    \caption{State and its derivative.}
    \label{fig:state_derivative}
  \end{figure}
\end{minipage}

\vspace*{0.5cm}
Lastly, we mention one such lens which is not an adapter: a lens which performs multiplication on the forward, and a derivative thereof on the backward pass.

\begin{minipage}{0.4\textwidth}
  \begin{example}[Multiplication]
    \label{ex:mul_derivative}
    The reverse derivative of the multiplication map\footnote{Formalisable in any cartesian distributive category (\cite[Def.\ 4]{wilson_categories_2022}).} $(x, y) \mapsto xy$ is the map: $R[\cdot]((x, y), \alpha) = (\alpha y, \alpha x)$.
  \end{example}
\end{minipage}
\begin{minipage}{0.6\textwidth}
  \begin{figure}[H]
    \scaletikzfig[0.8]{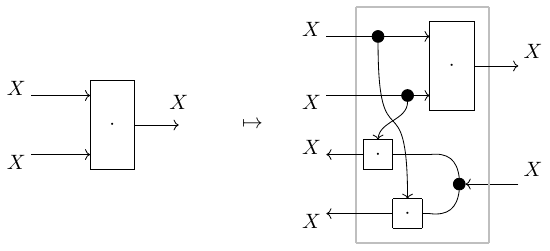}
    \caption{Multiplication and its derivative.}
    \label{fig:mul_derivative}
  \end{figure}
\end{minipage}

\vspace{2em}

\subsection{Feedforward neural networks}

Feedforward neural networks are one of the central building blocks of deep learning.
Such networks often consist of a sequence of alternating linear and non-linear layers.
Below we focus on $\Smooth$, but many of these layers can be defined in more general settings.

\begin{definition}[Linear layer]
  \label{def:linear_layer}
  Fix input and output types $x, y : \N$.
  Then a linear layer $x \to y$ is a morphism in $\Para(\Smooth)(\R^x, \R^y)$ whose parameter space is  $\R^{x \times y}$ and the implementation is 
  \begin{align*}
    \eval & : \R^x \times \R^{x \times y} \to \R^y\\
    \eval & (X, W) \mapsto W^\top X
  \end{align*}
  It is called $\eval$ is because it evaluates a linear map (represented as a matrix $W$) at a point $X$ (represented as a vector).\footnote{In category theory parlance, it is the counit of the tensor-hom adjunction in the subcategory $\FVectR$ of $\Smooth$.}
  Its reverse derivative is $R[\eval](X, W, Y') = (Y'W, Y'^\top X)$
\end{definition}

In practice, the term linear layer is often used interchangeably with an \emph{affine} layer, which is a composition of a linear layer and a bias layer defined below.

\begin{definition}[Bias layer]
  \label{def:bias_layer}
Given a number of neurons $y : \N$, the bias layer is the endomorphism in $\Para(\Smooth)(\R^y, \R^y)$ whose parameter space is $\R^y$, and implementation is pointwise addition:
\begin{align*}
  + & : \R^y \times \R^y \to \R^y\\
  + & (Y, B) \mapsto Y + B
\end{align*}
Its reverse derivative is given by the reverse derivative of $+$, i.e.\ copy: $R[+](Y, B, \alpha) = (\alpha, \alpha)$.\footnote{Note the similarities with \cref{fig:sum_derivative}. In the bias layer, one of the summands lives on the vertical axis.}
\end{definition}

\begin{example}[Activation function]
  \label{ex:activation_functions}
  An activation function is usually a function $a : \R \to \R$ in $\Smooth$ applied pointwise (i.e.\ a function $a^n : \R^n \to \R^n$)
  Notable examples are below, and graphed in \cref{fig:plots_activation}.
  
  \begin{align*}
    \id(x) &= x & \text{Identity}\\
    \sigma(x) &= \frac{\exp(x)}{\exp(x) + 1} & \text{Sigmoid (or logistic) function}\\
    \tanh(x) &= \frac{e^{\exp(2x)} - 1}{e^{\exp(2x)} + 1}& \text{Hyperbolic tangent}\\
    \relu(x) &= \max(0, x) & \text{Rectified linear unit (\cite{fukushima_cognitron_1975})}\\
    \leakyrelu(x) &= \max(\alpha x, x) & \text{``Leaky'' ReLU (\cite{maas_rectifier_2013})\footnote{ReLU precomposed with a linear layer $\R \to \R$}}\\
    \gelu(x) &= x\sigma(1.702x) & \text{Gaussian  Error Linear Unit (\cite{hendrycks_gaussian_2020})}
  \end{align*}
  Though, not all examples are like this.
  The commonly used $\Softargmax$ function\footnote{Commonly known as ``softmax'' though this is a misleading name. See \cite[Sec. 6.2.2.3]{goodfellow_deep_2016}} is of type  $\R^n \to \R^n$, and can't be factored as a product of $n$ functions $\R \to \R$.
  For every output $i : N$ it's defined as a function of type $\R^n \to \R$ whose implementation is 
  \begin{align*}
    \Softargmax(x)_i &= \frac{\exp(x_i)}{\sum_{j=1}^{n}\exp(x_j)}
  \end{align*}
\end{example}

\begin{figure}[h]
  \centering
  \begin{tikzpicture}
    \begin{axis}[
      axis lines=left,
      xmin=-2, xmax=2, ymin=-2, ymax=2,
      xlabel=$x$, ylabel=$f(x)$,
      legend pos=south east,
      legend cell align={right},
      legend style={draw},
      samples=100,
      domain=-2:2,
      thick,
      grid=major,
      width=1.5*\axisdefaultwidth,
      height=\axisdefaultheight,
      ]
      \addplot[blue] ({x}, {\identityfn(x)}); \addlegendentry{Identity}
      \addplot[red] ({x}, {\sigmoidfn(x)}); \addlegendentry{Sigmoid}
      \addplot[green] ({x}, {\tanhfn(x)}); \addlegendentry{Tanh}
      \addplot[orange] ({x}, {\relufn(x)}); \addlegendentry{ReLU}
      \addplot[purple] ({x}, {\leakyrelufn(x)}); \addlegendentry{Leaky ReLU}
      \addplot[brown] ({x}, {\gelufn(x)}); \addlegendentry{GELU}
    \end{axis}
  \end{tikzpicture}
  \caption{Graphs of activation functions: Identity, Sigmoid, Tanh, ReLU, Leaky ReLU, and GELU.}
  \label{fig:plots_activation}
\end{figure}
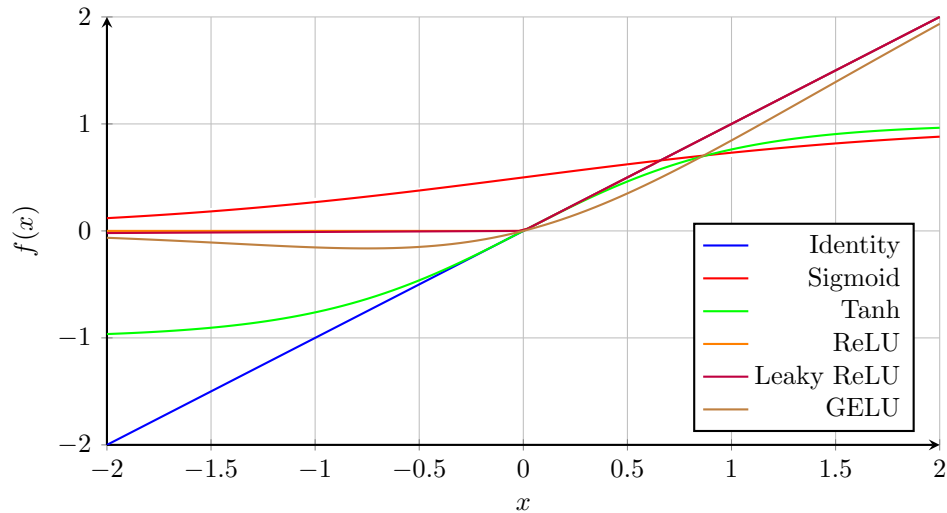

\begin{remark}
  Softargmax is a particularly interesting function, having numerous uses in probability theory, statistical mechanics and reinforcement learning.
  In the case when $n=2$, Softargmax recovers the logistic function.
  Since its output is always a vector whose entires sum up to $1$ it can also be interpreted as a probability distribution.
  This suggests probability monads as a useful abstraction in providing more refined characterisation of this map.
\end{remark}

Given all the components defined above, a standard ``fully-connected'' layer usually consists of a sequential composition of a linear layer, a bias layer and an activation function. (\cref{fig:linear_layer})

\begin{figure}[H]
  \scaletikzfig[0.8]{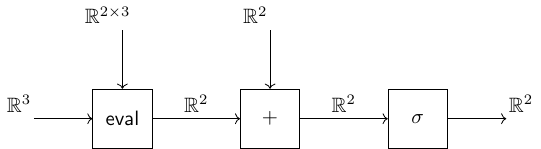}
  \caption{String diagram representation of a linear layer with some activation function $\sigma$.}
  \label{fig:linear_layer}
\end{figure}

As we will see in the next section, most of the interesting neural networks are composed out of something else than just a sequence of standard layers.

\subsection{Weight tying}
\label{sec:networks_with_shared_context}

The concept of \emph{encapsulation} is one of the most fundamental aspects of computer science.
It allows us to abstract away the details of a system, and reuse functionality by simply calling the appropriate method.
This reduces the amount of bugs and allows us to write more understandable and composable code.
Encapsulation is in neural networks achieved by \emph{weight tying}, the method by which we use multiple copies of a neuron in the same place.
We thus only have to learn a functionality once, and can reuse it in different places.
This consequently makes learning easier, as the search space is reduced, and can be seen as the core idea behind the recent success of deep learning.\footnote{In many ways, weight tying is a neural network equivalent of higher-order functions, and is something we hope to explore in future work.}

While most often considered in $\Smooth$, the essential idea behind weight tying works in any cartesian category.

\begin{definition}[Weight tying]
  \label{def:weight_tying}
  Fix a cartesian category $\cC$.\footnote{This statement can be generalised to any distributive algebroidal actegory (\cite[Def.\ 5.2.4]{capucci_actegories_2023}).}
  Given any two objects $\mi, \mo$ and a parametric map $f : \Para(\cC)(\mi, \mo)$ whose parameter space is $P \times P$ (for any $P : \cC$) the \textbf{weight tying} of $f$ is the parametric morphism
  \[
    (P, f^{\Delta_P}) : \Para(\cC)(\mi, \mo) \quad \text{where} \quad f^{\Delta_P} = \boxed{X \times P \xrightarrow{X \times \Delta_P} X \times (P \times P) \xrightarrow{f} X}
  \]
  obtained by reparameterising $f$ with the copy map $\Delta_P : P \to P \times P$. (the notation $f^{\Delta_P}$ follows the convention outlined in \cref{fig:para_reparam}.)\footnote{A galaxy-brain definition of weight tying is that it is precisely \emph{the functor} described in \cref{lemma:cokl_to_para}.}
\end{definition}

This seemingly simple definition (whose string diagram representation is shown in \cref{fig:weight_tying}) we believe is the cornerstone of modern deep learning.
It takes a parametric morphism with two independent parameters and couples them together, removing a degree of freedom.
This makes the ``neuron'' that the parameter is implementing be used in two different places, and reduces the search space as whatever is learned can now be reused.

\begin{figure}[H]
  \scaletikzfig[0.8]{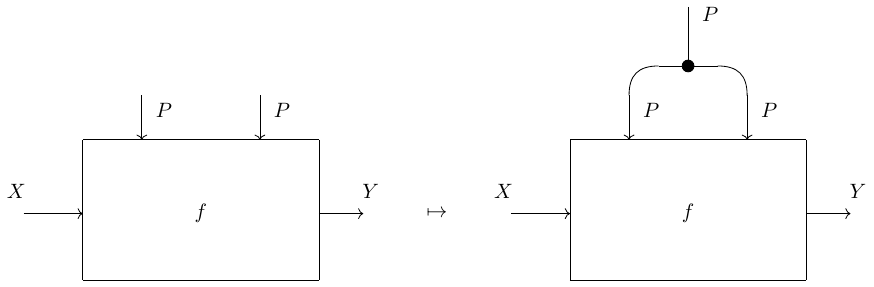}
  \caption{String diagram representation of weight tying.}
  \label{fig:weight_tying}
\end{figure}

Weight tying is most often applied in settings of sequential (\cref{def:para}) and parallel (\cref{prop:para_monoidal}) composition of parametric morphisms.
In sequential composition (\cref{fig:weight_tying_both} (left)), starting with two composable parametric morphisms that share the parameter, i.e.\ 
\[
 (P, f)  : \Para(\cC)(\mi, \mo) \quad \text{and} \quad (P, g) : \Para(\cC)(\mo, Z)
\]
we can form the composite morphism $(f \compPara{\comp} g)^{\Delta_P}$ which now has coupled parameters.
This is often used in recurrent neural networks (\cref{subsec:backpropagation_through_time}), where we want to ensure that the same action is implemented at every time step.
For parallel composition (\cref{fig:weight_tying_both} (right)) the idea is the same.
Starting with two parametric maps that have the same parameter type (here $P$):
\[
  (P, f) : \Para(\cC)(\mi, \mo) \quad \text{and} \quad (P, g) : \Para(\cC)(Z, W)
\]
the weight tying of $f$ and $g$ is the morphism $(f \compPara{\times} g)^{\Delta_P}$ with coupled parameters, now used in parallel.
This is used in generative adversarial networks (\cref{def:gan}), for instance, where the weight tying constrains the discriminator to use the same parameter while processing samples from the generator and those from the dataset.

\begin{figure}[h]
  \begin{subfigure}[c]{0.45\textwidth}
    \scaletikzfig[0.7]{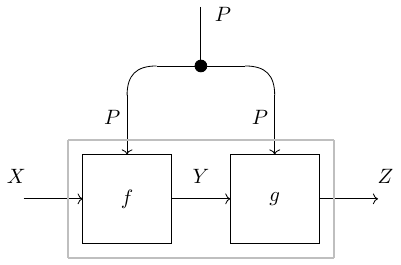}
    \label{fig:weight_tying_sequential}
  \end{subfigure}\hfill
  \begin{subfigure}[c]{0.45\textwidth}
    \scaletikzfig[0.7]{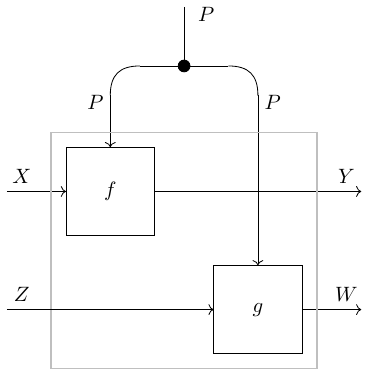}
    \label{fig:weight_tying_parallel}
  \end{subfigure}
  \caption{Weight tying a a) sequential and b) parallel composition of morphisms.}
  \label{fig:weight_tying_both}
\end{figure}

Weight tying can also be used to model \emph{batching}, a process by which from a neural network we obtain a new one that can process multiple inputs at the same time.

\begin{definition}[Batching as weight tying]
 \label{def:batching}
 Let $(P, f) : \Para(\cC)(A, B)$ be a morphism. Then the \newdef{$n$-fold batching} of $f$ is the morphism $(P, f^{\Delta^n_P}) : \Para(\cC)(A^n, B^n)$, where $\Delta_P^n : P \to P^n$ is the $n$-fold copy map.
\end{definition}

In practice we can't merely copy neurons arbitrarily: we usually have to exploit some structure in our data, and perform weight tying in a principled way.
In the next sections we will see a number of patterns of weight tying, describing recurrent, recursive, generative-adversarial, and many other neural network architectures.

\begin{cthelp}
  Can we formulate the notion of a linear layer internal to any generalised cartesian reverse derivative category?
  What are the categorical semantics of activation functions?
  Can we think of them as (co)effects applied to linear maps, much like those that arise from (co)Kleisli categories of comonads?
  Is there a formal connection between weight-tying and higher order functions, as suggested in \cite{olah_neural_2015}?
\end{cthelp}

\section{Backpropagation through time: recurrent neural networks}
\label{subsec:backpropagation_through_time}

So far we did not explicitly acknowledge the structure of inputs of neural networks.
Often, the input to a neural network is an entire sequence of characters, frames, tokens or various other temporal inputs.
Recurrent neural networks (\cite{hochreiter_long_1997, schmidt_recurrent_2019}, \cite[Sec. 10.2]{goodfellow_deep_2016}), or just RNNs, are those that consume an input of this form, and backpropagation in this setting is often referred to as \emph{backpropagation through time}.
Thinking of a RNN as a morphism, it is a \emph{stateful morphism}: it consumes a sequence of inputs $X^n$ and produces another sequence $Y^n$, where here $n$ is the length of the sequence, and each $j : n$ of the output depends on all $i : n$  of the input for $i \leq j$.
Most often, this is done by repeating a single recurrent neural network \emph{cell} across every time step.

Recurrent neural networks have thoroughly been characterised in \cite{sprunger_differentiable_2021} in terms of cartesian differential categories (\cite{blute_cartesian_2009}) and the concept of the \emph{delayed trace} introduced in the same paper.
For more details we refer the interested reader to the aforementioned paper, and here we instead outline its key details.
The paper is based on the definition of $\St(\cC)$: the category of stateful morphisms.

\begin{definition}[{Causal extension of a cartesian category (compare \cite[Def.\ 11]{sprunger_differentiable_2021})}]
  \label{def:stc}
  Let $\cC$ be a cartesian category.
  Its \newdef{causal extension} is a category denoted by $\St(\cC)$ whose objects are $\N$-indexed families of objects, and a morphisms are stateful morphism sequences \footnote{That is, \emph{equivalence classes} of stateful morphism sequences} (see \cite[Fig. 1]{sprunger_differentiable_2021})
  Likewise, $\St_0(\cC)$ is a subcategory of $\St(\cC)$ whose family of objects are all the same, and morphisms given by iterating $f$. (See \cite[Def.\ 15]{sprunger_differentiable_2021} for a precise specification).
\end{definition}

There are two important considerations related to the framework of this thesis.
Is there a parametric version of $\St(\cC)$, and can it be equipped with the coalgebra structure of $\LensA$?

As \cite{sprunger_differentiable_2021} does not mention $\Para$, the first question is left unanswered.
Nonetheless, as $\St(\cC)$ is a cartesian category, it is possible to form $\Para$ over it.
On the other hand, $\Para(\St(\cC))$ has too many degrees of freedom: in recurrent neural networks it is important that each recurrent cell inside the layer shares the same parameter (\cref{fig:recurrent_twolayer} (right)).
We hypothesize there is an parametric category $\Para_{\Delta}(\St(\cC))$ capturing this.

\begin{figure}[H]
  \scaletikzfig[0.7]{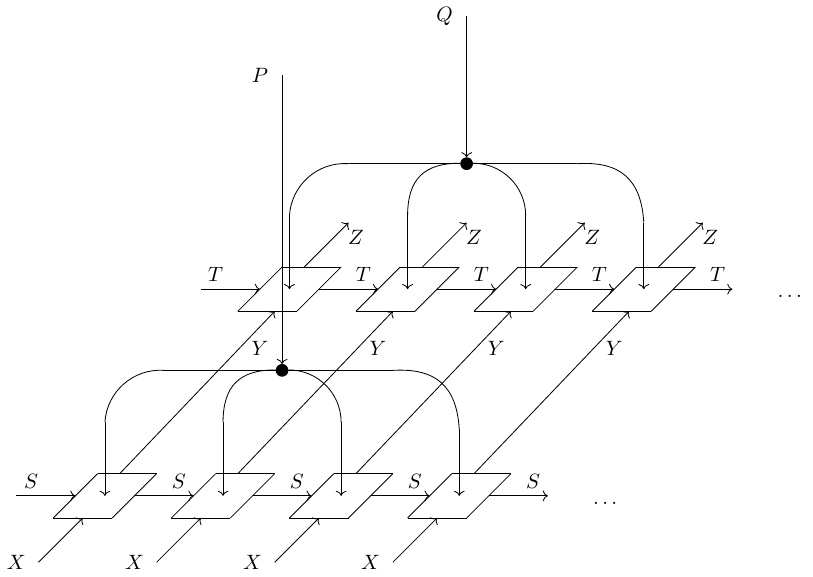}
  \caption{Two-layer recurrent neural network. Note that parameters are shared throughout time steps, but not throughout layers. This two-dimensional representation reinforces the idea that \emph{parametric double categories} are a suitable framework for studying recurrent neural networks.}
  \label{fig:recurrent_twolayer}
\end{figure}

Pertaining to the second question, the paper shows that $\St(\cC)$ is a cartesian differential category when $\cC$ is.
It does not show $\St(\cC)$ is a \emph{reverse} derivative category\footnote{Because \cite{sprunger_differentiable_2021} originally appeared before \cite{cockett_reverse_2020}, the reverse derivative categories paper}, though there do not seem to be any major obstacles to doing so.

Additionally, \cite{sprunger_differentiable_2021} also introduces \emph{the delayed trace}, an operator generalising the categorical operation of trace with an implicit guardedness guarantee.
This operator has potential in being applicable to more than just recurrent neural networks, as the authors hypothesize it could be useful for describing parameter update and meta-learning (\cite[Sec. VI]{sprunger_differentiable_2021}).

The most popular kind of recurrent neural networks are Long Short-Term Memory (LSTM) networks (\cite{hochreiter_long_1997}) which introduce various gating mechanisms aimed at reducing the effect of vanishing and exploding gradients (\cite{pascanu_difficulty_2013}).
They appear in countless variants in practice (\cite{yu_review_2019}), none which appear to have been studied through the lens of category theory.

\begin{cthelp}
  The work of \cite{sprunger_differentiable_2021} provides a comprehensive characterisation of recurrent neural networks through double categories.
  As mentioned in this section, avenues for exploration include establishing an action on $\St(\cC)$ which yields the appropriate bicategory $\Para_{\Delta}(\St(\cC))$, and formulation of $\St(\cC)$ as a generalised reverse derivative category.
  Additionally, to the best of our knowledge there is no work characterising recurrent neural networks to any of the recursion schemes or generalised kinds of (un)folds such as catamorphisms or anamorphisms.
  In other words, the structurally recursive component of $\St_0(\cC)$ --- being given as a repeat of just a single recurrent cell throughout the layer --- is not captured by any of the usual categorical tools for doing so.
  We touch on this in the next chapter.
\end{cthelp}

\section{Recursive neural networks, and more.}
\label{subsec:recursive_encoding_decoding}

There is an interesting generalisation of recurrent neural networks called \emph{recursive} neural networks (\cite{socher_recursive_2013, goller_learning_1996}).\footnote{From the latter reference comes the term ``Backprop through structure''.}
Instead of consuming input data in the form of a list, recursive neural networks consume data in the form of a tree.
This makes them amenable for use in a variety of settings, most notably that of parse trees in natural language processing.
This was the original motivation for their introduction in \cite{socher_recursive_2013} and seemingly the motivation for their only existing categorical model (\cite{lewis_compositionality_2019}).
While this model establishes the relationship of recursive neural networks to the categorical compositional vector space semantics (\cite{coecke_mathematical_2010}), it does not establish any relationship of recursive neural networks to recurrent ones through the language of category theory as suggested by \cite{olah_neural_2015}.
Inspired by \cite{olah_neural_2015} and unpublished research we have done on this topic, we provide hypotheses about category-theoretic formulation of recursive neural networks.

It is well-known in the category theory literature that lists and trees are examples of inductive data types whose categorical semantics are those of initial algebras of endofunctors.\footnote{A more refined semantics of inductive data types and structural recursion in general is given by recursive coalgebras (\cite{capretta_recursive_2006}).}
More precisely, the set $\List{X}$ of lists of elements of type $X$ is the initial object of the category $(1 + X \times (-))\dsh\Alg$, and the set $\BTree{X}$ of finite binary trees with $X$-labelled leaves is the initial object of the category $(X + (-)^2)\dsh\Alg$.
As it happens, these endofunctors (and many more of interest) are all strong, enabling us to see them as actegory morphisms (\cref{ex:strong_endofunctor_morphism_of_actegories}).
Consequently, via \cref{prop:para_morphism} this morphism induces an endomorphism $\Para(F)$ on $\Para(\Set)$.
And since every $\Para(\cC)$ holds a copy of $\cC$ inside it (\cref{lemma:c_embeds_into_para}) we might expect that the category of $\Para(F)$-algebras holds a copy of $F$-algebras inside of it.

This is indeed the case, and part of our forthcoming work on this topic.
We believe the claims of \cite{olah_neural_2015} neural networks can be substantiated for encoding, decoding and recursive neural networks.
Intuitively, if $\BTree{X} : \Set$ arises as the initial algebra of the strong endofunctor $X + (-)^2 : \Set \to \Set$, then it can be shown that this is an appropriate notion of a higher-dimensional inital algebra of $\Para(X + (-)^2) : \Para(\Set) \to \Para(\Set)$.\footnote{We believe that ``quasi initial algebra`` is the appropraite notion here.}
In such a case, given any other algebra for $(Y, f)$ for $\Para(X + (-)^2)$ there is a catamorphism unfolding the tree and performing weight sharing, roughly as in \cref{fig:recursive}.

\begin{figure}[H]
  \scaletikzfig[0.85]{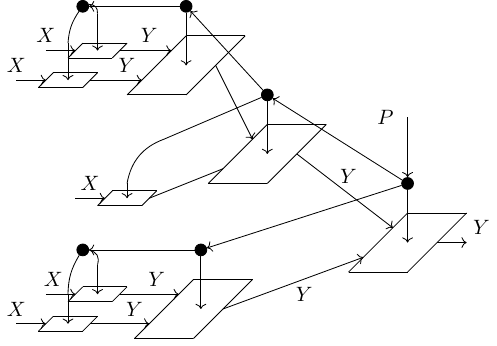}
  \caption{String diagram representation of a recursive neural network as a morphism in the category $\Para(X + (-)^2)\dsh\Alg$. The same parameter, and the same components of the underlying map $f : P \times (X + Y^2) \to Y$ are used throughout the network.}
  \label{fig:recursive}
\end{figure}

The generality of this initial algebra formulation suggests plenty of other examples.
The initial algebra of $1 + X \times -$ is $\List{X}$, and its parametric form given some algebra $(S, f)$ comes equipped with universal maps which describe \emph{encoding neural networks}, as suggested by \cite{olah_neural_2015}.
We depict an example thereof in \cref{fig:para_encoding_decoding} (left) where we see a network consuming a list of inputs and folding over them with the map $f$ in order to produce an accumulated value of type $S$.
Dually, we can also consider \emph{coalgebras} of $Y \times -$ where the terminal coalgebra is that of \emph{streams} (\cref{fig:para_encoding_decoding} (right)) which generate an infinite sequence of outputs of type $Y$ from a single starting state. 

\begin{figure}[h]
  \begin{subfigure}[c]{0.45\textwidth}
    \scaletikzfig[0.8]{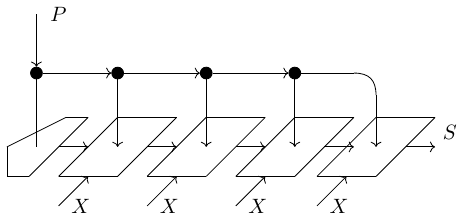}
    \label{subfig:para_encoding}
  \end{subfigure}\hfill
  \begin{subfigure}[c]{0.45\textwidth}
    \scaletikzfig[0.8]{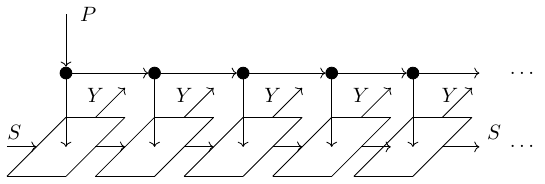}
    \label{subfig:para_decoding}
  \end{subfigure}
  \caption{Encoding and decoding neural networks as morphisms in $\Para(1 + X \times (-))\dsh\Alg$ and $\Para(Y \times (-))\dsh\coAlg$, respectively. As in \cref{fig:recursive}, the same parameter value and the same type of implementation map is used througout.}
  \label{fig:para_encoding_decoding}
\end{figure}

\begin{cthelp}
This section suggests a general categorical framework for studying structurally recursive neural networks: as that given by a category of algebras for a given endofunctor, and opens up a wide array of questions: what other (co)algebras can be studied in this way?
It is known that Moore and Mealy machines can be seen as coalgebras for endofunctors \cite{rutten_universal_2000}.
Are these suitable candidates for arhitectures? What other, potentially branching or non-deterministic architectures exist?
\end{cthelp}

\section{Graph neural networks}
\label{subsec:graph_neural_networks}

In previous section we have described a categorical framework modelling neural networks that process a list, or a tree of inputs.
Graphs are a more general data structure that can represent molecules, social neworks and transporation networks. %
In recent years graph neural networks became a sizeable subfield of deep learning, spawning a myriad of research and practical applications (\cite{wu_comprehensive_2021, dudzik_graph_2022, velickovic_message_2022, mirhoseini_graph_2021}).
\forlater{conceptually, what kind of extra power do graphs give us? what kinds of abstract structure do they contain that trees and lists don't? cycles? }

Intuitively, consuming a graph as an input allows us to process each of its nodes by taking into account all of those around it.
Most crudely, this can be described as a global underlying context (\cref{subsec:local_vs_global}) the neural network layers have access to.
These layers then compute updates for each node based on the connectivity pattern encoded by the graph.
This was the category-theoretic treatment of \emph{graph convolutional neural networks} (GCNNs) in \cite{gavranovic_graph_2022} whose definition of a GCNN layer we present in \cref{def:prim_gcnn_layer}.

\begin{definition}[GCNN Layer]
  \label{def:prim_gcnn_layer}
  Fix the number of nodes $n : \N$ in a graph, input and output types $\mil, \mol : \N$, and an activation function $\sigma : \R^\mol \to \R^\mol$.
  A graph convolutional neural network layer $\mil \to \mol$ is a morphism in $\Para(\CoKl(- \times \R^{n \times n}))(\R^{\mil \times n}, \R^{\mol \times n})$ whose parameter space is $\R^{\mil \times \mol + \mol}$
  and the implementation is
  \begin{align*}
    f : \R^{\mil \times n} \times \R^{\mil \times \mol + \mol} \times \R^{n \times n}&\to \R^{\mol \times n}\\
    (\mi, (W, B), A) &\mapsto \sigma(W^\top \mi A + B)
  \end{align*}
\end{definition}

\begin{figure}[H]
  \centering
  \includegraphics[width=\textwidth]{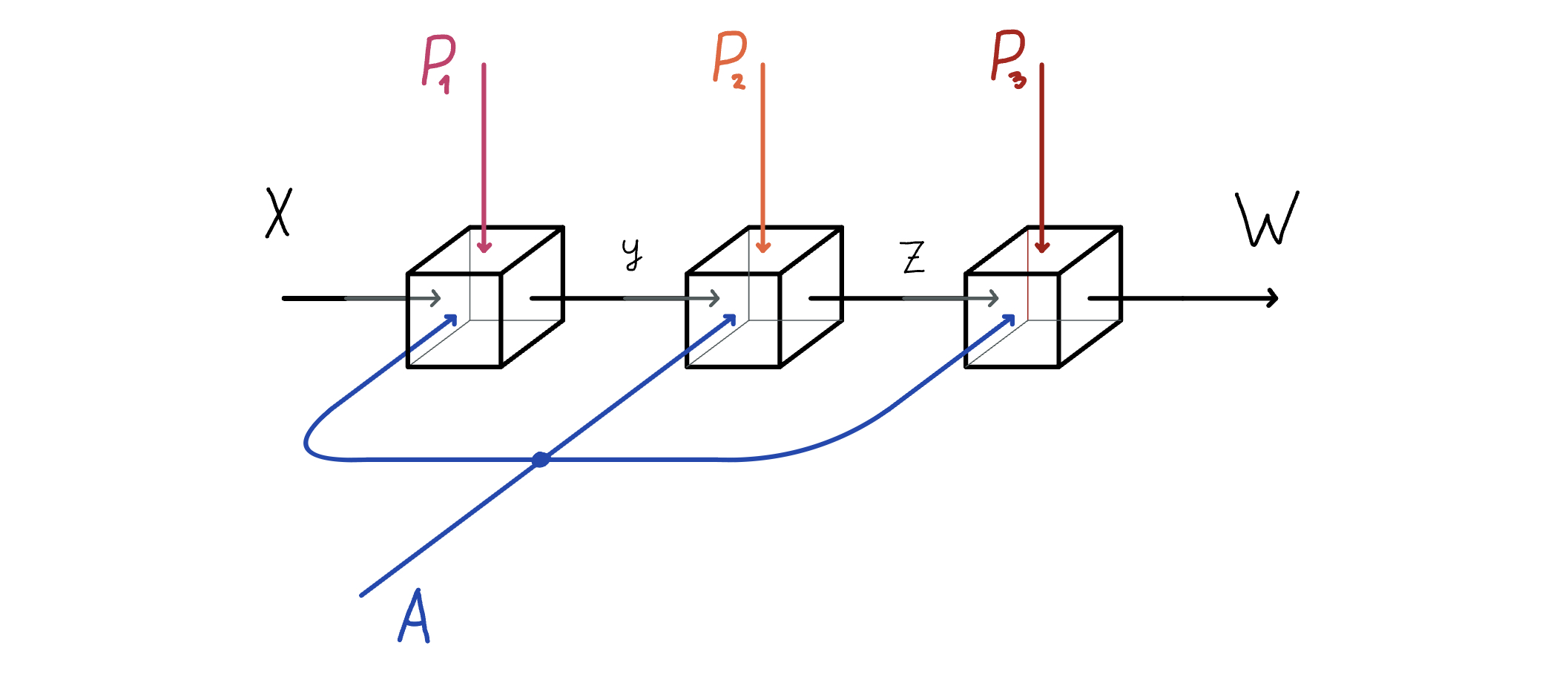}
  \caption{Composition of three layers in a graph convolutional neural network.
    On the vertical axis, we see three parameter spaces: $P_1$, $P_2$, and $P_3$
    of each layer. On the horizontal axis, we see the adjacency matrix $A$ as a
    ``parameter'' to each layer. By composing  new layers on the right the number
    of parameters increases, but the adjacency matrix is just copied.}
  \label{fig:graph_convolutional}
\end{figure}

The global context of $\CoKl(- \times \R^{n \times n})$ ensures that the same adjacency matrix used in one layer is used throughout the whole graph neural network.
In \cite{gavranovic_graph_2022} it has been shown that $n$-node convolutional graph neural networks are morphisms in $\Para(\CoKl(- \times \R^{n \times n}))$, and that the base $\CoKl(- \times \R^{n \times n})$ is a generalised reverse derivative category (\cref{ex:cokleisli_of_grdc}).
Despite this kind of a formulation having convolutional neural networks (\cite{bengio_convolutional_1995}) as a special case, in many ways its categorical form is limited.
Graphs are encoded as an adjacency matrix and their geometry is not explicitly accounted for in category-theoretic terms.
Furthermore, the implemented multiplication by an adjacency matrix implements only a special form of \emph{message passing} (\cite{velickovic_message_2022}).
It is a special case because only nodes in this graph are assumed to have features attached to them, and edge-level and graph-level features are ignored.

General message passing equips graphs with node, edge and graph-level features, describing structured ways information is propagated (\cite[Eq. 8, 9, 10]{velickovic_everything_2023} and \cite{dudzik_graph_2022}).
We do not study this further except to say that message passing (specifically, \cite[Eq. 10]{velickovic_everything_2023}) can be formulated as a lens whose forward part is tasked with broadcasting the information out of a node, and the backward pass is tasked with aggregating received information from other nodes.

We also mention the concept of equivariance.
It can be thought of as a formal notion of consistency under transformation. 
For instance, convolutional neural networks the above example models are translation equivariant, meaning they allow us to state how translations of features in the input correspond to a translation of the features of the output.
In other words, in the task of image segmentation we get the same result if we a) translate an object in the original image, and then apply the convolutional neural network, or b) apply the convolutional neural network, and then translate the generated mask.
But translations are not the only kind of transformation.
Transformations can be equivariant with respect to a group of rotations, reflections, scaling, and many others.
Equivariance with respect to an arbitrary group has been studied in \cite{cohen_group_2016}, and in more general terms in \cite{harris_characterizing_2019, de_haan_natural_2020, bronstein_geometric_2021}.

Lastly, graph neural networks have recently related to the concept of dynamic programming through the language of category theory (\cite{dudzik_graph_2022}).
As they use the formalism of polynomial functors whose morphisms are dependent lenses, we believe this is an exciting research direction that could place them on even more solid categorical foundations.

\begin{cthelp}
  There is a well-understood way of generalising the category $\CoKl(- \times \R^{n \times n})$ to a fibred setting, analogously how $\CoKl(- \times X)$ can be embedded into the slice category $\cC/X$.
  This allows the dimensionality of vector space of features to depend on the node it is fibred over.
  Likewise, the abstract machinery of the Grothendieck construction here yields a particular instantiation of message-passing.\footnote{This came up in a conversation with Matteo Capucci.}
  Can this be used as a generalisation of the above formulation of graph convolutional neural networks?

  Lastly, \cite{bronstein_geometric_2021, cohen_group_2016} host numerous examples of group equivariant neural networks described as homomorphisms of group actions.
  All of these arise as morphisms of monad algebras for the group action monad.
  By generalising these homomorphisms to account for possibly different monads on the domain and codomain, we obtain group invariant neural networks.
  This suggests a host of other examples, and a generalised theory of (co)equivariance as that of (co)monad/endofunctor algebra homomorphisms.
  Such a formulation further suggests connections to structurally (co)recursive neural network formulations in \cref{subsec:recursive_encoding_decoding} where they are also formulated as (co)algebra homomorphisms.
  Can this theory pave way for a theory of architectures of neural networks?
  We believe so.
\end{cthelp}

\section{Generative Adversarial Networks}
\label{sec:gan}

Generative Adversarial Networks (\cite{goodfellow_generative_2014}), or GANs are an important architecture that lies in the centre of the intersection of deep learning and game theory.
Unlike the architectures described above, GANs do not arise as an instantiation of the learning framework in a different kind of a category.
Instead, a GAN is a system of two neural networks trained with ``competing'' optimisers.
One neural network is called \emph{the generator} whose optimiser is, as usual, tasked with moving in the direction of the negative gradient of the loss.
However, the other network --- called \emph{the discriminator} --- has an optimiser which is tasked with moving in the \emph{positive}, i.e.\ ascending direction of the gradient of the total loss --- maximising the loss.
The actual networks are wired in such a way (\cref{fig:gan_box}) where the discriminator effectively serves as a loss function to the generator, i.e.\ being the generator's only source of information on how to update.
Dually, taking the vantage point of the discriminator, the generator serves as an ever changing source of training data.

\begin{figure}[H]
  \centering
  \includegraphics[width=0.9\textwidth]{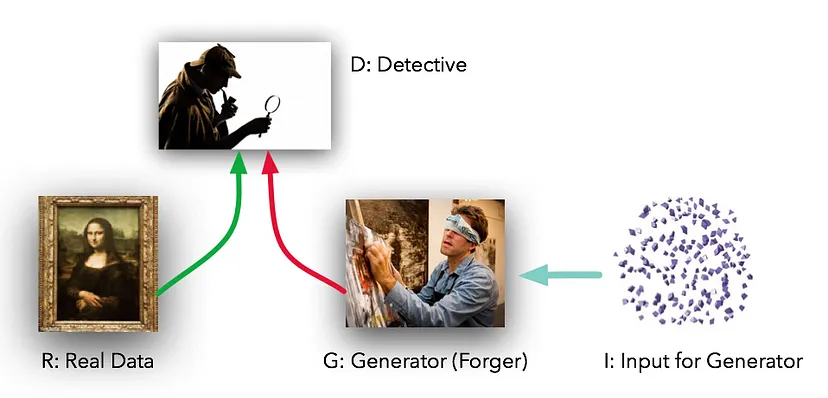}
  \caption{GANs can intuitively be understood as a cat-and-mouse game.
Think of the generator as a forger producing counterfeit paintings from a particular artist, and the discriminator as the detective distinguishing between real and fake artwork.
They both start without knowing how real artwork looks like, and slowly improve as they interact.
For instance, the detective might first see a real painting.
If they misclassify it as fake, they receive a training signal instructing them how to classifying paintings similar to it as real next time.
An analogous story works if they see a forgery.
But in this case, the forger also receives a training signal, and learns how to imrpove \emph{so as to fool the detective} better next time.
This causes a feedback loop where both of them continue improving, enabling the forger to create plausible looking artwork. Figure from \protect\footnotemark.}
  \label{fig:gan_forger_detective}
\end{figure}

\footnotetext{\url{https://medium.com/@devnag/generative-adversarial-networks-gans-in-50-lines-of-code-pytorch-e81b79659e3f}}

GANs have been a popular architecture that has received wide attention and has been applied in numerous domains \cite{pan_recent_2019}.
Following our previous category-theoretic work on GANs in \cite[Sec. 4.1]{capucci_towards_2022}, we present their novel definition below.\footnote{Noting that their forward part can be stated in any cartesian category $\cC$.}

\begin{definition}[GAN]
  \label{def:gan}
  Fix three objects $Z, X$ and $L$ in $\cC$ (respectively called ``the latent space'', ``the data space'' and ``the payoff space'').
  Then given two parametric morphisms
  \[
    (P, g) : \Para(\cC)(Z, X) \quad \text{and} \quad (Q, d) : \Para(X, L)
  \]
  a \textbf{generative adversarial network} is a morphism $(P \times Q, \GAN_{g, d}) : \Para(Z \times X, L \times L)$ where $\GAN_{g, d}$ is defined as the composite
  \[
    \GAN_{g, d} \coloneqq \boxed{Z \times X \xrightarrow{g \compPara{\times} \id_X} X \times X \xrightarrow{(d \compPara{\times} d)^{\Delta_Q}} L \times L}
  \]
\end{definition}

Its string diagram representation is shown in \cref{fig:gan_box} where we see that a GAN consists of two parallel tracks.
We will see in \cref{ex:gan_gda_dotproduct} how the first one will be used to process latent vectors, and the second one to process samples from a chosen dataset.
Despite the fact that there are two boxes labeled $d$, they are weight tied (\cref{def:weight_tying}), making them behave like a singular unit.

\begin{figure}[H]
  \scaletikzfig[0.6][1.1]{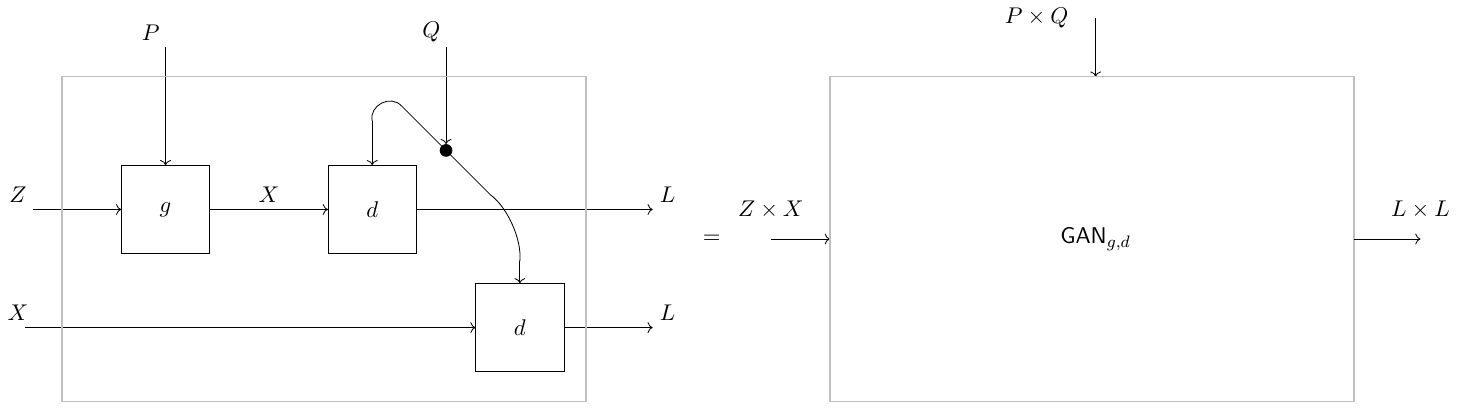}
  \caption{A generative adversarial network.}
  \label{fig:gan_box}
\end{figure}

We can easily state what the reverse derivative of $\GAN_{g, d}$ is in terms of its components:
\begin{align}
  \label{ex:gan_reverse_derivative}
  \begin{split}
    R[\GAN_{g, d}](z, x_r, p, q, \alpha_g, \alpha_r) = (z',x_r',p',q_g' + q_r') \quad \text{where} \quad %
      (x_g', q_g') &= R[d](g(z, p), q, \alpha_g)\\
      (x_r', q_r') &= R[d](x_r, q, \alpha_r)\\
      (z', p') &= R[g](z, p, x_g')
    \end{split}
\end{align}

The pair $(\GAN_{g, d}, R[\GAN_{g, d}])$ yields a parametric lens of type $\diset{Z \times X}{Z' \times X'} \to \diset{L \times L}{L' \times L'}$ (\cref{fig:gan_differentiated}), which we interpret as follows.

\begin{figure}[h]
  \scaletikzfig[0.6][1.2]{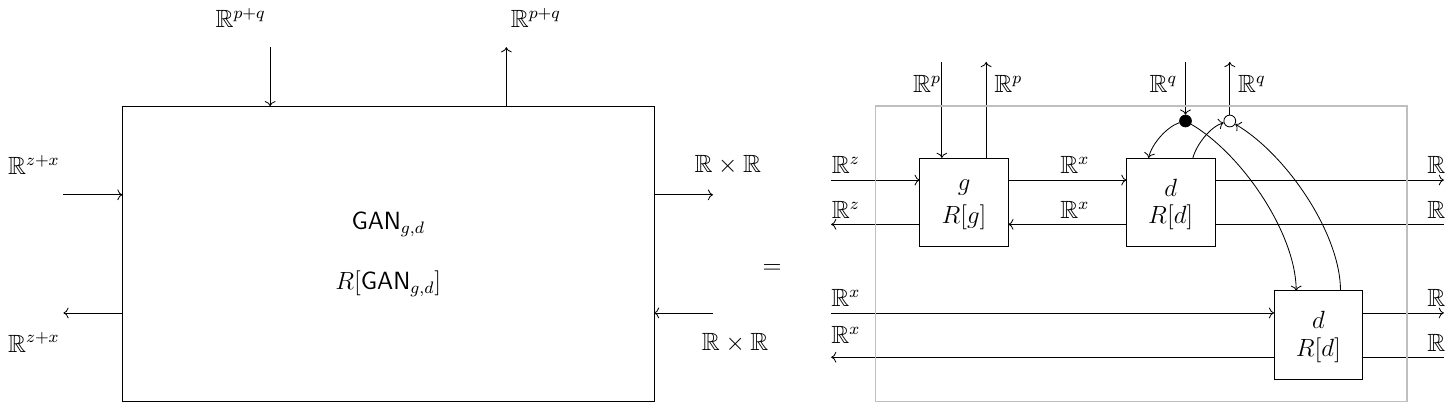}
  \caption{A generative adversarial network under the image of $\Para(\RC)$.}
  \label{fig:gan_differentiated}
\end{figure}

It consumes two pieces of data, ``a latent vector'' $z : Z$, a ``real'' sample from the dataset $x_r : X$, in addition to the parameter $p : P$ for the generator and a parameter $q : Q$ for the discrmiminator.
What happens then are two independent evaluations done by the discriminator.
The first one uses generator's attempt of producing a sample from the dataset (the latent vector which was fed into it, producing $g(z, p) : X$) as input to the discriminator, producing a payoff $d((g, z, p), q) : L$ for this particular sample.
The second one uses the actual sample from the dataset $x_r$, producing the payoff $d(x_r, q) : L$.
We will see in \cref{ex:gan_gda_dotproduct} how the learning context this GAN is trained in produces dynamics which push the generator towards generating real-looking samples, and the discriminator towards discriminating between real and generated samples.

There are various other generalisations of GANs (\cite{hindupur_gan_2023}) out of which we point out CycleGAN (\cite{zhu_unpaired_2017}).
This is an architecture that combines autoencoders and generative adversarial networks in a non-trivial manner using the concept of ``cycle-consistencies'', a loss function enforcing high-level composition invariants of neural networks. 
In (\cite{gavranovic_learning_2020}) it has been argued that this system can be seen as one arising out of a presentation of a category through generators and relations, where each relation gives rise to a corresponding cycle-consistency loss.

\begin{cthelp}
  \label{gan_future_work}
An avenue for future work is GANs in the setting of a base category with coproducts. %
In such a setting we'd have the coproduct injection $\nabla_X : X + X \to X$ at our disposal.
Intuitively, $\nabla_X$ consumes either a datapoint created by the generator, or a datapoint from the training data, and erases its label before passing it on to the discriminator (\cref{fig:gan2}).

\begin{figure}[H]
  \scaletikzfig[0.8]{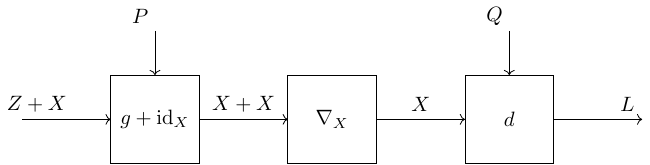}
  \caption{A generative adversarial network defined using coproducts.}
  \label{fig:gan2}
\end{figure}

This provides a more natural interpretation of GANs as it does not necessitate the processing of both $Z$ and $X$ at the same time (unlike \cref{def:gan}), and additionally simplifies the treatment of loss functions (see discussion after (\cref{ex:gan_gda_dotproduct}))
Further work needs to be done providing categorical semantics of differentiation of this architecture in terms of lenses, as lenses do not have all coproducts even if their base does (\cref{subsec:dependent_optics}).
Another avenue for work is formalising GAN-specific optimisers such as \emph{competitive gradient descent} (\cite{schaefer_competitive_2019}) which we mention in more detail in \cref{can_we_compose_optimisers}.
\end{cthelp}

\section{Transformers}
\label{sec:transformers}

Transformers are one of the most influential architectures in deep learning \cite{vaswani_attention_2017, tay_efficient_2022,phuong_formal_2022}.
They're a recent model, originally proposed as a component of a sequence-to-sequence model for machine translation in the paper \emph{Attention is All you Need} (\cite{vaswani_attention_2017}).
They have since been used to achieve the state-of-the-art performance on various tasks in natural language processing (\cite{brown_language_2020}), computer vision (\cite{dosovitskiy_image_2021}), and audio processing (\cite{chen_neural_2018, dong_speech-transformer_2018}), most notably being used as a central piece of the large language model ChatGPT (\cite{openai_introducing_2022}).

One of the most interesting properties of transformers is their ability to learn during the inference stage.
This property --- called \emph{in-context learning} (\cite{brown_language_2020, von_oswald_transformers_2023}) --- is a kind of learning that happens on the \emph{forward pass}, without any parameter updates. Learning instructions (possibly with input-output examples) are encoded as part of the input prompt which the already trained transformers then consumes.
Afterwards, the transformer can be prompted with previously unseen examples to which it in in many cases successfully generalises by producing the desired output as a response.\footnote{The manner by which they learn in-context was shown to be, surprisingly, \emph{gradient descent} (\cite{von_oswald_transformers_2023})!}
In-context learning additionally blurred the lines between training and test stages, motivating research on the meta-learning ability of transformers (\cite{mishra_simple_2018,melo_transformers_2022}).

So, what is a Transformer?
Just like a recurrent neural network (\cref{subsec:backpropagation_through_time}), it is a model which a) consumes a sequence of inputs, b) can be repeated sequentially in layers, and c) consists of a modular component repeated across every time step.
But importantly, unlike in a recurrent neural network, in transformers the modular component does not have a concept of \emph{state} which is iteratively updated at every time step (\cref{fig:recurrent_twolayer}).
Likewise, unlike the ``clean`` design of a recurrent neural network, the transformer layer contains many other details that we remark on at the end of this subsection.

The aforementioned modular component is the \emph{attention} mechanism: the centerpiece of the Transformer.
This is a mechanism meant to mimic cognitive attention found in humans and animals: allocating more resources towards processing a particular part of the input, while ignoring the others.
Variants of attention have been present in architectures such as Neural Turing Machines (\cite{graves_neural_2014}) and Differentiable Neural Computer \cite{graves_hybrid_2016}, though they were not explicitly defined, nor a centerpiece of the paper.
We define the attention component below, and explain the manner by which it's embedded within the transformer architecture.

\begin{definition}[{Attention (\cite{vaswani_attention_2017})}]
  \label{def:attention}
  Fix $\seq, \val$ and $\key : \N$ denoting, respectively, the length of the sequence to be processed (whose elements are called \emph{tokens}), the dimensionality of feature vectors associated with each token, and the dimensionality of \emph{key}\footnote{One can think of this as an identifier of the feature, possibly even an encoding of its \emph{type}.} vectors.
  Then attention is the following morphism in $\Smooth$:
  \begin{align*}
    \Attend: \R^\key \times \R^{\seq \times \key} \times \R^{\seq \times \val} &\to \R^\val\\
    (Q, K, V) &\mapsto \Softargmax(\frac{QK^\top}{\sqrt{\key}})V
  \end{align*}
\end{definition}

We leave out a full description of its semantics to (\cite{vaswani_attention_2017,phuong_formal_2022}), instead here merely pointing out that the inputs $Q, K$ and $V$ are called \emph{query}, \emph{key}, and \emph{value} vectors, and that $Q$ represents the query of the token to the rest of the sequence which the $\Attend$ morphism evaluates. %

Here $\Attend$ is a computation done \emph{for each token} in the sequence.
The manner by which this is extended to the entire sequence involves an additional step of \emph{reuse of the $K$ and $V$ matrices}.
Formally, we can do this by treating $\Attend$ as a morphism in $\CoKl(- \times \R^{\seq \times \key} \times \R^{\seq \times \val})(\R^\key, \R^\val)$ and taking its parallel product $\seq$ times inside this category.
This produces $\Attend^\seq : \CoKl(- \times \R^{\seq \times \key} \times \R^{\seq \times \val})(\R^{\seq \times \key}, \R^{\seq \times \val})$ whose type in $\Smooth$ is $\Attend^\seq : \R^{\seq \times \key} \times \R^{\seq \times \key} \times \R^{\seq \times \val} \to \R^{\seq \times \val}$, and implementation is
\[
  (\{Q_i\}_i^{\seq}, K, V) \mapsto \{\Attend(Q_i, K, V)\}_i^{\seq}
\]

Interestingly, this is not a morphism in $\Para(\Smooth)$: there are no parameters here!
If there are no parameters, how does the self-attention layer learn?
The idea is simple: there is a fully-connected layer precomposed with attention that does the learning.\footnote{Of course, we first need to embed attention into $\Para(\Smooth)$ via the functor $\cC \to \Para(\cC)$ (\cref{lemma:c_embeds_into_para}), which treats it a trivially parametric morphism.}
This fully-connected layer learns how to produce the correct query, key, and value information that then allows the entire transformer architecture to learn how to self-attend.

This is a very rough story.
The actual transformer architecture contains multiple self-attention layers, each with multiple \emph{attention heads} (meaning a token can produce multiple queries), residual connections, and various forms of normalisation (see \cref{rem:normalisation}).
Furthermore, its details often vary from implementation from implementation.
For more detail, we point to \cite{beynon_rainbow_2023,tay_efficient_2022,phuong_formal_2022}.

\begin{remark}[Normalisation]
  \label{rem:normalisation}
Normalisation is another active area of study in deep learning. It refers to the transformation of input data, activations or weights in a particular way to be on a similar scale, helping stabilies and speed up learning. They can be performed across the batch dimension (\cite{ioffe_batch_2015}), across the layer dimension (\cite{ba_layer_2016}), or across weights (\cite{salimans_weight_2016}).
The Transformer architecture in particular uses layer normalisation.
\end{remark}

\begin{cthelp}
  Very little is known about the categorical semantics of transformers.
  Does any of the versions of transformers --- potentially with some ad-hoc details removed --- possess a universal property?
  What is the categorical semantics of in-context learning, and is it related in any way to the microcosm principle in category theory?
  Is the fact that vision transformers \emph{learn to be equivariant} (\cite{gruver_lie_2023}) related in any way to parametric algebras from \cref{subsec:recursive_encoding_decoding} whose homomorphisms are conjectured to model generalised notion of equivariance?
\end{cthelp}

\section{Loss functions}
\label{sec:loss_functions}

Loss functions are a central component of deep learning systems.
The loss function consumes the output of our model and quantifies how well its prediction match the actual data.
By differentiating the composite (\cref{fig:para-nn-loss-losslens}) we can provide a signal for the learning process to follow.

\begin{figure}[H]
  \scaletikzfig[0.8]{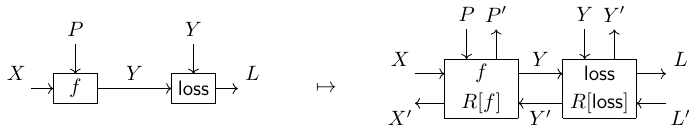}
  \caption{Loss function composed with a model.}
  \label{fig:para-nn-loss-losslens}
\end{figure}

In standard learning in $\Smooth$ the loss function is viewed as a map of type $\mo \times \mo \to \R$.
Here $\mo$ is the output type of the model, and $\R$ is \emph{a payoff} (or \emph{loss}) object.
In different kinds of learning settings there will be a different kind of a payoff object which quantifies how well the model is faring --- for instance, in categories like $\PolyZ$ the payoff object will be $\Z$.\footnote{At the present moment the choice of the payoff object is still done in an ad-hoc way, often as the base field of the underlying vector space. It is not what properties it needs to satisfy.}
We thus abstract away this payoff object and denote it with $L$.
Furthermore, in our setup it will be natural to view the loss map as a parametric map $Y \to L$ whose parameter space is $Y$.
Thus the loss function on $Y$ is here simply a morphism in $\Para(\cC)(Y, L)$.

We now give a series of examples.
In all of them $\groundtruth$ will represent the ground \textbf{t}ruth, $\predicted$ model's \textbf{p}rediction, and $\loss$ the underlying implementation of this loss function.

\begin{example}[Mean squared error]
  \label{ex:mean_squared_error}
  One of the simplest and most common loss functions is mean squared error defined in $\Smooth$.
  Given an object $\R^\mol$ mean squared error is a parametric morphism $(\R^\mol, \loss) : \Para(\Smooth)(\R^\mol, \R)$ where
  \[
    \loss(\predicted, \groundtruth) = \frac{1}{2} \sum^y_{i = 1}((\predicted)_i - (\groundtruth)_i)^2
  \]
  Its reverse derivative is $R[e](\predicted, \groundtruth, \alpha) =  (\alpha(\predicted - \groundtruth), \alpha(\groundtruth - \predicted))$.
\end{example}

\begin{example}[Boolean error]
  \label{ex:xor_error}
  In $\PolyZ$ the loss function on $\Z^\mol$ that is implicitly used in \cite{wilson_reverse_2021} is given by the parametric morphism $(\Z^\mol, \loss) : \Para(\PolyZ)(\Z^\mol, \Z)$ where
  \[
    \loss(\predicted, \groundtruth) = \predicted + \groundtruth
  \]
  Note that $+$ in $\Z_2$ is given by $\XOR$.
  Its reverse derivative, following \cref{fig:sum_derivative}, unpacks to $R[e](\predicted, \groundtruth, \alpha) = (\alpha, \alpha)$.
\end{example}

\begin{example}[Softargmax cross-entropy]
  \label{ex:softargmax_cross_entropy}
  One of the canonical functions used in classification is softargmax cross-entropy (often called softmax cross-entropy).
  Given an output object $\R^\mol$, it is a parametric map $(\R^\mol, \loss) : \Para(\Smooth)(\R^\mol, \R)$ where
  \[
    \loss(\predicted, \groundtruth) = \frac{1}{2} \sum^y_{i = 1}\Softargmax((\groundtruth)_i)((\predicted)_i - \log \Softargmax((\predicted)_i))
  \]
  and where $\Softargmax$ is defined in example \cref{ex:activation_functions}.\footnote{Observe that in \cite[Ex. 8]{cruttwell_categorical_2022} the formula is slightly different because of the assumption that $\groundtruth$ is a probability distribution.
    Our approach is type-safe, making it appear distinct from those in the literature.}
\end{example}

There are some stranger examples as well!
Unbeknownst to most, the loss function that is at the heart of \emph{Deep Dreaming} (\cref{subsec:deep_dreaming}) is the one given by the dot product.
\begin{example}[Dot product]
  \label{ex:dot_product}
  In deep dreaming we often want to focus on a particular element $i : \predicted$ of the network output $\predicted : R^\mol$ and ignore all the others.
  This is done by supplying a one-hot vector $\groundtruth$ as the ground truth to this loss function.
  The dot product then performs masking of all elements of $\predicted$, except the one at which one-hot vector $\groundtruth$ was active at.
  More precisely, given an output type $\R^\mol$, the dot product is a parametric map $(\R^\mol, \loss) : \Para(\Smooth)(\R^\mol, \R)$ where
  \[
    \loss(\predicted, \groundtruth) = \predicted \cdot \groundtruth
  \]
  Its reverse derivative is $R[e](\predicted, \groundtruth, \alpha) = (\alpha \groundtruth, \alpha \predicted)$.
\end{example}

\begin{cthelp}
  Very little is known about the categorical semantics of loss functions.
  Do any of them --- especially the ones involving cross-entropy --- have a universal property?
  What is the minimal structure needed to express each one?
  Likewise, what is the minimal structure needed of the payoff object $L$ for learning?
\end{cthelp}

	\chapter{Supervised Learning}
\label{ch:supervised_learning}

\epigraph{The question of whether a computer can think is no more interesting than the question of whether a submarine can swim.}{Edsger W. Dijkstra}

\newthought{We have done a lot of things so far.}
We explained how to understand parts of neural networks categorically: how to model their weights, backpropagation, how to model various architectures, all in ways that are somehow vastly general --- for instance not tied to euclidean spaces, nor restricting ourselves to an implementation thereof --- while still capturing the essence of backpropagation, architecture, loss functions and so on.
What remains is to describe how the puzzle pieces fit together.

In \cref{fig:para-nn-loss-losslens} we have seen how to compose the model and the loss together.
By applying the general machinery of \cref{ch:para_optic}, specifically \cref{prop:para_morphism}, we can take the underlying $\LensA$-coalgebra, i.e.\ a functor $\RC : \cC \to \LensA(\cC)$ and apply $\Para$ to it, yielding a pseudofunctor
\[
  \Para(\RC) : \Para(\cC) \to \Para(\LensA(\cC))
\]

which takes a parametric map and augments it with its backward derivative.
We note that, despite our approach being completely architecture agnostic, it is still a precise specification of the derivative computation.
For instance, if our base category $\cC \coloneqq \Smooth$ with its usual derivative structure, we get standard backpropagation in the usual sense.
For $\cC \coloneqq \CoKl(\R^{n \times n} \times -)$ and the induced differential structure from $\Smooth$ we model backpropagation of neural networks with a shared context, such as graph convolutional neural networks (\cref{subsec:graph_neural_networks}).
For $\cC \coloneqq \St_0(\cC)$ and its induced differential structure we model backpropagation on recurrent neural networks (\cref{subsec:backpropagation_through_time}).
Of course, the base category can be changed, and we could have studied $\St_o(\PolyZ)$ for instance --- backpropagation of recurrent neural networks in the setting of boolean circuits.
And so on.
After differentiating, in any of these cases we will end up with a parametric lens $\diset{\mi}{\mi '} \to \diset{\lo}{\lo '}$ like the one below.

\begin{figure}[h]
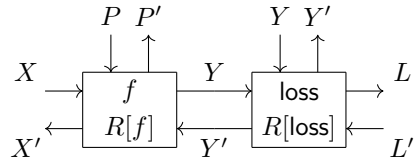

  \scaletikzfig{paralens-loss-function-with-model-2}
  \caption{Model composed with a loss function, augmented with the reverse derivative.}
  \label{fig:model_loss_with_derivative}
\end{figure}

This lens has a forward part $f  \compPara{\comp} \loss : \mi \times P \times \mo \to \lo$ which consumes an input datapoint, a parameter and a label and produces a measure of how well the neural network did on this example.
Its backward pass (which we will denote by $\update$ in this chapter), is a morphism of type $\mi \times P \times \mo \times \lo' \to \mi' \times P' \times \mo'$.
It consumes an input datapoint, a parameter, a label and a learning rate.
What it then does is, via the composition rules of parametric maps (\cref{def:para}) and  lenses (\cref{def:lens_composition_identity}), is produce gradients of the loss with respect to: a) ground truth label (usually ignored in practice), parameter (used to update the network), and input (usually ignored in practice, unless one is doing deep dreaming (\cref{subsec:deep_dreaming})).
More precisely,

\begin{align*}
  \update(x, p, \groundtruth, \alpha) = (x', p', \groundtruth') \quad \text{where} \quad \quad \quad \predicted &= f(x, p)\\
  (\predicted ', \groundtruth ')&=R[\loss](\predicted, \groundtruth, \alpha)\\
  (x', p') &=R[f](x, p, \predicted ')
\end{align*}

In the next section we describe how to ``close off'' this $\update$ map from the left side, the right side, and even from the top.

\begin{contributions}
  Almost the entirety of this chapter is novel contribution, centering around the application of the base change of $\Para$ to the reverse derivative functor $\cC \to \LensA(\cC)$.
  Novel contribution is the proof that $\Para(\cC)$ has quasi-initial objects (\cref{box:para_initial}), optimiser composition (\cref{can_we_compose_optimisers}), and the framing of GANs as a supervised learning system (\cref{ex:gan_gda_dotproduct}).
  Most of this was previously published in ``Categorical Foundations of Gradient-Based Learning'' (\cite{cruttwell_categorical_2022}).
  For a more nuanced discussion of originality, see \cref{sec:supervised_learning_in_the_literature}.
\end{contributions}

\begin{epistemicstatus}
  The fact that the resulting theory of supervised learning takes on a rather simple form makes me confident that the contents of this chapter are on the right track: there are no major issues as I see it.
  The categorical formalism between some smaller components has directions for improvement, though.
  For instance, the formulation of learning rates as parametric lenses is unsatisfactory because it is not in the image of a reverse derivative functor.
  It is plausible that better foundations of differentiation will help shed light on this issue, though a more conclusive answer might be obtained by providing a better theory of iteration of learning.
  Iteration of learning might be something that definitively informs whether the formulation of supervised learning in this chapter is correct, as meta-learning mentioned in \cref{box:optimisers_vs_neural_nets} involves unrolling an iterated learner which processes not just datapoints, but entire datasets in one update step.
  Likewise, I envision that the refinement of optimisers by the way of dependent types is another necessary theoretical step, as currently many artificial identifications are made in their definition.
\end{epistemicstatus}

\section{Corners, learning rates, and optimisers}

In order to perform learning we will need to supply datapoints, a learning rate, and importantly, an optimiser that will tell us how to update the parameter in light of gradient information.

\subsection{Left side: corners}
\label{subsec:left_side_corners}

From the left side, we need \emph{a state}: a parametric lens of type $\diset{1}{1} \to \diset{\mi}{\mi '}$.
Interestingly, such a lens always exists!

\begin{mybox}[label=box:para_initial]{ForestGreen}{Categorical aside: $\Para[false](\cC)$ has a quasi-initial object.}
  \begin{proposition}
    \label{prop:quasi_initial}
    Let $(\cC, \otimes, I)$ be a monoidal category.
    Then $I$ is a quasi-initial object (\cref{def:quasi_initial}) of $\Para(\cC)$, where for every object $X : \Para(\cC)$ the strict initial object of the hom-category $\Para(\cC)(I, X)$ is $(X, \lambda_X)$ where $\lambda_X : I \otimes X \to X$ is the left unitor of $\cC$.\footnote{This is something I learned from Dario Stein.}
  \end{proposition}

  \begin{proof}
    To prove that $(X, \lambda_X)$ is initial in $\Para(\cC)(I, X)$ we need to show that for any other $(P, f) : \Para(\cC)(I, X)$ there is a unique reparameterisation $r : (X, \lambda_X) \Rightarrow (P, f)$ i.e.\ that there exists a unique map $r : P \to X$ such that $\lambda_X^r = f$.
    This map is 
    \[r \coloneqq \boxed{P \xrightarrow{\lambda_P^{-1}} I \otimes P \xrightarrow{f} X}
    \]
    and reparameterisation condition can be shown using naturality of $\lambda$.
  \end{proof}
  Note that since $\Para(\cC)$ is a bicategory, this also means that given any $(Q, g) : \Para(\cC)(X, Y)$ there is a unique 2-cell $(Y, \lambda_Y) \Rightarrow ((X, \lambda_X)  \compPara{\comp} (Q, g))$ arising out of postcomposition of $(X, \lambda_X)$ with $(Q, g)$.
  Its underlying reparameterisation is a map of type  $X \otimes Q \to Y$ and as it turns out this map is precisely $g$.
  Dually, it can be shown that $\Copara(\cC)$ has a quasi-\emph{terminal} object.
  A variant of this was explored in \cite{huot_universal_2019}, in which the ``affine reflection'' is an instance of $\Copara$.

\end{mybox}

It is a parametric lens whose parameter space is that of the domain, and its implementation $\lambda_X$ does nothing --- it is isomorphic to the identity map on $X$.
In string diagram notation it resembles a ``corner'', as it twists the input coming from the top to the right side.
Like the ones before, this lens also arises as the image under the differentiation functor.
In this case, it is the image of the identity map (\cref{ex:identity_derivative}), drawn in \cref{fig:para_state_to_paraoptic_state} in its parametric form.

\begin{figure}[h]
  \scaletikzfig[0.8]{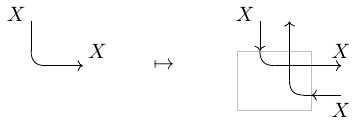}
  \caption{Quasi-initial object in $\Para(\cC)$ can be interpreted as a ``corner''.}
  \label{fig:para_state_to_paraoptic_state}
\end{figure}

Composing it with the rest of the system we obtain \cref{fig:paralens-model-loss-state}.

\begin{figure}[h]
  \scaletikzfig{paralens-model-loss-state}
  \caption{Model, loss function and a state, forming a parametric lens $\diset{1}{1} \to \diset{\lo}{\lo '}$.}
  \label{fig:paralens-model-loss-state}
\end{figure}

\subsection{Right side: learning rates}
\label{subsec:right_side_learning_rate}

From the right side, we need \emph{a costate}: a parametric lens of type $\diset{\lo}{\lo '} \to \diset{1}{1}$.
This is a lens that consumes the produced the output $\lo$ of the loss function, and ``turns it backward'' into a gradient $\lo '$ to be backpropagated.
We will see how this corresponds tightly to the concept of \emph{the learning rate}: a tuning parameter that determines the step size at each iteration while moving towards the optimum of the loss function.
In all the cases of interest it will be a non-parametric lens, hence the definition below.

\begin{minipage}{0.45\textwidth}
  \begin{definition}
    \label{def:learning_rate}
    A \emph{learning rate} on a payoff object $\lo : \cC$ is a morphism $\diset{\lo}{\lo '} \to \diset{1}{1}$ in $\Lens(\cC)$.
  \end{definition}
\end{minipage}
\begin{minipage}{0.45\textwidth}
  \begin{figure}[H]
    \scaletikzfig[1]{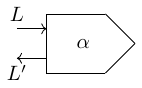}
    \caption{Learning rate as a lens with trivial parameters.}
    \label{fig:learning_rate}
  \end{figure}
\end{minipage}

Note that, in addition to its lack of parameters, there is no requirement that the backward pass of this lens is additive in the second component.
If it were additive, then there would exist only one learning rate: the one given in \cref{ex:delete_derivative} whose backward pass is constant at zero.
This is also the derivative of the forward pass, as it is the only plausible candidate.
But as we will see --- learning rate is hardly ever in the image of the differentiation functor $R$.

To see how this definition resembles a learning rate, it is useful to recall (\cref{box:scs_lens}) which tells us that a costate $\alpha : \diset{\lo}{\lo '} \to \diset{1}{1}$ in $\Lens(\cC)$ is isomorphic to a morphism $\alpha : \lo \to \lo '$ in $\cC$.
This is precisely how we will think of this lens from now on: as a morphism $\alpha : \lo  \to \lo '$ in $\cC$.

\begin{example}
  \label{ex:learning_rate_smooth_standard}
  In standard supervised learning in $\Smooth$ with payoff object $\R$ we fix the learning rate $\R \to \R$ as constant one at some $\alpha \sim 10^{-2}$.
\end{example}

\begin{example}
  \label{ex:learning_rate_poly_idenity}
  In $\PolyZ$ the relevant learning rate will $\alpha : \Z \to \Z$ is the one given by the identity map.
\end{example}

Other learning rate morphisms are possible as well.
One could fix some $\epsilon > 0$ and define the learning rate in $\Smooth$ by $\alpha(l) = \epsilon l$.
Such a learning rate would take into account how far away the network is from the desired goal and adjust the output accordingly.
Composing the learning rate with the model and loss in \cref{fig:model_loss_with_derivative} we obtain \cref{fig:paralens-model-loss-cap}.

\begin{figure}[h]
  \scaletikzfig{paralens-model-loss-cap}
  \caption{Model, loss function and a learning rate forming a parametric lens $\diset{\mi}{\mi '} \to \diset{1}{1}$.}
  \label{fig:paralens-model-loss-cap}
\end{figure}

\subsection{Top side: optimisers}
\label{sec:optimisers}

In addition to closing off this system from the left and the right side, we will also need to manipulate it from the top.
Recall what happens at the ports $\diset{P}{P'}$ on the top.
The parameter $p : P$ is supplied to the network, and the gradient $p' : P'$ is produced.
Once this happens, we are interested in computing a new parameter as a function of these two, often by moving a small step size in the direction $p'$.
This can succintly be captured as a lens $\diset{P}{P} \to \diset{P}{P'}$ which we think of being stacked on top of the model --- it is precisely a reparameterisation!
\footnote{Despite the fact that $P = P'$, here $P'$ informally denotes \emph{the tangent spoce} of $P$, to alleviate confusion.}
This is what we call a gradient update, and is in neural network terminology called \emph{an optimiser} (\cite{shiebler_compositionality_2023}).

In this section we explore how to model optimisers using lenses and reparameterisations, modelling both non-stateful and stateful variants (\cref{fig:model_and_optimiser}).
We begin by modelling gradient \emph{ascent} --- an update that can be captured in any cartesian left-additive category $\cC$.

\begin{figure}[h]
  \centering
  \begin{subfigure}{.5\textwidth}
    \centering
    {\scalefont{0.75}
      \input{model_and_optimiser.tikz}
    }
    
  \end{subfigure}%
  \begin{subfigure}{.5\textwidth}
    \centering
    {\scalefont{0.75}
      \input{model_and_stateful_optimiser.tikz}
    }

  \end{subfigure}
  \caption{Model reparameterised by basic gradient ascent (left) and a
    generic stateful optimiser (right).}
  \label{fig:model_and_optimiser}
\end{figure}

\vspace*{0.4cm}
\begin{minipage}{0.6\textwidth}
  \begin{definition}[Gradient ascent]
    \label{def:gradient_ascent}
    Let $\cC$ be a cartesian left-additive category.
    Gradient ascent on $P : \cC$ is a lens
    \[
      \diset{\id_P}{+_P} : \diset{P}{P} \to \diset{P}{P'}
    \] 
    where $+_P : P \times P' \to P$ is the monoid structure of $P$.\footnote{Here we ``tag'' the second argument of $+_P$ with a superscript, denoting that it should be thought of as a ``cotangent vector'', though for us really $P' = P$.
      For a type-safe(r) formulation see \cite{capucci_diegetic_2023}.}
  \end{definition}
\end{minipage}
\begin{minipage}{0.4\textwidth}
  \begin{figure}[H]
    \scaletikzfig[1.3]{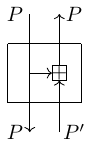}
    \caption{Gradient ascent}
    \label{fig:gradient_ascent}
  \end{figure}
\end{minipage}

\begin{example}[Gradient ascent in $\Smooth$]
  \label{ex:gradient_ascent_smooth}
In $\Smooth$ gradient ascent uses the monoid of $P$ which takes the cotangent vector $p'$ and adds it to the current point $p$.
\end{example}

\begin{example}[Gradient update in $\PolyZ$]
  \label{ex:gradient_ascent_polyz}
  In $\PolyZ$ gradient update will also performs ``addition'', but one implemented by $\XOR$, following the implementation in \cite{wilson_reverse_2021}.
\end{example}

On the other hand, the commonly used gradient \emph{descent} can only be captured in cartesian left-additive categories where the monoid structure on objects is additionally a group.\footnote{Since a homomorphism between groups needs to satisfy \emph{less} equations than a monoid homomorphism, this means that such an assumption is a fairly easy one to add!}

\begin{minipage}{0.6\textwidth}
  \begin{definition}[Gradient descent]
    \label{def:gradient_descent}
    Gradient descent on $P$ is a lens
    \[
      \diset{\id_P}{-_P} : \diset{P}{P} \to \diset{P}{P'}
    \]
    where $-_P:(p, p') = p - p'$.
  \end{definition}
\end{minipage}
\begin{minipage}{0.4\textwidth}
  \begin{figure}[H]
    \scaletikzfig[1.3]{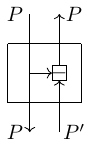}
    \caption{Gradient descent}
    \label{fig:gradient_descent}
  \end{figure}
\end{minipage}
\vspace*{0.3cm}

In $\Smooth$ this instantiates to usual gradient descent.
Gradient ascent in $\PolyZ$ is equal to gradient descent because $\XOR$ is its own inverse.
Intuitively in $\PolyZ$ there is always only one direction we can move (other than staying still): its flipping the bit.
Gradient descent and ascent are not usually seen as a lens --- but they fit precisely into this picture that we are creating.

Other variants of gradient descent also fit naturally into this framework by allowing for additional input/output data with P.
In particular, many of them keep track of the history of previous updates and use that to inform the next one.
This is easy to model in our setup: instead of asking for a lens $\diset{P}{P} \to \diset{P}{P'}$, we ask instead for a lens $\diset{S \times P}{S \times P} \to \diset{P}{P'}$ where S is some “state” object.
These turn out to be parametric lenses as well!

\begin{definition}[Stateful gradient update]
  \label{def:stateful_gradient_update}
  A stateful gradient update of $P$ consists of a choice of an object $S : \cC$ (the state object) and a lens
  \[
    \diset{S \times P}{S \times P} \to \diset{P}{P'}
  \]
\end{definition}

Following the idea of drawing parameters in a different dimension, adding another higher-order ``parameter'' implies that stateful parameter update requires a three-dimensional graphical language in order to properly separate the ``hyper-parameter''state object $S$ from the parameter object $S$.
This is a sensible direction to investigate (partially explored in \cite{gavranovic_meta-learning_2021}), it is not the one we take in this thesis.
Here we consider a stateful optimiser merely as just a lens whose domain is a product (\cref{fig:model_and_optimiser}, (right)).

Several well-known optimisers can be implemented this way.
All of these reside in $\Smooth$.

\begin{example}[{Momentum (\cite{polyak_methods_1964})}]
  \label{ex:momentum}
  In the momentum variant of gradient descent, we keep track of previous gradients used to update the parameter and use this information to inform how future updates should be computed.
  More precisely, set $S = P$ and fix $\gamma : \R$ (usually $\gamma > 0$).
  Then momentum is a lens

  \[
    \diset{\pi_P}{U'} : \diset{P \times P}{P \times P} \to \diset{P}{P'}
  \]
  with $U'(s, p, p') = (s', p + s')$ where $s' = - \gamma s + p'$.
  Note momentum recovers simple gradient update when $\gamma = 0$.
\end{example}

In both standard gradient ascent/descent and momentum our lens representation has a trivial forward part.
This raises the question of whether lenses really capture the essence of optimisers.
However, as soon as we move on to more complicated variants, having non-trivial forward part of the lens is important, and Nesterov momentum is a key example of this.

\begin{example}[{Nesterov momentum (\cite{nesterov_method_1993})}]
  \label{ex:nesterov_momentum}
  In Nesterov momentum we make a small modification to the forward pass of the previous example.
  While in momentum the forward pass simply discards the accumulated state, Nesterov momentum computes a ``lookahead'' value by tewaking the input parameter supplied to the network.
  More precisely, we again set $S = P$, fix $\gamma : \R$ (usually $\gamma > 0$) and define the \textbf{Nesterov momentum} lens
  \[
    \diset{U}{U'} : \diset{P \times P}{P \times P} \to \diset{P}{P'}
  \]
  where  $U(s, p) = p + \gamma s$ and $U'$ as in the previous example.
\end{example}

\begin{example}[{Adagrad (\cite{duchi_adaptive_2011})}]
  \label{ex:adagrad}
  Given any $\epsilon > 0$ and $\delta \sim 10^{-7}$ , Adagrad is given by $S = P$ and a lens whose forward part is defined as $(g, p) \mapsto p$.
  The backward part is $(g, p, p') \mapsto (g', p + \frac{\epsilon}{\delta + \sqrt{g'}} \odot p')$ where $g' = g + p' \odot p'$ and $\odot$ is the elementwise (Hadamard) product.
  Unlike with other optimization algorithms where the learning rate is the same
  for all parameters, Adagrad divides the learning rate of each individual
  parameter with the square root of the past accumulated gradients.
\end{example}

\begin{example}[{ADAM (\cite{kingma_adam_2015})}]
  \label{ex:adam}
  Adaptive Moment Estimation is another method that computes adaptive
  learning rates for each parameter by storing exponentially decaying average of past
  gradients ($m$) and past squared gradients ($v$).
  For fixed $\beta_1, \beta_2 \in [0, 1)$, $\epsilon > 0$, and $\delta \sim 10^{-8}$, Adam is given by $S = P \times P$ with the lens whose forward part is $(m, v, p) \mapsto p$ and whose backward part is
  $(m, v, p, p') = (\widehat{m}', \widehat{v}', p + \frac{\epsilon}{\delta + \sqrt{\widehat{v}'}} \odot \widehat{m}')$, where
  \begin{align*}
    m' &= \beta_1m + (1 - \beta_1)p'           & \widehat{m}' &= \frac{m'}{1 - \beta_1^t}\\
    v' &= \beta_2v + (1 - \beta_2)p'^2         & \widehat{v}' &= \frac{v'}{1 - \beta_2^t}
  \end{align*}
\end{example}

We note that these are just a select few out of many more (\cite{kashyap_survey_2022}) such as regularised gradient descent (\cite{botev_nesterovs_2017}) and Nesterov ADAM \cite{dozat_incorporating_2016}.
The last example is notable because it suggests that ``Nesterov'' is a general principle for modifying the forward passes of optimisers.

\begin{mybox}[label=box:optimisers_vs_neural_nets]{ForestGreen}{Both optimisers and neural networks are parametric lenses!}
  Stateful optimisers such as Nesterov momentum (\cref{ex:nesterov_momentum}) can be seen as parametric lenses, whose parameters are states. 
  Since we've seen that parameters of a parametric lens can be learned, is there any sense in which we can \emph{learn an optimiser}?
  The work of \cite{andrychowicz_learning_2016} --- appearing before any definitions of parametric lenses --- confirmed this.
  Inspired by the ability to learn any parametrised map, the authors parameterise a stateful optimiser, and \emph{learn} the gradient update function.
  Categorical foundations of such meta-learning is something we hope to thoroughly explore in future work.

  \scaletikzfig[0.8]{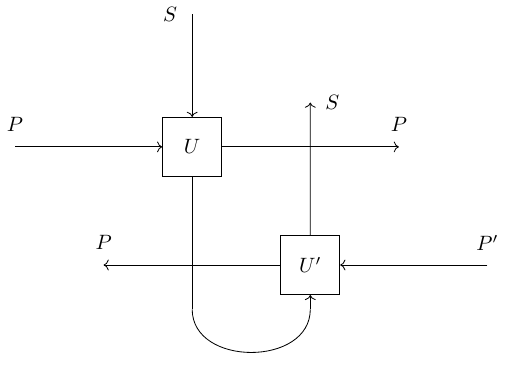}
  \captionof{figure}{A stateful optimiser (\cref{def:stateful_gradient_update}) can be seen as a parametric lens, hinting at the possibility that gradient updates can be learned. (\cite{andrychowicz_learning_2016})}
\end{mybox}

\subsubsection{Can we compose optimisers?}
\label{can_we_compose_optimisers}

Even though not explicitly acknowledged in the literature, optimisers can be composed, and this composition plays an important role in settings where deep learning intersects with multivariable optimisation.
In such settings we're interested in their \emph{parallel} composition, therefore giving a positive answer to the above question.\footnote{One might wonder whether optimisers can be composed in sequence as well. The apparent sequential composability of optimisers is unfortunately an artefact of our limited view without dependent types.}
Parallel composition of optimisers arises out of the fact that optimisers are lenses, and lenses have a monoidal product (\cref{sec:optic_monoidal}).
In such settings we might have an optimiser of two variables which descends in one direction, and ascends in the other one, for instance. 

\begin{definition}[Gradient descent-ascent (GDA)]
  \label{ex:gradient_descent_ascent}
  Given objects $P$ and $Q$, gradient descent-ascent on $P \times Q$ is a lens
  \[
    \diset{\id_{P \times Q}}{\gda} : \diset{P \times Q}{P \times Q} \to \diset{P \times Q}{P \times Q}
  \]
  where $\gda(p, q, p', q') = (p - p', q + q')$.
\end{definition}

In $\Smooth$ this gives an optimiser which descends on $P$ and ascends on $Q$.
In $\PolyZ$ this map ends up computing the same function on both parameter spaces.
This is something that ends up preventing us from modelling GANs in this setting (compare \cref{ex:gan_gda_dotproduct} where both positive and negative polarity of the optimiser map is needed).

When it comes to optimisers of two parameters, gradient descent-ascent is a particular type of an optimiser that is a product of two optimisers.
But not all optimisers can be factored in such a way, much like a general monoidal product doesn't necessarily have to be cartesian.
A good example of this is an optimiser on two parameters called \emph{competitive gradient descent} (\cite{schaefer_competitive_2019}).
We don't explicitly define or use it in this thesis, instead inviting the reader to the aforementioned reference for more information.

Of course, these were examples of composition of non-stateful optimisers.
For stateful ones the story becomes more complex, but appears to fully be captured by the already defined monoidal structure of $\Para$.
Fully defining and understanding the implications of these kinds of optimisers is an area of categorical machine learning ripe for research.

\section{Supervised learning}
\label{sec:supervised_learning}

We have now described all the main components of learning: neural networks, loss functions, learning rates and optimisers.
We are ready to put the pieces together and describe supervised learning.
Combining a model and a loss \cref{fig:model_loss_with_derivative} with the state on the left side \cref{fig:para_state_to_paraoptic_state} and costate on the right we obtain a ``scalar'', i.e.\ a closed system: a parametric lens $\diset{1}{1} \to \diset{1}{1}$, drawn in \cref{fig:learner_with_input}.

\begin{figure}[h]
  \scaletikzfig[1][0.75]{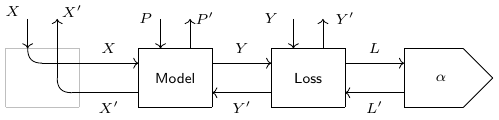}
  \caption{Closed parametric lens whose all inputs and outputs are now vertical wires.}
  \label{fig:learner_with_input}
\end{figure}

Following \cref{box:scs_para_optic} we can see that a scalar in $\Para(\Lens(\cC))$ is simply a costate in $\Lens(\cC)$ which is via \cref{box:scs_lens} isomorphic to a morphism in $\cC$.
What morphism is it?
It is one of type $\mi \times P \times \mo \to \mi' \times P' \times \mo'$ whose implementation we still suggestively call $\update$, as it is the backward map of this closed parametric lens.
That is, given a model with a forward pass $f : \mi \times P \to \mo$, a loss function $\loss : \mo \times \mo \to \lo$, a learning rate $\alpha : \lo \to \lo '$ the algebra of lens composition defines $\update$ as

\begin{align*}
  \update(x, p, \groundtruth) = (x', p', \groundtruth') \qquad \text{where} \quad \qquad \predicted  &= f(x, p) \\
  (\predicted', \groundtruth') &= R[\loss](\predicted, \groundtruth, \alpha(\loss(\predicted, \groundtruth))) \\
  (x', p') &= R[f](x, p, \predicted ')
\end{align*}

In the next two subsections we will see how this closed system can be reparameterised with an optimiser on the parameter ports to describe supervised learning of parameters, but also how it can be reparameterised with an optimiser on \emph{the input ports} describing \emph{deep dreaming}.

\subsection{Supervised learning of parameters given a loss}
\label{subsec:supervised_learning_parameters}

The most common type of learning performed on \cref{fig:learner_with_input} is supervised learning of parameters.
This is done by reparameterising it in the following manner.
The parameter ports $\diset{P}{P'}$ are reparameterised by one of the (potentially stateful) optimisers described in \cref{sec:optimisers}, while the backward wires $\mi '$ of inputs and $\mo '$ of labels are discarded.
This finally gives us a complete picture of the system which performs one step of supervised learning of parameters (\cref{fig:full_supervised_learning})

\begin{figure}[h]
  \scaletikzfig[1][0.75]{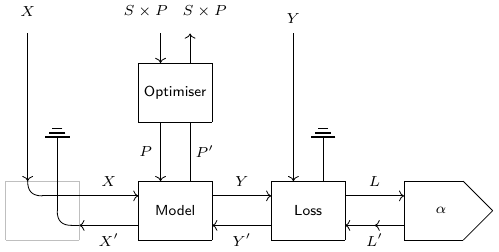}
  \caption{One step of supervised learning of parameters as a closed parametric lens. An animation of this supervised learning is available at this \href{https://giphy.com/gifs/xXl4CSLvGsd7fmeTOJ}{hyperlink}.}
  \label{fig:full_supervised_learning}
\end{figure}

Fixing a particular optimiser $\diset{U}{U*} : \diset{S \times P}{S \times P} \to \diset{P}{P'}$ we again unpack the entire construction.
By analogous arguments as before it now reduces to a map of type $\mi \times S \times P \times \mo \to S \times P$ and unpacks to:
\begin{align}
  \label{ex:supervised_learning_parameters_lens}
  \begin{split}
  \update(x, s, p, \groundtruth) = U^*(s, p, \overline{p}') \qquad \text{where} \quad \qquad \overline{p}  &= U(s, p) \\
  \predicted &= f(x, \overline{p}) \\
  (\predicted ' , \groundtruth ') &= R[\loss](\predicted, \groundtruth, \alpha(\loss(\predicted, \groundtruth))) \\
  (x', \overline{p}') &= R[f](x, \overline{p}, \predicted ')
  \end{split}
\end{align}

While this formulation might seem daunting, we note that it just explicitly specifies the entire computation performed by this supervised learning system.\footnote{It is notable that lenses forming the $\update$ in  \cref{ex:supervised_learning_parameters_lens} were composed as optics. See \cref{subsec:operational}.}
The variable $\overline{p}$ represents the parameter supplied to the network by the stateful gradient update rule (in many cases this is equal to $p$); $\predicted$ represents the prediction of the network (constrast this with $\groundtruth$ which represents the ground truth from the dataset).
Variables with a tick $'$ represent changes: $\predicted '$ and $\groundtruth '$ are the changes on predictions and ground truth label respectively, while $\mil '$ and $p'$ are changes on the inputs and the parameters.
Note that this entire update arises automatically out of the rule for lens composition: all we did was supply concrete lenses.

We justify and illustrate our approach on a series of case studies drawn from the machine learning literature, showing how in each case the parameters of our framework (in particular, loss functions and gradient updates) instantiate to familiar concepts.
This presentation has the advantage of treating all these case studies unformly in terms of our basic constructs, highlighting their similarities and differences.
We start in $\Smooth$ and fix some parametric map $(\R^p, f) : \Para(\Smooth)(\R^\mil, \R^\mol)$ and the constant learning rate $\alpha : \R$ (\cref{def:learning_rate}).
We then vary the loss function and gradient update, seeing how the update map above reduces to many of the known cases.

\begin{example}[Quadratic loss, basic gradient descent]
  \label{ex:quadratic_gradient_descent}
  Fix the quadratic error (\cref{ex:mean_squared_error}) as the loss map, and gradient descent (\cref{def:gradient_descent}) as the optimiser.
  Then the aforementioned $\update$ map simplifies.
  Since there is no state, its type reduces to $\mi \times P \times \mo \to P$, and its implementation to
  \[
    \update(x, p, \groundtruth) = p - p' \qquad \text{where} \quad \qquad (x', p') = R[f](x, p, \alpha \cdot (f(x, p) - \groundtruth))
  \]
  Note that $\alpha$ here is simply a constant, and due to the linearity of the reverse derivative we can slide the $\alpha$ from the costate into the gradient descent lens.
  Rewriting this update, and performing this sliding we obtain a closed form update step, resembling the way gradient descent with quadratic loss is usually formulated.
  \[
    \update(x, p, \groundtruth) = p - \alpha \cdot (R[f](x, p, f(x, p) - \groundtruth)) \comp \pi_1
  \]
\end{example}

\begin{figure}[h]
  \scaletikzfig[1][0.75]{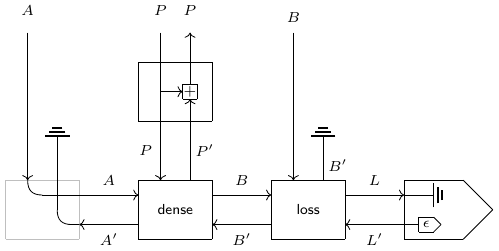}
  \caption{Closed parametric lens whose parameters are being learned.}
  \label{fig:full_supervised_learning_descent}
\end{figure}

For each parametric map $(\R^p, f) : \Para(\Smooth)(\R^\mil, \R^\mol)$ this example gives us a variety of \emph{regression} algorithms solved iteratively by gradient descent.
If the map $f$ is linear and $\mol = 1$, we recover simple linear regression with gradient descent.
If the codomain $\mol$ is multi-dimensional, i.e.\ we are predicting multiple scalars, then we recover multivariate linear regression.
Of course, by changing the underlying parametric map we can model arbitrarily complex neural network architectures.

\begin{example}[Softargmax cross-entropy, basic gradient descent]
  Fix softargmax cross-entropy (\cref{ex:softargmax_cross_entropy}) as the loss map and basic gradient descent (\cref{def:gradient_descent}) as the optimiser.
  Again the put map simplifies.
  The type reduces to $\mi \times P \times \mo \to P$ and the implementation to
  \[
    \update(x, p, \groundtruth) = p - p' \qquad \text{where} \quad \qquad (x', p') = R[f](x, p, \alpha \cdot (\Softargmax(f(x, p) - \groundtruth))
  \]
  The same rewriting performed on the previous example can be done here.
\end{example}

\begin{example}[Mean squared-error, Nesterov momentum]
  Fix the quadratic error (\cref{ex:mean_squared_error}) as the loss map, and Nesterov momentum (\cref{ex:nesterov_momentum}) the optimiser.
  This time the $\update$ map does not have simplified type, as it is still $\mi \times S \times P \times \mo \to S \times P$.
  Its implementation reduces to
  \begin{align*}
    \update(x, s, p, \groundtruth) = (s', p + s') \qquad \text{where} \quad \qquad \overline{p} &= p + \gamma s\\
    (x', \overline{p}') &= R[f](x, \overline{p}, \alpha \dot (f(x, \overline{p}) - \groundtruth))\\
    s' &= -\gamma s + \overline{p}'
  \end{align*}
\end{example}

This example with Nesterov momentum differs in two key points from all the other ones: i) the optimiser is stateful, and ii) its forward map is not trivial.
While many other optimisers are stateful, the non-triviality of the forward part here showcases the importance of lenses.
They allow us to make precise the notion of computing a ``lookahead'' value for Nesterov momentum, something that is in practice usually handled in ad-hoc ways.
Here the algebra of lens composition handles it naturally by using the forward map, a seemingly trivial, unused piece of data for previous optimisers.

Many kinds of systems that are traditionally considered unsupervised can be recast to their supervised form.
One example is GANs (\cref{sec:gan}).
By choosing $\GAN_{g, d}$ as the parametric map representing our supervised learning model, we can differentiate it as in \cref{fig:gan_differentiated}, and, with the appropriate choice of a loss function, produce the learning system in the literature called \emph{Wasserstein} GAN (\cite{arjovsky_wasserstein_2017}).

\begin{example}[GANs, Dot product, GDA]
  \label{ex:gan_gda_dotproduct}
  Fix $\GAN_{g, d}$ as the parametric map (\cref{def:gan}), gradient descent-ascent (\cref{ex:gradient_descent_ascent}) as the optimiser and dot product (\cref{ex:dot_product}) as the loss function.
  Then the update becomes a map of type $Z \times X \times P \times Q \times L \times L \to P \times Q$ and its implementation reduces to
  \begin{align*}
    \update(z, x, p, q, \groundtruth) = (p - p', q + q') \qquad \text{where} \quad \qquad (z', x', p', q') &= R[\GAN_{g, d}](z, x, p, q, \alpha \cdot \groundtruth)
  \end{align*}
  We can further unpack the label $\groundtruth = ({\groundtruth}_g, {\groundtruth}_r)$\footnote{Whose interpretation here is more of a ``mask'', see \cref{subsec:deep_dreaming}.} and via \cref{ex:gan_reverse_derivative} express the update in terms of reverse derivatives of $g$ and $d$:
  \begin{align*}
    \update(z, x_r, p, q, {\groundtruth}_g, {\groundtruth}_r) = (p - p', q + q_g' + q_r') \qquad \text{where} \quad \qquad (x_g', q_g') &= R[d](g(z, p), q, \alpha \cdot {\groundtruth}_g)\\
       (z', p') &= R[g](z, p, x_g')\\
       (x_r', q_r') &= R[d](x_r, q, \alpha \cdot {\groundtruth}_r)
  \end{align*}
  This brings us to the last step, where by linearity of the backward pass we can extract $\alpha$ and components of $\groundtruth$ out:
  \begin{align*}
    \update(z, x_r, p, q, {\groundtruth}_g, {\groundtruth}_r) = (p - \alpha {\groundtruth}_g p', q + \alpha({\groundtruth}_g q_g' + {\groundtruth}_r q_r')) \qquad \text{where} \quad (x_g', q_g') &= R[d](g(z, p), q, 1)\\
    (z', p') &= R[g](z, p, x_g')\\
    (x_r', q_r') &= R[d](x_r, q, 1)
  \end{align*}
\end{example}

The ultimate representation is a form in which it makes it possible to see how the update recovers that of Wasserstein GANs.
The last missing piece is to note that the supervision labels $y_t$ here are effectively ``masks''.
Just like in standard supervised learning an input-output pair $(x_i, y_i)$ consisted of an input value and a corresponding label which guided the direction in which the output $f(x_i, p)$ should've been improved, here the situation is the same.
Given any latent vector $z$ its corresponding ``label'' is the learning signal ${\groundtruth}_g = 1$ which does not change anything in the update, effectively signaling to the generator's and discriminator's optimisers that they should descend (minimizing the assigned loss, making the image more realistic next time), and respectively ascend (maximizing the loss, becoming better at detecting when its input  is a sample generated by the generator).
On the other hand, given any real sample $x_r$ its corresponding ``label'' is the learning signal ${\groundtruth}_r = -1$ which signals to the discriminator's optimiser that it should do the opposite of what it usually does; it should descend, causing it to assign a lower loss value actual samples from the dataset.
In other words, the input-output pairs are here always of the form $((z, x)_i, (1, -1)_i)$, making this in many ways GANs a \emph{constantly} supervised model.
Nonetheless, these different ``forces'' that pull the discriminator in different directions depending on the source of the input, coupled with the ever-changing generated inputs make GANs have intrinsically complex dynamics that are still being studied.

  The fact that we were able to encode Wasserstein GAN in this form in our framework is a consequence of its simple formulation of its loss function, which is effectively given by subtraction \cite[Theorem 3]{arjovsky_wasserstein_2017}.
  There are many things here left to future work.
  These include studying GANs outside of $\Smooth$, studying GANs with different loss functions (such as the original paper (\cite{goodfellow_generative_2014}) which uses JS-divergence), and studying GANs with different optimisers (such as \emph{competitive gradient descent} (\cite{schaefer_competitive_2019})).
  A lot of these are predicated on a categorical framework that includes dependent types in its formalism (see the end of \cref{sec:gan}), and an integration of categorical frameworks of differentiation with those of probability, which would allow a refinement of the output type of GANs (and many other neural networks) into the type $D(n)$ if a probability distribution on $n : \N$ outputs.

We finish off these examples by moving to a different base category $\PolyZ$.
This example shows that our framework describes learning in not just continuous, but discrete settings too.
Again, we fix a parametric map $(\Z^p, f) : \PolyZ(\Z^\mil, \Z^\mol)$ but this time we fix the identity learning rate (\cref{ex:learning_rate_poly_idenity}) instead of a constant one.

\begin{example}[Basic update in Boolean circuits]
  \label{ex:boolean_circuits_learning}
  Fix $\XOR$ as the loss map (\cref{ex:xor_error}), and gradient ascent (\cref{ex:gradient_ascent_polyz}) as the optimiser.
  The type of the update map again simplifies to $\mi \times P \times \mo \to P$.
  Its implementation reduces to
  \[
    \update(x, p, \groundtruth) = p + p' \qquad \text{where} \quad \qquad (x', p') = R[f](x, p, f(x, p) + \groundtruth)
  \]
\end{example}

\subsection{Deep Dreaming: Supervised Learning of Inputs}
\label{subsec:deep_dreaming}

We have seen that reparameterising the parameter port with gradient descent allows us to capture supervised learning of parameters.
In this subsection we describe how reparameterising the \emph{input port} provides us with a way to update the input datapoint --- usually an image, used to elicit a particular intepretation. (\cref{fig:dogception})
This is the idea behind the technique of deep dreaming, appearing in the literature in many forms (\cite{dosovitskiy_inverting_2016, mahendran_understanding_2015, nguyen_deep_2015, simonyan_deep_2014}).

\begin{figure}[h]
  \scaletikzfig[1][0.75]{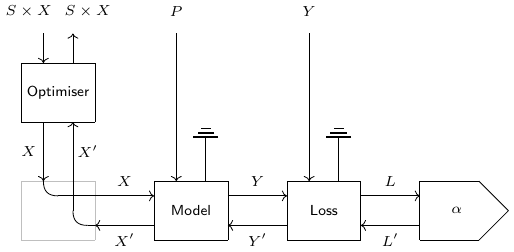}
  \caption{Deep dreaming: supervised learning of inputs.}
  \label{fig:full_deep_dreaming}
\end{figure}

Deep dreaming is a technique which uses the parameters $p$ of some trained
classifier network to iteratively dream up, or amplify some features of a class $\mol_i$ on a chosen input $\mil$.
For example, if we start with an image of a landscape $\mil_0$, a label $\mol_i$ of a ``cat'' and a parameter $p$ of a sufficiently well-trained classifier, we can start performing ``learning'' as usual: computing the predicted class for the landscape $\mil_0$ for the network with parameters $p$, and then computing the distance between the prediction and our label of a cat $\mol_i$.
When performing backpropagation, the respective changes computed for each layer tell us how the activations of that layer should have been changed to be more ``cat'' like.
This includes the first (input) layer of the landscape $\mil_0$.
Usually, we discard these changes and apply gradient update only to the parameters.
In deep dreaming we \textit{discard the parameter gradients} and apply gradient update to the input (\cref{fig:full_supervised_learning}).
Gradient update here takes these changes and computes a new image $\mil_1$ which
is the same image of the landscape, but changed slightly so to look more like
whatever the network thinks a cat looks like.
This is the essence of deep dreaming, where iteration of this process allows networks to dream up features and shapes on a particular chosen image \cite{mordvintsev_inceptionism_2015}. 

\begin{figure}[h]
  \centering
  \includegraphics[width=.7\textwidth]{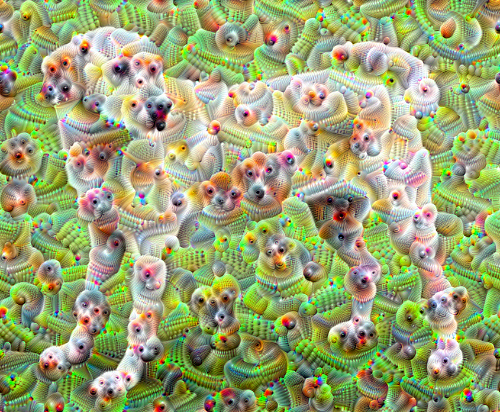}
  \caption{An image of a dog with features enhanced through the process of deep dreaming. Image taken from an \href{https://www.tensorflow.org/tutorials/generative/deepdream}{official TensorFlow tutorial}.}
  \label{fig:dogception}
\end{figure}

Just like in the previous subsection, we can write this deep dreaming system as a map in $\Para(\Lens(\cC))$ of type $\diset{1}{1} \to \diset{1}{1}$, and in an analogous way show it reduces to a map of type $S \times \mi \times P \times \mo \to S \times \mi$ whose implementation is
\begin{align*}
  \update(s, x, p, \mol_i) = U^*(s, x, \overline{x}') \qquad \text{where} \quad \qquad \overline{x}  &= U(s, x) \\
  \predicted &= f(\overline{x}, p) \\
  (\predicted ' , \mol_i ') &= R[\loss](\predicted, \mol_i , \alpha(\loss(\predicted, \mol_i))) \\
  (\overline{x}', p') &= R[f](\overline{x}, p, \predicted ')
\end{align*}

We note that deep dreaming is usually presented without any loss function as a
maximisation of a particular activation in the last layer of the network output
\cite[Sec. 2]{simonyan_deep_2014}.
As it is a maximisation, it is performed using gradient ascent, as opposed to gradient descent.
However, this is just a special case of our framework where the loss function is the dot product (\cref{ex:dot_product}).
The choice of the particular activation is encoded as a one-hot vector, and the loss function in that case essentially masks the network output, leaving active only the particular chosen activation.
We explicitly unpack this in the following example inside the base category $\Smooth$, where again we fix a parametric map $(\R^p, f) : \Para(\Smooth)(\R^\mil, \R^\mol)$.

\begin{example}[Deep dreaming, dot product, basic gradient update]
  Fix the dot product loss (\cref{ex:dot_product}), and gradient ascent (\cref{ex:gradient_ascent_smooth}) as the optimiser.
  Then the above $\update$ map simplifies to $\mi \times P \times \mo \to \mi$ as there is no state.
  Its implementation reduces to
  \[
    \update(x, p, y_i) = x + x'  \qquad \text{where} \quad \qquad (x', p') = R[f](x, p, \alpha \cdot y_i)
  \]
  Following \cref{ex:quadratic_gradient_descent}, this update can be rewritten as
  \[
    \update(x, p, y_i) = x + \alpha \cdot (R[f](x, p, y_i) \comp \pi_1)
  \]
  making a few things apparent.
  This update does not depend on the prediction $f(x, p)$: no matter what the network has predicted, the goal is always to maximize particular activations.
  Which activations?
  The ones chosen by $y_i$.
  When $y_i$ is a one-hot vector, this picks out the activation of just one class to maximize, which is often done in practice.
\end{example}

While we present only the most basic image, there is plenty of room left for exploration.
The work of \cite[Section 2.]{simonyan_deep_2014} adds an extra regularization term to the image.
In general, the neural network $f$ is sometimes changed to copy a number of internal activations which are then exposed on the output layer.
Maximizing all these activations often produces more visually appealing results.
In the literature we did not find an example which uses the softmax-cross
entropy (\cref{ex:softargmax_cross_entropy}) as a loss function in deep dreaming, which seems like the more natural choice in this setting.
Furthermore, while deep dreaming commonly uses basic gradient descent, there is nothing
preventing the use of any of the optimiser lenses
discussed in the previous section, or even doing deep dreaming in the context of
Boolean circuits.
Lastly, learning iteration which was described in at the end
of previous subsection can be modelled here in an analogous way.

\section{Supervised learning in the literature}
\label{sec:supervised_learning_in_the_literature}

Categorical models of supervised learning in the literature are scarce.
The seminal and by far the most important paper to mention is ``Backprop as Functor'' \cite{fong_backprop_2021}.

This was the first paper to explicitly tackle many of the neural network components using category theory at the same time: parameterisation, backpropagation, loss functions, and gradient descent.
It was also the first paper to use string diagrams to depict the information flow in a composite neural network.
Being seminal, it spawned a lot of subsequent research, and demonstrated viability of category theory as a modelling tool for deep learning in a manner that was not considered before.

To study parameterisation the authors defined the $\Para$ construction, which is to me the first known name of this construction.
They also defined an analogue of $\Para(\Lens(\Set))$ --- a category they called $\Learn$.
Both of these were defined as categories and not bicategories.\footnote{Though, it was acknowledged that the bicategorical setting is more natural}
Their $\Para$ definition is specialised to the self-cartesian action of $\Smooth$, and arises out of the quotienting of the bicategory by isomorphism classes (\cref{box:from2cattocat}).
The category $\Learn$ (whose morphisms they call \emph{learners}) could be seen as a special case of $\Para(\Lens(\Set))$, also quotiented out by isomorphism classes.
But there is an important distinction between $\enrichedbasechange{\Iso}(\Para(\Lens(\Set)))$ and $\Learn$: the latter arises out of a quotient by a different kind of 2-cells than those of $\Para(\Lens(\Set))$.
The 2-cells in $\Para(\Lens(\Set))$ are lenses, but those in $\Learn$ are \emph{charts} (\cref{ex:slice_charts_lenses}).

They defined backpropagation as a functor $L : \Para \to \Learn$ (\cite[Theorem III.2]{fong_backprop_2021}), and used it to model one update step of supervised learning.
However, the definition of $L$ contains a number of problems.
Firstly, it is not well-defined (see \cref{rem:backprop_not_functor}). 
Intuitively, since they do not deal with explicit 2-cells (but equivalence classes thereof) it was not noticed that this functor is not well-defined on these equivalence classes.
This is precisely because the 2-cells in $\Learn$ are not those that arise out of 2-cells of $\Para$.
Compare this with the formulation in this thesis, where reverse differential structure is a functor $\RC : \cC \to \LensA(\cC)$, and its parametric variant arises by application of $\Para$ to it (\cref{prop:para_morphism}), ensuring well-definedness by construction.

This definition is also tied to the setting of Euclidean spaces, as they do not use any of the existing categorical frameworks for differentiation such as tangent or differential categories (\cref{sec:backprop_literature}).
Since we do, we can thus go well beyond them in terms of examples --- their example of smooth functions is just one of the examples we cover in this thesis.

Thirdly, their functor does not augment the forward map with just its reverse derivative, but also with the loss map and the optimiser as well.
The functoriality condition of this functor then adds extra conditions on the partial derivatives of the loss function: it's required to be invertible in the 2\textsuperscript{nd} variable.
This constraint was not justified, nor it is a constraint that appears in deep learning practice.
For instance, it precludes usage or the L1 loss or the Hinge loss as loss functions.
When it comes to the optimiser, only simple gradient descent (\cref{def:gradient_descent}) is considered, and it is not clear whether abstracting away the optimiser would add additional conditions on its backward pass.

Related work to this is \emph{Lenses and Learners} (\cite{fong_lenses_2019}) which shows that learners can be seen as asymmetric lenses.
This arises because every parametric morphism $f : A \times P \to B$ can be seen as a span whose left leg is a projection (\cref{eq:para_span}).
\begin{equation}
  \label{eq:para_span}
  \begin{tikzcddiag}[ampersand replacement=\&]
      \& {A \times P} \\
      A \&\& B
      \arrow["f", from=1-2, to=2-3]
      \arrow["{\pi_A}"', from=1-2, to=2-1]
    \end{tikzcddiag}
\end{equation}
We belive studying $\Para$ in its dependent variant (\cref{sec:act_para_literature}) might shed light on the connection of this work to this thesis.
Lenses and Learners are related to the formalisation of supervised learning in terms of Delta lenses (\cref{sec:bidirectionality_literature}) that appeared around the same time  (\cite{diskin_general_2020}) whose relationship to the work in this thesis is not clear.

The work that provides a prescient categorical characterisation of gradient based learning, solving many of the issues outlined above is \emph{Dioptics} \cite{dalrymple_dioptics_2019}.
Despite containing only sketches and conjectures, in hindsight many of them seem to be the right ones.
In addition to formulations of differentiation with covectors described in \cref{sec:backprop_literature}, this is the first place I have personally seen $\Para$ written using two dimensional notation and the three-legged shape of neural networks. %
It is also the first work to provide a full picture of supervised learning (\cite[Slide 28]{dalrymple_presentation_2019}) with concepts of neural network, optimiser, and the loss function conceptually separated on two axes, each of them bidirectional. Dioptics did not consider the loss function as a parametric map, but did break down the optimiser box into two components: the one given by the inner product on the tangent vector space, and the one given by integrating the produced tangent vectors along a geodesic and evaluating it at $1$.
Like this thesis, it also assumed a non-dependently typed setting (i.e.\ trivialisable bundles), but on the other hand, its main definition \cite[Def.\ 3.0.1]{dalrymple_dioptics_2019} is given in terms of a coend that trivialises in many cases of interest (see \cref{box:from2cattocat}).
The author does acknowledges that the coend definition may not be the right abstraction.
As in Backprop as Functor, the analogous functor to $\Para \to \Learn$ is also defined by hand \cite[Conjecture 4.2.1]{dalrymple_dioptics_2019} and not factored through the $\Para$ construction as we do in this thesis (\cref{prop:para_morphism}).

	\part{Conclusions}

\chapter{To boldly go}
\label{ch:to_boldly_go}

\epigraph{Hofstadter's Law: It always takes longer than you expect, even when you take into account Hofstadter's Law.}{Douglas Hoftstadter}

\newthought{At the end of spring of 2018,} after I pointed out some of the issues of ``Backprop as Functor'' (\cite{fong_backprop_2021}) to one of its coauthors David Spivak (articulated in a much less clear manner than in \cref{sec:supervised_learning_in_the_literature}), he acknowledged them and enthusiastically told me ``This might be a good summer project to work on.''.
In my infinite wisdom I thought ``There is no way this takes that long''.
Four years later, here I am, writing the conclusion of my PhD thesis on this matter.

It has been a fascinating ride.
In this thesis I've described what I believe to be the start of a good foundation for studying deep learning, and beyond.
I've taken great care to distill as many conceptual moving parts as I could, but some things are still stuck together.
Organised from the most immediate to the most ambitious, the directions for further research are outlined below.

\begin{itemize}
    \item \textbf{Locally graded categories.} I expect locally graded categories to subsume actegories as a foundation of parameterisation.
    They unify both actegories and enriched categories, providing a general foundation that captures both internal and external parameterisation.
    This consequently informs the chapter on bidirectionality, as well as the chapter on differentiation\footnote{I note a resemblance of hom-sets of diffeological spaces (mentioned in \cref{sec:backprop_literature} of the chapter on backpropagation) and hom-objects of locally graded categories.}.
    Some preliminary work on this topic is in \cref{sec:locally_graded_categories,sec:locally_indexed_categories}. 
    Locally graded categories give rise to a significantly different way of thinking about some of the content in this thesis, and I'm excited to explore it.
    \item \textbf{Complete characterisation of differentiation.} This thesis contains three assumptions in modelling derivatives.
        \begin{itemize}
            \item I assume that the type of the backward pass (i.e.\ the tangent space of a particular point) of some space is isomorphic to the space itself (i.e.\ $X \cong T_x(X)$, for all $x:X$). This is the case for euclidean spaces and rings of polynomials, but not true for manifolds and many other interesting structures.
            \item I assume that the cotangent vector space (i.e.\ the space of linear functionals) at a point is isomorphic to the tangent space at that point (i.e.\ $T_x(X) \cong T_x^*(X)$ for all $x:X$). This is the case if our vector space comes equipped with a non-degenerate bilinear form, but is an assumption that prevents us form establishing a formal connection to settings where this isomorphism can't even be stated: for instance that of game theory.
            \item I assume the base category has \emph{products}.
            This models models only \emph{deterministic} systems, and precludes us from dealing with settings where we do not have at our disposal the ability to copy and delete information such as probabilistic or quantum machine learning.
        \end{itemize}
    Relaxing the first two assumptions leads to the work in \cite{capucci_diegetic_2023}\footnote{Though, in its current form at the expense of an operational view on differentiation.}, and points to a generalised theory of ``differentiation'' which exhibits fundamental connections to game theory.
    Relaxing the last assumption leads to the work of \cite{toumi_diagrammatic_2021} describing current categorical approaches to quantum machine learning.
    I hypothesise that all three generalisations in tandem can pave the way to a useful foundation of generalised theory of differentiation which covers a a wide-variety of learning schemes --- from reinforcement learning, game theory, over bayesian learning to even the lesser known automata learning (\cite{jacobs_automata_2014, urbat_automata_2020}).
    \item \textbf{Meta-learning}. I believe that findings from \cref{sec:optimisers} about the resemblance of the shape of optimisers and neural networks themselves might inform future work on meta-learning \cite{andrychowicz_learning_2016, hospedales_meta-learning_2022, finn_model-agnostic_2017}.  In \cite{gavranovic_meta-learning_2021} a question is posed: how much of the theoretical foundation of parametric optics can be used to formalise \emph{learning to learn by gradient descent by gradient descent} (\cite{andrychowicz_learning_2016})?
    This necessitates more precise categorical semantics of optimisation, and we believe that following up on progress in \cite{shiebler_kan_2022, hinze_kan_2012} can make that happen.
    \item \textbf{Taking dependent types seriously.} The \emph{types of outputs} of many processes we study are dependent on the \emph{values of their inputs}.
    Dependent type theory is the formal theory that studies this fundamental, and incredibly expressive idea.
    It allows us to treat \emph{proofs} about our systems in the same manner we treat the systems themselves: as data which can be plugged into functions as inputs, or produced by functions as outputs.\footnote{A particular reason why I find dependent types useful is because they incentivise us to be cognisant of the arrow of time, necessitating the tracking of the chain of causal dependencies. This sometimes --- almost if by magic --- prevents unwanted states or consequences \emph{from even being expressible}.}
    This paves the way for the construction of neural reasoning systems which learn to output not just arrays of numbers, but formal \emph{proofs}. This enables more rigor and reliability, consequently unlocking the usage of deep learning in new markets and sectors where this reliability is crucial.
    Most of the concepts in this thesis can be made \emph{dependent}, and I expect the complete theory of categorical deep learning to become far simpler by doing so.
    \item \textbf{New regulatory paradigms.} Current regulatory practices for AI are mostly \emph{retroactive}: they are focused on explaining decisions a model has already made.
    This is because there is no formal language to express the kinds of reasoning we'd proactively want to uphold a model to. %
    For instance, most unifying approaches to the theory of architectures of deep learning --- such as Geometric Deep Learning (\cite{bronstein_geometric_2021}) --- give us a mathematical theory of equivariance and invariance of deep learning models.
    This theory enables regulatory agencies to construct falsifiable tests for invariance to protected classes such as race and gender in large-scale deployed systems.
    But geometric deep learning inherently deals only with explainability and formal guarantees of models with respect to simple kinds of transformations: those given given by the algebraic structure of a group, such as translation, rotation, and scaling.
    Because of the group constraints, all of these transformations are required to be invertible, meaning geometric deep learning cannot handle recursion, transition functions of automata/dynamical systems, aggregation steps of a dynamic programming algorithm, or general programs which write to or read from memory.
    Geometric deep learning thus deals with transformations which are inherently \emph{geometric}, as opposed to \emph{algebraic} in nature.
    As outlined in \cref{subsec:recursive_encoding_decoding} and my forthcoming work, category theory provides us with a theory of generalised equivariance, and the theory of initial algebras with formal specification of what it means to reason \emph{inductively}.
    In the long run I envision this enabling the construction of falsifiable tests of deep learning models with respect to more than just invertible actions such as translations, but also with respect to \emph{the kinds of reasoning} these models can use when coming to their conclusions.
    What if we could explicitly decide whether a deep learning model gets to use particular kinds of inductive arguments, or particular kinds of logics in its reasoning?
    Can we end the era of unstructured models by making structured as capable as the unstructured ones, but verifiably so?
\end{itemize}

The last direction is the one I am the most excited about.
Category theory has given me and many other people a unique vantage point to understand the world through: a language so robust and compositional, that the distinction between different scientific fields almost becomes erased.
Can we imbue our deep learning models with these insights, and enable them to model the world using categories?
This, in many ways, is the north star of deep learning: construction of models that reason in structured ways.
What better way to do that than to use the theory of structure: category theory?

  \appendix
	\chapter{Monoidal and enriched categories}

\section*{Monoidal categories}
\label{appendix:monoidal_categories}

In this thesis, monoidal, braided monoidal, and symmetric monoidal categories play a central role.

\begin{definition}[{Monoidal category (compare \cite[Def.\ 12.1]{johnson_2-dimensional_2021})}]
  \label{def:monoidal_category}
  Let $\cM$ be a category.
  We call $\cM$ a \newdef{monoidal category} if it comes equipped with
  \begin{itemize}
  \item A functor $\otimes : \cM \times \cM \to \cM$ called \emph{the monoidal product};
  \item An object $I : \cM$ called \emph{the monoidal unit};
  \end{itemize}
  and following three natural isomorphisms, called the \emph{left unitor}, \emph{right unitor} and \emph{associator}, respectively:
  \begin{equation}
    \label{eq:monoidal_isos}
\begin{tikzcddiag}[ampersand replacement=\&]
	\cM \&\& {\cM \times \cM} \&\& \cM \&\& {\cM \times \cM \times \cM} \&\& {\cM \times \cM} \\
	\\
	\&\& \cM \&\&\&\& {\cM \times \cM} \&\& \cM
	\arrow["\otimes", from=1-9, to=3-9]
	\arrow["\otimes"', from=3-7, to=3-9]
	\arrow["{\otimes \times \cM}"', from=1-7, to=3-7]
	\arrow["{\cM \times \otimes}", from=1-7, to=1-9]
	\arrow["\alpha", shorten <=15pt, shorten >=15pt, Rightarrow, 2tail reversed, from=1-9, to=3-7]
	\arrow[""{name=0, anchor=center, inner sep=0}, Rightarrow, no head, from=3-3, to=1-5]
	\arrow["{\langle \cM, I \rangle}"', from=1-5, to=1-3]
	\arrow[""{name=1, anchor=center, inner sep=0}, Rightarrow, no head, from=3-3, to=1-1]
	\arrow["{\langle I, \cM \rangle}", from=1-1, to=1-3]
	\arrow["\otimes"', from=1-3, to=3-3]
	\arrow["\rho", shorten >=6pt, Rightarrow, from=1-3, to=0]
	\arrow["\lambda"', shorten >=6pt, Rightarrow, from=1-3, to=1]
\end{tikzcddiag}
\end{equation}
whose components at each $C, D, E : \cM$ are explicitly the maps
\[
  \lambda_C : I \otimes C \xrightarrow{\cong} C \quad \text{,} \quad \rho_C : C \otimes I \xrightarrow{\cong} C \quad \text{and} \quad \alpha_{C, D, E} : (C \otimes D) \otimes E \xrightarrow{\cong} C \otimes (D \otimes E)
\]
satisfying the following coherence laws.
\begin{itemize}
\item \textbf{The pentagon identity.} This law has to hold for all $C, D, E, F : \cM$.
  It tells us that, starting with all-left-association $((C \otimes D) \otimes E) \otimes F$, all the different ways of rebracketing to all-right association $C \otimes (D \otimes (E \otimes F))$ are the same.
  \Cref{eq:monoidal_pentagonator} expresses this as a commutative diagram.
  \begin{equation}
    \label{eq:monoidal_pentagonator}
\begin{tikzcddiag}[ampersand replacement=\&]
	\& {((C \otimes D) \otimes E) \otimes F} \\
	{(C \otimes (D \otimes E)) \otimes F} \&\& {(C \otimes D) \otimes (E \otimes F)} \\
	\\
	{C \otimes ((D \otimes E) \otimes F)} \&\& {C \otimes (D \otimes (E \otimes F))}
	\arrow["{\alpha_{C, D, E} \otimes F}"'{pos=0.6}, from=1-2, to=2-1]
	\arrow["{\alpha_{C, D \otimes E, F}}"', from=2-1, to=4-1]
	\arrow["{C \otimes \alpha_{D, E, F}}"', from=4-1, to=4-3]
	\arrow["{ \alpha_{C \otimes D, E, F}}"{pos=0.7}, from=1-2, to=2-3]
	\arrow["{\alpha_{C, D, E \otimes F}}", from=2-3, to=4-3]
\end{tikzcddiag}
\end{equation}
\item \textbf{Triangle identity.} These have to hold for all $A, B : \cM$.
  Analogously to above, \cref{eq:mon_triangle} tells us that any two ways of getting rid of $I$ in $(A \otimes I) \otimes B$ are the same.
  \begin{equation}
    \label{eq:mon_triangle}
    \begin{tikzcddiag}[ampersand replacement=\&]
        {(A \otimes I) \otimes B} \&\& {A \otimes (I \otimes B)} \\
        \& {A \otimes B}
        \arrow["{\alpha_{A, I, B}}", from=1-1, to=1-3]
        \arrow["{A \otimes \lambda_B}", from=1-3, to=2-2]
        \arrow["{\rho_A \otimes B}"', from=1-1, to=2-2]
      \end{tikzcddiag}
  \end{equation}
\end{itemize}
\end{definition}

This definition, especially the coherence conditions might seem a bit contrived.
We note that they arises out of the categorification\footnote{Taken to mean \emph{vertical} categorification} of the concept of a monoid in the monoidal category $(\Set, \times, 1)$ of sets and functions to a \emph{pseudomonoid} (\cite[Sec. 3]{day_monoidal_1997}) in the monoidal 2-category $(\Cat, \times, \CatTerm)$ of categories and functors.
This pseudomonoid comes equipped with pasting diagrams that have a geometric interpretation, which give rise to the above pentagonator and triangle identities.
This will be important when we study braided and symmetric monoidal categories, as their coherences can analogously be justified.

\begin{notation}
  We often refer to the monoidal category $(\cM, \otimes, I, \alpha, \lambda, \rho)$ as simply $(\cM, \otimes, I)$, or even just $\cM$, when the monoidal structure is clear from context.
\end{notation}

Morphisms of monoidal categories are called lax monoidal functors.

\begin{definition}[{Lax monoidal functor (\cite[Def 1.2.14]{johnson_2-dimensional_2021})}]
  \label{def:lax_monoidal_functor}
  Let $(\cM, \otimes, I)$ and $(\cN, \otimes', I')$ be monoidal categories.
  A functor $F : \cM \to \cN$ is said to be \newdef{lax monoidal} if it comes equipped with:
  \begin{itemize}
  \item A natural transformation
    \begin{tikzcddiag}[ampersand replacement=\&]
        {\cM \times \cM} \& {\cN \times \cN} \\
        \cM \& \cN
        \arrow["\otimes"', from=1-1, to=2-1]
        \arrow["F"', from=2-1, to=2-2]
        \arrow["{\otimes'}", from=1-2, to=2-2]
        \arrow["{F \times F}", from=1-1, to=1-2]
        \arrow["\mu", shorten <=6pt, shorten >=6pt, Rightarrow, from=1-2, to=2-1]
      \end{tikzcddiag}
    whose component at every $X, Y : \cM$ is a morphism $\mu_{X, Y} : F(X) \otimes' F(Y) \to F(X \otimes Y)$ in $\cN$;
  \item A natural transformation
    \begin{tikzcddiag}[ampersand replacement=\&]
        \CatTerm \\
        \cM \& \cN
        \arrow["I"', from=1-1, to=2-1]
        \arrow[""{name=0, anchor=center, inner sep=0}, "{I'}", from=1-1, to=2-2]
        \arrow["F"', from=2-1, to=2-2]
        \arrow["\epsilon", shorten <=3pt, shorten >=3pt, Rightarrow, from=0, to=2-1]
    \end{tikzcddiag}
    which to the unique object $\bullet$ of $\CatTerm$ assigns a morphism $\epsilon_{\bullet} : I' \to F(I)$ in $\cN$
  \end{itemize}
  such that the following coherence laws are satisfied.
  \begin{itemize}
  \item \textbf{Associativity.} This diagram has to commute for every $X, Y, Z : \cM$.
\[\begin{tikzcddiag}[ampersand replacement=\&]
	{(F(X) \otimes' F(Y)) \otimes' F(Z)} \&\& {F(X \otimes Y) \otimes' F(Z)} \&\& {F((X \otimes Y) \otimes Z)} \\
	\\
	{F(X) \otimes' (F(Y) \otimes' F(Z))} \&\& {F(X) \otimes' F(Y \otimes Z)} \&\& {F(X \otimes (Y \otimes Z))}
	\arrow["{\alpha'_{F(X), F(Y), F(Z)}}"', from=1-1, to=3-1]
	\arrow["{\mu_{X, Y} \otimes' F(Z)}", from=1-1, to=1-3]
	\arrow["{\mu_{X \otimes Y, Z}}", from=1-3, to=1-5]
	\arrow["{F(\alpha_{X, Y, Z})}", from=1-5, to=3-5]
	\arrow["{F(X) \otimes' \mu_{Y, Z}}"', from=3-1, to=3-3]
	\arrow["{\mu_{X, Y \otimes Z}}"', from=3-3, to=3-5]
\end{tikzcddiag}\]
\item \textbf{Left and right unitality.} These diagrams have to commute for every $X : \cM$.
\[\begin{tikzcddiag}[ampersand replacement=\&]
	{I' \otimes' F(X)} \&\& {F(X)} \&\&\& {F(X) \otimes' I'} \&\& {F(X)} \\
	{F(I) \otimes' F(X)} \&\& {F(I \otimes X)} \&\&\& {F(X) \otimes' F(I)} \&\& {F(X \otimes I)}
	\arrow["{\lambda'_{F(X)}}", from=1-1, to=1-3]
	\arrow["{\epsilon \otimes' F(X)}"', from=1-1, to=2-1]
	\arrow["{\mu_{I, X}}"', from=2-1, to=2-3]
	\arrow["{F(\lambda_X)}"', from=2-3, to=1-3]
	\arrow["{\rho'_{F(X)}}", from=1-6, to=1-8]
	\arrow["{F(X) \otimes' \epsilon}"', from=1-6, to=2-6]
	\arrow["{\mu_{X, I}}"', from=2-6, to=2-8]
	\arrow["{F(\rho_{X})}"', from=2-8, to=1-8]
\end{tikzcddiag}\]
\end{itemize}
  If $\mu$ and $\epsilon$ are natural isomorphisms, we call $F$ a \newdef{strong monoidal functor}.
  If they are identities, we call it a \newdef{strict monoidal functor}.
\end{definition}

Monoidal categories and strong monoidal functors form a category $\MonCat$.
Next, we define the \emph{reverse} of a monoidal category, and show how it can be used to define the notion of a \emph{braided monoidal category}.

\begin{definition}[Reversed monoidal category, {\cite[Example 1.2.9.]{johnson_2-dimensional_2021}}]
  \label{def:reversed_monoidal_category}
  Given a category $\cM$ with a monoidal structure $(\otimes, I, \alpha, \lambda, \rho)$, we define the \newdef{reversed monoidal structure} on $\cM$ as the tuple $(\otimes^{\rev}, I, \alpha^{-1}, \rho, \lambda)$ where
  $\otimes^{\rev} \coloneqq \boxed{\cM \times \cM \xrightarrow{\swap_{\cM, \cM}} \cM \times \cM \xrightarrow{\otimes} \cM}$.\footnote{Here $\swap_{\cC, \cD} : \cC \times \cD \to \cD \times \cC$ is a functor that swaps the order of the categories in a product.}
  We often label the entirety of this structure just as $\cM^{\rev}$.
\end{definition}

\begin{definition}[{Braided monoidal category (compare \cite[Def.\ 1.2.34]{johnson_2-dimensional_2021})}]
  \label{def:braided_monoidal_category}
  Let $(\cM, \otimes, I, \alpha, \lambda, \rho)$ be a monoidal category.
  We call it a \newdef{braided monoidal category} if it comes equipped with a natural isomorphism
  \[
    \begin{tikzcddiag}[ampersand replacement=\&]
      {\cM \times \cM} \& \cM
      \arrow[""{name=0, anchor=center, inner sep=0}, "\otimes", curve={height=-12pt}, from=1-1, to=1-2]
      \arrow[""{name=1, anchor=center, inner sep=0}, "{\otimes^{\rev}}"', curve={height=12pt}, from=1-1, to=1-2]
      \arrow["\beta", shorten <=3pt, shorten >=3pt, Rightarrow, 2tail reversed, from=0, to=1]
    \end{tikzcddiag}
  \]
  whose component at every $X, Y : \cM$ is explicitly the map
  \[
    \beta_{X, Y} : X \otimes Y \to Y \otimes X
  \]
  making the following two diagrams (called the \emph{hexagon equations}) commute:

  \begin{equation}
    \label{eq:braided_hexagon}
\begin{tikzcddiag}[ampersand replacement=\&,column sep=tiny]
	\& {(A \otimes B) \otimes C} \&\&\& {A \otimes (B \otimes C)} \\
	{A \otimes (B \otimes C)} \&\& {C \otimes (A \otimes B)} \& {(A \otimes B) \otimes C} \&\& {(B \otimes C) \otimes A} \\
	\\
	{A \otimes (C \otimes B)} \&\& {(C \otimes A) \otimes B} \& {(B \otimes A) \otimes C} \&\& {B \otimes (C \otimes A)} \\
	\& {(A \otimes C) \otimes B} \&\&\& {B \otimes (A \otimes C)}
	\arrow["{\beta_{A \otimes B, C}}", from=1-2, to=2-3]
	\arrow["{\alpha_{A, B, C}}"', from=1-2, to=2-1]
	\arrow["{A \otimes \beta_{B, C}}"', from=2-1, to=4-1]
	\arrow["{\alpha^{-1}_{A, C, B}}"', from=4-1, to=5-2]
	\arrow["{\beta_{A, C} \otimes B}"', from=5-2, to=4-3]
	\arrow["{\alpha_{C, A, B}}"', from=4-3, to=2-3]
	\arrow["{\beta_{A, B \otimes C}}", from=1-5, to=2-6]
	\arrow["{\alpha^{-1}_{A, B, C}}"', from=1-5, to=2-4]
	\arrow["{\beta_{A, B} \otimes C}"', from=2-4, to=4-4]
	\arrow["{\alpha_{B, A, C}}"', from=4-4, to=5-5]
	\arrow["{B \otimes \beta_{A, C}}"', from=5-5, to=4-6]
	\arrow["{\alpha^{-1}_{B, C, A}}"', from=4-6, to=2-6]
\end{tikzcddiag}
\end{equation}
\end{definition}

The hexagon equations tell us that braiding a monoidal product of objects is the same as braiding the objects individually one by one.

We have seen that monoidal categories are pseudomonoids in $(\Cat, \times, \CatTerm)$.
Interestingly, it turns out that pseudomonoids in $(\MonCat, \times, \CatTerm)$ are precisely braided monoidal categories!
For a full formal account of this deep fact about monoidal categories we refer the interested reader to \cite{joyal_braided_1986}, here just presenting the high-level intuition.

Since a monoidal category $(\cM, \otimes, I)$ is an object of $\MonCat$, a pseudomonoid structure on it necessitates the provision of strong monoidal functors $\cM \times \cM \to \cM$ and $\CatTerm \to \cM$ (and the corresponding unitors and associators, which are required to be monoidal natural transformations).
Since $\cM$ is monoidal, we have suitable candidates for both: functors $\otimes$ and $\name{I}$.
The latter can routinely be checked to already be strong monoidal, but there are some subtleties with the former.
Inducing strong monoidal structure on $\otimes$ necessitates the provision of the laxator and unitor natural isomorphisms.
The unitor can also trivially be provided, but we need to take special care with the laxator of $\otimes$.
This is a natural transformation that is in the aforementioned reference and \cite[p. 12]{kelly_basic_1982} called the  ``middle-four interchange''.
It contains the data needed to define braiding (\cite[Prop. 3]{joyal_braided_1986}), and conversely, given a braided monoidal category, one can define the corresponding middle-four interchange \cite[Prop. 2]{joyal_braided_1986}.
In the newer literature (such as \cite[p. 12]{cruttwell_monoidal_2022}), the ``middle-four interchange'' is often just referred to as the interchanger, which is the terminology we adopt in this thesis.

\begin{definition}[{Interchanger (compare \cite[p. 12]{cruttwell_monoidal_2022} and \cite[p. 12]{kelly_basic_1982})}]
  \label{def:interchanger}
  In any braided monoidal category $(\cM, \otimes, I)$, for all $A, B, C, D : \cM$ we can form \newdef{the interchanger}, a natural isomorphism $\interchanger_{A, B, C, D} : A \otimes B \otimes C \otimes D \to A \otimes C \otimes B \otimes D$ defined as the composite
  \[
   \interchanger_{A, B, C, D} \coloneqq \boxed{A \otimes B \otimes C \otimes D \xrightarrow{A \otimes \beta_{B, C} \otimes D} A \otimes C \otimes B \otimes D}
 \]
 Here we have omitted the associators, trusting it is clear how to formally account for them.\footnote{Otherwise, see \cite[Prop. 2]{joyal_braided_1986}}
 If all the underlying objects are given by the same object $A$, we will write $\interchanger_A$ for this interchanger.
\end{definition}

Here the interchanger is not to be confused with the interchange law, a property that holds in any monoidal category (not necessarily braided).
When the monoidal categories are braided, it is possible for a morphism between them to additionally preserve this braided structure.
These are called braided monoidal functors, and are relevant to us for studying reparameterisations of monoidal actegories, crucial in defining weighted optics (see \cref{subsec:a_piece_of_notation,def:weighted_optic}).

\begin{definition}[{Braided monoidal functor (compare \cite[Def.\ 1.2.30]{johnson_2-dimensional_2021})}]
  \label{def:braided_monoidal_functor}
  Let $(\cM, \otimes, I)$ and $(\cN, \otimes', I')$ be braided monoidal categories.
  A lax monoidal functor $(F, \phi, \epsilon) : \cM \to \cN$ is \newdef{braided monoidal} if for all $X, Y : \cM$ following diagram commutes:
\begin{equation}
  \label{eq:braided_monoidal_functor}
  \begin{tikzcddiag}[ampersand replacement=\&]
      {F(X) \otimes'F(Y)} \&\& {F(Y ) \otimes' F(X)} \\
      \\
      {F(X \otimes Y)} \&\& {F(Y \otimes X)}
      \arrow["{\phi^{\mathcal{M}}_{X, Y}}"', from=1-1, to=3-1]
      \arrow["{F(\beta^{\mathcal{M}}_{X, Y})}"', from=3-1, to=3-3]
      \arrow["{\mu^{\mathcal{N}}_{Y, X}}", from=1-3, to=3-3]
      \arrow["{\beta^{\mathcal{N}}_{F(X), F(Y)}}", from=1-1, to=1-3]
    \end{tikzcddiag}
\end{equation}
We call the braided monoidal functor \newdef{strong} (resp. \newdef{strict}) if the lax monoidal functor $(F, \phi, \epsilon)$ is strong (resp. strict).
\end{definition}

Braided monoidal categories and strong braided monoidal functors also form a category.
We can now continue the pattern and ask: if pseudomonoids in $(\Cat, \times, \CatTerm)$ are monoidal categories, pseudomonoids in $(\MonCat, \times, \CatTerm)$ are braided monoidal categories, then what are pseudomonoids in this newly defined category of braided monoidal categories and strong braided monoidal functors?

The answer is \emph{symmetric monoidal categories}.
We give an explicit definition.
\begin{definition}[{Symmetric monoidal category (compare \cite[Def.\ 1.2.24]{johnson_2-dimensional_2021})}]
  \label{def:symmetric_monoidal_category}
  Let $(\cM, \otimes, I, \beta)$ be a braided monoidal category.
  We call it a \newdef{symmetric monoidal category} if braiding twice is the same as not doing anything at all; that is, if the following diagram commutes:
  \[\begin{tikzcddiag}[ampersand replacement=\&]
      {A \otimes B} \&\& {A \otimes B} \\
      \& {B \otimes A}
      \arrow["{\beta_{A, B}}"', from=1-1, to=2-2]
      \arrow["{\beta_{B, A}}"', from=2-2, to=1-3]
      \arrow[Rightarrow, no head, from=1-1, to=1-3]
    \end{tikzcddiag}\]
\end{definition}

When $\cM$ is symmetric, then either one of the hexagon identities is redundant (\cite[p. 2]{joyal_braided_1986}).
\begin{remark}
  \label{rem:affine_traversals}
While many braided monoidal categories are in practice symmetric, this is not always the case.
A notable exception is the monoidal product used to define a class of optics called \emph{affine traversals} (see \cite[Prop. 5.2.7]{capucci_actegories_2023} and \cite[Prop. 3.25]{clarke_profunctor_2022}).
\end{remark}

Symmetric monoidal functors, i.e.\ morphisms between symmetric monoidal categories are, interestingly, still just braided monoidal functors, and no extra condition is needed here.

And lastly, if we try iterating the previous process, and unpack what pseudomonoids in the 2-category of symmetric monoidal categories and symmetric monoidal functors is, we would get what we had already: a symmetric monoidal category.
This curiousity is a part of the general periodic table of higher categories, and we refer the reader to \cite[Sec. 2.1]{baez_lectures_2010}.

\section*{Enriched categories}
\label{appendix:enriched_categories}

Monoidal categories allow us to define the notion of an \emph{enriched} category, a generalisation of a category where morphisms between objects need not form a set.

\begin{definition}[{Enriched category (compare \cite[Sec. 1.2]{kelly_basic_1982} or \cite[Def. 1.3.1]{johnson_2-dimensional_2021})}]
  \label{def:enriched_category}
  Let ${(\cV, \otimes, I, \alpha, \lambda, \rho)}$  be a monoidal category.
  Then a \newdef{$\cV$-enriched category} $\cC$ consists of the following data.
  \begin{itemize}
  \item A set $\Ob[false](\cC)$ of \emph{objects}. We will often denote this set with just $\cC$, trusting the meaning is clear from the context;
  \item For every $A, B : \cC$ an object of $\cV$ denoted as $\cC(A, B)$. We think of it as the \emph{hom-object}, or \emph{object of morphisms} from $A$ to $B$;
  \item For every $A : \cC$ a morphism of type $I \to \cC(A, A)$ in $\cV$ denoted as $\name{\id_A}$ and called the \emph{name}\footnote{It's called the \emph{name} of identity because $\name{\id_A}$ is not the identity morphism, but instead it ``picks it out'', or ``names it'' in $\cC(A, A)$.} of the identity morphism on $A$;
  \item For every $A, B, C : \cC$, a morphism of type $\cC(A, B) \otimes \cC(B, C) \to \cC(A, C)$ in $\cV$ denoted as $\comp_{A, B, C}$ and called \emph{the composition}.
  \end{itemize}
  This data has to satisfy the following coherence conditions expressed as commutative diagrams in $\cV$.
  The first one (\cref{eq:enriched_associativity}) which is defined for all $A, B, C, D : \cC$ ensures the composition is associative.
  The second one (\cref{eq:enriched_unitality}), defined for all $A, B : \cC$ ensures that it is unital from both left and right side.
  \begin{equation}
    \label{eq:enriched_associativity}
\begin{tikzcddiag}[ampersand replacement=\&,column sep=scriptsize]
	{(\cC(A, B) \otimes \cC(B, C)) \otimes \cC(C, D)} \&\&\& {\cC(A, C) \otimes \cC(C, D)} \&\& {\cC(A, D)} \\
	\\
	{\cC(A, B) \otimes (\cC(B, C) \otimes \cC(C, D))} \&\&\& {\cC(A, B) \otimes \cC(B, D)} \&\& {\cC(A, D)}
	\arrow["{\alpha_{\cC(A, B), \cC(B, C), \cC(C, D)}}"', from=1-1, to=3-1]
	\arrow["{\comp_{A, B, C} \otimes \cC(C, D)}", from=1-1, to=1-4]
	\arrow["{\cC(A, B) \otimes \comp_{B, C, D}}"', from=3-1, to=3-4]
	\arrow["{\comp_{A, C, D}}", from=1-4, to=1-6]
	\arrow["{\comp_{A, B, D}}"', from=3-4, to=3-6]
	\arrow[Rightarrow, no head, from=1-6, to=3-6]
\end{tikzcddiag}
\end{equation}

\begin{equation}
  \label{eq:enriched_unitality}
\begin{tikzcddiag}
	{I \otimes \cC(A, B)} &&&&& {\cC(A,B) \otimes I} \\
	\\
	{\cC(A, A) \otimes \cC(A, B)} && {\cC(A, B)} & {\cC(A, B)} && {\cC(A, B) \otimes \cC(B, B)}
	\arrow["{\cC(A, B) \otimes \name{\id_B}}", from=1-6, to=3-6]
	\arrow["{\comp_{A, A, B}}"', from=3-1, to=3-3]
	\arrow["{\lambda_{\cC(A, B)}}", from=1-1, to=3-3]
	\arrow["{\name{\id_A} \otimes \cC(A, B)}"', from=1-1, to=3-1]
	\arrow["{\rho_{\cC(A, B)}}"', from=1-6, to=3-4]
	\arrow["{\comp_{A, B, B}}", from=3-6, to=3-4]
\end{tikzcddiag}
\end{equation}

\end{definition}

\begin{notation}
  In the nomenclature ``$\cV$-enriched category $\cC$''(or simply a ``$\cV$-category $\cC$'') we trust that the monoidal structure of $\cV$ is clear from the context.
  When multiple enriched categories are in scope, we might superscript the composition ($\comp$) and identity maps ($\id$) with the respective category to differentiate between them.
\end{notation}

In place of functors, we have \emph{enriched} functors.
\begin{definition}[{Enriched functor (compare \cite[Sec. 1.2]{kelly_basic_1982} or \cite[Def. 1.3.5]{johnson_2-dimensional_2021})}]
  \label{def:enriched_functor}
  Let $(\cV, \otimes, I)$ be a monoidal category.
  Let $\cC, \cD$ be $\cV$-enriched categories.
  A \newdef{$\cV$-enriched functor} (or just a \newdef{$\cV$-functor}) $F : \cC \to \cD$ consists of the following data.
  \begin{itemize}
  \item A function $F : \Ob[false](\cC) \to \Ob[false](\cD)$ between the underlying sets of objects\footnote{We are overloading the notation $F$ here, trusting it is clear from the context what $F$ is applied to.};
  \item For every $A, B : \cC$ a morphism in $\cV$ of type $\cC(A, B) \to \cD(F(A), F(B))$, often denoted as $F_{A, B}$ or, when it is clear from the context, just $F$;
  \end{itemize}
  This data has to, for every $A, B, C : \cC$ satisfy the following coherence laws ensuring $F$ is coherent with composition and identities, respectively.
\[\begin{tikzcddiag}[ampersand replacement=\&]
	{\cC(A, B) \otimes \cC(B, C)} \&\& {\cC(A, C)} \& I \\
	\\
	{\cD(F(A), F(B)) \otimes \cD(F(B), F(C))} \&\& {\cD(F(A), F(C))} \& {\cC(A, A)} \& {\cD(F(A), F(A))}
	\arrow["{\comp^{\cC}_{A, B,C}}", from=1-1, to=1-3]
	\arrow["{F_{A, B} \otimes F_{B, C}}"', from=1-1, to=3-1]
	\arrow["{\comp^{\cD}_{F(A), F(B), F(C)}}"', from=3-1, to=3-3]
	\arrow["{F_{A, C}}", from=1-3, to=3-3]
	\arrow["{\name{\id^{\cC}_A}}"', from=1-4, to=3-4]
	\arrow["{\name{\id^{\cD}_{F(A)}}}", from=1-4, to=3-5]
	\arrow["{F_{A, A}}"', from=3-4, to=3-5]
\end{tikzcddiag}\]
\end{definition}

\begin{definition}[{Category of $\cV$-enriched categories (compare \cite[Sec. 1.2]{kelly_basic_1982})}]
Given any monoidal category $(\cV, \otimes, I)$, $\cV$-enriched categories and $\cV$-enriched functors form a category $\cV\dsh\Cat$.
\end{definition}

In this thesis there are two pertinent bases of enrichment.
For the base of enrichment $(\Set, \times, 1)$ we obtain $\Set\text{-}\Cat$, the category of $\Set$-enriched categories and $\Set$-functors: these are precisely just categories and functors!
That is, $\Cat \coloneqq \Set\text{-}\Set$\footnote{We remind the reader that we completely ignore universe levels here.}.
Analogously, for the base of enrichment $(\Cat, \times, \CatTerm)$ we obtain $\Cat\text{-}\Cat$, the category of $\Cat$-enriched categories and $\Cat$-functors: these are precisely 2-categories and 2-functors.

\begin{definition}[{Enriched base change (compare \cite[Sec. 1]{kelly_basic_1982})}]
  \label{def:enriched_base_change}
  Let $(\cV, \otimes, I)$  and $(\cW, \oplus, J)$ be two monoidal categories.
  Then a lax monoidal functor
  \[
    (F, \mu, \epsilon) : (\cV, \otimes, I) \to (\cW, \oplus, J)
  \]
  induces the \newdef{enriched base change}, an assignment which sends a $\cV$-enriched category $\cC$ to the $\cW$-enriched category $\enrichedbasechange{F}(\cC)$ whose data is defined as follows:
  \begin{itemize}
  \item Its objects are those of $\cC$;
  \item Its hom-object is for every $A, B : \enrichedbasechange{F}(\cC)$ defined as $\enrichedbasechange{F}(\cC)(A, B) \coloneqq F(\cC(A, B))$;
  \item The identity at $A : \enrichedbasechange{F}(\cC)$ is the composite $J \xrightarrow{\epsilon} F(I) \xrightarrow{F(\id_A)} F(\cC(A, A)) = \enrichedbasechange{F}(\cC)(A, A)$;
  \item For every $A, B, C : \enrichedbasechange{F}(\cD)$ the composition morphism $\comp^{\enrichedbasechange{F}(\cC)}_{A, B, C}$ is defined as the composite below, expressed as the commutativity of the diagram.
\[\begin{tikzcddiag}[ampersand replacement=\&]
	{\enrichedbasechange{F}(\cC)(A, B) \otimes \enrichedbasechange{F}(\cC)(B, C)} \&\&\&\& {\enrichedbasechange{F}(\cC)(A, C)} \\
	\\
	{F(\cC(A, B)) \otimes F(\cC(B, C))} \&\& {F(\cC(A, B) \otimes \cC(B, C))} \&\& {F(\cC(A, C))}
	\arrow[Rightarrow, no head, from=1-1, to=3-1]
	\arrow["{\mu_{\cC(A, B), \cC(B, C)}}"', from=3-1, to=3-3]
	\arrow["{F(\comp^{\cC}_{A, B, C})}"', from=3-3, to=3-5]
	\arrow[Rightarrow, no head, from=3-5, to=1-5]
	\arrow["{\comp^{\enrichedbasechange{F}(\cC)}_{A, B, C}}", from=1-1, to=1-5]
\end{tikzcddiag}\]

  \end{itemize}

  The assignment $\enrichedbasechange{F}$ forms a functor
  \[
    \enrichedbasechange{F} : \cV\dsh\Cat \to \cW\dsh\Cat
  \]
  whose details we do not unpack here.
\end{definition}

\section*{Actegories}
\label{appendix:actegories}

\begin{definition}[{Coherence laws for a right actegory (compare \cite[Def.\ 3.1.1]{capucci_actegories_2023})}]
  \label{def:act_coherence}
  Coherence laws for a $\cM$-actegory $(\cC, \act, \mu, \eta)$ are analogous to the coherence laws of a monoidal category (\cref{def:monoidal_category}).
  In other words, we have two laws:
  \begin{itemize}
    \item \textbf{Pentagonator.} This law has to hold for all $M, N, P : \cM$ and $C : \cC$.
      It tells us that, starting with $C \act (M \otimes (N \otimes P))$, any two ways of acting on $C$ give us the same result.
      We can either apply actions of $\mu$ in the only way possible on this, or alpply the associator and \emph{then} apply the actions of $\mu$ in the only way possible.
      This is what the diagram \cref{eq:act_pentagonator} describes: a condition that has to hold.\footnote{Its name ``pentagonator'' arises out of the squashing of the righmost identity into a single object. While this seems justified in the monoidal case, the added distinction between the associator and multiplicator in an actegory in my opinion suggests a different conceptualisation of the diagram.}
\begin{equation}
    \label{eq:act_pentagonator}
\begin{tikzcddiag}[ampersand replacement=\&]
	{C \act (M \otimes (N \otimes P))} \&\& {(C \act M) \act (N \otimes P)} \&\& {((C \act M) \act N) \act P} \\
	\\
	{C \act ((M \otimes N) \otimes P))} \&\& {(C \act (M \otimes N)) \act P} \&\& {((C \act M) \act N) \act P}
	\arrow["{\mu_{C, M, N \otimes P}}", from=1-1, to=1-3]
	\arrow["{\mu_{C \act M, N, P}}", from=1-3, to=1-5]
	\arrow["{C \act \alpha^{-1}_{M, N, P}}"', from=1-1, to=3-1]
	\arrow["{\mu_{C, M \otimes N, P}}"', from=3-1, to=3-3]
	\arrow["{\mu_{C, M, N} \act P}"', from=3-3, to=3-5]
	\arrow[Rightarrow, no head, from=1-5, to=3-5]
\end{tikzcddiag}
\end{equation}
  \item \textbf{Unitors.} These laws have to hold for all $C : \cC$ and $M : \cM$.
    The diagrams in \cref{eq:act_unitors} tell us that any two ways of getting rid of $I$ in $\cC \act (M \otimes I)$ (resp.\ $\cC \act (I \otimes M)$) are equal.
    That is, we get the same result as if we a) apply the right unitor $\rho_M$ (resp.\ left unitor $\lambda_M$) of the monoidal category, or b) apply the multiplicator of the actegory, and then the unitor of the actegory.
    See also \cite[Remark 3.1.2]{capucci_actegories_2023}.
    \begin{equation}
      \label{eq:act_unitors}
      \begin{tikzcddiag}
      	{C \act (M \otimes I)} &&&&& {C \act (I \otimes M)} \\
      	\\
      	{(C \act M) \act I} && {C \act M} & {C \act M} && {(C \act I) \act M}
      	\arrow["{\mu_{C, M, I}}"', from=1-1, to=3-1]
      	\arrow["{C \act \rho_M}", from=1-1, to=3-3]
      	\arrow["{\eta_{C \act M}}"', from=3-1, to=3-3]
      	\arrow["{C \act \lambda_M}"', from=1-6, to=3-4]
      	\arrow["{\mu_{C, I, M}}", from=1-6, to=3-6]
      	\arrow["{\eta_C \act M}", from=3-6, to=3-4]
      \end{tikzcddiag}
    \end{equation}
\end{itemize}
  
\end{definition}

\BraidedRightToLeft*
\begin{proof}
Let $\beta : \otimes \Rightarrow \otimes^{\rev}$ be the braiding of $(\cM, \otimes, I)$.
Then given a right $(\cM, \otimes, I)$-actegory $(\cC, \act, \eta, \mu)$ we can define a right $(\cM, \otimes^{\rev}, I)$ actegory (i.e.\ a left $(\cM, \otimes, I)$-actegory) with the same base $\cC$, whose structure maps are $(\actfw, \eta, \mu^{\rev[false]})$.
That is, the action functor and the unitor are the same as the starting actegory, because they can be defined without a reference to the underlying monoidal product of the acting actegory.
The only difference is in the multiplicator $\mu^{\rev[false]}$ which is defined by the following pasting diagram.
  \begin{equation}
    \label{eq:act_right_to_left}
\begin{tikzcddiag}[ampersand replacement=\&]
	{\cC \times \cM \times \cM} \&\& {\cC \times \cM} \\
	\\
	{\cC \times \cM} \&\& \cC
	\arrow["\act", from=1-3, to=3-3]
	\arrow["\act"', from=3-1, to=3-3]
	\arrow["{\act \times \cM}", from=1-1, to=1-3]
	\arrow["\mu", shorten <=15pt, shorten >=15pt, Rightarrow, 2tail reversed, from=1-3, to=3-1]
	\arrow[""{name=0, anchor=center, inner sep=0}, "{\cC \times \otimes}", curve={height=-18pt}, from=1-1, to=3-1]
	\arrow[""{name=1, anchor=center, inner sep=0}, "{\cC \times \otimes^{\rev}}"', curve={height=18pt}, from=1-1, to=3-1]
	\arrow["{\cC \times \beta}", shorten <=7pt, shorten >=7pt, Rightarrow, 2tail reversed, from=1, to=0]
\end{tikzcddiag}
\end{equation}

Componentwise, for $C : \cC$ and $M, N : \cM$ it gives us a natural isomorphism
\[
  \mu^{\rev[false]}_{C, M, N} \coloneqq \boxed{C \act (M \otimes^{\rev} N) \xrightarrow{C \times \beta_{M, N}} C \act (M \otimes N) \xrightarrow{\mu_{C, M, N}} (C \act M) \act N}
\]

To show that $(\cC, \act, \eta, \mu^{\rev[false]})$ is indeed an actegory we have to show that the appropriate coherence conditions (\cref{def:act_coherence}) are satisfied.
This follows routinely by unpacking the corresponding pasting diagrams and applying the naturality of $\beta$.

\end{proof}

\chapter{Locally graded and indexed categories}

\section*{Locally graded categories}
\label{sec:locally_graded_categories}

A common generaliation of enriched categories and actegories are \emph{locally graded} categories, originally defined in \cite{wood_indicial_1976,johnstone_v-indexed_1978}.\footnote{See also \cite[Lemma 12]{garner_embedding_2018}, \cite[Def. 5.1.9]{gowers_crossroads_2020}, \cite[Prop. 2]{mellies_parametric_2022}, and \cite[Def.\ 9]{mcdermott_flexibly_2022}}
They are categories enriched in $\internalHom{\Cat}{\cM^{\op}}{\Set}$, for a given monoidal category $\cM$.
They capture both internal parameterisation (one given by actegories) and external parameterisation (one given by enriched categories).

Analogously to actegories, they have a corresponding $\Para$ construction which constructs a bicategory of parametric morphisms.
We start by defining the monoidal structure of the category $\internalHom{\Cat}{\cM^{\op}}{\Set}$ which is needed to talk about enrichment in it.
It is given by \emph{Day Convolution}.

\begin{definition}[{Day convolution (compare \cite[Prop. 6.2.1]{loregian_coend_2021}, \cite[Def. 11]{garner_embedding_2018})}]
  Let $(\cM, \otimes, I)$ be a monoidal category.
  Then the functor category $\internalHom{\Cat}{\cM}{\Set}$ can be equipped with a monoidal structure given by \newdef{Day convolution}.
  It is defined by the following data:
  \begin{itemize}
  \item The monoidal product $\otimes_{D}$ as
    \begin{align*}
      \internalHom{\Cat}{\cM}{\Set} \times  \internalHom{\Cat}{\cM}{\Set} &\to \internalHom{\Cat}{\cM}{\Set}\\
      (F, G) &\mapsto \int^{(M, N) : \cM \times \cM} F(M) \times G(N) \times \cM(M \otimes N, -)
    \end{align*}
  \item The monoidal unit as $\cM(I, - ) : \cM \to \Set$
  \end{itemize}
  Associators and unitors are defined in \cite[Prop 6.2.1]{loregian_coend_2021}.
\end{definition}

\begin{remark}
  To understand why this is called Day \emph{convolution}, it is helpful to consider the case where $\cM$ is a discrete monoidal category, i.e.\ a monoid.
  In such a case the coend reduces to a coproduct, and morphisms in $\cM$ to equalites.
  In other words, we have
  \begin{align*}
    (F \otimes_D G)(X) &= \int^{M, N} F(M) \times G(N) \times \cM(M \otimes N, X)\\
                       &= \sum_{\substack{M, N : \cM \\ M \otimes N = X}} F(M) \times G(N)
  \end{align*}
  Furthermore specialising $\cM$ to, for instance, the group $(\Z, +, 0)$ of integers with addition allows us to write the above as
  \[
     (f \otimes_D g)(x) = \sum_{m : \Z} f(m) \times g(x - m)
    \]
  making the connection with classical convolution formula more apparent (we have taken the liberty of turning variable names lowercase).
\end{remark}

To better understand this monoidal structure, we study one if its interesting properties: monoids in this monoidal category are lax monoidal functors.
And to understand that, we will first study the set of natural transformations from $\cM(I, -)$ to any other functor $F$ (\cref{prop:day_conv_maps_from_unit}) and the set of natural transformations $F \otimes_D G \Rightarrow H$ (for any three functors $F, G, H$) (\cref{prop:day_conv_maps_monoid_like}), showing both of these admit a natural reduction.

\begin{proposition}
  \label{prop:day_conv_maps_from_unit}
  Let $(\cM, \otimes, I)$ be a monoidal category, and $F : \cM \to \Set$ a functor.
  Then
  \[
    \internalHom{\Cat}{\cM}{\Set}(\cM(I, -), F) \cong F(I)
  \]
\end{proposition}

\begin{proof}
  Follows by unpacking the set of natural transformations as an end, and applying Yoneda. 
\end{proof}

Concretely, this means that every element $i : F(I)$ can be identified with the natural transformation $F(-)(i) : \internalHom{\Set}{\cM(I, -)}{F}$, i.e.\ an assignment $F(-)(i) : \cM(I, X) \to F(X)$ for every $X : \cM$, and vice versa.

\begin{proposition}
  \label{prop:day_conv_maps_monoid_like}
  Let $(\cM, \otimes, I)$ be a monoidal category.
  Then
  \[
    \internalHom{\Cat}{\cM}{\Set}(F \otimes_D G, H) \cong \int_{(M, N): \cM \times \cM}\Set(F(M) \times G(N), H(M \otimes N))
  \]
\end{proposition}

\begin{proof}
  Follows by the coend calculus below.
  \begin{alignat*}{2}
    & &&\internalHom{\Cat}{\cM}{\Set}(F \otimes_D G, H)\\
    &\text{(Natural transformation as end)} &\cong&\int_{P : \cM}\Set((F \otimes_D G)(P), H(P))\\
    &\text{(Day convolution)} &\cong&\int_{P : \cM}\Set(\int^{(M, N) : \cM \times \cM} \cM(M \otimes N, P) \times F(M) \times G(N), H(P))\\
    & \text{(Cocontinuity)}  &\cong& \int_{(P, M, N): \cM \times \cM \times \cM}\Set(\cM(M \otimes N, P) \times F(M) \times G(N), H(P))\\
    & \text{(Tensor-hom adjunction)}  &\cong& \int_{(P, M, N): \cM \times \cM \times \cM}\Set(\cM(M \otimes N, P), \internalHom{\Set}{F(M) \times G(N)}{H(P)})\\
    & \text{(Yoneda)}  &\cong& \int_{(M, N): \cM \times \cM}\Set(F(M) \times G(N), H(M \otimes N))
  \end{alignat*}
\end{proof}

\begin{proposition}[Lax monoidal functors as monoids]
  There is a 1-1 correspondence between monoids in $(\internalHom{\Cat}{\cM}{\Set}, \otimes_D, \cM(I, -))$ and lax monoidal functors of type $(\cM, \otimes, I) \to (\Set, \times, 1)$.
\end{proposition}

\begin{proof}
  A lax monoidal functor consists of a functor $F : \cM \to \Set$ and two natural transformations $\mu : \int_{(M, N) : \cM \times \cM}\Set(F(M) \times F(N), F(M \otimes N))$ and $\epsilon : \int_{\bullet : \CatTerm}\Set(1, F(I))$.
  A monoid in $\internalHom{\Cat}{\cM}{\Set}$ consists of a functor $F$ and two natural transformations $\mu : F \otimes_D F \Rightarrow F$ and $\epsilon : \cM(I, -) \Rightarrow F$.
  Via \cref{prop:day_conv_maps_monoid_like,prop:day_conv_maps_from_unit} it becomes straightforward to translate these representations one to another, and confirm that they are indeed well defined, and inverses.
\end{proof}

Since this is a monoidal category, we can use it as the base of enrichment.

\begin{definition}[{Locally graded category (\cite[Def.\ 9]{mcdermott_flexibly_2022})}]
  Let $(\cM, \otimes, I)$ be a monoidal category.
  Then we call a $(\internalHom{\Cat}{\cM^{\op}}{\Set}, \otimes_D, \cM^{\op}(I, -))$-enriched category a \newdef{locally $\cM$-graded} category.
\end{definition}

Let us unpack, following \cref{def:enriched_category} what such an enriched category (call it $\cC$) consists of.
\begin{itemize}
\item \textbf{Objects.} Like any enriched category, it consists of a set of objects.
\item \textbf{Morphisms.} Given any two objects $A, B$ of the hom-object $\cC(A, B)$ is a \emph{functor} of type $\cM^{\op} \to \Set$.
  Since a functor does not have elements, this $\internalHom{\Cat}{\cM^{\op}}{\Set}$-enriched category $\cC$ does not have morphisms in the conventional sense.
  But there are ``graded morphisms'', i.e.\ for each object $M : \cM^{\op}$ a set of morphisms $\cC(A, B)(M)$;%
\item \textbf{Identity.} To talk about identities, we make use of \cref{prop:day_conv_maps_from_unit} which tells us that the set of possible identity natural transformations $\cM^{\op}(I, -) \Rightarrow \cC(A, A)$ is isomorphic to the set $\cC(A, A)(I)$.
  Thus we refer to the identity morphism $\id_A$ as an element of $\cC(A, A)(I)$;
\item \textbf{Composition.} Analogous story follows here. Via \cref{prop:day_conv_maps_monoid_like} the set of possible composition natural transformations $\cC(A, B) \otimes_D \cC(B, C) \Rightarrow \cC(A, C)$ is isomorphic to the set $\int_{(M, N): \cM^{\op} \times \cM^{\op}}\Set(\cC(A, B)(M) \times \cC(B, C)(N), \cC(A, C)(M \otimes N))$.
  Thus we refer to the composition morphism $\comp_{A, B, C}$ as a natural assignment of a function
  \[
    \cC(A, B)(M) \times \cC(B, C)(N) \to \cC(A, C)(M \otimes N)
  \]
  to every $A, B, C : \cC$ and $M, N : \cM^{\op}$.
\end{itemize}

  We can now see how locally graded categories are a generalisation of actegories and enriched categories.

\begin{example}[Actegories are locally graded categories]
  \label{ex:actegories_are_locally_graded_categories}
  Let $(\cC, \act)$ be a $\cM$-actegory.
  Then we can construct a locally $\cM$-graded category $\cC'$ with the following data.
  \begin{itemize}
  \item Its objects are same as that of $\cC$;
  \item For any $A, B : \cC'$, the hom functor $\cC'(A, B)$ is the composite $\cM^{\op} \xrightarrow{A \act -} \cC^{\op} \xrightarrow{\cC(-, B)} \Set$;
  \item For every $A : \cC'$ the set of potential identity morphisms $\cC'(A, A)(I)$ reduces to the set $\cC(A \act I, A)$ for which there is a clear candidate: the actegory unitor $\eta^{-1}_A : \cC(A \act I, A) = \cC'(A, A)(I)$;
  \item For every $A, B, C : \cC'$ the set of potential composition natural transformations reduces to the set $\int_{(M, N) : \cM^{\op} \times \cM^{\op}}\Set(\cC(A \act M, B) \times \cC(B \act N, C) , \cC(A \act (M \otimes N), C))$ for which there is a clear candidate: a function 
    \begin{align*}
      \cC(A \act M, B)) \times \cC(B \act N, C)) &\to \cC(A \act (M \otimes N), C))\\
      (f,   g) &\mapsto \boxed{A \act (M \otimes N) \xrightarrow{\mu_{C, M, N}} (A \act M) \act N \xrightarrow{f \act N} B \act N  \xrightarrow{g} C}
    \end{align*}
    defined naturally for each $(M, N) : \cM^{\op} \times \cM^{\op}$.
    Note that this is precisely the $\Para$ composition --- describing the manner by which two parametric morphisms are composed.
  \end{itemize}
\end{example}

\begin{remark}
  The unitor and the multiplicator of the actegory --- which are isomorphisms --- are used to define the identity and composition of a locally graded category.
  But note that only one direction of this isomorphism is used to define a locally graded category.
  This suggests that lax actegories too embed into locally graded categories (see \cite[Def. 5.1.9]{gowers_crossroads_2020}).
\end{remark}

Actegories, seen as locally graded categories come with an additional property: they have copowers by all representable functors.
As it also turns out, every locally graded category with copowers by all representables can be turned into an actegory.

\begin{proposition}[{\cite[Prop. 13]{garner_embedding_2018}}]
  \label{prop:copower_representable}
  Let $\cM$ be a symmetric monoidal category.
  Then there is a 1-1 correspondence between $\cM$-actegories and locally $\cM$-graded categories admitting powers by representables.
\end{proposition}

\begin{example}[Enriched categories are locally graded categories]
  \label{ex:enriched_categories_are_locally_graded_categories}
  Let $\cC$ be a $\cM$-category.
  Then we can construct a locally $\cM$-graded category $\cC'$ with the following data.
  \begin{itemize}
  \item Its objects are same as that of $\cC$;
  \item For any $A, B : \cC'$, the hom functor $\cC'(A, B)$ is the composite $\cM^{\op} \xrightarrow{\cM(-, \cC(A, B))} \Set$;
  \item For every $A : \cC'$ the set of identity morphisms $\cC'(A, A)(I)$ reduces to the set $\cM(I, \cC(A, A))$ for which there is a clear candidate: the name of the identity map $\name{\id_A} : \cM(I, \cC(A, A))$ of $\cC$;
  \item For every $A, B, C : \cC'$ set of potential composition natural transformations reduces to the set $\int_{(M, N) : \cM^{\op} \times \cM^{\op}}\Set(\cM(M, \cC(A, B)) \times \cM(N, \cC(B, C)), \cM(M \otimes N, \cC(A, C))$ for which there is a clear candidate: a function 
    \begin{align*}
      \cM(M, \cC(A, B)) \times \cM(N, \cC(B, C)) &\to \cM(M \otimes N, \cC(A, C))\\
      (f,   g) &\mapsto \boxed{M \otimes N \xrightarrow{f \otimes g} \cC(A, B) \otimes \cC(B, C) \xrightarrow{\comp^{\cC}_{A, B, C}} \cC(A, C)}
    \end{align*}
    defined naturally for each $(M, N) : \cM^{\op} \times \cM^{\op}$.
    This too is $\Para$ composition --- but for externally parametric morphisms.
  \end{itemize}
\end{example}

We do not fully prove it here, but both examples above assemble, respectively, into functors ${\cM\dsh\Act \to \internalHom{\Cat}{\cM}{\Set}\dsh\Cat}$, and ${\cM\dsh\Cat \to \internalHom{\Cat}{\cM}{\Set}\dsh\Cat}$.

These two embeddings show that locally graded categories can model both actegories and enriched categories.
What remains now is to define, analogously to $\Para$, a construction that consumes a locally graded category and produces a bicategory of parametric morphisms.
We will see that to be the case, and, interestingly, we will see that this produces a 2-category, as opposed to a bicategory.

This will be done by taking the base of change along the category of elements functor.
Recall that the category of elements is a functor $\El : \internalHom{\Cat}{\cM}{\Set} \to \Cat$ defined for any category $\cM$.
If the domain is equipped with the structure of Day convolution, and codomain with the cartesian product, we can show that this functor is lax monoidal.

\begin{proposition}[{$\El$ is a lax monoidal functor}]
  Let $(\cM, \otimes, I)$ be a monoidal category.
  Then the functor $\El : \internalHom{\Cat}{\cM}{\Set} \to \Cat$ becomes a lax monoidal functor
  \[
    (\El, \mu, \epsilon) : (\internalHom{\Cat}{\cM}{\Set}, \otimes_D, \cM(I, -)) \to (\Cat, \times, \CatTerm)
  \]
  where
  \begin{itemize}
  \item The laxator $\mu_{F, G} : \El(F) \times \El(G) \to \El(F \otimes_D G)$ is given by
    \[
      ((M, M_F), (N, N_G)) \mapsto (M \otimes N, ((M, N), (M_F, N_G, \id_{M \otimes N})))
    \]
  \item The unit $\epsilon : \CatTerm \to \El(\cM(I, -))$ picks out $(I, \id_I : \cM(I, I))$.
  \end{itemize}
\end{proposition}

This means that we $(\El, \mu, \epsilon)$ is a suitable functor we can perform base change along.
In doing so, we obtain a 2-functor $\enrichedbasechange{(\El, \mu, \epsilon)} : \internalHom{\Cat}{\cM}{\Set}\dsh\Cat \to \TwoCat$.
In the following example, we study its action on objects, which results in a locally graded version of $\Para$.

\begin{example}
  Let $\cC$ be a $\internalHom{\Cat}{\cM}{\Set}$-enriched category.
  Then the action of $\enrichedbasechange{(\El, \mu, \epsilon)}$ on $\cC$ results in a 2-category with the following data.
  \begin{itemize}
  \item \textbf{Objects} are the same as those of $\cC$;
  \item \textbf{Morphisms.}  A morphism $A \to B$ consists of an object $M : \cM$ and an element $f : \cC(A, B)(M)$.
  \item \textbf{2-morphisms.} A map $(M, f) \Rightarrow (N, g)$ is given by $r : M \to N$ such that $\cC(A, B)(r)(g) = f$;
  \item \textbf{Identity morphism.} For every object $A$ the identity morphism is given by the composite
    \[
      \CatTerm \xrightarrow{\epsilon} \El(\cM^{\op}(I, -)) \xrightarrow{\El(\id_A)} \El(\cC(A, A))
    \]
    We note that precomposing $\El(\id_A)$ with $\epsilon$ amounts exactly to applying Yoneda lemma (\cref{prop:yoneda}), and selecting the identity morphism $\id_A$ ---  already considered to be an element of $\cC(A, A)(I)$.
    More precisely, the name of the identity $\epsilon \comp \El(\id_A) : 1 \to \El(\cC(A, A))$ picks out $(I : \cM^{\op}, \id_A : \cC(A, A)(I))$;
  \item \textbf{Morphism composition.} For every $A, B , C : \cC$ the composition morphism is given by enriched base change which unpacks to a composite functor of type
    \begin{align*}
      \El(\cC(A, B)) \times \El(\cC(B, C)) &\xrightarrow{\mu_{\cC(A, B), \cC(B, C)}} \El(\cC(A, B) \otimes_D \cC(B, C)) &\xrightarrow{\El(\comp^{\cC}_{A, B, C})} \El(\cC(A, C))\\
      ((M, f), (N, g)) &\mapsto (M \otimes N, ((M, N), (f, g, \id_{M \otimes N}))) &\mapsto (M \otimes N, f \comp^{\cC}_{A, B, C} g)
    \end{align*}
  \end{itemize}
\end{example}

It can be shown that in both cases of an actegory and an enriched category seen as a locally graded category, applying the base change along the category of elements produces the analogous 2-category to $\Para(\cC)$.

\section*{Locally indexed categories}
\label{sec:locally_indexed_categories}

In the previous section we equipped the category $\internalHom{\Cat}{\cM}{\Set}$ with the Day convolution monoidal product.
This required $\cM$ itself to be a monoidal category. 
But there is another monoidal product we can equip $\internalHom{\Cat}{\cM}{\Set}$ with that does not necessitate any monoidal structure on $\cM$.

\begin{definition}[Pointwise product]
  Let $\cM$ be a category.
  Then the category $\internalHom{\Cat}{\cM}{\Set}$ can be given cartesian monoidal structure where
  \begin{itemize}
  \item Its cartesian monoidal product $\times$ is defined as
    \begin{align*}
      \internalHom{\Cat}{\cM}{\Set} \times \internalHom{\Cat}{\cM}{\Set} &\to \internalHom{\Cat}{\cM}{\Set}\\
      (F, G) &\mapsto \boxed{\cM \xrightarrow{\Delta_{\cM}} \cM \times \cM \xrightarrow{F \times G} \Set \times \Set \xrightarrow{\times} \Set}
    \end{align*}
  \item The unit is $\const{1} : \cM \to \Set$
  \end{itemize}
  Associators and unitors can be defined routinely from the pointwise associators and unitors of $(\Set, \times, 1)$.
  It can be checked that this monoidal product is indeed cartesian, following similarly from the cartesian structure of $\Set$.
\end{definition}

When the category $\cM$ is cartesian, then we can also define the Day convolution product on $\internalHom{\Cat}{\cM}{\Set}$, and these monoidal products turn out to be equivalent.

\begin{definition}[{Locally indexed category, compare \cite[Def.\ 75]{levy_call-by-push-value_2013}}]
  Let $\cM$ be a category.
  Then a we call a category enriched in $(\internalHom{\Cat}{\cM}{\Set}, \times, \const{1})$ a \newdef{locally $\cM$-indexed} category.
\end{definition}

Let us unpack, following \cref{def:enriched_category} what such an enriched category (call it $\cC$) consists of.
Like a locally graded category, it consists of a set of objects, and its hom-object $\cC(A, B)$ is a functor of type $\cM^{\op} \to \Set$.
This means that a locally indexed category also does not have morphisms in the conventional sense, but instead has ``graded'' (or more appropriately, ``indexed'') morphisms, i.e.\ a set of morphisms for every object $M : \cM^{\op}$.
Where locally indexed categories differ from graded ones is in identities and composition.
\begin{itemize}
\item \textbf{Identity.} The identity on an object $A$ is a natural transformation $\id_A : \const 1 \Rightarrow \cC(A, A)$.
  To every object $M : \cM^{\op}$ it assigns the map of type $\cC(A, A)(M)$ which we suggestively call \emph{projection} and label as $\pi_A$;
\item \textbf{Composition.} The composition is a natural transformation $\comp_{A, B, C} : \cC(A, B)(-) \times \cC(B, C)(-) \Rightarrow \cC(A, C)(-)$.
  To every $M : \cM$ it assigns a map
    \begin{align*}
      \comp_{A, B, C}^M : \cC(A, B)(M) \times \cC(B, C)(M) &\to \cC(A, C)(M)
    \end{align*}
    which composes two morphisms of the same grade.
\end{itemize}

If locally graded categories are analogous to $\Para(\cC$, then locally indexed categories are analogous to $\CoKl(X \times -)$.
Locally indexed categories define a composition of morphisms only if they are of the same grade, unlike locally graded categories.
Likewise, note how locally indexed categories define the identity map for every object of $\cM$, while locally graded categories only defined it for the monoidal unit $I$.

What follows is a locally indexed variant of the $\CoKl(X \times -)$ construction from \cref{def:cokl}.

\begin{example}[{compare \cite[Def.\ 77]{levy_call-by-push-value_2013}}]
  Let $\cC$ be a cartesian category.
  Then we can form a locally $\cC$-indexed category $\cC$ with the following data.
  \begin{itemize}
  \item Its objects are the same as those of $\cC$;
  \item Given any $A, B : \cC$ the hom functor is $\cC(A, B) \coloneqq \boxed{\cC^{\op} \xrightarrow{A \times -} \cC^{\op} \xrightarrow{\cC(-, B)} \Set}$;
  \item For every $A$ the identity maps unpacks to a natural transformation $\id_A : \const{1} \Rightarrow \cC(A \times -, B)$ whose component at every $M : \cM^{\op}$ is a morphism $\name{\pi_A} : 1 \to \cC(A \times M, A)$ picking out the projection morphism $\pi_A : A \times M \to A$;
  \item For every $A, B, C : \cC$ and $M : \cM^{\op}$ composition unpacks to a natural transformation $\cC(A \times -, B) \times \cC(B \times -, C) \Rightarrow \cC(A \times -, C)$ which is for every $M : \cM^{\op}$ defined as the function
    \begin{align*}
      \cC(A \times -, B) \times \cC(B \times M, C) &\to \cC(A \times M, C)\\
         (f, g) \mapsto (f \compPara{\comp} g)^{\Delta_M}
    \end{align*}
  \end{itemize}
\end{example}

\chapter{The Para construction}
\label{app:para}

\ParaBiCatToParaTwoCat*

\begin{proof}
  Showing that a bicategory is a 2-category reduces to showing that the unitors and associators are identity natural transformations.
  When it comes to the unitors, let's study what happens when we compose
  \[
    (I, \eta_A) : \Para_{\actfw}(\cC)(A, A) \quad \text{and} \quad (P, f) : \Para_{\actfw}(\cC)(A, B)
  \]
  where $\eta_A : A \actfw I \to A$ and $f : A \actfw P \to B$.
  The resulting composite (which we call $g$) is
  \[
    A \actfw (I \otimes P) \xrightarrow{\mu_{A, I, P}} (A \act I) \act P \xrightarrow{\eta_A \actfw P} A \actfw P \xrightarrow{f} B
  \]
  In order to show that the left unitor is the identity natural isomorphism, we need to show that the reparameterisation $(I \otimes P, g) \Rightarrow (P, f)$ is identity.
  This is witnessed by the identity morphism $\lambda^{-1}_P : P \to I \otimes P$ (as $\cM$ is strict) for which the reparameterisation condition reduces to showing the diagram \cref{eq:para_bicat_unitor} commutes.
  And this diagram is precisely the left unitor coherence condition of an actegory (\cref{eq:act_unitors}).
  Right unitor follows analogously.
  \begin{equation}
    \label{eq:para_bicat_unitor}
    \begin{tikzcddiag}[ampersand replacement=\&]
    	\&\& {A \actfw (I \otimes P)} \\
    	\\
    	{A \actfw P} \&\& {(A \actfw I) \actfw P}
    	\arrow["{\mu_{A, I, P}}", from=1-3, to=3-3]
    	\arrow["{A \actfw \lambda_P}"', from=1-3, to=3-1]
    	\arrow["{\eta_A \actfw P}", from=3-3, to=3-1]
    \end{tikzcddiag}
  \end{equation}
  For the associator, given
  \[
    (P, f) : \Para_{\actfw}(\cC)(A, B) \quad \text{and} \quad (Q, g) : \Para_{\actfw}(\cC)(B, C) \quad \text{and} \quad (R, h) : \Para_{\actfw}(\cC)(C, D) 
  \]
  we need to show that two different ways of composing these morphisms
  \[
    ((P \otimes Q) \otimes R, (f  \compPara{\comp} g)  \compPara{\comp} h) \quad \text{and} \quad (P \otimes (Q \otimes R), f  \compPara{\comp} (g  \compPara{\comp} h))
  \]
  are equal on the nose.
  Since $\cM$ is strict it is straightforward to provide the underlying identity map.
  All that remains is to check the reparameterisation condition, which reduces to the diagram \cref{eq:para_bicat_associator}.
  
  \begin{equation}
    \label{eq:para_bicat_associator}
\begin{tikzcddiag}[ampersand replacement=\&]
	{A \act ((P \otimes Q) \otimes R)} \&\& {(A \act (P \otimes Q)) \act R} \\
	\\
	{A \act (P \otimes (Q \otimes R))} \& {(A \act P) \act (Q \otimes R)} \& {((A \act P) \act Q) \act R} \\
	\\
	\& {B \act (Q \otimes R)} \& {(B \act Q) \act R}
	\arrow["{\mu_{A, P \otimes Q, R}}", from=1-1, to=1-3]
	\arrow["{\mu_{A, P, Q} \act R}", from=1-3, to=3-3]
	\arrow["{ A \act \alpha_{P, Q, R}}"', from=1-1, to=3-1]
	\arrow["{\mu_{A, P, Q \otimes R}}"', from=3-1, to=3-2]
	\arrow["{\mu_{B, Q, R}}"', from=5-2, to=5-3]
	\arrow["{(f \act Q) \act R}", from=3-3, to=5-3]
	\arrow["{\mu_{A \act P, Q, R}}"', from=3-2, to=3-3]
	\arrow["{f \act (Q \otimes R)}", from=3-2, to=5-2]
\end{tikzcddiag}
\end{equation}

  The top square commutes because of the actegory pentagonator \cref{eq:act_pentagonator}, and the bottom square commutes because of naturality of $\mu$.
\end{proof}

\ParaMonoidal*

\begin{proof}
  \label{proof:para_monoidal_bicat}
  The monoidal bicategory structure (\cite[Def.\ 12.1.2]{johnson_2-dimensional_2021}) of $\Para_{\actfw}(\cC)$ consists of the following data.
  
  \begin{itemize}
  \item The monoidal product pseudofunctor $\compPara{\boxtimes} : \Para_{\actfw}(\cC) \times \Para_{\actfw}(\cC) \to \Para_{\actfw}(\cC)$ defined by
    \begin{itemize}
    \item An action on objects mapping $X, Y$ to $X \boxtimes Y$;
    \item For every $A, B, C, D : \Para_{\actfw}(\cC)$ a local functor
      \[
        \Para_{\actfw}(\cC)(A, B) \times \Para_{\actfw}(\cC)(C, D) \to  \Para_{\actfw}(\cC)(A \boxtimes C, B \boxtimes D)
      \]
      mapping objects i.e.\ parametric morphisms $(P, f : A \act P \to B)$ an $(Q, g : C \act Q \to D)$ to $(P \otimes Q, f \compPara{\boxtimes} g)$ where
      \[
        f \compPara{\boxtimes} g \coloneqq \boxtimes{(A \boxtimes C) \act (P \otimes Q) \xrightarrow{\interchangeract_{A, C, P, Q}} (A \act P) \boxtimes (B \act Q) \xrightarrow{f \boxtimes g} B \boxtimes D}
      \]
      and morphisms i.e.\ reparameterisations $r : (P, f) \Rightarrow (Q, g)$ and $s : (R, h) \Rightarrow (S, i)$ to $r \otimes s$.
    \item A natural isomorphism $\mu^{\compPara{\boxtimes}}$  %
      \begin{equation*}
        \label{eq:para_bicat_laxator}
        \begin{tikzcddiag}[ampersand replacement=\&]
          {\mathbf{P}(A, B) \times \mathbf{P}(X, Y) \times \mathbf{P}(B, C) \times \mathbf{P}(Y, Z)} \&\& {\mathbf{P}(A, C) \times \mathbf{P}(X, Z)} \\
          \\
          {\mathbf{P}(A \boxtimes X, B \boxtimes Z) \times \mathbf{P}(B \boxtimes Y, C \boxtimes Z)} \&\& {\mathbf{P}(A \boxtimes X, C \boxtimes Z)}
          \arrow["\comp", from=1-1, to=1-3]
          \arrow["{\compPara{\boxtimes}_{A, C, X, Z}}", from=1-3, to=3-3]
          \arrow["{\compPara{\boxtimes}_{A, B, X, Z} \times \compPara{\boxtimes}_{B, C, Y, Z}}"', from=1-1, to=3-1]
          \arrow["\comp"', from=3-1, to=3-3]
          \arrow["{\mu^{\compPara{\boxtimes}}}"{description}, Rightarrow, from=3-1, to=1-3]
        \end{tikzcddiag}
\end{equation*}

      This natural isomorphism (where here $\mathbf{P}$ stands for $\Para_{\actfw}(\cC)$) is a weakened version of the interchange law (\cref{eq:interchange})).
      Componentwise it describes the relation between different orders of composition of parametric maps
      \begin{align*}
        (P, f) &: \Para_{\actfw}(A, B) \quad \text{and} \quad  (R, h) : \Para_{\actfw}(B, C)\\
        (Q, g) &: \Para_{\actfw}(X, Y) \quad \text{and} \quad  (S, i) : \Para_{\actfw}(Y, Z)
      \end{align*}
      We can either a) compose the maps in parallel, and then the results sequentially, or b) compose the maps sequentially, and then the results in parallel.
      These two options yield:
      \[
        (f \compPara{\boxtimes} g)  \compPara{\comp} (h \compPara{\boxtimes} i) \quad \text{and} \quad (f  \compPara{\comp} h) \compPara{\boxtimes} (g  \compPara{\comp} i)
      \]
      whose parameter spaces are $(P \otimes Q) \otimes (R \otimes S)$ and $(P \otimes R) \otimes (Q \otimes S)$, respectively.
      The invertible 2-cell $\mu^{\compPara{\boxtimes}}$ ensures these are isomorphic, and is witnessed by the reparamaterisation of type $(P \otimes R) \otimes (Q \otimes S) \to (P \otimes Q) \otimes (R \otimes S)$ given by the interchanger $\interchanger_{P, R, Q, S}$ (\cref{def:interchanger}).
    \item A natural isomorphism $\epsilon^{\compPara{\boxtimes}}$ which relates the product of two identity maps $(I, \eta_A)$ on $A$ and $(I, \eta_B)$ on $B$ with the identity $(I, \eta_{A \boxtimes B})$ on the product $A \boxtimes B$.
      The former unpacks to a parametric map
      \[
        \eta_A \compPara{\boxtimes} \eta_B \coloneqq \boxtimes{(A \boxtimes B) \actfw (I \otimes I) \xrightarrow{\interchangeract_{A, B, I, I}} (A \actfw I) \boxtimes (B \actfw I) \xrightarrow{f \boxtimes g} A \boxtimes B}
      \]
      It is routine to show that this is isomorphic to $(I, \eta_{A \boxtimes B})$ witnessed by the reparameterisation of either the left or the right unitor of $I$.
    \end{itemize}
  \item Identity pseudofunctor $\CatTerm \xrightarrow{\compPara{J}} \Para_{\actfw}(\cC)$ defined by:
    \begin{itemize}
    \item A choice of $J : \Para(\cC)$, its identity object;
    \item A morphism $(I, \eta_J) : \Para_{\actfw}(\cC)(J, J)$ where $\eta_J : J \act I \to I$;
    \item An invertible 2-cell $\eta_J  \compPara{\comp} \eta_J \Rightarrow \eta_J$ in $\Para_{\actfw}(\cC)(J, J)$.
      Its domain can be unpacked to
      \[
        J \actfw (I \otimes I) \xrightarrow{\mu_{J, I, I}} (J \actfw I) \actfw I \xrightarrow{\eta_J} J \actfw I \xrightarrow{\eta_J} J.
      \]
      and analogously to above it can be reparameterised by the isomorphism of type $I \otimes I \to I$ given by either the left or the right unitor of $I$, yielding $\eta_J$.
    \item Invertible 2-cell $\id_J \Rightarrow (I, \eta_J)$.
      As these are equal on the nose, this 2-cell is identity.\footnote{This is why in the definition of a monoidal actegory the opunitor of $\actfw$ was redundant (\cref{def:monoidal_actegory}).}
    \end{itemize}
  \item The associator and the left and right unitor adjoint equivalences, all defined by the images of the associator, unitor and pentagonator isomorphisms of the monoidal category $\cC$ under the embedding of $\cC$ into $\Para_{\actfw}(\cC)$ (\cref{lemma:c_embeds_into_para}).
  \item The pentagonator and 2-unitor invertible 2-cells, all defined by iterated application of the isomorphism $\lambda_I : I \otimes I \to I$.
  They establish a coherence coherence of the associators and the unitors which in our case all have trivial parameters.
\end{itemize}

Lastly, it remains to prove this data satisfies the appropriate coherence conditions \cite[11.2.14, 11.2.16, and 11.2.17]{johnson_2-dimensional_2021}.
We note that all of the coherence conditions establish an equality between pasting diagrams involving only associators/unitors, pentagonators/2-unitors, and the unitor $\epsilon^{\compPara{\boxtimes}}$.
As in our case all of them are identity reparameterisations, it is easy to see that all of the required diagrams are equal.
This concludes the proof that $\Para_{\act}(\cC)$ is a monoidal bicategory.
\end{proof}

\ParaMonoidalTwoCat*
\begin{proof}
  To show $\Para_{\act}(\cC)$ is a monoidal 2-category we additionally need $\compPara{\boxtimes}$ and $\compPara{J}$ defining the monoidal bicategory structure to be strict 2-functors.
  It is easy to see that strictness of $\cM$ implies strictness of $\compPara{J}$, as its natural isomorphisms are defined via the unitors of $\cM$.
  When it comes to $\compPara{\boxtimes}$, only one of the natural isomorphisms becomes identity when $\cM$ is strict: $\epsilon^{\compPara{\boxtimes}}$.
  The other natural isomorphism $\mu^{\compPara{\boxtimes}}$ is defined using the braiding of $\cM$ which is not identity when $\cM$ is strict --- only when $\cM$ is \emph{commutative monoidal}.
  It is routine to check that making $\cM$ commutative monoidal makes $\mu^{\compPara{\boxtimes}}$ identity, $\compPara{\boxtimes}$ a strict 2-functor, and $\Para_{\act}(\cC)$ a monoidal 2-category.\footnote{This was first pointed out to me by Dan Shiebler.}
  It can also be seen that if $\Para_{\act}(\cC)$ is a monoidal 2-category, then $\compPara{\boxtimes}$ must be a strict 2-functor, $\mu^{\compPara{\boxtimes}}$ identity, meaning the braiding of $\cM$ is strict, i.e.\ $\cM$ is commutative monoidal.
\end{proof}

\ParaMonCat*
\begin{proof}
  Following the reasoning in the proof of \cref{prop:para_monoidal_twocat} --- we see that strictification of the monoidal structure boils down to ensuring $\compPara{\boxtimes}$ is a strict 2-functor, which consequently boils down to the requirement that the interchange law holds strictly, i.e.\ such that either of the ways of composing the parametric maps yield the same result.
  Since we are dealing with equivalence classes, ``same'' means up to a corresponding reparameterisation.
  This is something that holds for $\conncomp, \Core$ and $\Epi$ as they can all use the braiding $\beta_{Q, R}$ of $\cM$ to define their quotient (which is an isomorphism, and thus an epimorphism).
  On the other hand, this does not hold for $\Ob$: it is easy to see that the objects $(P \otimes Q) \otimes (R \otimes S)$ and $(P \otimes R) \otimes (Q \otimes S)$ fall into different equivalence classes.
\end{proof}

\ParaFunctorial*
\begin{proof}
As this functor is only \emph{pseudo}, we have to define the respective composition and identities 2-cells, and show that they are invertible.
For composition, starting with two composable 1-cells in $\Para_{\act}(\cC)$
  $$
  A \xrightarrow{\quad (P, f)\quad} B \xrightarrow{\quad (Q, g)\quad} C
  $$
  i.e.\ with
  \begin{equation*}
    \begin{aligned}
      &P : \cM \\
      &A \act P \xrightarrow{f} B
    \end{aligned}
    \qquad\text{and}\qquad
    \begin{aligned}
      &Q : \cM \\
      &B \act Q \xrightarrow{g} C
    \end{aligned}  
  \end{equation*}
  there are two ways to map them to $\Para_{\actc}(\cD)$ using $\Para((R, R^\sharp, \ell))$.
  We can either map them individually and then compose them, or first compose them, and them map the result.
  In the first scenario, mapping them individually yields
  \[
    R^\sharp(A) \actc R(P) \xrightarrow{R^\sharp(f)} R^\sharp(B) \quad \text{and} \quad R^\sharp(B) \actc R(Q) \xrightarrow{R_{\ell}^\sharp(g)} R^\sharp(C).
  \]
  When composed, these two maps yield a morphism in $\Para_{\boxtimes}(\cD)$
  \[
    R^\sharp(A) \boxtimes (R(P) \otimes R(Q)) \xrightarrow{\mu^{\boxtimes}_{R^\sharp(A), R(P), R(Q)}} (R^\sharp(A) \boxtimes R(P)) \boxtimes R(Q) \xrightarrow{R_{\ell}^\sharp(f) \boxtimes R(Q)} R^\sharp(B) \boxtimes R(Q) \xrightarrow{R_{\ell}^\sharp(g)} R^\sharp(C)
  \]
  whose parameter space is $R(P) \otimes R(Q)$.\footnote{We're overloading notation by using $\otimes$ for the monoidal structure of $\cM$ and $\cN$.}
  In the second scenario, first composing them yields
  \[
    A \act (P \otimes Q) \xrightarrow{\mu^{\act}_{A, P, Q}} (A \act P) \act Q \xrightarrow{f \act Q} B \act Q \xrightarrow{g} C
  \]
  and the image of this composite under the mapping yields the following morphism in $\Para_{\boxtimes}(\cD)$:
  \[
    R^\sharp(A) \boxtimes (R(P \otimes Q)) \xrightarrow{R_{\ell}^\sharp(\mu^{\act}_{A, P, Q} \comp (f \act Q) \comp g)} R^\sharp(C)
  \]
  As these two morphisms have different parameter objects they clearly aren't strictly equal.
  Though there is an invertible 2-cell connecting them --- $R(P \otimes Q) \to R(P) \otimes R(Q)$ --- the one given by the oplaxator $\mu^{-1}_{P, Q}$.
  Similar arguments can be given for the unitor of this lax functor which is given by the opunitor of the same monoidal functor $R$.\footnote{Note that for a lax actegory the induced functor would be a lax functor instead of a pseudofunctor.}
  Showing composition and identities are preserved up to a coherent isomorphism is routine.
\end{proof}

\begin{proposition}
  Let $\cC$ be a cartesian category, treated as a locally discrete 2-category.
  Then we can define an identity-on-objects oplax functor $(S, \mu, \epsilon)$
\[\begin{tikzcddiag}
	{\cC } && {\Copara(\cC)} \\
	A && A \\
	\\
	B && B
	\arrow["S", from=1-1, to=1-3]
	\arrow["f"', from=2-1, to=4-1]
	\arrow["{\graph(f)}", from=2-3, to=4-3]
	\arrow[maps to, from=2-1, to=2-3]
	\arrow[maps to, from=4-1, to=4-3]
\end{tikzcddiag}\]

As it is only oplax, its identities and composition are preserved only up to a coherent 2-cell.
  \begin{itemize}
  \item \textbf{Identity.} For every $A : \cC$ we have the opunitor: a natural transformation $\epsilon_A : S(\id^{\cC}_A) \Rightarrow \id^{\Copara(\cC)}_A$.
    It unpacks to a reparameterisaton of type $(A, \graph(f)) \to (1, \lambda_A)$, and is given by the terminal map $\terminal_A : A \to 1$;
  \item \textbf{Composition.} For every $f : A \to B$ and $g : B \to C$ in $\cC$ the oplaxator is a natural transformation $\mu_{f, g} : S(f \comp g) \Rightarrow S(f) \comp S(g)$ which unpacks to a reparameterisation of type $(A, \graph(f \comp g)) \Rightarrow (A \times B, \graph(f) \compCopara{\comp} \graph(g))$.
    This reparameteristaion is given by the map $\graph(f) : A \to A \times B$.\footnote{There is also a reparameteristaion $\pi_A : A \times B \to A$ going the other way, but it is not the inverse of $\graph(f)$. It is only a one-sided inverse, i.e.\ $\graph(f)$ is a section of $\pi_A$ but not vice versa.}
  \end{itemize}
  It is routine to check that the oplaxator and opunitor satisfy the constraints of an oplax functor.
\end{proposition}

We note that there is also a \emph{laxator} for $S$ which is inverse to $\mu$, but there's no unitor.

\begin{proposition}
  Let $\cC$ be a cartesian category, treated as a locally discrete 2-category.
  Then we can define an identity-on-objects strict functor $P : \Copara(\cC) \to \cC$ mapping every morphism $(M, f : A \to M \times B)$ to $f \comp \pi_B$, and every 2-cell to the only possible choice: the identity 2-cells in $\cC$.
  It's routine to check that this is indeed a strict 2-functor.
\end{proposition}

\chapter[Category of elements and the Grothendieck construction]{Category of elements and the Grothendieck construction}
\chaptermark{Category of elements and the Groth. construction}

\section*{Category of elements}

An essential construction in category theory and in the formulation of weighted $\CoparaW{W}{}$ is the \emph{category of elements}.

\begin{definition}[Category of elements]
  \label{def:cat_of_elem}
  Let $F : \cC \to \Set$ be a functor.
  The \newdef{category of elements} of $F$ --- which we denote by $\El(F)$ --- is a category with the following data:
  \begin{itemize}
  \item \textbf{Objects.} An object is a pair $\diset{C}{c'}$, where $C$ is an object of $\cC$ and $c'$ is an element of $F(C)$;
  \item \textbf{Morphisms.} A morphism $\diset{C}{c'} \to \diset{D}{d'}$ consists of a morphism $f : C \to D$ in $\cC$ such that $F(f)(c') = d'$. Alternatively, 
    \[
      \El(F)\OpticHom{C}{c'}{D}{d'} \coloneqq \sum_{f : \cC(C, D)}F(D)(F(f)(c'), d')
    \]
    Note that since $F(D)$ is a set, there is only one possible inhabitant of $F(D)(F(f)(c'), d')$.
    This allows us to present the set of all morphisms in $\El(F)$ in a potentially surprising form, as the set $\cC(C, D) \times F(D)$ (see \cite[p. 4]{moser_internal_2023});
  \item \textbf{Identity.} Identity morphisms are identities in $\cC$;
  \item \textbf{Composition.} Composition is given by composition of morphisms in $\cC$.
  \end{itemize}
  We often call $\cC$ \emph{the base} of the category of elements, and sets $F(c)$ its fibers.
  Each category of elements has an associated projection $\pi_F : \El(F) \to \cC$ which forgets the data associated to $F$.
\end{definition}

\begin{example}[Twisted arrow category]
  \label{ex:twisted_arrow}
  The category of elements of the hom-functor of a category called the \emph{twisted arrow category}, and given some category $\cC$ it is denoted by $\tw(\cC)$.
  \[
    \tw(\cC) \coloneqq \El(\Hom_\cC)
  \]
  An object in $\tw(\cC)$ is a map $f : A \to B$, and a morphism $f \to g$ in this category is a way of factorizing $g$ through $f$.
  Following \cref{def:cat_of_elem}, this category comes equipped with the projection $\pi_{\cC} : \tw(\cC) \to \cC^{\op} \times \cC$.
\end{example}

Given a category $\cC$ and a functor $F : \cC \to \Set$, we have defined the category $\El(F)$.
As $F$ is merely an object of $\internalHom{\Cat}{\cC}{\Set}$ we might wonder whether $\El$ can naturally be defined with respect to its morphisms.
The answer yes.

\begin{definition}
  Let $\cC$ be a category, $F, G : \cC \to \Set$ two functors, and $\alpha : F \Rightarrow G$ a natural transformation between them.
  Then we can define a functor
  \[
    \El(\alpha) : \El(F) \to \El(G)
  \]
  as follows.
  \begin{itemize}
  \item \textbf{Objects.} It maps every $\diset{X}{x'} : \El(F)$ to $\diset{X}{\alpha_X(x')}$ i.e.\ it keeps base points the same;
  \item \textbf{Morphisms.} It maps $\diset{f}{f'} : \diset{X}{x'} \to \diset{Y}{y'}$ ($f'$ is here explicit notation for the proof $f' : F(f)(x') = y'$) to $\diset{f}{\alpha_f \comp a_Y(f')} : \diset{X}{\alpha_X(x')} \to \diset{Y}{\alpha_y(y')}$, where the second component is the composite equality
    \[
      G(f)(\alpha_X(x')) \stackrel{\alpha_f}{=} \alpha_Y(F(f)(x')) \stackrel{\alpha_Y(f')}{=} \alpha_Y(y')
    \]
    and $\alpha_f : \alpha_X \comp G(f) = F(f) \comp \alpha_Y$ is the naturality condition of $\alpha$. 
  \end{itemize}
  Note that the action of this functor is identity on the base category, both on object and morphism level.
\end{definition}

What is the intuitive explanation of this definition?
Recall: to each component $X$ in $\cC$ it assigns a morphism $\alpha_X : F(X) \to G(X)$.
We can think of the results of the action of the functor $\El$ on an object ---  a category of elements --- as a pair consisting of a choice of a component of this natural transformation (say, $X$) and an input $x : F(X)$ to the morphism $\alpha_X$.
Then the action of $\El$ on morphisms simply keeps the component $X$ the same, while applying $\alpha_X$ to the input $x$.

\subsection*{Monoidal structure of the category of elements}

Categories of elements are the acting categories of actegories in the definition of weighted $\Copara$.
This makes us interested in conditions under which they are monoidal, and braided monoidal, as well as conditions under which the associated projection preserves the (braided) monoidal structure.

\begin{proposition}[Monoidal structure of the category of elements]
  \label{prop:monoidal_structure_category_of_elements}
  Let $\cC$ have monoidal structure $(\otimes, I)$ and let $F : \cC \to \Set$ have lax monoidal structure $(\phi, \epsilon)$ where $\phi_{X, Y} : F(X) \times F(Y) \to F(X \otimes Y)$ and $\epsilon : 1 \to F(I)$.
  Then $\El(F)$ is a monoidal category.
  Its data is as follows:
  \begin{itemize}
  \item The monoidal product of $\diset{C}{c'}$ and $\diset{D}{d'}$ is $\diset{C \otimes D}{\phi_{c, d}(c', d')}$;
  \item The monoidal product of morphisms is given by that of $\cC$;
  \item The monoidal unit is $\diset{I}{\epsilon(\bullet)}$, where $\bullet : 1$ is the unique element of monoidal unit $1$ in in $\Set$;
  \item Associators and unitors are those of $\cC$
  \end{itemize}
  Furthermore, the functor $\pi_F : \El(F) \to \cC$ is strict monoidal.

  If $\cC$ is a braided monoidal category and $(F, \phi, \epsilon)$ is a braided monoidal functor, then $\El(F)$ becomes a braided monoidal category, and the projection $\pi_F : \El(F) \to \cC$ becomes a strict braided monoidal functor.
\end{proposition}

Previously we gave the example of the category of elements of the hom-functor, i.e.\ the twisted arrow category (\cref{ex:twisted_arrow}).
In order for this category to be a monoidal category, the hom-functor of a category needs to be lax monoidal.
And that indeed happens if the base category is monoidal.

\begin{proposition}
  \label{prop:monoidal_cat_hom_functor}
  Let $(\cM, \otimes, I)$ be a monoidal category.
  Then $\Hom_{\cM} : (\diset{\cM^{\op}}{\cM}, \diset{\otimes}{\otimes}, \diset{I}{I}) \to (\Set,
  \times, 1)$ can be equipped with the lax monoidal structure $(\phi, \epsilon)$ where
  \begin{itemize}
  \item The laxator $\phi_{(M, M'), (N, N')} : \cM(M, M') \times \cM(N, N') \to \cM(M \otimes N, M' \otimes N')$ tensors the maps together, and is formally given by the action of $\otimes$ on morphisms.
  \item The unit $\epsilon : 1 \to \cM(I, I)$ picks out the identity morphism on $I$.
  \end{itemize}
  Additionally, if $\cM$ is braided, then $(\Hom_{\cM}, \phi, \epsilon)$ is too. 
\end{proposition}

\begin{corollary}
  \label{corr:hom_functor_projection_braided}
  If $\cC$ is a monoidal category, then so is $\tw(\cC)$, and the functor $\pi_\Hom : \tw(\cC) \to \cC^{\op} \times \cC$ is strict monoidal.
  Additionally, if $\cC$ is braided monoidal, then so is $\tw(\cC)$, and the projection is too.
\end{corollary}

\section*{Grothendieck construction}

A category of elements is defined for a functor $F : \cC \to \Set$.
But often, the codomain of $F$ is instead $\Cat$.
Here are some examples.

\begin{example}[{\cite[Sec. 3.2]{spivak_generalized_2022}}]
  Let $\cC$ be a category with products.
  Then we can form the functor
  \[
\CoKl(= \times -) : \cC^{\op} \to \Cat
    \]
  
  mapping an object $X$ to the coKleisli category $\CoKl(X \times -)$, and a morphism $f : X \to Y$ to the functor $\CoKl(f \times -) : \CoKl(Y \times -) \to \CoKl(X \times -)$ given by reparameterisation.
  Furthermore, because $\CoKl(X \times -)$ is always a cartesian monoidal category (\cref{def:cokl_is_cartesian_left_additive}) we can form a restriction of this functor to commutative monoids
  \[
\asc{\CoKl(= \times -)} : \cC^{\op} \to \Cat
    \]
  mapping an object $X$ to $\asc{\CoKl(X \times -)}$, and analogously for reparameterisation.
\end{example}

This example is a subfunctor of the following example of the slice functor --- which adds an extra layer of dependence on objects of the base category $\cC$.

\begin{example}[{\cite[Ex. 3.5]{spivak_generalized_2022}}]
  Let $\cC$ be a category with pullbacks.
  Then we can form the contravariant slice functor
  \[
\cC/- : \cC^{\op} \to \Cat
  \]
  mapping an object $X$ to the slice category $\cC / X$ and a morphism $f : X \to Y$ to the functor $f^* : \cC / Y \to \cC / X$ defined by pullback.
  Analogously to the example above, because $\cC/X$ is always a cartesian monoidal category, we can form a restriction of this functor to commutative monoids
  \[
\asc{\cC/-} : \cC^{\op} \to \Cat
  \]
  which maps an object $X$ to $\asc{\cC / X}$, and analogously for morphisms.
\end{example}

When the base $\cC$ is $\Set$, there is an functor isomorphic to the one above which gives us the indexed perspective on dependence.

\begin{example}[Families]
  \label{ex:fam_functor}
  For any category $\cC$ we can form the following functor
  \[\begin{tikzcddiag}
	  {\Set^{\op}} && \Cat \\
	  X && { \internalHom{\Cat}{\discr(X)}{\cC}} \\
	  \\
	  Y && { \internalHom{\Cat}{\discr(Y)}{\cC}}
	  \arrow["{\Fam(\cC)}", from=1-1, to=1-3]
	  \arrow[maps to, from=2-1, to=2-3]
	  \arrow["f"', from=2-1, to=4-1]
	  \arrow["{\internalHom{\Cat}{\discr(f)}{\cC}}"', from=4-3, to=2-3]
	  \arrow[maps to, from=4-1, to=4-3]
  \end{tikzcddiag}\]
  whose action on objects is given by a category whose objects are families of objects in $\cC$, and action on morphisms is given by precomposition with $\discr(f)$.
  Just like with examples above, it's possible to restrict this functor to commutative monoids.
\end{example}

In such cases, we can talk about a higher-dimensional generalisation of the category of elements: the Grothendieck construction.
It comes in two variants.

\begin{definition}[(Covariant) Grothendieck construction {(compare \cite[Def.\ 3.1]{spivak_generalized_2022} and \cite[Def.\ A.4]{sati_entanglement_2023})}]
  \label{def:grothendieck_construction}
  Let $\cC$ be a category.
  Let $(F, \mu, \epsilon) : \cC \to \Cat$ be an oplax functor.
  Then the Grothendieck construction of $F$ is a category we denote by $\Gr(F)$ (note that arrow goes clockwise), defined by the following data.
  \begin{itemize}
  \item \textbf{Objects.} The set of objects is $\sum_{X : \cC}F(X)$. We denote a particular object as $\diset{X}{X'}$ where $(X : \cC$ and $X' : F(X))$;
  \item \textbf{Morphisms.} For every two objects $\diset{X}{X'}$ and $\diset{Y}{Y'}$ the hom-set is
    \[
      \Gr(F)\OpticHom{X}{X'}{Y}{Y'} \coloneqq \sum_{f : \cC(X, Y)}F(Y)(F(f)(X'), Y')
    \]
  \item \textbf{Identity.} For any $\diset{X}{X'}$ the identity morphism is the pair $\diset{\id_X}{{\epsilon_X}_{X'}}$;
  \item \textbf{Composition.} The composition of morphisms
    $$
    \diset{X}{X'} \xrightarrow{\enskip \diset{f}{f'}\enskip} \diset{Y}{Y'} \xrightarrow{\enskip \diset{g}{g'}\enskip} \diset{Y}{Y'}
    $$
    i.e.\ of
    \begin{equation*}
      \begin{aligned}
        &f : X \to Y\\
        &f' : F(f)(X') \to Y'
      \end{aligned}
      \qquad\qquad\text{and}\qquad\qquad
      \begin{aligned}
        &g : Y \to Z\\
        &g' : F(g)(Y') \to Z'
      \end{aligned}  
    \end{equation*}
    is the pair $\diset{f \comp g}{{\mu_{f, g}}_{X'} \comp F(g)(f') \comp g'}$ where
    \begin{align*}
      &X \xrightarrow{f} Y \xrightarrow{g} Z&\\
      &F(f \comp g)(X') \xrightarrow{{\mu_{f, g}}_{X'}} (F(f) \comp F(g))(X') = F(g)(F(f)(X')) \xrightarrow{F(g)(f')} F(g)(Y') \xrightarrow{g'} Z'
    \end{align*}
  \end{itemize}
\end{definition}

\begin{definition}[(Contravariant) Grothendieck construction {(compare \cite[Def.\ 3.1]{spivak_generalized_2022} and \cite[Def.\ 10.1.2]{johnson_2-dimensional_2021})}]
  \label{def:grothendieck_construction}
  Let $\cC$ be a category.
  Let $(F, \mu, \epsilon) : \cC^{\op} \to \Cat$ be a lax functor.
  Then the contravariant Grothendieck construction of $F$ is a category we denote by $\GrC(F)$ (note that arrow goes counter-clockwise), defined by the following data.
  \begin{itemize}
  \item \textbf{Objects and identities.} Same as in $\Gr(\cC)$;
  \item \textbf{Morphisms.} For every two objects $\diset{X}{X'}$ and $\diset{Y}{Y'}$ the hom-set is
    \[
      \GrC(F)\OpticHom{X}{X'}{Y}{Y'} \coloneqq \sum_{f : \cC(X, Y)}F(X)(X', F(f)(Y'))
    \]
  \item \textbf{Composition.} The composition of morphisms
    $$
    \diset{X}{X'} \xrightarrow{\enskip \diset{f}{f'}\enskip} \diset{Y}{Y'} \xrightarrow{\enskip \diset{g}{g'}\enskip} \diset{Y}{Y'}
    $$
    i.e.\ of
    \begin{equation*}
      \begin{aligned}
        &f : X \to Y\\
        &f' : X' \to F(f)(Y')
      \end{aligned}
      \qquad\qquad\text{and}\qquad\qquad
      \begin{aligned}
        &g : Y \to Z\\
        &g : Y' \to F(g)(Z')
      \end{aligned}  
    \end{equation*}
    is the pair $\diset{f \comp g}{f' \comp F(f)(g') \comp {\mu_{f, g}}_{Z'}}$ where
    \begin{align*}
      &X \xrightarrow{f} Y \xrightarrow{g} Z&\\
      &X \xrightarrow{f'} F(f)(Y') \xrightarrow{F(f)(g')} F(f)(F(g)(Z')) = (F(f) \comp F(g))(Z') \xrightarrow{{\mu_{f, g}}_{Z'}} F(f \comp g)(Z')
    \end{align*}
  \end{itemize}
  It is tedious, but routine to check that this is indeed a category (\cite[Lemma 10.1.8]{johnson_2-dimensional_2021}).
\end{definition}

To establish a relationship between these two definitions, we first need to formalise $(-)^{\op}$ as a functor.

\begin{lemma}[Op functor]
  \label{lemma:op_functor}
  The $\op$ construction taking a category to its opposite forms a 2-functor
  $(-)^{\op} : \Cat^\co \to \Cat$.
  Notably, this functor is \emph{not} contravariant, i.e.\ it does not flip the direction of 1-cells.
  It is, however, contravariant on 2-cells, hence the superscript $^\co$.
\end{lemma}

\begin{proposition}[{(compare \cite[Prop. 3.2]{spivak_generalized_2022})}]
  In the following we fix $\cC$ to be a category.\\
  Let $F : \cC^{\op} \to \Cat$ be an oplax functor.
  Then there is an isomorphism of categories
  \[
    \Gr(\cC^{\op} \xrightarrow{F} \Cat)^{\op} \quad \text{and} \quad \GrC(\cC^{\co\op} \xrightarrow{F^\co} \Cat^\co \xrightarrow{(-)^{\op}} \Cat)
  \]
  Dually, let $G : \cC^{\op} \to \Cat$ be a lax functor.
  Then there is an isomorphism of categories
  \[
    \Gr(\cC^{\co\op} \xrightarrow{G^{\co}} \Cat^{\co} \xrightarrow{(-)^{\op}} \Cat) \quad \text{and} \quad  \GrC(\cC^{\op} \xrightarrow{G} \Cat)
  \]
  Technically, the type of $F^\co$ and $G^\co$ is $\cC^{\co\op} \to \Cat^\co$, so we can remove the $^\co$ superscript from $\cC$.
  Compare with \cite[Prop. 3.2]{spivak_generalized_2022} which assumes these are strict 1-functors, and thus a) removes the $^\co$ superscript entirely, and b) constructs only the top isomorphism.
\end{proposition}

In the definition below we refer to the right-hand side of these isomorphisms.
Equivalently we could have referred to categories on the left-hand side.

\begin{definition}[{$F$-charts (\cite[Def.\ 3.5.0.6 and 3.5.0.8]{myers_categorical_2022}) and $F$-lenses (\cite[Def.\ 3.3]{spivak_generalized_2022})}]
  \label{def:f_charts_lenses}
  Let $\cC$ be a category.
  Let $F : \cC^{\op} \to \Cat$ be\\
  \\
  \begin{minipage}{0.33\textwidth}
    \begin{center}
      a lax functor
    \end{center}
  \end{minipage}%
  \hfill
  \begin{minipage}{0.33\textwidth}
    \begin{center}
      (resp.)
    \end{center}
  \end{minipage}%
  \hfill 
  \begin{minipage}{0.33\textwidth}
    \begin{center}
      an oplax functor
    \end{center}
  \end{minipage}
  \\
  \\
  Then the category of \\
  \\
  \begin{minipage}{0.33\textwidth}
    \begin{center}
      \newdef{$F$-charts}
    \end{center}
  \end{minipage}%
  \hfill
  \begin{minipage}{0.33\textwidth}
    \begin{center}
      (resp.)
    \end{center}
  \end{minipage}%
  \hfill 
  \begin{minipage}{0.33\textwidth}
    \begin{center}
      \newdef{$F$-lenses}
    \end{center}
  \end{minipage}
  \begin{center}
   is the contravariant Grothendieck construction of
  \end{center}
  \begin{minipage}{0.33\textwidth}
  \[
    \cC^{\op} \xrightarrow{F} \Cat
  \]
  \end{minipage}%
  \hfill
  \begin{minipage}{0.33\textwidth}
    \begin{center}
      (resp.)
    \end{center}
  \end{minipage}%
  \hfill 
  \begin{minipage}{0.33\textwidth}
  \[
    \cC^{\op} \xrightarrow{F^\co} \Cat^\co \xrightarrow{(-)^{\op}} \Cat
  \]
  \end{minipage}
  \\
\end{definition}

In all of the examples below, the distinction between lax and oplax dissapears: all the functors are strict (1-)functors.

\begin{example}[{$\CoKl(= \times -)$-charts and $\CoKl(= \times -)$-lenses}]
 \label{ex:cokl_charts_lenses}
 Denoted respectively as $\Chart(\cC)$ and $\Lens(\cC)$, these are categories whose objects are pairs $\diset{X}{X'}$ where $X, X': \cC$.
 Where they differ is in morphisms.

  \begin{minipage}{0.45\textwidth}
   A $\CoKl(= \times -)$-chart of type $\diset{X}{X'} \to \diset{Y}{Y'}$ consists of a pair $\diset{f}{f'}$ where
   \begin{itemize}
    \item $f : \cC(X, Y)$;
    \item $f' : \cC(X \times X', Y')$.
   \end{itemize}
  \end{minipage}
  \hfill %
  \begin{minipage}{0.45\textwidth}
   A $\CoKl(= \times -)$-lens of type $\diset{X}{X'} \to \diset{Y}{Y'}$ consists of a pair $\diset{f}{f'}$ 
   \begin{itemize}
    \item $f : \cC(X, Y)$;
    \item $f' : \cC(X \times Y', X')$.
   \end{itemize}
  \end{minipage}
  \newline
  \\
  The category of $\CoKl(= \times -)$-charts is referred to in \cite[Def. 1.3.1]{jacobs_categorical_1998} as the \emph{simple fibration on} $\cC$.
  If we restrict to $\asc{\CoKl(= \times -)}$-charts and $\asc{\CoKl(= \times -)}$-lenses, we obtain analogous categories to above (denoted as $\Chart_A(\cC)$ and $\LensA(\cC)$), except the backward objects $X'$ and $Y'$ are now commutative monoids, and backward maps $f'$ are now additive in the second variable.
\end{example}

\begin{example}[{$(\cC/-)$-charts and $(\cC/-)$-lenses}]
 \label{ex:slice_charts_lenses}
 Denoted as $\DChart(\cC)$ and $\DLens(\cC)$, these are categories whose objects are pairs $\diset{X}{p}$ where $X : \cC$ and $p : X' \to X$.
 Where they differ is, like before, in morphisms.

  \begin{minipage}{0.45\textwidth}
   A $(\cC/-)$-chart of type $\diset{X}{p} \to \diset{Y}{q}$ consists of a pair $\diset{f}{f'}$ where
   \begin{itemize}
    \item $f : \cC(X, Y)$;
    \item $f' : \cC/X(X', X \times_Y Y')$
   \end{itemize}
   such that the diagram below commutes:

  \[\begin{tikzcddiag}
  	{X'} & {X \times_YY'} & {Y'} \\
  	X & X & Y
  	\arrow["f"', from=2-2, to=2-3]
  	\arrow["q", from=1-3, to=2-3]
  	\arrow["{\pi_{Y'}}", from=1-2, to=1-3]
  	\arrow["\lrcorner"{anchor=center, pos=0.125}, draw=none, from=1-2, to=2-3]
  	\arrow[Rightarrow, no head, from=2-1, to=2-2]
  	\arrow["p"', from=1-1, to=2-1]
  	\arrow["{\pi_X}", from=1-2, to=2-2]
  	\arrow["{f'}", from=1-1, to=1-2]
  \end{tikzcddiag}\]
  \end{minipage}
  \hfill %
  \begin{minipage}{0.45\textwidth}
   A $(\cC/-)$-lens of type $\diset{X}{p} \to \diset{Y}{q}$ consists of a pair $\diset{f}{f'}$ 
   \begin{itemize}
    \item $f : \cC(X, Y)$;
    \item $f' : \cC/X(X \times_Y Y', X')$
   \end{itemize}
   
   such that the diagram below commutes:
  \[\begin{tikzcddiag}
  	{X'} & {X \times_YY'} & {Y'} \\
  	X & X & Y
  	\arrow["f"', from=2-2, to=2-3]
  	\arrow["q", from=1-3, to=2-3]
  	\arrow["{\pi_{Y'}}", from=1-2, to=1-3]
  	\arrow["\lrcorner"{anchor=center, pos=0.125}, draw=none, from=1-2, to=2-3]
  	\arrow[Rightarrow, no head, from=2-1, to=2-2]
  	\arrow["p"', from=1-1, to=2-1]
  	\arrow["{\pi_X}", from=1-2, to=2-2]
  	\arrow["{f'}"', from=1-2, to=1-1]
  \end{tikzcddiag}\]
  \end{minipage}
  \newline
  \\
Analogously to the example above, we can restrict to $\asc{\cC/-}$-charts and $\asc{\cC/-}$-lenses (denoting them as $\DChartA(\cC)$ and $\DLensA(\cC)$) obtaining analogous categories except the backward object $p : X' \to X$ is fibrewise a commutative monoid, and the backward map $f'$ is fibrewise additive.
\end{example}

When $\cC$ is $\Set$ there is an isomorphic, indexed formulation of dependent lenses using the $\Fam$ construction, relying on the fact that $\Fam(\Set)(X) \cong \Set / X$.\footnote{Dually, we can also say that slice functor is a generalisation of the indexed world to an arbitrary category; it necessitates that the contravariant slice functor can be coherently defined on the category $\cC$ which we want to replace $\Set$ by, which happens if $\cC$ has all pullbacks.}

\begin{example}[{$\Fam(\cC)$-charts and $\Fam(\cC)$-lenses}]
 \label{ex:fam_charts_lenses}
 Often denoted also as $\DChart(\cC)$ and $\DLens(\cC)$, these are categories whose objects are pairs $\diset{X}{X'}$ where $X : \Set$ and $X' : \discr(X) \to \Set$.
 Where they differ is in morphisms.
 \newline

  \begin{minipage}{0.45\textwidth}
   A $\Fam(\cC)$-chart of type $\diset{X}{X'} \to \diset{Y}{Y'}$ is a lax triangle
  \[\begin{tikzcddiag}
	  X && Y \\
	  & \cC
	  \arrow["f", from=1-1, to=1-3]
	  \arrow[""{name=0, anchor=center, inner sep=0}, "{Y'}", from=1-3, to=2-2]
	  \arrow[""{name=1, anchor=center, inner sep=0}, "{X'}"', from=1-1, to=2-2]
	  \arrow["{f'}", shift left, shorten <=10pt, shorten >=10pt, Rightarrow, from=1, to=0]
  \end{tikzcddiag}\]
 consisting of:
\begin{itemize}
  \item A function $f : X \to Y$;
  \item A (vacuously) natural transformation $f' : X' \Rightarrow F(f)(Y')$, i.e.\ for each $x : X$ a morphism $X'(x) \to F(f)(Y')(x)$ in $\cC$.
\end{itemize}
  \end{minipage}
  \hfill %
  \begin{minipage}{0.45\textwidth}
   A $\Fam(\cC)$-lens of type $\diset{X}{X'} \to \diset{Y}{Y'}$ is a lax triangle
  \[\begin{tikzcddiag}[ampersand replacement=\&]
      X \&\& Y \\
      \& \cC
      \arrow["f", from=1-1, to=1-3]
      \arrow[""{name=0, anchor=center, inner sep=0}, "{Y'}", from=1-3, to=2-2]
      \arrow[""{name=1, anchor=center, inner sep=0}, "{X'}"', from=1-1, to=2-2]
      \arrow["{f'}"', shift right, shorten <=10pt, shorten >=10pt, Rightarrow, from=0, to=1]
    \end{tikzcddiag}\]
 consisting of:
  \begin{itemize}
    \item A function $f : X \to Y$;
    \item A (vacuously) natural transformation $f' : F(f)(Y') \Rightarrow X'$, i.e.\ for each $x : X$ a morphism $F(f)(Y')(x) \to X'(x)$ in $\cC$.
  \end{itemize}
  \end{minipage}
  \newline
  \\
  Restriction to commutative monoids (as in previous examples) gives us analogous categories, except that the backward object $X' : \discr(X) \to \Set$ is now an indexed commutative monoid, and the backward map $f'$ is additive for each $x : X$.
\end{example}

\begin{notation}
  \label{not:poly}
  When the base category is $\Set$, then $\Fam(\Set)$-lenses are isomorphic to $(\cC/-)$-lenses, and notation sometimes used for this category is also $\PolySpivak$ (following \cite{niu_polynomial_2023}).
\end{notation}

These examples show us the power of Grothendieck construction as applied to a variety of functors.
In \cref{prop:para_bicategorical_grothendieck} we remarked that Grothendieck construction can be defined for bicategories.
We do not unpack the construction here, rather only mention that its details can be found in (\cite{bakovic_grothendieck_2009}).

\section*{Cartesian left-additive structure}
\label{appendix:cart_left_additive}

\begin{proposition}[$\CoKl(X \times - )$ is cartesian left-additive]
  \label{def:cokl_is_cartesian_left_additive}
  Let $\cC$ be a cartesian left-additive category, and let $X : \cC$.
  Then the category $\CoKl(X \times -)$ is cartesian left-additive.
\end{proposition}

\begin{proof}
  We need to prove 1) that $\CoKl(X \times -)$ is cartesian monoidal, and 2) that it is additionally left-additive.
  \begin{enumerate}
    \item  The monoidal product on objects $A, B : \CoKl(X \times -)$ is given by their underlying product in $\cC$.
  The monoidal product of $f : \CoKl(X \times -)(A, B)$ and $g : \CoKl(X \times -)(C, D)$ is given by the composite $(f \compPara{\times} g)^{\Delta_X}$ which by notation in \cref{def:para} unpacks to the composite
  \[
    X \times (A \times C) \xrightarrow{\Delta_X \times (A \times C)} (X \times X) \times (A \times C) \cong (X \times A) \times (X \times C) \xrightarrow{f \times g} B \times D
  \]
  To show that this is a cartesian monoidal product, we need to equip every object with a commutative comonoid.
  This can be done by using the commutative comonoids in $\cC$ and the canonical embedding $\cC \to \CoKl(X \times -)$.
  It is straightforward to show that such comonoids are natural.
  \item We need to give a coherent choice of monoids to every object $A : \CoKl(X \times -)$.
  This is simply given by the embedding $\cC \to \CoKl(X \times -)$.
  It is routine to show that this satisfies the two commutative diagrams in \cref{def:clac}.
  \end{enumerate}
\end{proof}

\begin{lemma}
  \label{lemma:cla_preserves_additive}
  Let $F : \cC \to \cD$ be a morphism in $\CLACat$, i.e.\ a cartesian left-additive functor.
  Let $f : A \to B$ be an additive morphism in $\cC$. Then $F(f) : F(A) \to F(B)$ is also additive.
\end{lemma}

\begin{proof}
  To prove additivity of $F(f) : F(A) \to F(B)$, we need to prove it preserves the monoid structure.
  Since preservation of monoid structure is defined using two commutative diagrams (\cref{def:additive_morphism}), this follows from the fact that functors preserve commutative diagrams and the fact that $F$ is a product-preserving functor (allowing us to replace $F(A \times A)$ with $F(A) \times F(A)$).
\end{proof}

\LensFunctorCartLeftAdd*
\begin{proof}
  We need to prove that $\LensA(F)$ is a cartesian left-additive functor.
  To prove it is a functor, we need to:
  \begin{itemize}
  \item Define its action on objects and morphisms. We've done this in \cref{prop:lens_functor_cart_left_add} itself;
  \item Prove additivity of $\overline{f^\sharp}$. This follows from \cref{lemma:cla_preserves_additive};
  \item Prove identities are preserved.
    The identity $\diset{\id_A}{\pi_2} : \diset{A}{A'} \to \diset{A}{A'}$ in the domain gets mapped to $\diset{F(\id_A)}{F(\pi_2)}$.
    By preservation of identities and products of $F$ this is equal to the identity map on $\diset{F(A)}{F(A')}$.
  \item Prove composition is preserved. This can be be verified by routine.
  \end{itemize}
  To prove that it is additionally cartesian, we need to show that the image of every comonoid $(\diset{A}{A'}, !_{\diset{A}{A'}}, \Delta_{\diset{A}{A'}})$ is also a comonoid, and that all maps preserve comonoids.
     We can understand the first part in terms of actions on the counit and comultiplication of the comonoid.
    \begin{itemize}
    \item \textbf{Counit.} The action on the counit unpacks to the pair $\diset{F(!_{\diset{A}{A'}})}{F(!_{A \times 1} \comp 0_A)}$.
      By preservation of terminal and additive maps of $F$ this morphism is equal to the counit of $\diset{F(A)}{F(A')}$.
    \item \textbf{Comultiplication.} The action on the comultiplication unpacks to $\diset{(F(\Delta_A)}{F(\pi_{2, 3} \comp +_A)}$.
      By $F$'s preservation of products and additive morphisms this morphism is equal to the comultiplication of $\diset{F(A)}{F(A')}$.
    \end{itemize}
    It is routine to show it obey the corresponding laws and form a comonoid.
    
    For the second part we need to show that the image of every lens $\diset{f}{f^\sharp} : \diset{A}{A'} \to \diset{B}{B'}$ preserves these comonoids.
    For the forward part this is true because $F$ preserves products.
    For the backwards part this is true because $F$ is left-additive.
    
    Lastly, we need to prove that this functor is additionally left-additive.
    This means that it preserves the monoid $(\diset{A}{A'}, 0_{\diset{A}{A'}}, +_{\diset{A}{A'}})$ of every object.
    We unpack the action of $\LensA(F)$ on the unit $0_{\diset{A}{A'}}$ and multiplication $+_{\diset{A}{A'}}$ below.
    \begin{itemize}
    \item \textbf{Unit.} The action on the unit unpacks to the pair $\diset{F(0_A)}{F(!_{1 \times A})}$.
      By preservation of additive and terminal maps of $F$ this morphism is equal to the unit of $\diset{F(A)}{F(A)}$;
    \item \textbf{Multiplication.} The action on the multiplication unpacks to the pair $\diset{F(+_A)}{F(\pi_3 \comp \Delta_A)}$.
      By preservation of coadditive maps and products of $F$ this morphism is equal to  the multiplication of $\diset{F(A)}{F(A')}$.
    \end{itemize}
    Seeing as these monoids in the codomain are of the same form as those in the domain, it is routine to show that they obey the monoid laws.
  This concludes the proof that $\LensA(F)$ is a cartesian left-additive functor.
\end{proof}

\begin{definition}[{compare \cite[Def.\ 13]{cockett_reverse_2020}}]
  \label{def:crdc}
  A cartesian left-additive category $\cC$ is called a \newdef{cartesian reverse derivative category} if it comes equipped with a reverse differential combinator
\[
  R : \cC(A, B) \to \cC(A \times B, A)
\]
such that the following seven axioms hold expressed as commutative diagrams:
\begin{enumerate}
\item \textbf{Additivity of reverse differentiation.}\\
  For every $f, g : A \to B$ it is required that
  \[
    R[f + g] = R[f] + R[g] \quad \text{and} \quad R[0] = 0
  \]
  It is useful unpack the $+$ notation for morphisms (\cref{rem:why_left_additive}) in terms of commutative diagrams below
  
  \begin{minipage}{.5\textwidth}
    \centering
    \[\begin{tikzcddiag}
        {A \times B} && A \\
        \\
        {A \times B \times A \times B} && {A \times A}
        \arrow["{R[\Delta_A \comp (f \times g) \comp +_B]}", from=1-1, to=1-3]
        \arrow["{\Delta_{A \times B}}"', from=1-1, to=3-1]
        \arrow["{R[f] \times R[g]}"', from=3-1, to=3-3]
        \arrow["{+_A}"', from=3-3, to=1-3]
      \end{tikzcddiag}\]
  \end{minipage}%
  \begin{minipage}{0.5\textwidth}
    \centering
    \[\begin{tikzcddiag}[ampersand replacement=\&]
        {A \times B} \&\& A \\
        \\
        1
        \arrow["{R[!_A \comp 0_A]}", from=1-1, to=1-3]
        \arrow["{0_A}"', from=3-1, to=1-3]
        \arrow["{!_{A \times B}}"', from=1-1, to=3-1]
      \end{tikzcddiag}\]
  \end{minipage}
\item \textbf{Additivity of reverse derivative in the second component.}\\
  For every $f : A \to B$ it is required that
  \[
    \langle a, b + c \rangle R[f] = \langle a, b \rangle R[f] + \langle a, c \rangle R[f] \quad \text{and} \quad \langle a, 0 \rangle R[f] = 0
  \]
  Again, we find it instructive to unpack the underlying diagrams:
  
  \begin{minipage}{.5\textwidth}
    \centering
    \[\begin{tikzcddiag}
    	{A \times B \times B} && {A \times B} \\
    	{A \times A \times B \times B} && A \\
    	{A \times B \times A \times B} && {A \times A}
    	\arrow["{\Delta_A \times B \times B}"', from=1-1, to=2-1]
    	\arrow["{\interchanger_{A, A, B, B}}"', from=2-1, to=3-1]
    	\arrow["{R[f] \times R[f]}"', from=3-1, to=3-3]
    	\arrow["{+_A}"', from=3-3, to=2-3]
    	\arrow["{A \times +_B}", from=1-1, to=1-3]
    	\arrow["{R[f]}", from=1-3, to=2-3]
    \end{tikzcddiag}\]
  \end{minipage}%
  \begin{minipage}{0.5\textwidth}
    \centering
    \[\begin{tikzcddiag}
        {A \times 1} && 1 \\
        \\
        {A \times B} && A
        \arrow["{A \times 0_B}"', from=1-1, to=3-1]
        \arrow["{R[f]}"', from=3-1, to=3-3]
        \arrow["{!_{A \times 1}}", from=1-1, to=1-3]
        \arrow["{0_A}", from=1-3, to=3-3]
      \end{tikzcddiag}\]
  \end{minipage}
  
\item \textbf{Coherence with identities and projections.}\\
  For any two objects $A, B : \cC$ it is required that the following diagrams commute:
  
  \begin{minipage}{.3\textwidth}
    \centering
    \[\begin{tikzcddiag}[ampersand replacement=\&]
        {A \times A} \&\& A \\
        \\
        A
        \arrow["{R[\id_A]}", from=1-1, to=1-3]
        \arrow["{\pi_2}"', from=1-1, to=3-1]
        \arrow[Rightarrow, no head, from=3-1, to=1-3]
      \end{tikzcddiag}\]
  \end{minipage}%
  \begin{minipage}{0.5\textwidth}
    \centering
    \[\begin{tikzcddiag}[ampersand replacement=\&]
        {(A \times B) \times A} \&\& {A \times B} \&\& {(A \times B) \times B} \\
        \\
        A \&\&\&\& B
        \arrow["{\pi_2}"', from=1-1, to=3-1]
        \arrow["{i_1}"', from=3-1, to=1-3]
        \arrow["{R[\pi_1]}", from=1-1, to=1-3]
        \arrow["{\pi_2}", from=1-5, to=3-5]
        \arrow["{i_2}", from=3-5, to=1-3]
        \arrow["{R[\pi_2]}"', from=1-5, to=1-3]
      \end{tikzcddiag}\]
  \end{minipage}
\item \textbf{Coherence with pairings.}\\
  For every $f : A \to B$ and $g : A \to C$ it is required that the following diagram commutes:
  \[\begin{tikzcddiag}
  	{A \times B \times C} && A \\
  	{A \times A \times B \times C} \\
  	{A \times B \times A \times C} && {A \times A}
  	\arrow["{R[\langle f, g\rangle]}", from=1-1, to=1-3]
  	\arrow["{\Delta_A \times B \times C}"', from=1-1, to=2-1]
  	\arrow["{\interchanger_{A, A, B, C}}"', from=2-1, to=3-1]
  	\arrow["{R[f] \times R[g]}"', from=3-1, to=3-3]
  	\arrow["{+_A}"', from=3-3, to=1-3]
  \end{tikzcddiag}\]
  and for every $A : \cC$ it is required that the following diagram commmutes:
  \[\begin{tikzcddiag}[ampersand replacement=\&]
      {A \times 1} \&\& 1 \\
      \\
      A
      \arrow["{R[\terminal_A]}"', from=1-1, to=3-1]
      \arrow["{\terminal_{A \times 1}}", from=1-1, to=1-3]
      \arrow["{0_A}", from=1-3, to=3-1]
    \end{tikzcddiag}\]
\item \textbf{Reverse chain rule.}\\
  For all  composable pairs $f : A \to B$ and $g : B \to C$ the following diagram must commute:
  
  \[\begin{tikzcddiag}[ampersand replacement=\&]
      {A \times C} \&\& A \\
      \\
      {A \times B \times C} \&\& {A \times B}
      \arrow["{R[f \comp g]}", from=1-1, to=1-3]
      \arrow["{A \times R[g]}"', from=3-1, to=3-3]
      \arrow["{R[f]}"', from=3-3, to=1-3]
      \arrow["{\graph(f) \times C}"', from=1-1, to=3-1]
    \end{tikzcddiag}\]
\item \textbf{Linearity of the reverse derivative in its second component.} \\
  For every $f : A \to B$ the following diagram must commute:
  
  \[\begin{tikzcddiag}
      {A \times B} && {A \times 1 \times B} && {A \times B \times A \times A \times B} \\
      \\
      B &&&& {A \times B \times A}
      \arrow["{R[f]}"', from=1-1, to=3-1]
      \arrow["\cong", from=1-1, to=1-3]
      \arrow["{A \times 0_{B \times A \times A} \times B}", from=1-3, to=1-5]
      \arrow["{R^3[f]}", from=1-5, to=3-5]
      \arrow["{\pi_2}", from=3-5, to=3-1]
    \end{tikzcddiag}\]
\item \textbf{Symmetry of mixed partial derivatives.}\\
  Given the following compact, but unilluminating definition of $h$
  \[
    h \coloneqq (i_1 \times 1_{A \times A}) \comp R[R[(i_1 \times A) \comp R[R[f]] \comp \pi_2]] \comp \pi_2
  \]
  the symmetry of mixed partial derivatives can be simply stated as the requirement that the following diagram commutes:
  \[\begin{tikzcddiag}[ampersand replacement=\&]
      {A \times A \times A \times A} \& {A \times A \times A \times A} \\
      \& A
      \arrow["{c_A}", from=1-1, to=1-2]
      \arrow["h"', from=1-1, to=2-2]
      \arrow["h", from=1-2, to=2-2]
    \end{tikzcddiag}\]
  where $c_A$ is the interchanger (\cref{def:interchanger}).
\end{enumerate}

Note that RD6 and RD7 are independent from the others (\cref{rem:crdc_independent}).
\end{definition}

\chapter{Miscellaneous}

\begin{definition}[Initial object]
  An object $I$ in a category $\cC$ is called \newdef{initial} if for every $X : \cC$ it naturally holds that $\cC(I, X) \cong 1$, meaning that there is only one map of type $I \to X$.
  We often label this map as $i_X$.
\end{definition}

The naturality condition is superfluous, which can be seen in the following corollary.

\begin{corollary}
  If a category $\cC$ has an initial object $I$, then for every $f : X \to Y$ the triangle
\begin{tikzcddiag}[ampersand replacement=\&]
    {I } \& X \\
    \& Y
    \arrow["{i_X}", from=1-1, to=1-2]
    \arrow["{i_y}"', from=1-1, to=2-2]
    \arrow["f", from=1-2, to=2-2]
  \end{tikzcddiag}
  commutes.
\end{corollary}

\begin{proof}
  Since it holds that $\cC(I, Y) \cong 1$, this means that there is only one morphism that can be chosen as the composite $i_X \comp f$ --- precisely $i_Y$.
\end{proof}

\begin{definition}[{Quasi-initial object}]
  \label{def:quasi_initial}
  An object $I$ in a bicategory $\cB$ is \newdef{quasi-initial}\footnote{This is mostly referred to as a quasi-initial object in \cite{nlab_authors_quasi-limit_2023}, though it does not appear that these concepts have been studied sufficiently for consistent modern terminology to be adopted.} if for every object $X$ there is a laxly natural adjunction $R : \cB(I, X) \leftrightarrows \CatTerm : L$ where $\CatTerm$ is the terminal category in $\Cat$.
\end{definition}

Since we are in $\Cat$ the functor $R$ is the terminal functor.
And a choice of a left adjoint to it is equivalent to the choice of an initial object in the domain category $\cB(I, X)$.
We will suggestively label this object as $i_X$, and unique maps from it to any other $g : I \to X$ as $i_g$.
The lax naturality condition is superfluous, which can be seen in the following corollary.

\begin{corollary}
  If a bicategory $\cB$ has a quasi-initial object $I$, then for every $f : X \to Y$ the triangle
\begin{tikzcddiag}[ampersand replacement=\&]
    I \&\& X \\
    \\
    \&\& Y
    \arrow["f", from=1-3, to=3-3]
    \arrow["{i_X}", from=1-1, to=1-3]
    \arrow[""{name=0, anchor=center, inner sep=0}, "{i_Y}"', from=1-1, to=3-3]
    \arrow["{i_{i_X \comp f}}", shorten <=6pt, shorten >=6pt, Rightarrow, from=0, to=1-3]
  \end{tikzcddiag}
  laxly commutes up to the unique 2-cell $\i_{i_X \comp f}$.
\end{corollary}

\begin{proof}
Since $i_Y$ is initial in $\cB(I, Y)$, there is only one map of type $i_Y \Rightarrow i_X \comp f$ which we have already labeled as $i_{i_X \comp f}$ since we know that it exists and is uniquely defined.
\end{proof}

\begin{proposition}
  \label{prop:colim_terminal}
  Let $F : \cC \to \cD$ be a functor.
  If $\cC$ has a terminal object $1$, then $\colim F \cong F(1)$.
  Dually, if $\cC$ has an initial object $0$, then $\lim F \cong F(0)$
\end{proposition}

Note that the converse of the above propositions doesn't necessarily hold.

\begin{proposition}[{Colimit as the set of connected components of the category of elements }]
  \label{prop:colim_conn_comp_of_el}
  Let $F : \cC \to \Set$  be a functor.
  Then
  \[
    \colim F \cong \conncomp(\El(F))
  \]
\end{proposition}

\begin{proof}
  See \cite[(3.35)]{kelly_basic_1982} or \cite[Theorem 6.37.]{fong_invitation_2019}.
  The latter reference unpacks the above isomorphism as the well-known formula for computing colimits in $\Set$.
  We leave it to the reader to check that the equivalence relation described therein is the one arising as the image of $\conncomp$.
\end{proof}

\begin{proposition}[{From $\Set$-weighted colimits to colimits (compare \cite[(3.34)]{kelly_basic_1982}), \cite[Prop. 4.1.13 WC3]{loregian_coend_2021}}]
  \label{prop:from_set_weighted_colimits_to_colimits}
  Let $F : \cC \to \cD$ and $W : \cC^{\op} \to \Set$ be two functors.
  Then
  \begin{equation}
    \label{eq:weighted_colim_to_colim}
   \colim^W F \cong \colim (\El(W)^{\op} \xrightarrow{\pi^{\op}_W} \cC \xrightarrow{F} \cD) 
  \end{equation}
\end{proposition}

\begin{proof}
  \begin{alignat*}{2}
    & && \colim^W F\\
    &\text{(Duality of limits and colimits)}   &\cong& \limm^{W} F^{\op}\\
    &\text{(\cite[Prop. 4.1.11)]{loregian_coend_2021})}   &\cong& \limm (\El(W) \xrightarrow{\pi_W} \cC^{\op} \xrightarrow{F^{\op}} \cD^{\op}) \\
    &\text{(Duality of limits and colimits)}   &\cong& \colim (\El(W)^{\op} \xrightarrow{\pi^{\op}_W} \cC \xrightarrow{F} \cD) \tag*{\qedhere}
  \end{alignat*}
\end{proof}

Observe that if the domain of our functor is $\cC^{\op} \times \cC$, then there is always a choice of a canonical weight $\Hom_{\cC} : \cC^{\op} \times \cC \to \Set$.
As a matter of fact, the colimit of some $F : \cC \times \cC^{\op} \to \cD$ weighted by $\Hom_{\cC}$ is so ubiquitous that is has a special name: it's called the \emph{coend} of $F$.

This gives us two further characterisations of weighted colimits of functors of
type $\cC \times \cC^{\op} \to \cD$.

\begin{proposition}[{Weighted colimit as coend (compare \cite[Remark 4.1.13.1]{loregian_coend_2021})}]
  \label{prop:weight_colimit_as_coend}
  Let $F : \cC \to \cD$ and $W : \cC^{\op} \to \Set$ be two functors.
  Then
  \[
    \colim^W F \cong \int^{C : \cC} F(C) \actfw W(C)
  \]
  holds assuming $\cD$ admits the respective coend (which happens when $\cD$ is cocomplete), where $\actfw : \cD \times \Set \to \cD$ is the respective $\Set$-copower of $\cD$ arising from its coproduct structure (see \cite[Ex. 3.2.6]{capucci_actegories_2023}).
\end{proposition}

In many of our use cases (such as \cref{eq:weighted_optic_coend_form}) the domain of $F$ will be the product $\cM \times \cM^{\op}$, and codomain $\Set$, in which case the copower reduces to the cartesian product of $\Set$.

\begin{lemma}[Coend over a set]
  \label{lemma:coend_over_a_set}
  Let $X$ be a set, and let $F : \discr(X)^{\op} \times \discr(X) \to \Set$ be a functor.\footnote{Here the $^{\op}$ is unnecessary (as $\discr(X)$ does not have any non-identity morphisms) but we have left it in nonetheless.}
  Then
  \[
    \int^{A : \discr{X}} F(A, A) \cong \sum_{A : \discr(X)}F(A, A)
  \]
  
\end{lemma}

We unpack the Yoneda lemma, which comes in a total of four variants.\footnote{To really understand Yoneda lemma, one needs to understand how it is related to all the other concepts in category theory.}

\begin{proposition}[{Yoneda Lemma (\cite[Prop. 2.2.1]{loregian_coend_2021})}]
  \label{prop:yoneda}
  Let $F : \cC \to \Set$ be a functor (respectively $G : \cC^{\op} \to \Set$).
  Then the following isomorphisms hold, natural in $A$:
  \[
    F(A) \cong \int_{X:\cC}\internalHom{\Set}{\cC(A, X)}{F(X)} \quad \text{(and, respectively)} \quad G(A) \cong \int_{X : \cC^{\op}}\internalHom{\Set}{\cC(X, A)}{G(X)}
  \]
\end{proposition}

\begin{proposition}[{coYoneda (\cite[Prop. 2.2.1]{loregian_coend_2021})}]
  \label{prop:coyoneda}
  Let $F : \cC \to \Set$ be a functor (respectively $G : \cC^{\op} \to \Set$).
  Then the following isomorphisms hold, natural in $A$:
  \[
    \int^{X : \cC} F(X) \times \cC(X, A) \cong F(A) \quad \text{(and, respectively)} \quad \int^{X : \cC^{\op}} \cC(A, X) \times G(X) \cong G(A)
  \]
\end{proposition}

\begin{proof}
  We prove \cref{prop:yoneda} (left). The rest follows analogously. %
  \begin{itemize}
  \item \textbf{Left to right.}  
    Given an element $a : F(A)$ we need to produce a natural transformation $\cC(A, -) \Rightarrow F$, i.e.\ for every $X : \cC$ a function $\cC(A, X) \to F(X)$ natural in $X$.
    This function consumes a $f : A \to X$ and is required to produce an element of $F(X)$.
    We do the \emph{only thing we can}: apply $F$ to $f$, obtaining $F(f) : F(A) \to F(X)$ and then use $a : F(A)$ as its input, obtaining $F(f)(a)$.
    It's straightforward to check that this naturally defined.
  \item \textbf{Right to left.} 
    Going the other way, given a natural transformation $\alpha : \cC(A, -) \Rightarrow F$, to obtain an element of $F(A)$ we simply again do the only thing we can: look at the component $\alpha_A : \cC(A, A) \to F(A)$ and use $\id_A : A \to A$ as its input.
  \item \textbf{Isomorphism.}
    To confirm that this is indeed an isomorphism, we need to check that these are inverses.
    From one side it needs to hold that $(F(\id_A)(a))_A = a$, which is true since $F(\id_A) = \id_{F(A)}$.
    From the other side it needs to hold that $F(-)(\alpha_A(\id_A)) = \alpha$. At each component $X : \cC$ and every map $f : A \to X$ it needs to hold that  $F(f)(\alpha_A(\id_A))_X = \alpha_X(f)$.
    This follows from naturality of $\alpha$ (\cref{eq:yoneda_naturality}).
    \begin{equation}
      \label{eq:yoneda_naturality}
\begin{tikzcddiag}[ampersand replacement=\&]
	{\cC(A, A)} \&\& {F(A)} \&\& {\id_A} \&\& {\alpha_A(\id_A)} \\
	\\
	{\cC(A, X)} \&\& {F(X)} \&\& f \&\& {\alpha_X(f)= F(f)(\alpha_A(\id_A))}
	\arrow["{\alpha_A}", from=1-1, to=1-3]
	\arrow["{\cC(A, f)}"', from=1-1, to=3-1]
	\arrow["{\alpha_X}"', from=3-1, to=3-3]
	\arrow["{F(f)}", from=1-3, to=3-3]
	\arrow[maps to, from=1-5, to=1-7]
	\arrow[maps to, from=1-5, to=3-5]
	\arrow[maps to, from=1-7, to=3-7]
	\arrow[maps to, from=3-5, to=3-7]
\end{tikzcddiag}
    \end{equation}
  \end{itemize}
\end{proof}

The following lemma tells us when a profunctor is \emph{representable}, i.e.\ can be described as a composition of a hom-functor and a functor.

\begin{lemma}[{Representable profunctor, \cite[Sec. 5.2]{loregian_coend_2021}}]
  \label{lemma:profunctor_representable}
  A profunctor $P : \cM^{\op} \times \cM' \to \Set$  is \newdef{representable} by $i : \cM \to \cM'$ (resp. \newdef{corepresentable} by $j : \cM' \to \cM$) if the following isomorphism holds
  \[
    P \cong \cM'(i(-), -) \quad \text{(resp.)} \quad P \cong \cM(-, j(-))
  \]
  in the functor category $\internalHom{\Cat}{\cM^{\op} \times \cM'}{\Set}$.
\end{lemma}

\begin{definition}
  \label{def:graph}
  Let $\cC$ be a cartesian category.
  Then the \newdef{graph} of any morphism $f : A \to B$ in $\cC$ is the morphism
  \[
    \graph(f) \coloneqq \boxed{A \xrightarrow{\Delta_A} A \times A \xrightarrow{A \times f} A \times B}
  \]
\end{definition}

It is called the ``graph'' of f because its image is a set of pairs $(a, f(a))$ which we can think of as points in a coordinate plane to be graphed.

Another important concept in this thesis is the notion of a coalgebra over a copointed endofunctor.
It arises as a mid-point between the notions of a coalgebra over an endofunctor and coalgebra over a comonad.

\begin{definition}[Coalgebra over a copointed endofunctor]
  \label{def:coalgebra_copointed_endofunctor}
  Let $F : \cC \to \cC$ be endofunctor which is copointed, i.e.\ is equipped with a natural transformation $\epsilon : F \Rightarrow \id_{\cC}$.
  Then \newdef{a coalgebra over $(F, \epsilon)$} consists of a pair $(A, f)$, where $A : \cC$ and $f : A \to F(A)$ is a morphism in $\cC$ such that the diagram
  \begin{tikzcddiag}
      {F(A)} & A \\
      A
      \arrow["f", from=2-1, to=1-1]
      \arrow[Rightarrow, no head, from=2-1, to=1-2]
      \arrow["{\epsilon_A}", from=1-1, to=1-2]
    \end{tikzcddiag} commutes.
\end{definition}

\begin{proposition}[Adjoint quadruple]
  \label{prop:adjoint_quadruple}
  Between categories $\Cat$ and $\Set$ there exists a quadruple of functors, all adjoint one to the next:

\[\begin{tikzcddiag}[ampersand replacement=\&]
	\Cat \&\&\&\& \Set
	\arrow[""{name=0, anchor=center, inner sep=0}, "\discr", curve={height=-12pt}, from=1-5, to=1-1]
	\arrow[""{name=1, anchor=center, inner sep=0}, "\conncomp"', shift right=3, curve={height=30pt}, from=1-1, to=1-5]
	\arrow[""{name=2, anchor=center, inner sep=0}, "\codiscr"', shift right=3, curve={height=30pt}, from=1-5, to=1-1]
	\arrow[""{name=3, anchor=center, inner sep=0}, "\Ob", curve={height=-12pt}, from=1-1, to=1-5]
	\arrow["\dashv"{anchor=center, rotate=89}, draw=none, from=1, to=0]
	\arrow["\dashv"{anchor=center, rotate=90}, draw=none, from=0, to=3]
	\arrow["\dashv"{anchor=center, rotate=91}, draw=none, from=3, to=2]
\end{tikzcddiag}\]

More specifically, they are:
\begin{itemize}
  \item $\codiscr : \Set \to \Cat$, the functor that sends a set $X$ to a category whose objects are elements of $X$ such that there exists a unique morphisms between every two objects\footnote{Often also called a \emph{chaotic}, or \emph{indiscrete} category};
  \item $\Ob : \Cat \to \Set$, the left adjoint of $\codiscr$ that sends a category to the set of its underlying objects;
  \item $\discr : \Set \to \Cat$, the left adjoint of $\Ob$ that sends a set $X$ to a category whose objects are elements of $X$ and whose morphisms are only identity morphisms;
  \item $\conncomp : \Cat \to \Set$, the left adjoint of $\discr$ that sends a category to the set of its connected components.
\end{itemize}
\end{proposition}

\begin{definition}[Strong endofunctor]
  \label{def:strong_endofunctor}
  Let $(\cC, \otimes, I)$ be a monoidal category.
  Then an endofunctor $F : \cC \to \cC$ is called \newdef{(right) strong} if it comes equipped with a natural transformation

\[\begin{tikzcddiag}
	{\cC \times \cC} & {\cC \times \cC} \\
	\cC & \cC
	\arrow["{F \times \cC}", from=1-1, to=1-2]
	\arrow["{ \otimes}", from=1-2, to=2-2]
	\arrow["F"', from=2-1, to=2-2]
	\arrow["\otimes"', from=1-1, to=2-1]
	\arrow["s"', shorten <=9pt, shorten >=9pt, Rightarrow, from=1-2, to=2-1]
\end{tikzcddiag}\]
  whose component at each $X, Y : \cC$ is explicitly the map $s_{X, Y} : F(X) \otimes Y \to F(X \otimes Y)$ making the following diagrams commute for all $X, Y, Z : \cC$:

\[\begin{tikzcddiag}[ampersand replacement=\&]
	{F(X) \otimes I} \&\&\& {(F(X) \otimes Y) \otimes Z} \& {F(X \otimes Y) \otimes Z} \& {F((X \otimes Y) \otimes Z)} \\
	\\
	{F(X \otimes I)} \&\& {F(X)} \& {F(X) \otimes (Y \otimes Z)} \&\& {F(X \otimes (Y \otimes Z))}
	\arrow["{s_{X, Y} \otimes Z}", from=1-4, to=1-5]
	\arrow["{s_{X \otimes Y, Z}}", from=1-5, to=1-6]
	\arrow["{\alpha_{F(X), Y, Z}}"', from=1-4, to=3-4]
	\arrow["{s_{X, (Y \otimes Z)}}"', from=3-4, to=3-6]
	\arrow["{F(\alpha_{X, Y, Z})}", from=1-6, to=3-6]
	\arrow["{s_{X, I}}"', from=1-1, to=3-1]
	\arrow["{F(\rho_{X})}"', from=3-1, to=3-3]
	\arrow["{\rho_{F(X)}}", from=1-1, to=3-3]
\end{tikzcddiag}\]
\end{definition}

\begin{definition}[Monoidal closed category]
  \label{def:monoidal_closed}
  Let $(\cC, \otimes, I)$ be a monoidal category.
  We call it a \emph{monoidal closed category} if, for every $C : \cC$ the functor $- \otimes C$ has a right adjoint denoted by $\internalHom{\cC}{C}{-}$ which forms the internal hom functor $\internalHom{\cC}{-}{-} : \cC^{\op} \times \cC \to \cC$ such that there is isomorphism
  \[
    \cC(X \otimes Y, Z) \cong \cC(X, \internalHom{\cC}{Y}{Z})
  \]
  natural in $X, Y, Z$.
\end{definition}

\begin{remark}
  \label{rem:eval}
Notably, in any monoidal closed category for each $X, Y: \cC$ there exists the map $\eval_{X, Y} : \internalHom{\cC}{X}{Y} \otimes X \to Y$ called the \newdef{evaluation map}.
In $\Set$, $\eval$ consumes a function $f : X \to Y$, an element $x : X$ and produces $f(x) : Y$.
\end{remark}

\begin{remark}[{The functor $L_{\epsilon, e}$ in \cite[Theorem III.2]{fong_backprop_2021} not well-defined}]
  \label{rem:backprop_not_functor}
The functor $L_{\epsilon, e} : \Para \to \Learn$ is not well-defined with respect to the morphism equivalence classes.\footnote{This counterexample was pointed out to me by Geoffrey Cruttwell.}\textsuperscript{,}\footnote{Here we are using $\Para$ as defined in \cite[Def.\ III.1]{fong_backprop_2021}, which is isomorphic to $\enrichedbasechange{\Iso}(\Para(\Smooth))$ using the definitions of this thesis.}
That is, there are parametric maps where $(P, f) \sim (P', f')$ in $\Para$, but $L_{\epsilon, e}((P, f)) \not\sim L_{\epsilon, e}((P', f'))$ in $\Learn$.
For a specific example, consider two $\R$-parametric morphisms of type $\R \to \R$:
\begin{itemize}
  \item $f(a, p) = 3p + a$
  \item $f'(a, p) = p + a$;
\end{itemize}

They are equivalent in $\Para$, as witnessed by the invertible reparameterisation $\alpha : \R \to \R \coloneqq x \mapsto 3x$.
To show that they are equivalent in $\Learn$ we need a reparameterisation $\alpha$ as before, but now additionally by \cite[p.\ 4, right column, top]{fong_backprop_2021} it is required that the following equation

\begin{equation}
  \label{eq:learner_update_equiv_rel}
U_{f'}(\alpha(p), a, b) = \alpha(U_f(p, a, b))
\end{equation}

holds for all $p, a, b : \R$. By setting $\epsilon = 1$ and $e$ to the mean squared error we can calculate
\begin{itemize}
  \item $U_f(p, a, b) = -8p - 3a + 3b$;
  \item $U_{f'}(p, a, b) = -a + b$;
\end{itemize}

Substituting $\alpha$ and the two results above in \cref{eq:learner_update_equiv_rel} yields $-24p -9a + 9b = -a + b$.
This clearly doesn't hold for all $p, a, b \in \R$.
\end{remark}

  \backmatter

	\bibliographystyle{alphaurl}
	\bibliography{references}

\end{document}